\documentclass[10pt,twocolumn,letterpaper]{article}

\usepackage[pagenumbers]{cvpr}

\usepackage{tabularx}
\usepackage{etoolbox,siunitx}
\robustify\bfseries

\usepackage{tablefootnote}

\usepackage{indentfirst}

\usepackage{multirow}
\usepackage{comment}
\usepackage[dvipsnames]{xcolor}

\usepackage{amsthm}
\usepackage[accsupp]{axessibility}

\newtheorem{proposition}{Proposition}
\usepackage{multirow}

\usepackage{pgfplots}
\pgfplotsset{compat=1.18}

\newcommand*{\inparagraph}[1]{\noindent\textbf{#1}\hspace{0.5em}}

\definecolor{cvprblue}{rgb}{0.21,0.49,0.74}
\usepackage[pagebackref,breaklinks,colorlinks,citecolor=cvprblue]{hyperref}

\usepackage{colortbl}
\definecolor{lightred}{RGB}{255,240,240}
\definecolor{lightblue}{RGB}{225,239,240}
\definecolor{flat}{RGB}{128, 64, 128}
\definecolor{construction}{RGB}{70, 70, 70}
\definecolor{object}{RGB}{153, 153, 153}
\definecolor{nature}{RGB}{107, 142, 35}
\definecolor{sky}{RGB}{70, 130, 180}
\definecolor{human}{RGB}{220, 20, 60}
\definecolor{vehicle}{RGB}{0, 0, 142}
\definecolor{ignore}{RGB}{0, 0, 0}
\definecolor{true}{RGB}{0,114,178}
\definecolor{false}{RGB}{204, 121, 167}
\definecolor{left}{RGB}{19,44,95}

\newcolumntype{R}{>{\columncolor{lightred}}S[table-format=2.2]}
\newcolumntype{B}{>{\columncolor{lightblue}}S[table-format=2.2]}

\def\oureuc{Flat-Euc (ours)\xspace}
\def\ourhyp{Flat-Hyp (ours)\xspace}

\def\eucname{Flat-Euc\xspace}
\def\hypname{Flat-Hyp\xspace}

\def\paperID{7166} 
\def\confName{CVPR}
\def\confYear{2024}

\usepackage{fancyhdr}
\usepackage{setspace}

\title{Flattening the Parent Bias:\\[1mm] Hierarchical Semantic Segmentation in the Poincaré Ball}

\author{Simon Weber\textsuperscript{1,2} \hspace{2em}
Barış Zöngür\textsuperscript{1} \hspace{2em}
Nikita Araslanov\textsuperscript{1,2}  \hspace{2em}
Daniel Cremers\textsuperscript{1,2} \\[1mm]
\textsuperscript{1}Technical University of Munich \hspace{1cm} \textsuperscript{2}Munich Center for Machine Learning
}

\newcommand\blfootnote[1]{
  \begingroup
  \renewcommand\thefootnote{}\footnote{#1}
  \addtocounter{footnote}{-1}
  \endgroup
}

\begin{document}

\setlength{\belowdisplayskip}{6pt} \setlength{\belowdisplayshortskip}{6pt}
\setlength{\abovedisplayskip}{6pt} \setlength{\abovedisplayshortskip}{6pt}

\twocolumn[{
\renewcommand\twocolumn[1][]{#1}
\maketitle
\thispagestyle{fancy}
\centering
\begin{center}
\vspace{-1em}
\begin{minipage}{0.3\linewidth}
    \centering
    \def\svgwidth{1.0\linewidth}
\begingroup%
  \makeatletter%
  \providecommand\color[2][]{%
    \errmessage{(Inkscape) Color is used for the text in Inkscape, but the package 'color.sty' is not loaded}%
    \renewcommand\color[2][]{}%
  }%
  \providecommand\transparent[1]{%
    \errmessage{(Inkscape) Transparency is used (non-zero) for the text in Inkscape, but the package 'transparent.sty' is not loaded}%
    \renewcommand\transparent[1]{}%
  }%
  \providecommand\rotatebox[2]{#2}%
  \newcommand*\fsize{\dimexpr\f@size pt\relax}%
  \newcommand*\lineheight[1]{\fontsize{\fsize}{#1\fsize}\selectfont}%
  \ifx\svgwidth\undefined%
    \setlength{\unitlength}{127.73991797bp}%
    \ifx\svgscale\undefined%
      \relax%
    \else%
      \setlength{\unitlength}{\unitlength * \real{\svgscale}}%
    \fi%
  \else%
    \setlength{\unitlength}{\svgwidth}%
  \fi%
  \global\let\svgwidth\undefined%
  \global\let\svgscale\undefined%
  \makeatother%
  \begin{picture}(1,0.86529855)%
    \footnotesize
    \lineheight{1}%
    \setlength\tabcolsep{0pt}%
    \put(0,0){\includegraphics[width=\unitlength,page=1]{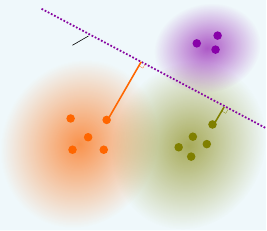}}%
    \put(0.05,0.68){\color[rgb]{0,0,0}\makebox(0,0)[lt]{hyperplane}}%
    \put(0.4,0.6){\color[rgb]{0,0,0}\makebox(0,0)[lt]{$d_\mathbb{E}^A$}}%
    \put(0.73,0.48){\color[rgb]{0,0,0}\makebox(0,0)[lt]{$d_\mathbb{E}^B$}}%
  \end{picture}%
\endgroup%
\\
    {\small (a) Embeddings in Euclidean space}
\end{minipage}
\hspace{5em}
\begin{minipage}{0.27\linewidth}
    \centering
    \def\svgwidth{1.0\linewidth}
\begingroup%
  \makeatletter%
  \providecommand\color[2][]{%
    \errmessage{(Inkscape) Color is used for the text in Inkscape, but the package 'color.sty' is not loaded}%
    \renewcommand\color[2][]{}%
  }%
  \providecommand\transparent[1]{%
    \errmessage{(Inkscape) Transparency is used (non-zero) for the text in Inkscape, but the package 'transparent.sty' is not loaded}%
    \renewcommand\transparent[1]{}%
  }%
  \providecommand\rotatebox[2]{#2}%
  \newcommand*\fsize{\dimexpr\f@size pt\relax}%
  \newcommand*\lineheight[1]{\fontsize{\fsize}{#1\fsize}\selectfont}%
  \ifx\svgwidth\undefined%
    \setlength{\unitlength}{123.27456449bp}%
    \ifx\svgscale\undefined%
      \relax%
    \else%
      \setlength{\unitlength}{\unitlength * \real{\svgscale}}%
    \fi%
  \else%
    \setlength{\unitlength}{\svgwidth}%
  \fi%
  \global\let\svgwidth\undefined%
  \global\let\svgscale\undefined%
  \makeatother%
  \begin{picture}(1,1)%
    \footnotesize
    \lineheight{1}%
    \setlength\tabcolsep{0pt}%
    \put(0,0){\includegraphics[width=\unitlength,page=1]{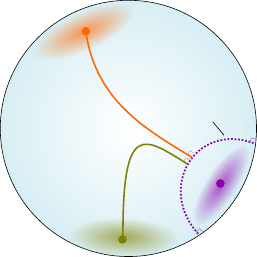}}%
    \put(0.7,0.6){\color[rgb]{0,0,0}\makebox(0,0)[lt]{gyroplane}}%
    \put(0.5,0.65){\color[rgb]{0,0,0}\makebox(0,0)[lt]{$d_\mathbb{H}^A$}}%
    \put(0.54,0.4){\color[rgb]{0,0,0}\makebox(0,0)[lt]{$d_\mathbb{H}^B$}}%
  \end{picture}%
\endgroup%
\\
    {\small (b) Embeddings in Poincaré ball}
\end{minipage}
\captionof{figure}{\textbf{Core idea.} Class embeddings in the Euclidean space \textit{(a)} exhibit non-uniform properties of the separation margin: the average distance of a pixel embedding of one class to the decision boundaries of the other classes varies substantially (\eg $d^A_\mathbb{E} > d^B_\mathbb{E} $).
This creates an implicit \emph{parent bias} in hierarchical segmentation, which prefers grouping one set of classes over the other, in terms of the parent-level segmentation accuracy.
In contrast, in hyperbolic space characterized by the Poincaré ball \textit{(b)}, the separation margins between the class embeddings are more uniform, \eg the hyperbolic distance of embeddings $A$ and $B$ of two different classes to the decision boundary (a gyroplane) of another class is approximately equal, $d^A_\mathbb{H} \approx d^B_\mathbb{H}$. This may explain the strong generalization of the parent-level predictions, observed in practice, in terms of the segmentation accuracy and calibration quality.}
\label{fig:teaser}
\vspace{0.5em}
\end{center}%
}]

\begin{abstract}
Hierarchy is a natural representation of semantic taxonomies, including the ones routinely used in image segmentation.
Indeed, recent work on semantic segmentation reports improved accuracy from supervised training leveraging hierarchical label structures.
Encouraged by these results, we revisit the fundamental assumptions behind that work.
We postulate and then empirically verify that the reasons for the observed improvement in segmentation accuracy may be entirely unrelated to the use of the semantic hierarchy.
To demonstrate this, we design a range of cross-domain experiments with a representative hierarchical approach.
We find that on the new testing domains, a flat (non-hierarchical) segmentation network, in which the parents are inferred from the children, has superior segmentation accuracy to the hierarchical approach across the board.
Complementing these findings and inspired by the intrinsic properties of hyperbolic spaces, we study a more principled approach to hierarchical segmentation using the Poincaré ball model.
The hyperbolic representation largely outperforms the previous (Euclidean) hierarchical approach as well and is on par with our flat Euclidean baseline in terms of segmentation accuracy.
However, it additionally exhibits surprisingly strong calibration quality of the parent nodes in the semantic hierarchy, especially on the more challenging domains.
Our combined analysis suggests that the established practice of hierarchical segmentation may be limited to in-domain settings, whereas flat classifiers generalize substantially better, especially if they are modeled in the hyperbolic space.\blfootnote{Project code: \href{https://github.com/tum-vision/hierahyp}{https://github.com/tum-vision/hierahyp}}
\end{abstract}
\section{Introduction}
\label{sec:intro}

Semantic knowledge is inherently structured, and organizing it in a hierarchy is both natural and expressive.
Unsurprisingly, hierarchical representations play an important role in computer vision \cite{li2020deep,liang2018dynamic,wang2019learning,meletis2018training,xiao2018unified,wang2020hierarchical}.
For instance, we may want to assign multiple labels to each pixel in the image, rather than a single one, to encode ancestral relations (\eg a ``car'' is also a ``vehicle'' and a ``means of transport'').
Adhering to a tree-based label hierarchy, this pixelwise classification task defines the so-called \emph{hierarchical} semantic segmentation and is the subject of this work.

In the literature, recent work addresses this problem as a supervised multi-label classification task \cite{Li:2022:DHS,Li:2023:LPV}.
In this formulation, the terminal leaf nodes and the internal nodes are modeled with individual one-\textit{vs}-all classifiers.
Remarkably, the empirical results of this approach appear to even exceed the standard supervised formulation (which only considers the leaf categories) in the evaluation of segmentation accuracy \emph{over the leaf categories themselves}.
Such an effect cannot be explained from the perspective of a learning algorithm, for which the semantic grouping of leaf nodes into parent classes is meaningless.

As our first step, we examine this phenomenon and reveal limited generalization of a state-of-the-art method for hierarchical semantic segmentation \cite{Li:2022:DHS}.
Rather surprisingly, we find that a \emph{flat} classifier, which only learns to classify the child (leaf) categories, largely outperforms the more sophisticated prior art on \emph{both the child and parent classes}.
We formulate the sufficiency of flat classifiers under the existing formulation of the hierarchical semantic segmentation, which establishes the link between model calibration and accuracy on the hierarchical prediction task.

Moving forward, we identify an inherent bias of flat classifiers toward particular groupings of child categories into parent meta-classes, as illustrated in \cref{fig:teaser}.
Specifically, the Euclidean distance between a decision boundary of one class and the class embeddings of the other categories is non-uniform.
However, classification errors of deep classifiers tend to occur near decision boundaries.\footnote{This is evidenced by the imperfect, yet fairly high calibration quality of segmentation networks, as we will also show in the experiments.}
This implies that defining a parent class comprising the two classes with the lowest separation margin will tend to produce a lower error rate in the parent-level classification, if we were to combine two classes with the largest margin of separation.
To mitigate this \emph{parent bias}, we would like the decision boundary of any class to be equidistant to the embeddings of other classes.
While additional regularization may be necessary to achieve this in the Euclidean space, we find that hyperbolic spaces provide such capacity naturally.
We embed pixel features in the Poincaré ball instead of the Euclidean space, which allows us to alleviate the parent bias and achieve a notable improvement in segmentation accuracy and calibration on the parent-level classification task.

We summarize our contributions as follows.
\textit{(i)} We reveal limited generalization of prior work on the hierarchical semantic segmentation task.
\textit{(ii)} Through a systematic analysis, we establish the sufficiency of flat classifiers for this task, which in Euclidean embedding space, however, may suffer from suboptimal accuracy on the parent classes.
\textit{(iii)} We show that the intrinsic properties of hyperbolic space, the Poincaré ball model, allow for mitigating this bias.
\textit{(iv)} We experimentally confirm the strong generalization of the Poincaré ball model, in terms of segmentation accuracy and calibration, especially on parent categories.

\section{Related Work}
\label{sec:related}

Research on semantic segmentation spans numerous problem domains, including deep network architectures \cite{Ronneberger:2015:UNC,Shelhamer:2017:FCN,Chen:2018:EDA,Cheng:2021:PPC}, training objectives \cite{Sudre:2017:GDO,Berman:2018:TLS,Cheng:2021:BIU} and strategies \cite{Luc:2016:SSA,Mittal:2021:SSS,Ouali:2020:SSS}.
Here, we are specifically interested in hierarchical semantic segmentation and, more generally, the hierarchical classification problem.
Therefore, our review of related literature below will focus only on these aspects, and we refer interested readers to surveys for a comprehensive overview of semantic segmentation research \cite{Mettes:2023:HDL}.

\inparagraph{Hierarchical classification with tree-like taxonomies.}
A multi-label classification problem is considered hierarchical if the label assignment must respect a pre-defined hierarchy \cite{Gordon:1987:RHC,Silla:2011:ASH}.
Hierarchical classifiers may be categorized into flat, local and global approaches.
\emph{Flat} classifiers only model the leaf nodes, thus completely ignoring the class hierarchy.
Following the tree structure in the bottom-up fashion, one can infer the parent label from the predictions of its children.
By contrast, \emph{global} (or ``big-bang'') methods explicitly represent each node in the tree, for example with a one-\textit{vs}-all classifier per node \cite{Kiritchenko:2006:LEP}.
Local approaches solve a number of smaller classification problems using only the local information available at each node or tree level \cite{Koller:1997:HCD,Eisner:2005:IPF}.
The success of these strategies appears to be domain-specific.
However, it is notable that flat classifiers are generally seen as competitive baselines \cite{Wang:2010:FHL,Babbar:2013:OFV,Valmadre:2022:HCM} -- the conclusion reached in our study too.
Learning individual classifiers for internal (non-leaf) nodes in the hierarchy may reduce semantically critical prediction errors \cite{Bertinetto:2020:MBM,Frome:2013:DVA}.

As a remark, hierarchical prediction has been the subject of research on problems in natural language processing and bionformatics \cite{Silla:2011:ASH},
where it is not uncommon to have large taxonomies -- in the order of tens or hundreds of thousands of labels \cite{Vens:2008:DTH,Partalas:2015:LSH}.
By contrast, a typical size of the label space in computer vision is substantially smaller \cite{Dimitrovski:2011:HAM}, especially for dense tasks, such as semantic segmentation considered here (\eg up to 30 in Cityscapes \cite{Cordts:2016:TCD}).

\inparagraph{Hierarchical semantic segmentation.}
Considering a hierarchy of image segments is a classic concept in computer vision.
Hierarchical image parsing helps to improve robustness to (self-)occlusions and to variation in object scale of early object recognition systems \cite{Schnitzspan:2009:DSL,Zhu:2010:LHS,Arbelaez:2011:CDH,Uijlings:2013:SSO}.
Similar to conditional random fields (CRFs) \cite{Cordts:2017:TSM}, deep networks can also benefit from hierarchical representations for advanced contextual reasoning \cite{Girshick:2014:RFH,Sharma:2015:DHP,Xu:2022:GVT}.

In contrast to the earlier work, where the hierarchy plays a facilitating role, training deep semantic segmentation networks producing a hierarchical label structure is relatively recent \cite{Li:2022:DHS,Li:2023:LPV}.
HSSN~\cite{Li:2022:DHS}, which we extensively use in our analysis, formalizes the hierarchical prediction task with auxiliary ``parent'' logits.
Note that this implies a training objective with more decision boundaries to learn than in the standard (child-only) case, since each parent logit requires a one-\textit{vs}-all classifier.
In a follow-up work, LogicSeg~\cite{Li:2023:LPV} formulates Boolean rules describing the hierarchical constraints and maps them to a differentiable loss using fuzzy logic.
While somewhat elegant, this approach does not improve over HSSN empirically in a significant way.

\inparagraph{Hyperbolic computer vision.} 
Deep learning on hyperbolic manifolds is still in its nascent stage \cite{Mettes:2023:HDL}.
In contrast to the Euclidean space, hyperbolic spaces possess properties lending themselves well to embedding hierarchical representations with minimal distortion \cite{Bridson:2013:MSN,Nickel:2017:PEL,Sala:2018:RTH}.
Previous research concentrated on generalizing the Euclidean models operating on the hyperbolic manifold, in terms of network models \cite{ganea2018,Ermolov:2022:HVT,Spengler:2023:PRN} and training specifics \cite{Guo:2022:CHC}.
Exploiting the properties of the hyperbolic embedding space has been of primary interest in (self-supervised) metric learning \cite{Suris:2021:LPF,Hsu:2021:CIH}.
Against the backdrop of this work, semantic segmentation has been studied rather marginally.
In a seminal work in this domain, \citet{Atigh:2022:HIS} learn pixel embeddings on the Poincaré ball and reach competitive segmentation accuracy \wrt the more established Euclidean formulation.
Concurrently, \citet{Franco:2023:HAL} report the correlation of the embedding norm with uncertainty in the context of active segmentation learning.

Overall, despite the recent progress, the benefits of the hyperbolic representation for semantic segmentation remain unclear.
Our study of hierarchical semantic segmentation exemplifies some compelling advantages of the Poincaré ball model, both theoretically and experimentally.

\section{The problem and motivation}
\label{sec:motivation}

Let us revisit the formulation of the hierarchical semantic segmentation problem from previous work \cite{Li:2022:DHS,Li:2023:LPV}, which we follow in our study.
Our label space is defined as the union of semantic categories at multiple levels of the semantic hierarchy, $\mathcal{S} := \cup_n \mathcal{S}_n$, where $\mathcal{S}_0$ defines the leaf classes.
Learning a semantic segmentation model with the finest label space, $\mathcal{S}_0$, reduces the problem to the conventional supervised scenario \cite{Shelhamer:2017:FCN,Chen:2018:EDA}, since it defines the granularity limit set by the available annotation in a given benchmark.
In addition to $\mathcal{S}_0$, we construct $\mathcal{S}_1$ by defining ``meta-classes'', which \emph{semantically} agglomerate one or more categories from $\mathcal{S}_0$ into a common parent class.
In Cityscapes \cite{Cordts:2016:TCD}, for example, one defines a parent class ``Human'' comprising child classes ``Person'' and ``Rider''.
While one could create deep hierarchical structures, the limited annotation in semantic segmentation only allows for hierarchies up to $n = 2$ -- more rarely $n = 3$, in practice \cite{Li:2022:DHS}.
With the label hierarchy thus defined, our goal now is to maximize the segmentation accuracy (\eg mIoU or mean pixel accuracy), evaluated separately for each level of the tree.

\begin{figure}
\centering
  \begin{subfigure}{\linewidth}
   \centering
   \def\svgwidth{1.0\linewidth}
  \input{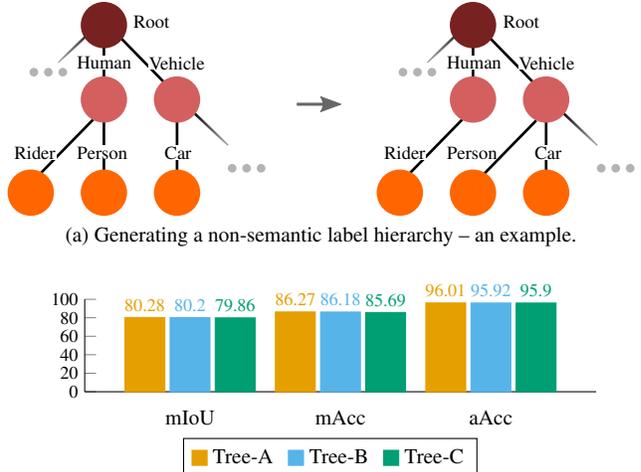}
  \caption{Generating a non-semantic label hierarchy -- an example.}
  \label{fig:hierarchy}
  \end{subfigure}
    
  \vspace{1em}

\begin{subfigure}{\linewidth}
 \centering
     \footnotesize
    \begin{tikzpicture}[every node/.style={font=\footnotesize}]
        \pgfplotstableread[row sep=\\,col sep=&]{
            name & acc  \\
            mIoU & 80.28 \\
            mAcc & 86.27 \\
            aAcc & 96.01 \\
        }\dataA
        
        \pgfplotstableread[row sep=\\,col sep=&]{
            name & acc  \\
            mIoU & 80.20 \\
            mAcc & 86.18 \\
            aAcc & 95.92 \\
        }\dataB
    
        \pgfplotstableread[row sep=\\,col sep=&]{
            name & acc  \\
            mIoU & 79.86 \\
            mAcc & 85.69 \\
            aAcc & 95.90 \\
        }\dataC
    
        \definecolor{city}{RGB}{230,159,0}
        \definecolor{coco}{RGB}{86, 180, 233}
        \definecolor{pots}{RGB}{0, 158, 115}
        
        \begin{axis}[
            width=\linewidth,
            height=10em,
            ybar,
            bar width=0.53cm,
            x=2cm,
            enlarge x limits=0.35,
            ymin=0, ymax=100,
            axis y line*=left,
            axis x line*=bottom,
            symbolic x coords={mIoU, mAcc, aAcc},
            xtick=data,
            ytick pos=left,
            xtick style={draw=none},
            x tick label style={text width=2cm,align=center, scale=1, transform shape},
            y tick label style={scale=0.85},
            legend style={at={(0.48,-0.5)},draw=black,anchor=north,legend columns=-1, scale=0.85},
            nodes near coords,
            nodes near coords style={font=\rm},
            every node near coord/.append,
            every node near coord/.append style={scale=0.85,  /pgf/number format/assume math mode=true},
            ytick={
            	0, 20, 40, 60, 80, 100
            },
            yticklabels={
            	0, 20, 40, 60, 80, 100
            },
            legend image code/.code={
            \draw [#1] (0cm,-0.1cm) rectangle (0.2cm,0.1cm); },
            ]
            \addplot[city,fill=city] table[x=name,y=acc]{\dataA};
            \addplot[coco,fill=coco] table[x=name,y=acc]{\dataB};
            \addplot[pots,fill=pots] table[x=name,y=acc]{\dataC};
            \legend{Tree-A\:\:,Tree-B\:\:,Tree-C}
        \end{axis}
    \end{tikzpicture}
    \caption{Segmentation accuracy does not change between semantic (Tree-A) and non-semantic hierarchies (Tree-B, Tree-C).}
    \label{fig:hierarchy_eucl}
  \end{subfigure}
  \caption{\textbf{Training with non-semantic hierarchies.} We train semantic segmentation models with non-semantic trees. \textit{(a)} For example, we define class ``Person'' as a parent of a ``vehicle'', which is clearly semantically meaningless. \textit{(b)} Non-semantic hierarchies (``Tree-B'' and ``Tree-C'') do not affect the segmentation accuracy of semantic hierarchy (``Tree-A'') in a significant way.}
  \label{fig:hierarchy_combined}
\end{figure}

Empirical observations from previous work \cite{Li:2022:DHS,Li:2023:LPV} suggest that defining the meta-classes, \eg $\mathcal{S}_1$, improve the accuracy of the children categories $\mathcal{S}_0$.
\emph{What explains this phenomenon?}
After all, the only additional supervision signal in our new hierarchical formulation is the semantic proximity of some classes in $\mathcal{S}_0$.
However, this does not immediately render the learning problem easier.
In fact, since we add an additional classification problem over categories in $\mathcal{S}_1$, optimization may become even more difficult.

The hypothesis that the semantic proximity between child categories provides a complementary supervision signal asks for empirical validation.
We design meta-classes in a semantically meaningless fashion and train a DeepLabv3+ with ResNet-101 backbone \cite{Chen:2018:EDA} following HSSN training objective \cite{Li:2022:DHS} on Cityscapes \textit{train} \cite{Cordts:2016:TCD}.
For example, we define ``Car'' as a child of ``Human'', and ``Terrain'' as a sibling of ``Sky''.
Observing the results in \cref{fig:hierarchy_eucl} on Cityscapes \textit{val}, we conclude that such semantically incoherent definitions of new meta-classes do not have much effect on the segmentation accuracy of the leaf categories (\ie the standard 19 semantic classes in Cityscapes \cite{Cordts:2016:TCD}, which remained unchanged).
This experiment suggests that the empirical benefits reported by \citet{Li:2022:DHS,Li:2023:LPV} may not be related to the \emph{semantic} definition of the label hierarchy.

\subsection{Flat classifiers are strong baselines}

As discussed in \cref{sec:related}, \emph{flat} classifiers offer a reasonable sanity check regardless of the problem domain.
Indeed, we found that the HSSN model exhibits a strong bias towards the specific traffic scenes of the training domain (Cityscapes), while performing poorly on the novel domains.
By contrast, a flat classifier consistently outperforms HSSN on the test datasets (\cf \cref{table:ood_deeplab}.).
These results may not strike us as surprising and are in line with our intuition developed in the previous section.
A less expected result, however, is that we also observed superior accuracy of flat classifiers on the parent categories.
Recall that we represent the semantic hierarchical relations as a fixed \emph{stationary} tree, defined by three properties: \textit{(i)} Every level of the label hierarchy forms a categorical distribution. 
\textit{(ii)} A prediction at a terminal node determines the complete hierarchical label (due to the uniqueness of the path to the root).
\textit{(iii)} The label hierarchy remains unchanged in training and testing.
As we are dealing with a closed-world taxonomy, it is straightforward that the conditional probability of a parent class can be expressed with only the conditional probabilities of its children, which we formalize as follows:

\begin{proposition}\label{prop:infer_parent}
Let $y$ be the parent class and $m_y$ denote the children of that class in image $\mathcal{I}$.
Given a stationary tree defined above, the optimal class posterior $p(\hat{y} = y \mid \mathcal{I})$ for a parent node $y$ with child class posterior $p(\hat{m}_{y} = m_y \mid \mathcal{I})$ is given by
\begin{align}\label{eq:parent_posterior}
\begin{split}
p(\hat{y} = y \mid I) & = \sum_m p(\hat{y} = y \mid m) \; p(\hat{m} = m \mid \mathcal{I}) \\ & = \sum_{m_y} p(\hat{m}_y = m_y \mid \mathcal{I}).
\end{split}
\end{align}
\end{proposition}

Note that prior work on hierarchical semantic segmentation \cite{Li:2022:DHS,Li:2023:LPV} independently models the parent posterior $p(\hat{y} = y \mid \mathcal{I})$.
The above proposition states the sufficiency of predicting the child nodes independently from the parent nodes, and then inferring the parent posterior with \cref{eq:parent_posterior}.

\subsection{The parent bias in Euclidean space}

Class embeddings produced in the Euclidean space tend to produce a multi-modal distribution, where each class forms an independent mode (a cluster).
The modes exhibit a particular rank-based arrangement and we can loosely establish that, for example, class ``A'' is closer to class ``B'', in terms of the Euclidean distance (\cf \cref{sec:norms} for empirical support).
Since typical classification errors occur at the decision boundaries, the spatial proximity of two class embeddings presents an inherent bias of that embedding space.
For instance, if classes ``A'' and ``B'' are close in the embedding space and end up in different parent categories, this will lead to suboptimal accuracy on the parent level.
Without any prior on the parent taxonomy, which is task-specific and would lead to the parent bias, we can encourage the modes of our class embeddings to be approximately equidistant.
In the Euclidean space, this would require additional regularization.
However, it can emerge naturally in hyperbolic space -- in the Poincaré ball model.  

\section{From Euclidean to hyperbolic geometry}

\inparagraph{Segmentation in Euclidean space.}
Let us formalize the training process of a deep network for semantic segmentation in Euclidean space.
Given an image $\mathcal{I} \in \mathbb{R}^{H \times W \times 3}$, we would like to predict label $l \in \mathcal{S}$ for each pixel $i \in \{1, ..., HW \}$, where $\mathcal{S}$ is the set of $\lvert \mathcal{S} \rvert$ class labels.
An encoder $f_\theta$ with parameters $\theta$ maps $\mathcal{I}$ to a set of pixel feature embeddings $\mathbf{X} = ( \mathbf{x}_i )_{i=1}^{HW} = f(\mathcal{I}) \in \mathbb{R}^{HW \times d}$.
In the last layer, for each pixel $i$, we obtain a segmentation map, modeled as the class posterior,
\begin{equation}
p(\hat{l} = l \mid \mathbf{x}_i) \propto \exp(a_{l}^{\top} \mathbf{x}_i + b_{l}) \, .
\end{equation}
Here, $\{ (a_l, b_l) \}_{l=1}^{\lvert \mathcal{S} \rvert}$ are classifier parameters and define hyperplanes for each class $l \in \mathcal{S}$ in the Euclidean space.
During training, we jointly optimize for $\Theta := \big\{ \theta, \{ (a_l, b_l) \}_{l=1}^{\lvert \mathcal{S} \rvert} \big\}$ with backpropagation, minimizing the (expected) cross-entropy loss for each pixel $i$,
\begin{equation}
\underset{\Theta}{\text{min}} \quad \mathbb{E}_\mathcal{I} \big[ - \log{ p(\hat{l} = l_i^\ast \mid \mathbf{x}_i) } \big],
\end{equation}
where $l_i^\ast$ is the ground-truth label for pixel $i$.

By analogy with the Euclidean setup, the per-pixel classification in the hyperbolic space involves gyroplanes (\cf \cref{fig:teaser}), defined by offsets and normals \cite{ganea2018}.
Let us revisit some basic notions of hyperbolic geometry.

\subsection{The Poincaré ball model}
\label{sec:poincare_segmentation}

\inparagraph{Poincaré Ball and Exponential Map.}
Hyperbolic geometry can be expressed with different conformal models \cite{cannon97hyperbolic}. We operate on the Poincaré ball $(\mathbb{D}_{c}^{n}, g^{\mathbb{D}_{c}^{n}})$ with $-c$ denoting the negative curvature, and $g^{\mathbb{D}_{c}^{n}}$ being the Riemannian metric associated with the manifold $\mathbb{D}_{c}^{n} = \{ x \in \mathbb{R}^{n} \mid c \lVert x \rVert < 1\}$. $g^{\mathbb{D}_{c}^{n}}$ can be linked to the Euclidian metric tensor $g^{\mathbb{E}} = \mathbf{I}_n$ via:
\begin{equation}
g^{\mathbb{D}_{c}^{n}} = (\lambda_{x}^{c})^{2}g^{\mathbb{E}}= \bigg(\frac{2}{1 - c \lVert x \rVert^{2}}\bigg)^{2} g^{\mathbb{E}} \, .
\end{equation}

The exponential map between the Euclidean space $\mathbb{R}^{n}$ and the Poincaré ball $\mathbb{D}_{c}^{n}$ with anchor $v$ is defined as:
\begin{equation}
\text{Exp}_{v}^{c}(x) = v \oplus_{c} \bigg(\tanh(\sqrt{c}\frac{\lambda_{v}^{c}\lVert x \rVert}{2})\frac{x}{\sqrt{c}\lVert x \rVert}\bigg) \, ,
\end{equation}
where $\oplus_{c}$ is the Möbius hyperbolic addition defined as:
\begin{equation}
v \oplus_{c} w = \frac{(1 + 2c \langle v, w \rangle + c \lVert w \rVert ^{2})v + (1 - c\lVert v \rVert^{2})w}{1 + 2c \langle v, w \rangle + c^{2}\lVert v \rVert ^{2} \lVert w \rVert ^{2}} \, ,
\end{equation}
for all $v, w \in \mathbb{D}^{n}_{c}$.
For simplicity, $v$ is set to the origin $0$ and then the considered exponential map is:
\begin{equation}
\text{Exp}_{0}^{c}(x) = \tanh(\sqrt{c}\lVert x \rVert)\frac{x}{\sqrt{c}\lVert x \rVert} \, .
\end{equation}

\inparagraph{Hyperbolic Distance.}
The hyperbolic distance between $x$ and $z$ on the Poincaré ball is given by:
\begin{equation}\label{eq:hyperbolic_distance}
d_{\mathbb{H}}(x,z) = \text{arcosh}\bigg(1 + 2 \frac{\lVert x - z \rVert ^{2}}{(1 - \lVert x \rVert^{2})(1 - \lVert z \rVert^{2})}\bigg) \, .
\end{equation}

\inparagraph{Hyperbolic Multinomial Logistic Regression (MLR).}
Given a hyperbolic vector $h$ and $K$ classes, \citet{ganea2018} provide a geometric interpretation of the hyperbolic MLR by defining gyroplanes as:
\begin{equation}
H_{k}^{c} = \{ h \in \mathbb{D}_{c}^{n} \mid \langle -r_{k} \oplus_{c} h, a_{k} \rangle = 0\}\,,
\end{equation}
where $k \in \{1,...,K\}$ and $r_{k}$ and $a_{k}$ are respectively the gyroplane offset and the normal associated with class $k$.
The hyperbolic distance between $h$ and the gyroplane of class $k$ is given by:
\begin{equation}\label{eq:hyperdist}
d_{\mathbb{H}}(h,H_{y}^{c}) = \frac{1}{\sqrt{c}} \text{arcsinh}\bigg(\frac{2\sqrt{c} \lvert \langle -r_{y}\oplus_{c}h, w_{y} \rangle \rvert}{(1-c\lVert -r_{y}\oplus_{c} h\rVert^{2})\lVert w_{y} \rVert}\bigg) \,.
\end{equation}
Based on this distance, we define the hyperbolic logit:
\begin{equation}\label{eq:hyperbolic_logit}
\zeta_{y}(h) := \frac{\lambda_{r_{y}}^{c}\lVert w_{y} \rVert }{\sqrt{c}} \text{arcsinh}\bigg(\frac{2\sqrt{c} \langle -r_{y}\oplus_{c}h, w_{y} \rangle}{(1-c\lVert -r_{y}\oplus_{c} h\rVert^{2})\lVert w_{y} \rVert}\bigg) \,,
\end{equation}
and the likelihood,
\begin{equation}\label{eq:hyp_likelihood}
p(\hat{y} := y \mid h) \propto \exp(\zeta_{y}(h)) \, .
\end{equation} 

\subsection{Calibrated segmentation in the Poincaré ball}
Following \citet{Atigh:2022:HIS}, we perform the per-pixel classification in the hyperbolic space. We project $\mathbf{X}$ onto the Poincaré ball with the mapping $\text{Exp}_{0}^{c}(\cdot)$ to get the hyperbolic embedding $\mathbf{H} \in \mathbb{D}^{H \times W \times n}$. Then, we optimize the obtained likelihood (\cf \cref{eq:hyp_likelihood}) with the standard cross-entropy loss.
Once the classification is performed in the Poincaré ball, we link the hyperbolic logits (\cref{eq:hyperbolic_logit}) to the confidence of the prediction, by analogy with Euclidean networks \cite{guo2016}. 
To our knowledge, while this extension of Euclidean calibration to hyperbolic networks is straightforward, we are the first to present its experimental analysis.
Nevertheless, the analogy with Euclidean space has its limitations.
Focusing on the hyperbolic distance, we demonstrate its concave property \wrt Euclidean distance.
This property allows us to establish a distinguishing feature of the hyperbolic space in modeling inter-class relationships.

\subsection{From concavity to flattening bias}\label{subsec:flattening}

\begin{figure}
    \centering
    \def\svgwidth{1.0\linewidth}
    \input{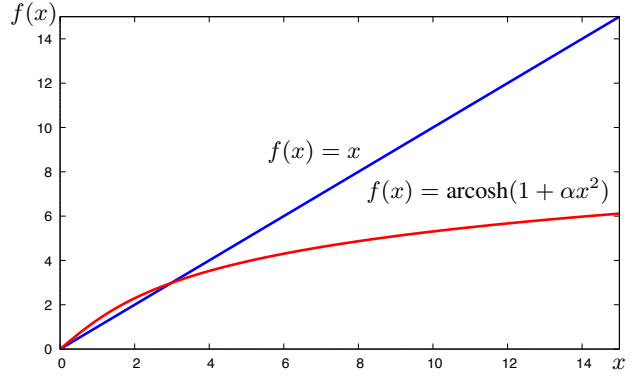} 
  \caption{\textbf{Hyperbolic distance exhibits strict concavity \wrt Euclidean distance.} Observe (assuming $\alpha=1$ for simplicity) that the hyperbolic distance has sublinear (logarithmic) growth. In practical terms, this implies that the difference in the distance between two hyperbolic representations as $x$ increases will diminish, while it remains constant in Euclidean space.}
  \label{fig:arcosh}
\end{figure}

\inparagraph{On the hyperbolic distance.}
Hyperbolic geometry naturally embeds hierarchical structure \cite{Sala:2018:RTH}.
However, we argue here that the hyperbolic space lends itself well also for \emph{flat} classification.
We observe that during training, the hyperbolic embeddings of the same class tend to be pushed to the periphery of the Poincaré ball and onto the same side of the associated gyroplane \cite{Atigh:2022:HIS}.
Studying inter-class hyperbolic distance (\cf \cref{eq:hyperdist}), we further find that embeddings of one class are approximately equidistant to the embeddings of any other class.
This implies that, in contrast to the Euclidean space, there is no parent bias in the Poincaré ball, \ie there is no preference for a specific grouping of child categories into parents.
The following proposition provides the formal argument explaining our observation:

\begin{proposition}\label{prop:concavity}
    The hyperbolic distance between two embeddings $h_{1},h_{2}$ is strictly concave in the Euclidean distance between $h_{1}$ and $h_{2}$.
\end{proposition}
\begin{proof}
We can write the hyperbolic distance $d_{\mathbb{H}}$ between two embeddings in the Poincaré ball as a function of the Euclidean distance $d_{\mathbb{E}}(h_{1},h_{2})$.
The derivative of $d_{\mathbb{H}}$ with respect to $d_{\mathbb{E}}(h_{1},h_{2})$ is then
\begin{equation}
\frac{2}{\sqrt{(1-\lVert h_{1} \rVert^{2})(1 - \lVert h_{2} \rVert^{2})} \sqrt{1 + \frac{d_{\mathbb{E}}^{2}}{(1-\lVert h_{1} \rVert^{2})(1-\lVert h_{2} \rVert^{2})}}} \, ,
\end{equation}
which is strictly decreasing in $d_{\mathbb{E}}$.
\end{proof}

\cref{fig:arcosh} illustrates this proposition. Given the concave property of the hyperbolic distance, we postulate that the Poincaré ball formulation facilitates distance uniformity between classes.
Since the hyperbolic embeddings are pushed to the border of the Poincaré ball during training, the norms $\| h_{1}\|$ and $\| h_{2}\|$ in \cref{eq:hyperbolic_distance} are close to $1$.
Therefore, we operate at the high-end spectrum of the domain in \cref{fig:arcosh} (see \cref{sec:norms} for further details).
Consequently, changes in the class embeddings have a diminished effect on the distance in the Poincaré ball, whereas in the Euclidean space, this relationship is linear.
Demonstrating a practical consequence, the next section experimentally confirms that the hyperbolic space leads to strong parent-level generalization in semantic segmentation.

\begin{table}
\footnotesize
\sisetup{detect-weight=true,detect-inline-weight=math}
\setlength\tabcolsep{3pt}
\begin{tabularx}{\linewidth}{
cX
R@{\hspace{1em}}B@{\hspace{1.8em}}
R@{\hspace{1em}}B
}
\toprule
\parbox[t]{2mm}{\multirow{2}{*}{\rotatebox[origin=c]{90}{Dataset}}}
 & \multirow{2}{*}{Method}
 & \multicolumn{4}{c}{cwECE} \\
 
\cmidrule(lr){3-6}
 & & \multicolumn{2}{c}{Level 1} & \multicolumn{2}{c}{Level 0} \\
\midrule
\parbox[t]{2mm}{\multirow{4}{*}{\rotatebox[origin=c]{90}{Cityscapes}}} &
HSSN$^\ast$ \cite{Li:2022:DHS} & 0.90 & {--} &  5.97 &  {--}  \\
& HSSN \cite{Li:2022:DHS} & 0.79 & 0.79 & 6.85 & 5.09  \\
& \oureuc  & \bfseries 0.52 & \bfseries 0.60 & 4.40 & \bfseries 3.97  \\
& \ourhyp  & 0.66 & 0.73 & \bfseries 4.35 & 4.09  \\
\midrule

\parbox[t]{2mm}{\multirow{4}{*}{\rotatebox[origin=c]{90}{Mapillary}}} &
HSSN$^\ast$ \cite{Li:2022:DHS} & 4.26 & {--} & 17.39 &  {--}  \\
& HSSN \cite{Li:2022:DHS} & 4.47 & 4.35 & 18.44 & 15.59 \\
& \oureuc  & \bfseries 2.84 & 3.62 & \bfseries 11.03 & \bfseries 11.26  \\
& \ourhyp  & 2.88 & \bfseries 3.38 & 11.44& 11.42   \\

\midrule
\parbox[t]{2mm}{\multirow{4}{*}{\rotatebox[origin=c]{90}{IDD}}} &
HSSN$^\ast$ \cite{Li:2022:DHS} & 7.85 & {--} &  17.74 &  {--}  \\
& HSSN \cite{Li:2022:DHS} & 7.86 & 7.56 & 19.39 & 17.24   \\
& \oureuc  & \bfseries 5.94 & 7.24 & \bfseries \bfseries 12.77 & 13.43 \\
& \ourhyp  & 6.05 & \bfseries 6.27 & 13.37 & \bfseries 10.87   \\
\midrule

\parbox[t]{2mm}{\multirow{4}{*}{\rotatebox[origin=c]{90}{ACDC}}} &
HSSN$^\ast$ \cite{Li:2022:DHS} & 8.12 & {--} & 23.54 &  {--}  \\
& HSSN \cite{Li:2022:DHS} & 6.78 & 8.40 & 21.18 & 20.92  \\
& \oureuc  & 6.55 & \bfseries 5.96 & \bfseries 16.63 & 16.49  \\
& \ourhyp  & \bfseries 6.34 & 6.03 & 17.58 & \bfseries 12.64   \\
\midrule

\parbox[t]{2mm}{\multirow{4}{*}{\rotatebox[origin=c]{90}{BDD}}} &
HSSN$^\ast$ \cite{Li:2022:DHS} & 8.68 & {--} &  26.37 &  {--}  \\
& HSSN \cite{Li:2022:DHS} & 8.82 & 6.77 & 29.13 & 23.71  \\
& \oureuc  & 7.40 & 7.05& 21.44 & 20.95  \\
& \ourhyp  & \bfseries 6.55 & \bfseries 6.09 & \bfseries 20.71 & \bfseries 18.58   \\
\midrule

\parbox[t]{2mm}{\multirow{4}{*}{\rotatebox[origin=c]{90}{Wilddash}}} &
HSSN$^\ast$ \cite{Li:2022:DHS} & 14.28 & {--} &  32.32 &  {--}  \\
& HSSN \cite{Li:2022:DHS} & 14.95 & 11.48 & 33.21 & 28.23  \\
& \oureuc  & 13.53 & 13.16 & 23.25 & 23.93  \\
& \ourhyp  & \bfseries 9.52 & \bfseries 10.04 & \bfseries 20.17& \bfseries 18.45  \\

\bottomrule

\end{tabularx}
\setlength{\fboxsep}{0.2pt}
\caption{\textbf{Calibration quality (cwECE).} We train \colorbox{lightred}{DeepLabv3+/ResNet-101} and \colorbox{lightblue}{OCRNet/HRNet-W48} on Cityscapes (train) and report the results for six datasets. HSSN$^\ast$ corresponds to the pretrained model provided by the authors, only available for DeepLabV3+.}
\label{table:ood_calib}
\end{table}
\section{Experiments}

\inparagraph{Datasets.}
Different from prior work, we perform our analysis by testing the models on 6 datasets: Cityscapes \cite{cordts2016cityscapes}, Mapillary \cite{neuhold2017mapillary}, IDD \cite{varma2019idd}, ACDC \cite{sakaridis2021acdc}, BDD \cite{yu2020bdd100k} and Wilddash \cite{zendel2018wilddash}.
Note that since we train our models on Cityscapes, the datasets with a larger visual domain shift (ACDC, BDD and Wilddash) are more challenging.

\inparagraph{Metrics.}
We compare the models in terms of segmentation accuracy and calibration quality, on both the child and parent nodes (19 and 7 classes, respectively).
To evaluate the accuracy, we use the mean Intersection-over-Union (mIoU), the mean accuracy over classes (mAcc) and the average pixel accuracy (aAcc).
We follow \citet{kull2019beyond} to derive a calibration metric by comparing the difference between the confidence and accuracy of class predictions.
We report the class-wise expected calibration error (cwECE).

\inparagraph{Models.}
We use DeepLabV3+ with ResNet-101 backbone \cite{chen2018encoder} and OCRNet with HRNet-W48 backbone \cite{yuan2020object}, with the backbones pre-trained on ImageNet \cite{Deng:2009:INA}.
We train each model on Cityscapes \textit{train} with fine annotations \cite{cordts2016cityscapes} for 80K iterations. 
Leveraging the child class posterior probabilities (Level 0), flat models compute the parent posterior (Level 1) using \cref{eq:parent_posterior}.

\begin{table*}
\footnotesize
\sisetup{detect-weight=true,detect-inline-weight=math}
\setlength\tabcolsep{3pt}
\begin{tabularx}{\linewidth}{
cX
R@{\hspace{1em}}B@{\hspace{1.8em}}
R@{\hspace{1em}}B@{\hspace{1.8em}}
R@{\hspace{1em}}B@{\hspace{1.8em}}
R@{\hspace{1em}}B@{\hspace{1.8em}}
R@{\hspace{1em}}B@{\hspace{1.8em}}
R@{\hspace{1em}}B
}
\toprule
\parbox[t]{2mm}{\multirow{2}{*}{\rotatebox[origin=c]{90}{Dataset}}}
 & \multirow{2}{*}{Method}
 & \multicolumn{6}{c}{Level 1} 
 & \multicolumn{6}{c}{Level 0} \\

\cmidrule(lr){3-8}\cmidrule(lr){9-14}
 & & \multicolumn{2}{c}{mIoU} & \multicolumn{2}{c}{mAcc} & \multicolumn{2}{c}{aAcc}  & \multicolumn{2}{c}{mIoU} & \multicolumn{2}{c}{mAcc} & \multicolumn{2}{c}{aAcc}  \\
\midrule
\parbox[t]{2mm}{\multirow{4}{*}{\rotatebox[origin=c]{90}{Cityscapes}}} &
HSSN$^\ast$ \ \cite{Li:2022:DHS} & 90.82 & {--}  & 94.92 & {--}  & 97.35 & {--}  & \bfseries 81.62 & {--}  &\bfseries 87.90 & {--}  & 96.16 &  {--} \\

& HSSN \cite{Li:2022:DHS} & 90.68 & 90.27 & 94.47 & 94.25 & 97.29 &  97.20 & 80.30 & \bfseries 80.21 & 86.11 & 86.62 & 95.93& 95.98 \\

& \oureuc  &  \bfseries 90.96 &  \bfseries90.57 & \bfseries95.24 & \bfseries 94.80 &   \bfseries 97.64 & \bfseries 97.52  &  80.89 & 79.82 & 87.53 & \bfseries 87.34 & \bfseries 96.56 & \bfseries 96.27  \\

& \ourhyp  & 90.91 & 90.47 & 95.08 & 94.67  & 97.60 & 97.48 & 80.36 & 79.28 & 86.83 & 86.80 & 96.41 & 96.26 \\

\midrule
\parbox[t]{2mm}{\multirow{4}{*}{\rotatebox[origin=c]{90}{Mapillary}}} &
HSSN$^\ast$ \ \cite{Li:2022:DHS} & \bfseries 83.26 & {--}  & 89.37 & {--}  & \bfseries 94.70 & {--}  & 62.32 & {--}  & 72.53 & {--}  & 90.37 &  {--} \\

& HSSN \cite{Li:2022:DHS} & 81.89 & 79.09 & 88.35 & 86.19 &93.77 & 91.91 & 59.34 & 59.09 &70.37 & 70.62 & 89.53& 87.54  \\

& \oureuc  &  83.11 &  80.27 & \bfseries 90.73 & 89.14 &   93.98 & 93.7  &  \bfseries 63.94 & \bfseries60.47 & \bfseries 76.44 & \bfseries 74.62 & \bfseries 90.47 & 90.12  \\

& \ourhyp  & 81.87 & \bfseries 81.89 & 90.15 & \bfseries 89.23 & 93.41 & \bfseries 94.9 &  60.34 & 58.66 & 74.88 & 73.55 & 89.60 &\bfseries 90.96 \\

\midrule
\parbox[t]{2mm}{\multirow{4}{*}{\rotatebox[origin=c]{90}{IDD}}} &
HSSN$^\ast$ \ \cite{Li:2022:DHS} & 79.03 & {--}  & 84.37 & {--}  & 94.46 & {--}  &  58.33 & {--}  & 70.05 & {--}  & 91.57 &  {--} \\

& HSSN \cite{Li:2022:DHS} & 78.52 & 77.69 & 84.00 & 83.01 & 94.45 & 94.29 & 55.21 & 58.14 &67.37 & 67.21 & 90.57& 91.79  \\

& \oureuc  &  \bf81.27 &  79.30 & \bfseries 87.08 & 85.03 &   95.29 & 94.90  &  \bfseries 61.64 & 58.50 & \bfseries 73.70 & 71.65 & \bfseries 92.22 & 91.89  \\

& \ourhyp  & 80.98 & \bfseries 79.83 & 86.70 & \bfseries 85.25 & \bfseries 95.54 & \bfseries 95.27 &  58.76 & \bfseries 59.01 & 71.55 & \bfseries 72.88 & 91.97 &\bfseries 92.14  \\
\midrule

\parbox[t]{2mm}{\multirow{4}{*}{\rotatebox[origin=c]{90}{ACDC}}} &
HSSN$^\ast$ \ \cite{Li:2022:DHS} & 65.56 & {--}  & 78.60 & {--}  & 86.62 & {--}  &  42.97 & {--}  & 57.33 & {--}  & 81.62 &  {--} \\

& HSSN \cite{Li:2022:DHS} & 73.45 & 69.43 & 82.00 & 78.78 & 90.64 & 87.41 & 52.71 & 49.20 & 62.62 & 61.40 & 84.95& 82.10  \\

& \oureuc  &  \bfseries 75.18 &  \bfseries 71.29 & \bfseries 83.86 & \bfseries 83.47 &   \bfseries 91.80 & \bfseries 90.38  &  \bfseries 54.34 & \bfseries 51.90 & \bfseries 66.75 &  64.56 & \bfseries 86.79 & \bfseries 85.99 \\

& \ourhyp  & 72.90 & 70.88 & 83.66 & 82.58 & 91.78 & 90.20 &  47.72 & 49.31  & 62.46 & \bfseries 67.19 & 85.37 &85.85  \\
\midrule

\parbox[t]{2mm}{\multirow{4}{*}{\rotatebox[origin=c]{90}{BDD}}} &
HSSN$^\ast$ \ \cite{Li:2022:DHS} & 73.53 & {--}  & 82.25 & {--}  & 89.60 & {--}  &  48.32 & {--}  & 60.54 & {--}  & 86.76 &  {--} \\

& HSSN \cite{Li:2022:DHS} & 74.28 & 73.66 & 81.88 & 81.10 & 90.11 &90.08 & 47.75 & 48.22 & 57.96 & 60.18 & 86.85& 87.26  \\

& \oureuc  &  76.27 &  74.47 & 84.64& 83.52 &   91.22 & 90.62  & \bfseries 51.54 & \bfseries 50.04 & \bfseries 64.58 & 63.12 & 88.52 & 88.05  \\

& \ourhyp  & \bfseries 76.49 & \bfseries 76.62 & \bfseries 85.12 &\bfseries 84.34 & \bfseries 91.91 & \bfseries 92.12 & 49.64 & 49.60 & 63.21 & \bfseries 64.37 & \bfseries 89.00 &\bfseries 89.31  \\
\midrule

\parbox[t]{2mm}{\multirow{4}{*}{\rotatebox[origin=c]{90}{Wilddash}}} &
HSSN$^\ast$ \ \cite{Li:2022:DHS} & 57.20 & {--}  & 71.42 & {--}  & 76.85 & {--}  &  36.55 & {--}  & 50.61 & {--}  & 71.62 &  {--} \\

& HSSN \cite{Li:2022:DHS} & 58.60 & 59.80 & 71.65 & 71.22 & 76.49 &79.74 & 37.07 & 39.01 & 50.03 & 50.86 & 71.12& 74.42  \\

& \oureuc  &  58.98 & 57.11 & 74.15 & 72.64 &   76.75 & 75.74  &  39.38 & \bfseries 39.97 & 55.97 & 54.31 & 71.84 & 70.84 \\

& \ourhyp  & \bfseries 62.53 & \bfseries 62.52 & \bfseries 77.02 & \bfseries 75.77 & \bfseries 81.91 & \bfseries 81.88 &  \bfseries 40.57 & 39.48 & \bfseries 57.17 & \bfseries 57.57 & \bfseries 76.29 &\bfseries 76.35  \\

\bottomrule
\end{tabularx}
\setlength{\fboxsep}{0.2pt}
\caption{\textbf{Segmentation accuracy (mIoU, mAcc).} We train \colorbox{lightred}{DeepLabv3+/ResNet-101} and \colorbox{lightblue}{OCRNet/HRNet-W48} on Cityscapes \textit{train} and test them on six datasets. HSSN$^\ast$ corresponds to the pretrained model given by the authors, only available for DeepLabV3+.}
\label{table:ood_deeplab}
\end{table*}

\inparagraph{Implementation details.}
For a fair comparison, we follow HSSN \cite{Li:2022:DHS} to set the training hyper-parameters. 
For the hyperbolic networks, we use the optimization method in \cite{ganea2018}.
We optimize the offsets with Riemannian SGD \cite{bonnabel2013stochastic}, and set the learning rate to $0.0001$.
The normals are optimized with SGD in the Euclidean space, with a learning rate $0.001$.
The projection onto the Poincaré ball uses the Geoopt library \cite{geoopt2020kochurov}, setting curvature to $c = 1$.

\subsection{Quantitative results} 

 \cref{table:ood_calib} and \cref{table:ood_deeplab} report the calibration quality and the segmentation accuracy, respectively, for the hierarchical training (HSSN \cite{Li:2022:DHS}), and the flat Euclidean (\eucname) and hyperbolic models (\hypname). To ensure a fair comparison, we train all models with a consistent codebase and training schedule.
 For reference, we also report the results for HSSN$^\ast$ with DeepLabV3+ using the weights provided by the authors \cite{HSSN:2008:Repo}.
 A pre-trained model for OCRNet is not available. 
 Level 1 and Level 0 refer to the parent (classes $S_{1}$) and child (classes $S_{0}$) nodes in the hierarchical tree, respectively.
 In the context of our study, we are particularly interested in Level 1.

\inparagraph{Calibration quality.}
Referencing \cref{table:ood_calib}, we inspect the calibration quality of DeepLabV3+/ResNet-101 (shaded in red).
For child nodes (Level 0), the flat classifiers exhibit calibration quality on par with or better than HSSN.
For parent nodes (Level 1), \hypname is better calibrated than \eucname on challenging datasets ACDC/BDD/Wilddash (-0.21/-0.85/-4.01). Notably, the gap grows toward the most challenging testbeds, BDD and Wilddash.
The Poincaré ball model outperforms HSSN on five datasets, and, notably, for the most challenging datasets -- BDD and Wilddash (-2.27/-5.43 \wrt HSSN). 

Similarly for OCRNet/HRNet-W48 (shaded in blue, \cf \cref{table:ood_calib}), for child nodes (Level 0), the Poincaré ball is better calibrated by a large margin for the datasets with a large domain shift, IDD/ACDC/BDD/Wilddash (-2.56/-3.85/-2.37/-5.48 \wrt \eucname; -6.37/-8.28/-5.13/-9.78 \wrt HSSN). 
For parent nodes (Level 1), \hypname is better calibrated than its competitors for datasets Mapillary/IDD/BDD/Wilddash (-0.24/-0.97/-1.04/-3.12 \wrt \eucname; -0.97/-1.29/-0.68/-1.44 \wrt HSSN).

\inparagraph{Segmentation accuracy.} 
Let us examine the segmentation accuracy of DeepLabV3+/ResNet-101 in \cref{table:ood_deeplab}.
For child nodes (Level 0), flat classifiers \eucname and \hypname substantially outperform HSSN on Mapillary/IDD/ACDC/BDD/Wilddash, in terms of mAcc, aAcc and mIoU. 
For parent nodes (Level 1), the Poincaré ball model outperforms the Euclidean model and HSSN, in terms of aAcc for Mapillary/IDD/BDD/Wilddash.
For the most challenging datasets, BDD/Wilddash, \hypname clearly reaches the best results in terms of mIoU (+0.22/+3.55 \wrt \eucname; +2.21/+3.93 \wrt HSSN) and mAcc (+0.48/+2.87 \wrt \eucname, +3.24/+5.37 \wrt HSSN).
Even on the less challenging datasets (Cityscapes/Mapillary/IDD/ACDC), where \eucname outperforms the hyperbolic model on the parent-level predictions, the difference in mIoU between the child and parent nodes in the Euclidean and the hyperbolic model reduces (from -0.53/-3.60/-2.88/-6.62 to -0.07/-1.24/-0.29/-2.28) in favor of the hyperbolic model. This observation supports our analytical analysis of the parent bias in \cref{subsec:flattening}. 

Similarly for OCRNet/HRNet-W48, we observe that \hypname shows the best mAcc for IDD/ACDC/BDD/Wilddash on Level 0, by a significant margin.
For parent nodes (Level 1), \hypname outperforms hierarchical training and \eucname in terms of aAcc on Mapillary/IDD/BDD/Wilddash.
For the most challenging datasets, \hypname also clearly outperforms its competitors in terms of mIoU (+2.15/+5.41 \wrt the \eucname; +2.96/+2.72 \wrt HSSN, on BDD/Wilddash) and mAcc (+0.82/+3.13 \wrt \eucname; +3.24/+4.55 \wrt HSSN, on BDD/Wilddash). Similarly to DeepLabV3+, on less challenging datasets, Cityscapes/Mapillary/IDD/ACDC, the difference in mIoU scores from child to parent levels is larger for \hypname.
Unlike DeepLabV3+, the hyperbolic model outperforms the Euclidean model (+1.62) on the Mapillary dataset.

\subsection{Qualitative results}

In \cref{fig:qaulrest}, we visualize an example of semantic segmentation on Level 1, comparing HSSN to our Euclidean and hyperbolic networks.
We observe that HSSN mislabels parts of the ``Building'' as a ``Vehicle''.
Notably, the confidence of this incorrect prediction is high.
By contrast, both Euclidean and the hyperbolic networks largely predict the ``Building'' correctly, although with higher uncertainty than HSSN.
Rather curiously, the hyperbolic network exhibits a spatially smoother confidence map, which suggests a higher leverage of spatial correlations in the Poincaré ball model.

\begin{figure*}
\centering
  \begin{subfigure}{0.33\linewidth}
    \centering
    \includegraphics[width=\linewidth]{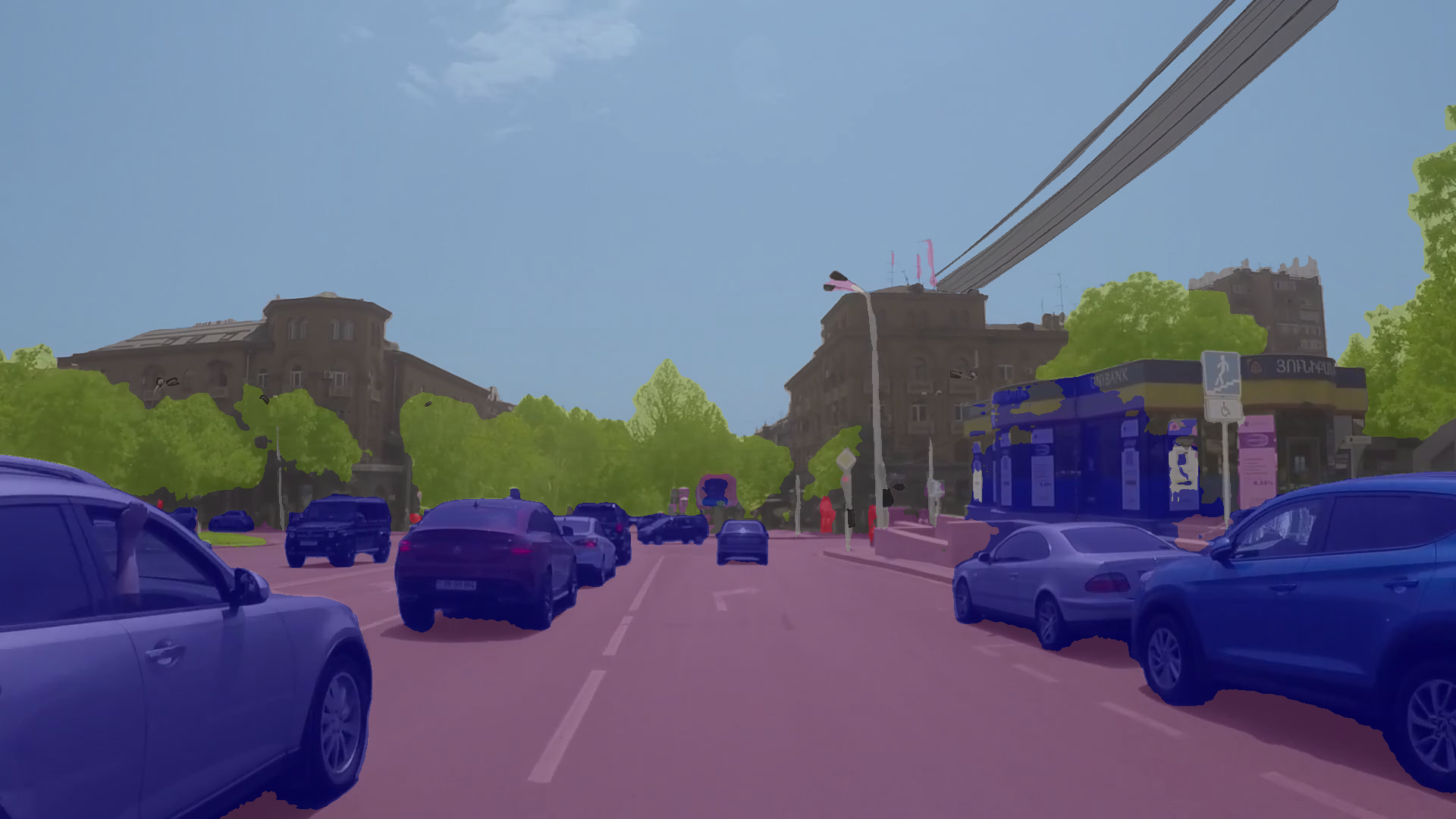}\\[1mm]
    \includegraphics[width=\linewidth]{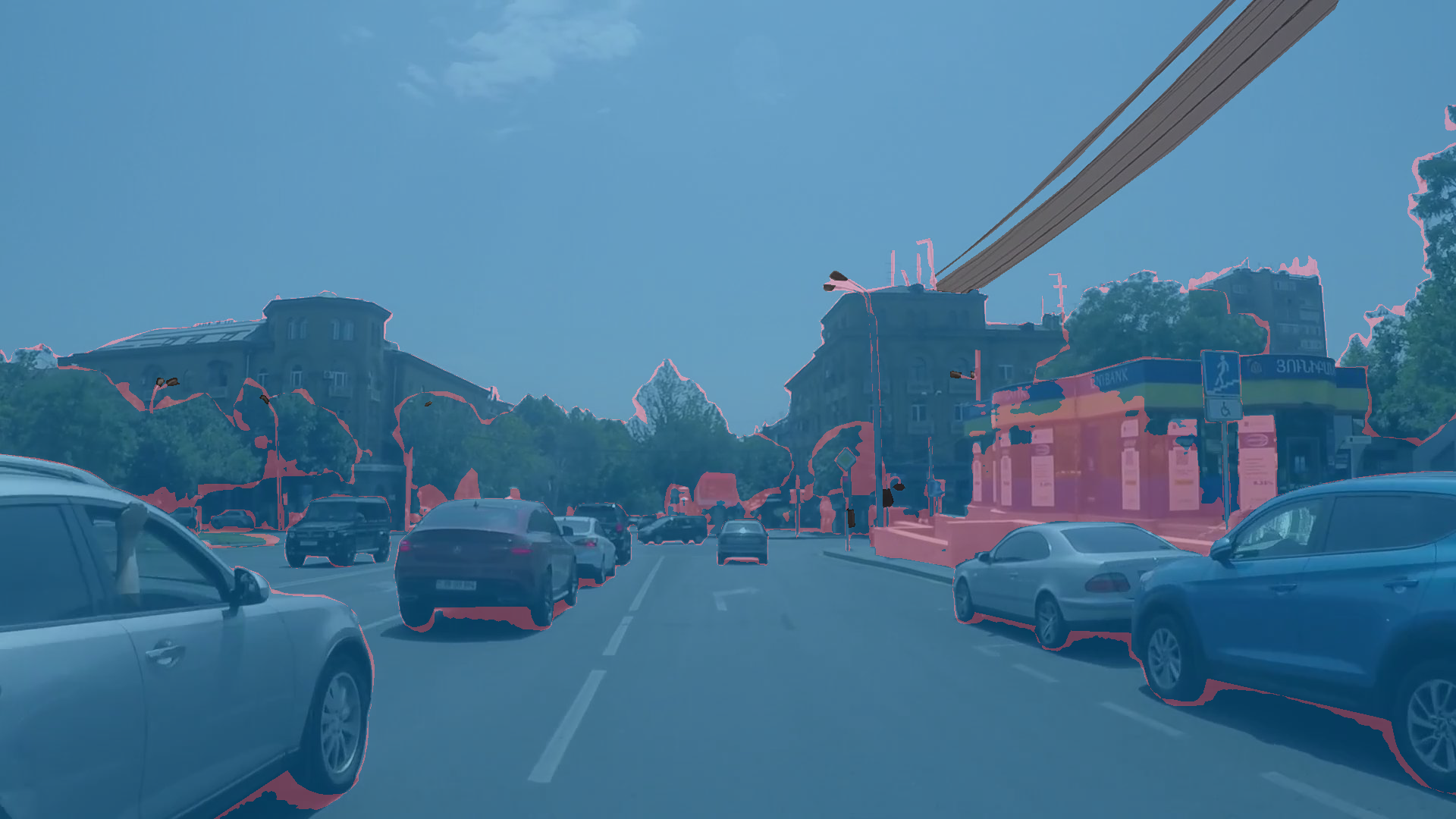}\\[1mm]
    \includegraphics[width=\linewidth]{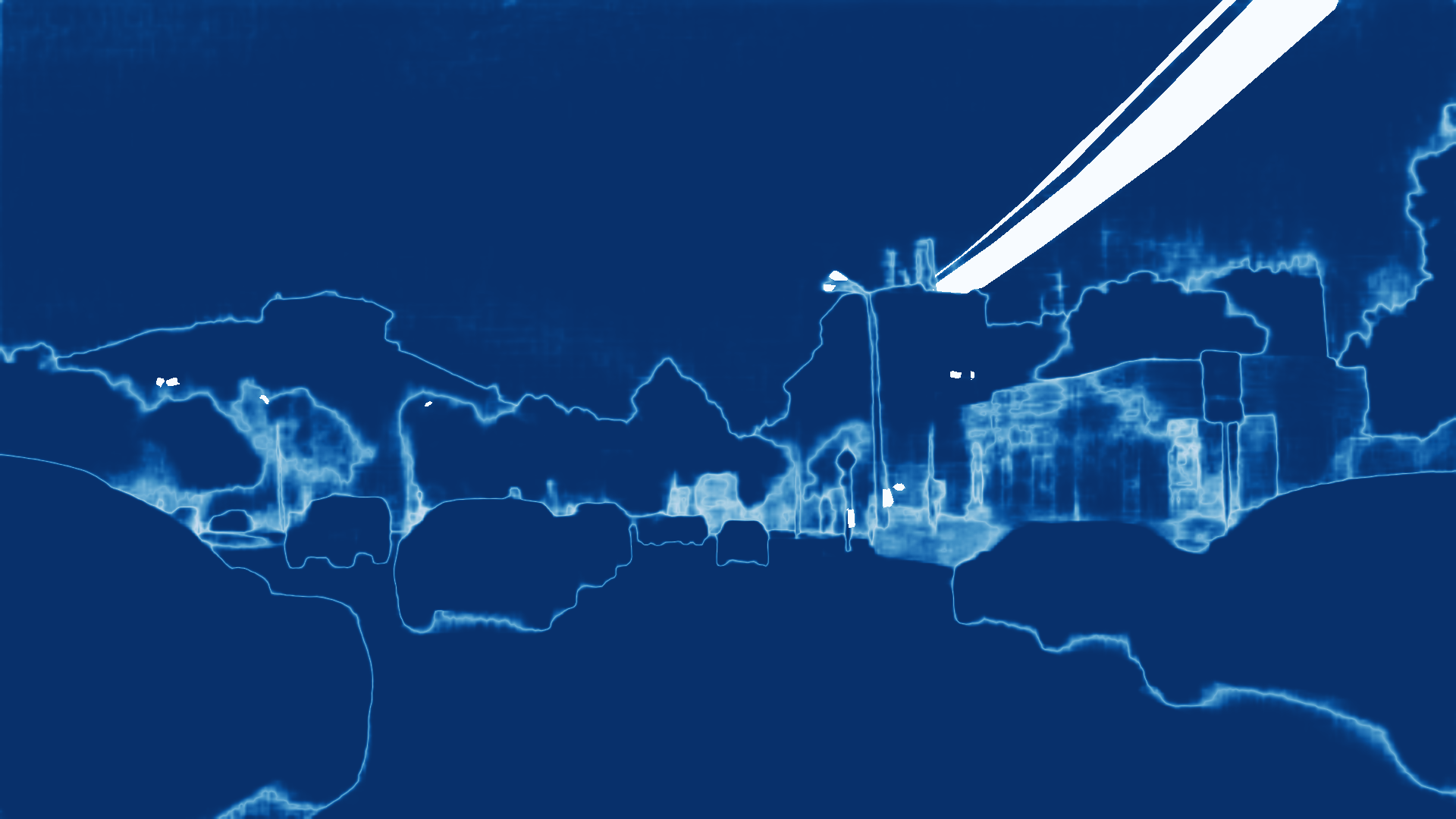}
  \label{fig:wd1}
  \end{subfigure}
  \begin{subfigure}{0.33\linewidth}
    \centering
    \includegraphics[width=\linewidth]{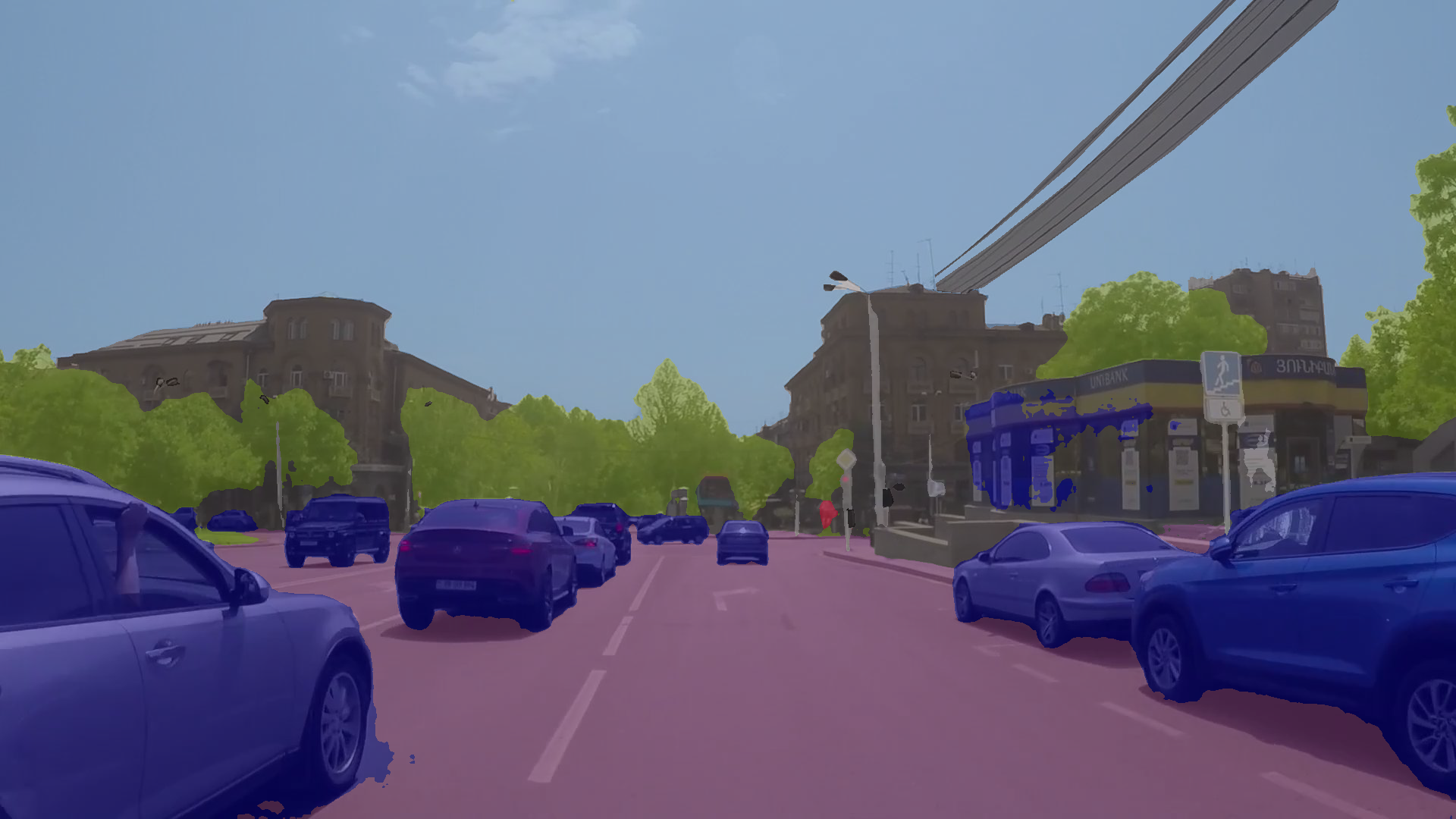}\\[1mm]
    \includegraphics[width=\linewidth]{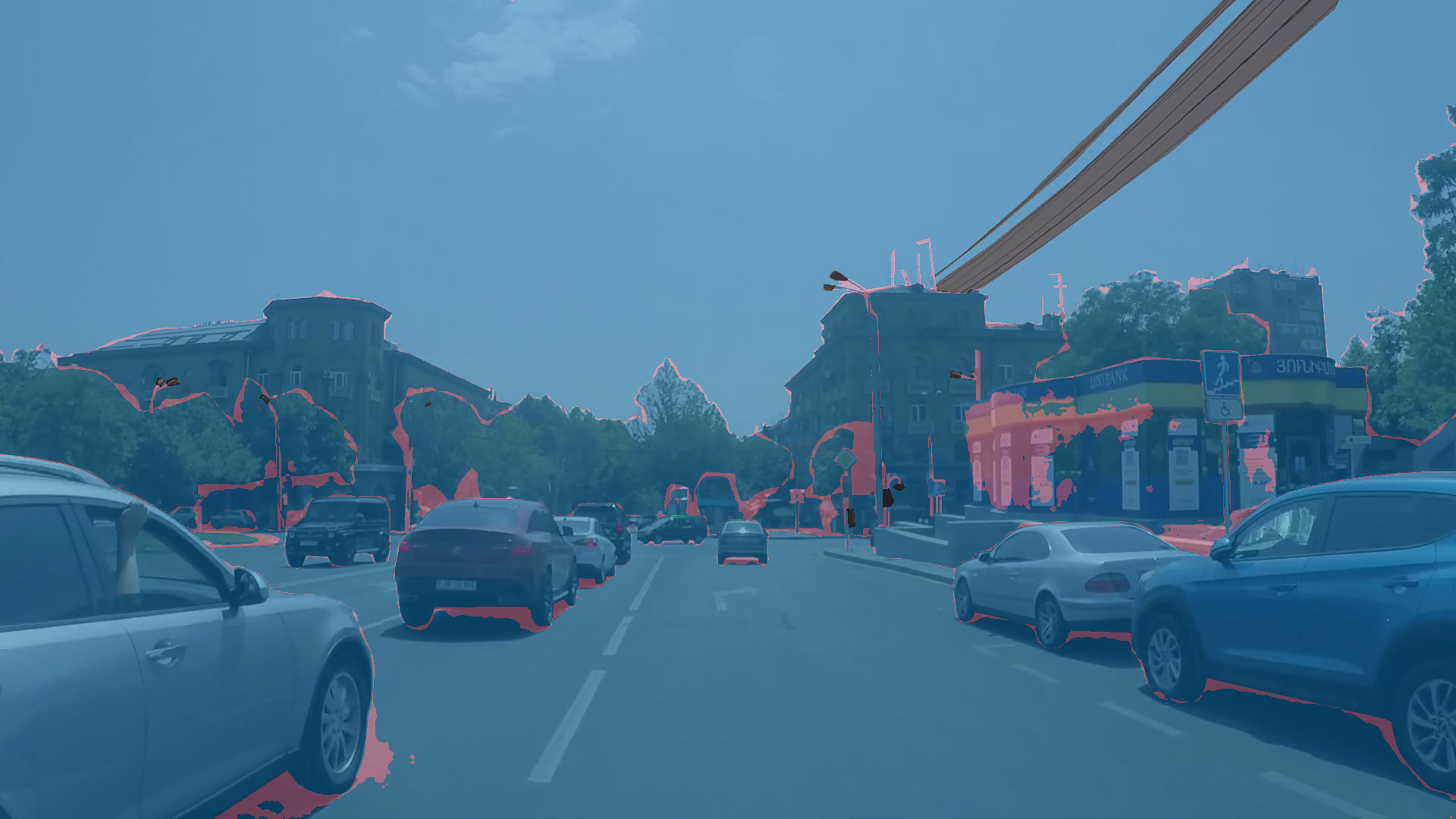}\\[1mm]
    \includegraphics[width=\linewidth]{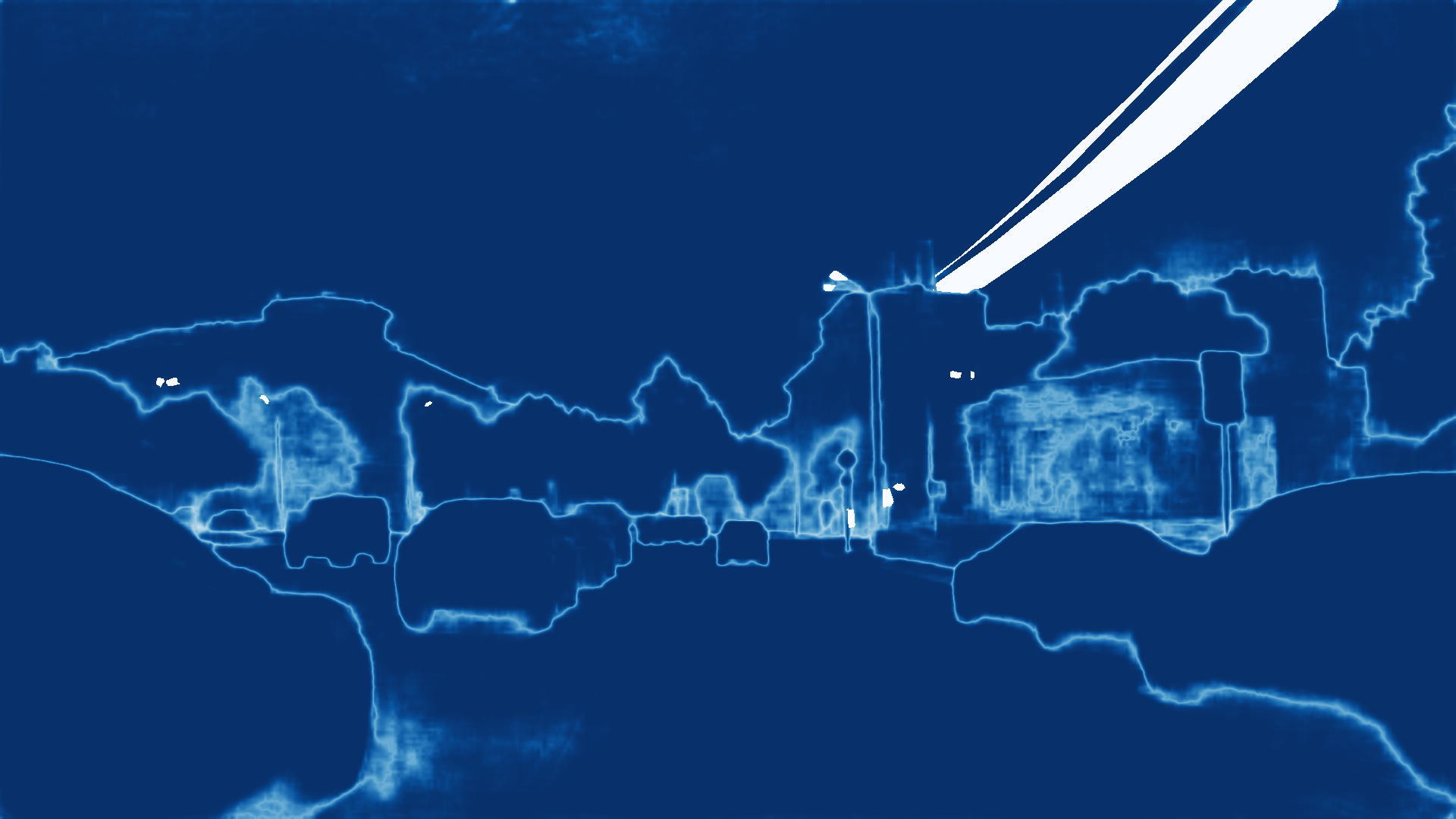}
    \label{fig:wd2}
  \end{subfigure}
  \begin{subfigure}{0.33\linewidth}
    \centering
    \includegraphics[width=\linewidth]{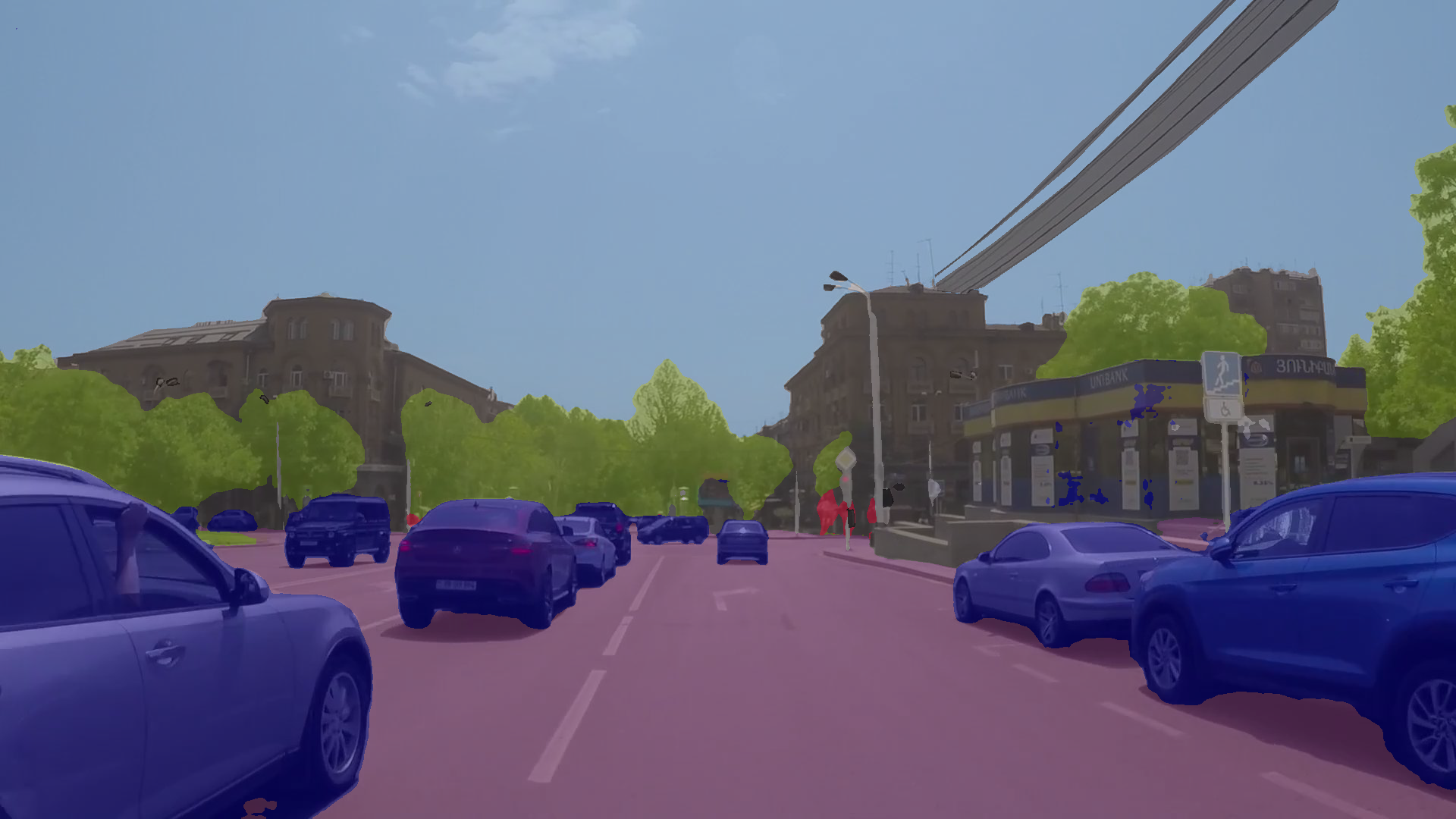}\\[1mm]
    \includegraphics[width=\linewidth]{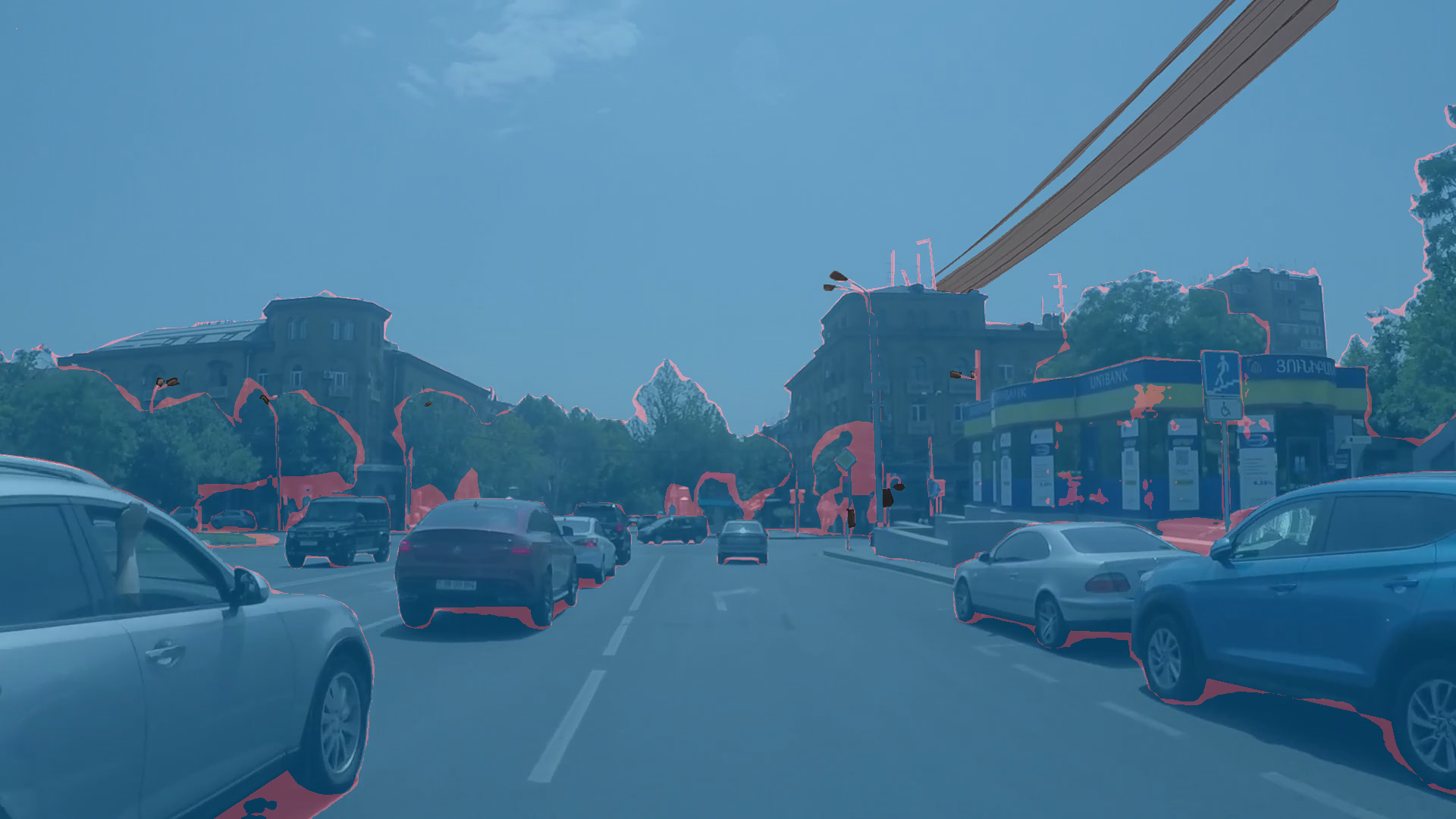}\\[1mm]
    \includegraphics[width=\linewidth]{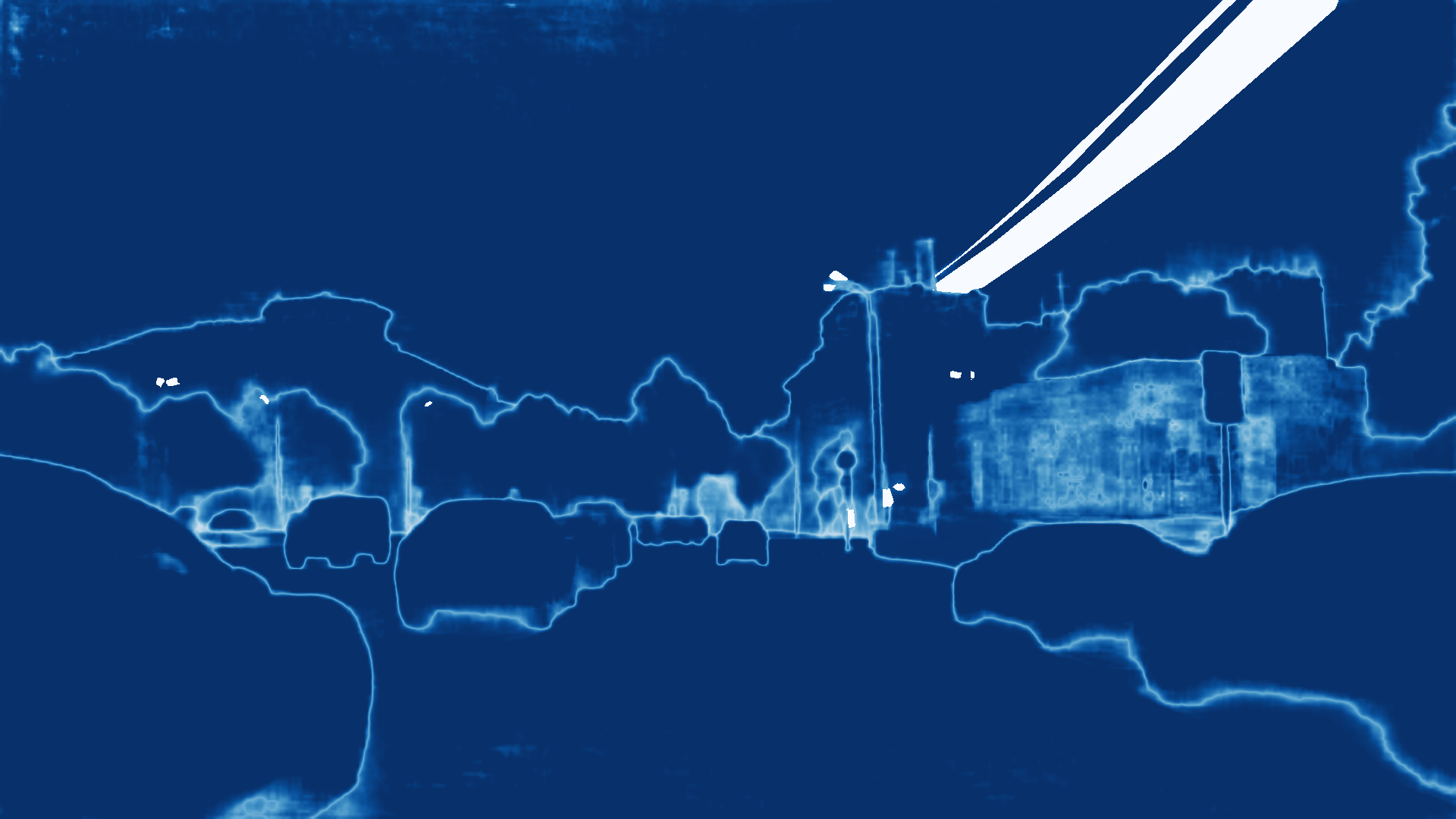}
    \label{fig:wd3}
  \end{subfigure}
  \caption{\textbf{Qualitative results on Wilddash.} We show parent-level segmentation results for HSSN (left), \eucname (middle), and \hypname (right) models by using label predictions (top), pixel-level accuracy maps (middle row), and confidence maps (bottom).}
  \label{fig:qaulrest}
\end{figure*}

\subsection{Discussion}

Overall, the empirical results of segmentation accuracy and calibration strongly support our analysis in \cref{sec:motivation} and \cref{sec:poincare_segmentation}. 
On the one hand, the hierarchical training in HSSN leads to poor performance on novel domains, presumably due to the large bias from the training on Cityscapes. 
By contrast, our flat classifiers, be they Euclidean or hyperbolic, are much more resilient to the domain shift, both in terms of calibration quality and segmentation accuracy, which is in line with our intuitive analysis in \cref{sec:motivation}.
The hyperbolic model exhibits a larger gap between the accuracy of the child and the parent nodes, compared to the Euclidean model. In cases, where both have comparable scores for child nodes, the hyperbolic model significantly outperforms the Euclidean model for parent nodes, and similarly for calibration.
This suggests the advantage of the Poincaré ball in flattening the Euclidean parent bias, as we conjectured in \cref{subsec:flattening}.
Importantly, these conclusions are consistent for both DeepLabV3+ and OCRNet.

\inparagraph{Limitations.}
We have only evaluated on datasets with a strong focus on traffic scenes.
This limitation comes from the lack of datasets with other taxonomic structures and of comparable visual complexity.
Our analysis also focuses on the Poincaré ball, although this is not the only possible realization of the hyperbolic space.
An extension of our study to other conformal models, such as the Lorentz model, offers exciting avenues for future research.
\section{Conclusion}

Our empirical investigations show that semantically meaningful hierarchical relations might not be the primary driver of the reported improvements in segmentation accuracy.
\textit{A contrario}, our cross-domain experiments reveal that flat segmentation networks, in which the parent categories are inferred from children, outperform hierarchical approaches consistently, for parent nodes most notably. 
However, we show that flat classifiers may suffer from the parent bias in the Euclidean space.
By contrast, the Poincaré ball model exhibits more uniform properties between class representations.
We demonstrate that a flat hyperbolic classifier coupled with a straightforward bottom-up inference generalizes surprisingly well across unseen test domains.
It tends to outperform, for parent categories, the Euclidean representation in terms of the segmentation accuracy and calibration on the more challenging datasets. 
To our knowledge, our work is also the first empirical analysis of dense calibration with hyperbolic networks.
We hope that our study will encourage future efforts toward more accurate and calibrated semantic segmentation models, extending beyond the currently mainstream Euclidean representation.

{
\small
\inparagraph{Acknowledgement.} 
This work was supported by the ERC Advanced Grant SIMULACRON.
}

{
    \small
    \bibliographystyle{ieeenat_fullname}
    \bibliography{main}

}

\clearpage
\pagenumbering{roman}
\appendix

\maketitlesupplementary

\noindent This supplemental material is organized as follows:\\
\noindent \textbf{\cref{sec:norms}} empirically complements our theoretical analysis of the hyperbolic distance in \cref{prop:concavity}, and confirms the inter-class uniformity in the Poincaré ball.\\
\noindent \textbf{\cref{sec:hyp_calib}} provides implementation details on calibrating hyperbolic networks.\\
\noindent \textbf{\cref{sec:cityscapes}} details the hierarchical two-level label structure used throughout our experiments.\\
\textbf{\cref{sec:qualitative}} provides additional qualitative examples from other datasets (Cityscapes, Mappilary, IDD, ACDC, BDD and Wilddash).

\section{Norms and concavity}\label{sec:norms}

\begin{table}[t]
    \footnotesize
    \begin{tabularx}{\linewidth}{X@{\hspace{5em}}S[table-format=1.8]@{\hspace{1em}}}
        \toprule
        {Class} & {Average norm} \\
        \midrule
        Road & 0.99594157 \\
        Sidewalk & 0.99593576 \\
        Building & 0.99596734 \\
        Wall & 0.99528885 \\
        Fence & 0.99568549 \\
        Pole & 0.99580016 \\
        Traffic light & 0.99598823 \\
        Traffic sign & 0.99578103 \\
        Vegetation & 0.99596470 \\
        Terrain & 0.99588626 \\
        Sky & 0.99596704 \\
        Person & 0.99590956 \\
        Rider & 0.99598813 \\
        Car & 0.99595947 \\
        Truck & 0.99581633 \\
        Bus & 0.99596589 \\
        Train & 0.99536107 \\
        Motorcycle & 0.99583602 \\
        Bicycle & 0.99596311 \\
        \bottomrule
    \end{tabularx}
    \caption{The average norm of hyperbolic embeddings in the Poincaré ball for each class. Note that the norms are close to $1$. Geometrically, this means that the embeddings reside at the periphery of the ball. Linking this observation to the concavity of the hyperbolic distance (\cf \cref{subsec:flattening}) explains the uniformity in the inter-classes distances, as a change in the embedding of a class, as long as its norm is preserved, does not have a significant effect on its relative distance to the embeddings of other classes.}
    \label{tab:embeddings_norm}
\end{table}

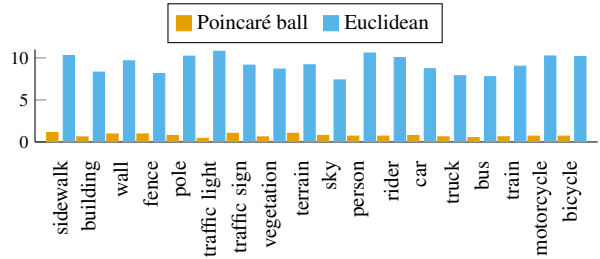
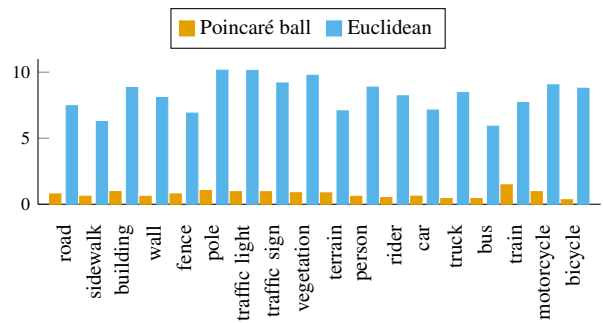
\begin{figure}
\centering
\begin{subfigure}{\linewidth}
  \footnotesize
  \begin{tikzpicture}[every node/.style={font=\footnotesize}]
    \pgfplotstableread[row sep=\\,col sep=&]{
        name & acc  \\
        sidewalk & 1.12 \\
        building & 0.60 \\
        wall & 0.95 \\
        fence &  0.95   \\
        pole   &  0.77   \\
        traffic light    & 0.43    \\
        traffic sign    &  1.03   \\
        vegetation    &  0.60   \\
        terrain    &  1.03   \\
        sky    &  0.78   \\
        person    & 0.69    \\
        rider    & 0.69    \\
        car    &  0.77   \\
        truck    &  0.61   \\
        bus    & 0.52    \\
        train    &  0.61  \\
        motorcycle    &  0.69   \\
        bicycle &  0.69   \\
    }\dataA

    \pgfplotstableread[row sep=\\,col sep=&]{
        name & acc  \\
        sidewalk & 10.28 \\
        building & 8.31 \\
        wall & 9.65 \\
        fence &  8.14   \\
        pole   &  10.21   \\
        traffic light    & 10.79    \\
        traffic sign    &  9.14   \\
        vegetation    &  8.67  \\
        terrain    &  9.18 \\
        sky    &  7.38 \\
        person    & 10.58  \\
        rider    & 10.03   \\
        car    &  8.71  \\
        truck    &  7.88 \\
        bus    & 7.76  \\
        train    &  9.01 \\
        motorcycle    &  10.23   \\
        bicycle &  10.17  \\
    }\dataB

    \definecolor{city}{RGB}{230,159,0}
    \definecolor{coco}{RGB}{86, 180, 233}
    \definecolor{pots}{RGB}{0, 158, 115}
        
    \begin{axis}[
        width=\linewidth,
        height=10em,
        ybar,
        bar width=0.15cm,
        x=0.40cm,
        enlarge x limits=0.05,
        ymin=0, ymax=11,
        axis y line*=left,
        axis x line*=bottom,
        symbolic x coords={sidewalk, building, wall, fence, pole, traffic light, traffic sign, vegetation, terrain, sky, person, rider, car, truck, bus, train, motorcycle, bicycle},
        xtick=data,
        ytick pos=left,
        xtick style={draw=none},
        x tick label style={
            text width=2cm,
            align=right,
            scale=1,
            transform shape,
            rotate=90,
            anchor=near xticklabel,
        },
        y tick label style={scale=0.85},
        legend style={at={(0.48,1.5)},draw=black,anchor=north,legend columns=-1, scale=0.85},
        nodes near coords={},
        nodes near coords style={font=\rm},
        every node near coord/.append,
        every node near coord/.append style={scale=0.85,  /pgf/number format/assume math mode=true},
        ytick={0, 5,10},
        yticklabels={0, 5,10},
        legend image code/.code={
            \draw [#1] (0cm,-0.1cm) rectangle (0.2cm,0.1cm);
        },
    ]
    \addplot[city,fill=city] table[x=name,y=acc]{\dataA};
    \addplot[coco,fill=coco] table[x=name,y=acc]{\dataB};
    \legend{Poincaré ball\:\:, Euclidean}
    \end{axis}
  \end{tikzpicture}
  \vspace{-1em}
  \caption{Measuring the coefficient of variation \wrt ``Road'' embeddings.}
  \label{fig:embedding_road}
  \end{subfigure}

  \vspace{1em}

  \begin{subfigure}{\linewidth}

    \begin{tikzpicture}[every node/.style={font=\footnotesize}]
    \pgfplotstableread[row sep=\\,col sep=&]{
        name & acc  \\
        road & 0.78 \\
        sidewalk & 0.61 \\
        building & 0.96 \\
        wall & 0.60 \\
        fence &  0.78   \\
        pole   &  1.04  \\
        traffic light    & 0.95   \\
        traffic sign    &  0.95  \\
        vegetation    &  0.87  \\
        terrain    &  0.86  \\
        person    & 0.60   \\
        rider    & 0.51  \\
        car    &  0.61 \\
        truck    &  0.43 \\
        bus    & 0.43\\
        train    &  1.47   \\
        motorcycle    &  0.95   \\
        bicycle &  0.34 \\
    }\dataA

    \pgfplotstableread[row sep=\\,col sep=&]{
        name & acc  \\
        road & 7.46\\
        sidewalk & 6.26 \\
        building & 8.84 \\
        wall & 8.08 \\
        fence &  6.90   \\
        pole   &  10.15   \\
        traffic light    & 10.13   \\
        traffic sign    &  9.18  \\
        vegetation    &  9.76 \\
        terrain    &  7.07   \\
        person    & 8.87   \\
        rider    & 8.21   \\
        car    &  7.13  \\
        truck    &  8.45  \\
        bus    & 5.91 \\
        train    & 7.70  \\
        motorcycle    & 9.04   \\
        bicycle &  8.78   \\
    }\dataB

    \definecolor{city}{RGB}{230,159,0}
    \definecolor{coco}{RGB}{86, 180, 233}
    \definecolor{pots}{RGB}{0, 158, 115}
        
    \begin{axis}[
        width=\linewidth,
        height=10em,
        ybar,
        bar width=0.15cm,
        x=0.40cm,
        enlarge x limits=0.05,
        ymin=0, ymax=11,
        axis y line*=left,
        axis x line*=bottom,
        symbolic x coords={road, sidewalk, building, wall, fence, pole, traffic light, traffic sign, vegetation, terrain, person, rider, car, truck, bus, train, motorcycle, bicycle},
        xtick=data,
        ytick pos=left,
        xtick style={draw=none},
        x tick label style={
            text width=2cm,
            align=right,
            scale=1,
            transform shape,
            rotate=90,
            anchor=near xticklabel,
        },
        y tick label style={scale=0.85},
        legend style={at={(0.48,1.35)},draw=black,anchor=north,legend columns=-1, scale=0.85},
        nodes near coords={},
        nodes near coords style={font=\rm},
        every node near coord/.append,
        every node near coord/.append style={scale=0.85,  /pgf/number format/assume math mode=true},
        ytick={0, 5, 10},
        yticklabels={0, 5, 10},
        legend image code/.code={
            \draw [#1] (0cm,-0.1cm) rectangle (0.2cm,0.1cm);
        },
    ]
    \addplot[city,fill=city] table[x=name,y=acc]{\dataA};
    \addplot[coco,fill=coco] table[x=name,y=acc]{\dataB};
    \legend{Poincaré ball\:\:, Euclidean}
    \end{axis}
  \end{tikzpicture}
  \caption{Measuring the coefficient of variation \wrt ``Sky'' embeddings.}
  \label{fig:embedding_sky}
  \end{subfigure}
  \vspace{-1em}
  \caption{The coefficient of variation of the distance between (a) 'Road' embeddings (resp. (b) 'Sky' embeddings) and the embeddings of other classes. The hyperbolic distance between the embeddings of a given class and the embeddings of other classes is much more uniform than the respective Euclidean distance.}
\end{figure}

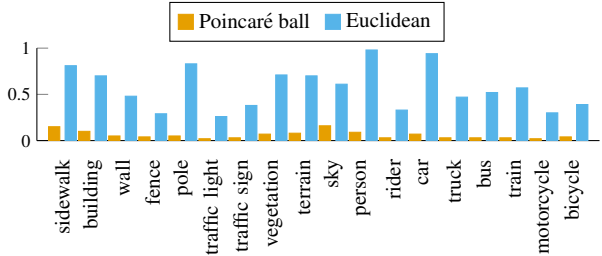
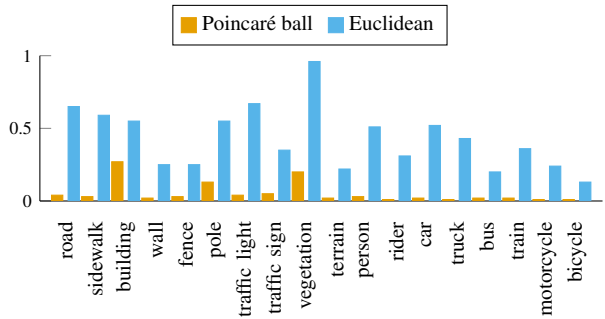
\begin{figure}
\centering
\begin{subfigure}{\linewidth}
  \footnotesize
  \begin{tikzpicture}[every node/.style={font=\footnotesize}]
    \pgfplotstableread[row sep=\\,col sep=&]{
        name & acc  \\
        sidewalk & 0.15 \\
        building & 0.10 \\
        wall & 0.05 \\
        fence &  0.04   \\
        pole   &  0.05   \\
        traffic light    & 0.02    \\
        traffic sign    &  0.03   \\
        vegetation    &  0.07   \\
        terrain    &  0.08   \\
        sky    &  0.16   \\
        person    & 0.09    \\
        rider    & 0.03    \\
        car    &  0.07   \\
        truck    &  0.03   \\
        bus    & 0.03    \\
        train    &  0.03  \\
        motorcycle    &  0.02   \\
        bicycle &  0.04   \\
    }\dataA

    \pgfplotstableread[row sep=\\,col sep=&]{
        name & acc  \\
        sidewalk & 0.81 \\
        building & 0.70 \\
        wall & 0.48 \\
        fence &  0.29   \\
        pole   &  0.83   \\
        traffic light    & 0.26    \\
        traffic sign    &  0.38   \\
        vegetation    &  0.71  \\
        terrain    &  0.70  \\
        sky    &  0.61  \\
        person    & 0.98  \\
        rider    & 0.33   \\
        car    &  0.94   \\
        truck    &  0.47 \\
        bus    & 0.52  \\
        train    &  0.57 \\
        motorcycle    &  0.30   \\
        bicycle &  0.39  \\
    }\dataB

    \definecolor{city}{RGB}{230,159,0}
    \definecolor{coco}{RGB}{86, 180, 233}
    \definecolor{pots}{RGB}{0, 158, 115}
        
    \begin{axis}[
        width=\linewidth,
        height=10em,
        ybar,
        bar width=0.15cm,
        x=0.40cm,
        enlarge x limits=0.05,
        ymin=0, ymax=1,
        axis y line*=left,
        axis x line*=bottom,
        symbolic x coords={sidewalk, building, wall, fence, pole, traffic light, traffic sign, vegetation, terrain, sky, person, rider, car, truck, bus, train, motorcycle, bicycle},
        xtick=data,
        ytick pos=left,
        xtick style={draw=none},
        x tick label style={
            text width=2cm,
            align=right,
            scale=1,
            transform shape,
            rotate=90,
            anchor=near xticklabel,
        },
        y tick label style={scale=0.85},
        legend style={at={(0.48,1.5)},draw=black,anchor=north,legend columns=-1, scale=0.85},
        nodes near coords={},
        nodes near coords style={font=\rm},
        every node near coord/.append,
        every node near coord/.append style={scale=0.85,  /pgf/number format/assume math mode=true},
        ytick={0, 0.5,1},
        yticklabels={0, 0.5,1},
        legend image code/.code={
            \draw [#1] (0cm,-0.1cm) rectangle (0.2cm,0.1cm);
        },
    ]
    \addplot[city,fill=city] table[x=name,y=acc]{\dataA};
    \addplot[coco,fill=coco] table[x=name,y=acc]{\dataB};
    \legend{Poincaré ball\:\:, Euclidean}
    \end{axis}
  \end{tikzpicture}
  \vspace{-1em}
  \caption{Measuring the coefficient of variation of hyper-/gyroplanes distance \wrt ``Road'' embeddings.}
  \label{fig:road_stdmu}
  \end{subfigure}

  \vspace{1em}

  \begin{subfigure}{\linewidth}

    \begin{tikzpicture}[every node/.style={font=\footnotesize}]
    \pgfplotstableread[row sep=\\,col sep=&]{
        name & acc  \\
        road & 0.04 \\
        sidewalk & 0.03 \\
        building & 0.27 \\
        wall & 0.02 \\
        fence &  0.03   \\
        pole   &  0.13  \\
        traffic light    & 0.04   \\
        traffic sign    &  0.05  \\
        vegetation    &  0.20  \\
        terrain    &  0.02  \\
        person    & 0.03   \\
        rider    & 0.01  \\
        car    &  0.02 \\
        truck    &  0.01 \\
        bus    & 0.02 \\
        train    &  0.02   \\
        motorcycle    &  0.01   \\
        bicycle &  0.01 \\
    }\dataA

    \pgfplotstableread[row sep=\\,col sep=&]{
        name & acc  \\
        road & 0.65\\
        sidewalk & 0.59 \\
        building & 0.55 \\
        wall & 0.25 \\
        fence &  0.25   \\
        pole   &  0.55   \\
        traffic light    & 0.67   \\
        traffic sign    &  0.35  \\
        vegetation    &  0.96 \\
        terrain    &  0.22   \\
        person    & 0.51   \\
        rider    & 0.31    \\
        car    &  0.52  \\
        truck    &  0.43   \\
        bus    & 0.20  \\
        train    &  0.36  \\
        motorcycle    &  0.24   \\
        bicycle &  0.13   \\
    }\dataB

    \definecolor{city}{RGB}{230,159,0}
    \definecolor{coco}{RGB}{86, 180, 233}
    \definecolor{pots}{RGB}{0, 158, 115}
        
    \begin{axis}[
        width=\linewidth,
        height=10em,
        ybar,
        bar width=0.15cm,
        x=0.40cm,
        enlarge x limits=0.05,
        ymin=0, ymax=1,
        axis y line*=left,
        axis x line*=bottom,
        symbolic x coords={road, sidewalk, building, wall, fence, pole, traffic light, traffic sign, vegetation, terrain, person, rider, car, truck, bus, train, motorcycle, bicycle},
        xtick=data,
        ytick pos=left,
        xtick style={draw=none},
        x tick label style={
            text width=2cm,
            align=right,
            scale=1,
            transform shape,
            rotate=90,
            anchor=near xticklabel,
        },
        y tick label style={scale=0.85},
        legend style={at={(0.48,1.35)},draw=black,anchor=north,legend columns=-1, scale=0.85},
        nodes near coords={},
        nodes near coords style={font=\rm},
        every node near coord/.append,
        every node near coord/.append style={scale=0.85,  /pgf/number format/assume math mode=true},
        ytick={0, 0.5, 1},
        yticklabels={0, 0.5, 1},
        legend image code/.code={
            \draw [#1] (0cm,-0.1cm) rectangle (0.2cm,0.1cm);
        },
    ]
    \addplot[city,fill=city] table[x=name,y=acc]{\dataA};
    \addplot[coco,fill=coco] table[x=name,y=acc]{\dataB};
    \legend{Poincaré ball\:\:, Euclidean}
    \end{axis}
  \end{tikzpicture}
  \vspace{-1em}
  \caption{Measuring the coefficient of variation of hyper-/gyroplanes distance \wrt ``Sky'' embeddings.}
  \label{fig:sky_stdmu}
  \end{subfigure}
  \caption{The coefficient of variation of the distance between (a) 'Road' embeddings (resp. (b) 'Sky' embeddings) and gyroplanes/hyperplanes of other classes. The hyperbolic distance between a given class and the gyroplanes of other classes is much more uniform than the respective Euclidean distance.}
\end{figure}

Recall from \cref{sec:poincare_segmentation} that during training, the hyperbolic likelihood,
\begin{equation}\label{eq:likelihood}
p(\hat{y} := y \mid h) \propto \exp(\zeta_{y}(h)) \, ,
\end{equation}
where
\begin{equation}
\zeta_{y}(h) := \frac{\lambda_{r_{y}}^{c}\lVert w_{y} \rVert }{\sqrt{c}} \text{arcsinh}\bigg(\frac{2\sqrt{c} \langle -r_{y}\oplus_{c}h, w_{y} \rangle}{(1-c\lVert -r_{y}\oplus_{c} h\rVert^{2})\lVert w_{y} \rVert}\bigg) \, ,
\end{equation}
is minimized via the cross-entropy loss.
Equivalently \cite{ganea2018}, \cref{eq:likelihood} can be written as:
\begin{align}
\begin{split}
\log{p(\hat{y} := y \mid h)} \propto \; &\text{sign}\big(\langle -r_{y}\oplus_{c}h, a_{y}\rangle\big) \times \\
\times &\sqrt{g_{r_{y}}^{c}(a_{y},a_{y})} \, d_{\mathbb{H}}(h,H_{y}^{c}) \, .
\end{split}
\end{align}

\noindent Thus, the closer a hyperbolic embedding is to the border (periphery) of the Poincaré ball, the higher its norm and the class posterior.
Indeed, when the embedding is located between the gyroplane and the periphery of the Poincaré ball, \ie the term $\langle -r_{y}\oplus_{c}h, a_{y}\rangle$ is strictly larger than zero, $p(\hat{y} := y \mid h)$ increases with the distance $d_{\mathbb{H}}(h,H_{y}^{c})$ between the embedding and the gyroplane.

We empirically confirm that hyperbolic embeddings end up at the boundary of the Poincaré ball after model training.
In \cref{tab:embeddings_norm}, we compute the average norm of the embeddings for each class and find that the embedding norm in the Poincaré ball indeed tends to be close to $1$, which indicates close proximity of the embeddings to the Poincaré ball's boundary.
Let us now recall the hyperbolic distance from \cref{eq:hyperbolic_distance}:
\begin{equation}\label{eq:dist_hyperbolic}
d_{\mathbb{H}}(h_{1},h_{2}) = \text{arcosh}\bigg(1 + 2 \frac{\lVert h_{1} - h_{2} \rVert ^{2}}{(1 - \lVert h_{1} \rVert^{2})(1 - \lVert h_{2} \rVert^{2})}\bigg) \, .
\end{equation}
After network training, the class embeddings are well-separated.
The denominator $(1 - \lVert h_{1} \rVert^{2})(1 - \lVert h_{2} \rVert^{2})$ becomes very small, hence the hyperbolic distance now operates at the high-end spectrum of the domain in \cref{fig:arcosh} (\ie when $x$ in \cref{fig:arcosh} is large).
It follows that the hyperbolic distance between embeddings of different classes tends to be uniform. 
To demonstrate this empirically, we provide an example analysis of the inter-classes distance for two classes, 'Road' and 'Sky'. We compare the Poincaré ball model and the Euclidean representation. For all embeddings in the Poincaré ball (resp. Euclidean space) corresponding to these two classes, we derive the distance to the embeddings and to the gyroplanes (resp. hyperplanes) associated with the other classes.  \cref{fig:embedding_road} (resp. \cref{fig:embedding_sky}) illustrates the coefficient of variation (standard deviation divided by the mean) of the distance between 'Road' (resp. 'Sky') embeddings and the embeddings of other classes, grouped by each class.
\cref{fig:road_stdmu} (resp. \Cref{fig:sky_stdmu}) compares the coefficient of variation of the distance between 'Road' (resp. 'Sky') embeddings and the gyroplanes (for a hyperbolic space) and hyperplanes (for a Euclidean space) associated to the other classes.
Observe that the coefficient of variation in the hyperbolic case is substantially lower across the board, which supports our analysis in the main text: The inter-class distance between the Poincaré embeddings is substantially more uniform than the embeddings in the Euclidean space.
The same conclusion holds \wrt other semantic categories.

\section{Calibration in the Poincaré ball}
\label{sec:hyp_calib}

A model is said to be \emph{calibrated} when the predictive probabilities correspond to the true probabilities \cite{guo2016}. Following the geometrical interpretation of \citet{ganea2018}, the predictive probabilities in hyperbolic networks correspond to the logits after the hyperbolic multinominal logistic regression.
Consequently, this allows us to extend popular calibration methods in Euclidean space to the Poincaré ball in a straightforward manner: Instead of considering the Euclidean predictions -- a softmax applied to Euclidean logits -- we apply a softmax to hyperbolic logits. 
Formally, in a hyperbolic space, the probability of class $y \in \{1,...,k\}$ for hyperbolic output $h$ is given by the softmax:
\begin{equation}\label{eq:accuracy}
p(\hat{y} := y \mid h) = \frac{\exp(\zeta_{y}(h))}{\sum_{i=1}^{k}\exp(\zeta_{i}(h))} \:,
\end{equation}
where $k$ is the number of classes. Training is performed via the cross-entropy loss, as in the Euclidean networks.

To evaluate the calibration quality, we follow \citet{kull2019beyond} to derive the class-wise Expected Calibration Error (cwECE).
We partition predictions into $M$ equally spaced bins for each class $\{B_{y,m}\}_{m=1,...,M}^{y=1,...,k}$ and take a weighted average of the difference between the accuracy and confidence for each bin:
\begin{equation}\label{eq:ECE}
\text{cwECE} = \frac{1}{k} \sum_{y=1}^{k} \sum_{m=1}^{M} \frac{\lvert B_{y,m} \rvert}{n}\lvert \text{acc}_{y}(B_{y,m}) - \text{conf}_{y}(B_{y,m}) \rvert  \,,
\end{equation}
where $n$ is the total number of samples of a given class across all bins. 
$\text{conf}_{y}(B_{y,m})$ and $\text{acc}_{y}(B_{y,m})$ are, respectively, the average prediction of class $y$ probability and the actual ratio of class $y$ in bin $B_{y,m}$.
cwECE represents the average gap between the predicted confidence and the expected accuracy, averaged over all classes.

\section{Hierarchical taxonomy}
\label{sec:cityscapes}

\begin{figure}
    \centering
    \includegraphics[width=\linewidth]{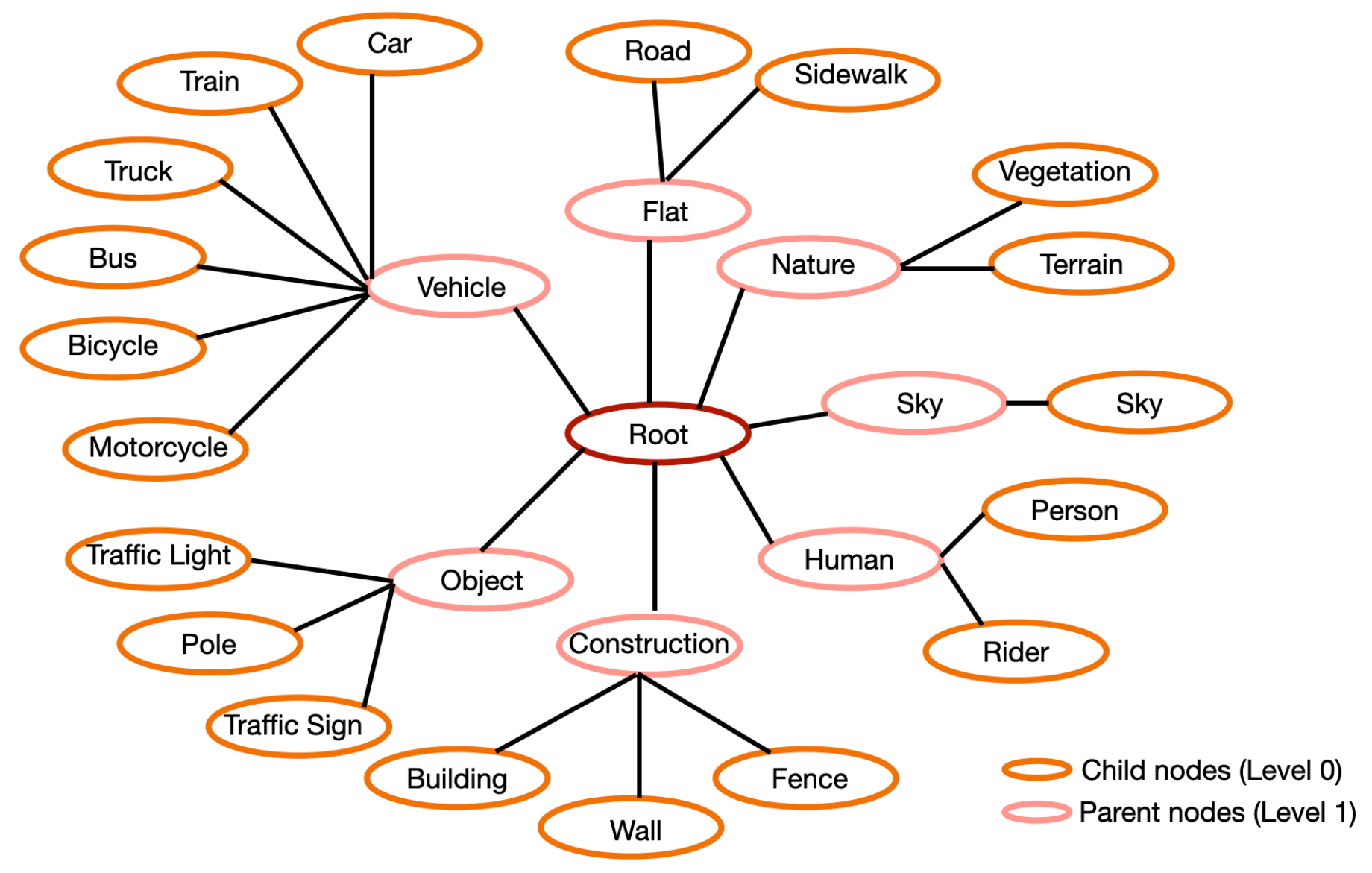}
  \caption{Hierarchical taxonomy derived from Cityscapes \cite{cordts2016cityscapes}, divided in $7$ parents (in pink) and $19$ children (in orange). Same hierarchy is used when testing on Mapillary \cite{neuhold2017mapillary}, IDD \cite{varma2019idd}, ACDC \cite{sakaridis2021acdc}, BDD \cite{yu2020bdd100k} and Wilddash \cite{zendel2018wilddash}.}
  \label{fig:cityscapes_architecture}
\end{figure}

For reference, in \cref{fig:cityscapes_architecture} we provide the hierarchical taxonomy used in our experiments. 
It follows the same structure from previous work \cite{Li:2022:DHS} and comprises 7 parent categories (Level 1) and 19 child categories (Level 0).

\section{Qualitative results}\label{sec:qualitative}
Here, we provide additional qualitative analysis. Following the same layout as in \cref{fig:qaulrest}, the additional examples here illustrate (top-down): label predictions, pixel-level accuracy, and confidence maps.
Results on Cityscapes (in-domain) in \cref{fig:qaulics1} show that each model has generally high accuracy and calibration quality.
Recall from our evaluation in the main text that parent-level predictions for flat models are on par with hierarchical models (HSSN) even when the hierarchical model has better accuracy on child nodes.
Moreover, all models are well calibrated in the in-domain settings both on the parent and child levels.

Under the domain shift, even of a moderate nature, we already observe the advantage of flat models to HSSN.
In \cref{fig:qaulimap2} for Mappilary, the HSSN model misclassifies ``Construction'' with ``Flat'' in both examples, while flat models have an overall higher accuracy in the respective area.

In a scenario with occlusion, flat models are less likely to misclassify compared to the HSSN model.
This becomes evident while inspecting the examples in \cref{fig:qaulidd2} for the IDD dataset.
In the first example, we observe that the HSSN model misclassifies the ``Construction'' class with ``Vehicle'', and in the second example, the ``Nature'' class with ``Construction''.
By contrast, a flat classifier exhibits higher accuracy in the occluded area.

When the evaluated models perform comparably in terms of accuracy, we observe that flat models are better calibrated. In the first example of \cref{fig:qauliacdc} for ACDC, we show that the HSSN model has low confidence in the ``Construction'' label (side of a building) even though the pixel accuracy is high in that area.
In the second example, HSSN misclassifies ``Construction'' with ``Road'', but has higher confidence than flat models in the area, even though the accuracy is lower. 

In more challenging scenarios, hyperbolic models exhibit increased accuracy compared to Euclidean models. In \cref{fig:qaulbdd2} for the BDD dataset, in both examples, the flat models show improved accuracy over the HSSN model and the hyperbolic model outperforms the Euclidean model.
In the first example, in the area where the HSSN model misclassifies ``Nature'' and ``Sky'', we see an increase in accuracy as we move from the leftmost example (HSSN) to the rightmost example (flat hyperbolic network). We can observe a similar increase in the accuracy where the HSSN model misclassifies ``Construction'' with ``Vehicle'' as we move from left to right. The HSSN model has high confidence in areas that it misclassified, indicating a higher degree is miscalibration than the flat models.

By observing the pixels around the synthetic occlusion on the first example in \cref{fig:qaulwd2} (top left corner) for Wilddash, we observe superior accuracy of the hyperbolic model compared to Euclidean and HSSN models. In terms of calibration, we observe that the Euclidean model has an overall higher confidence around the occluded area compared to the hyperbolic model, even though the accuracy is lower in this area.
It illustrates our empirical findings that the hyperbolic model has smoother confidence maps. This is also true for more challenging cases, such as the second example in \cref{fig:qaulwd2}.
HSSN and the Euclidean network have an overall lower accuracy than the hyperbolic model.
HSSN also has high confidence, despite misclassifying the ``Flat'' and ``Vehicle'' classes, while the Euclidean model has low confidence in the correctly predicted areas.
By contrast, the hyperbolic model has low certainty in the areas with lower pixel accuracy, suggesting good calibration of these models.

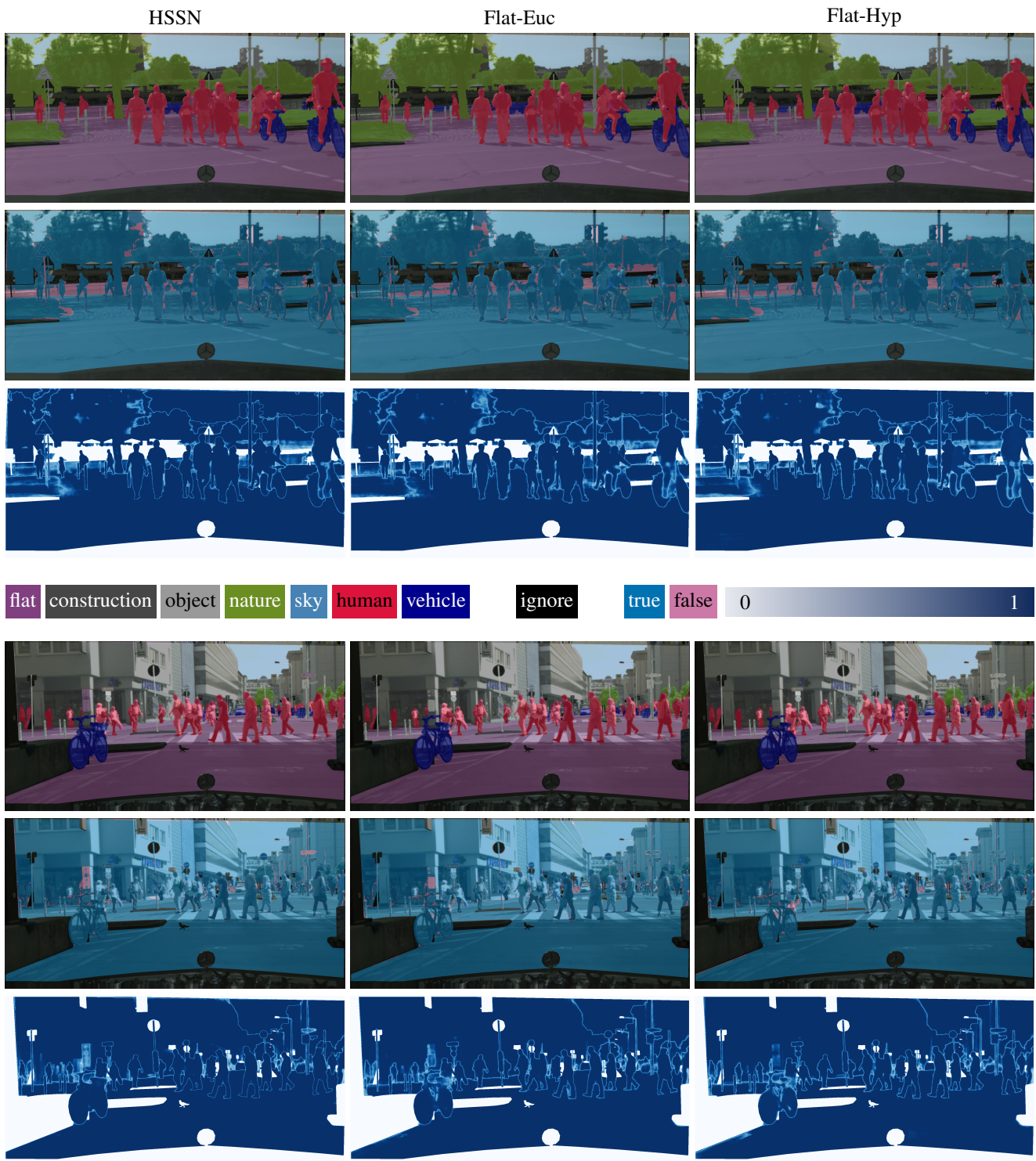
\begin{figure*}[t]
\centering
  \begin{subfigure}{0.33\linewidth}
    \centering
    HSSN\\\vspace{0.3em}
    \includegraphics[width=\linewidth]{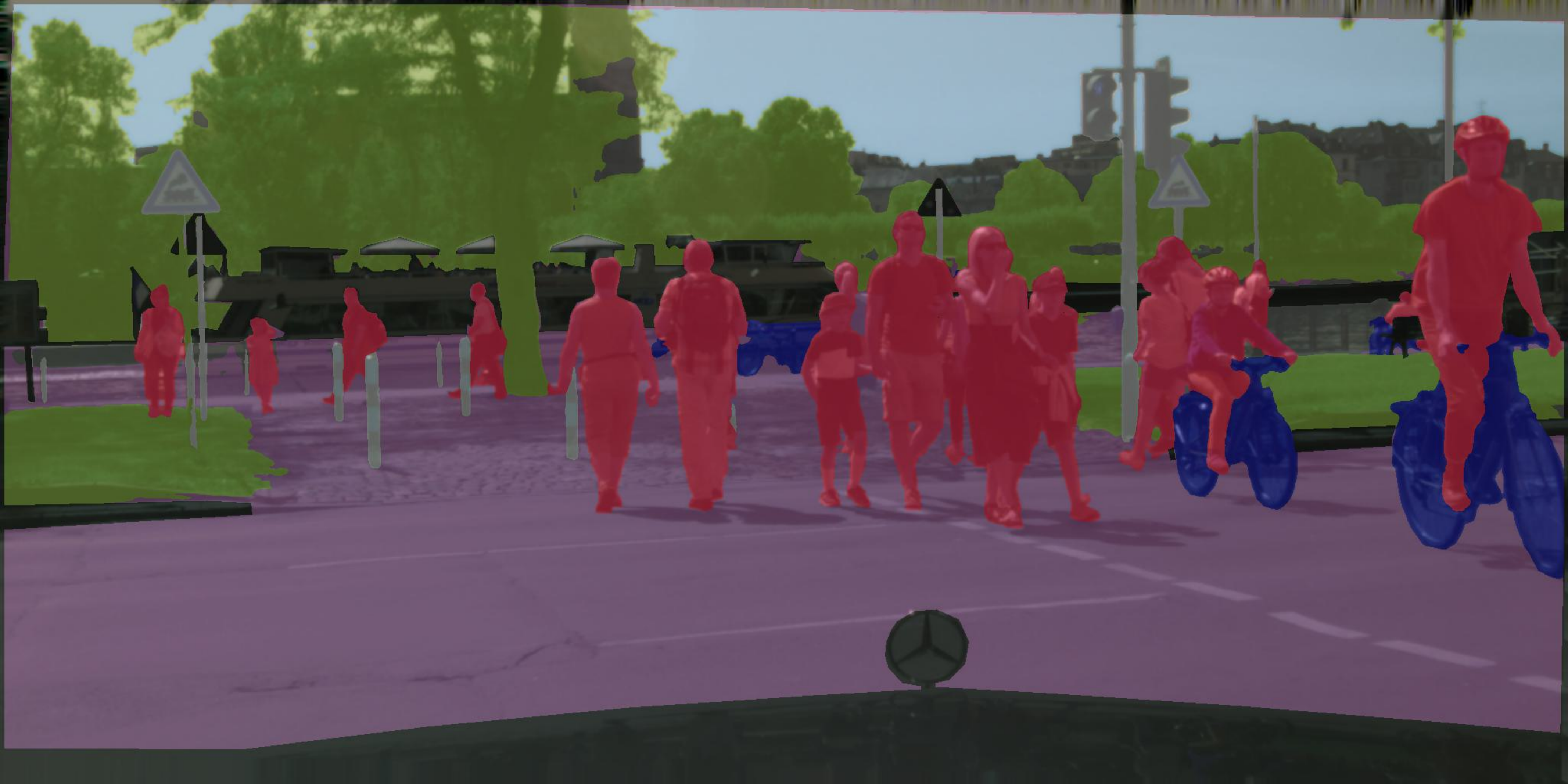}\\[1mm]
    \includegraphics[width=\linewidth]{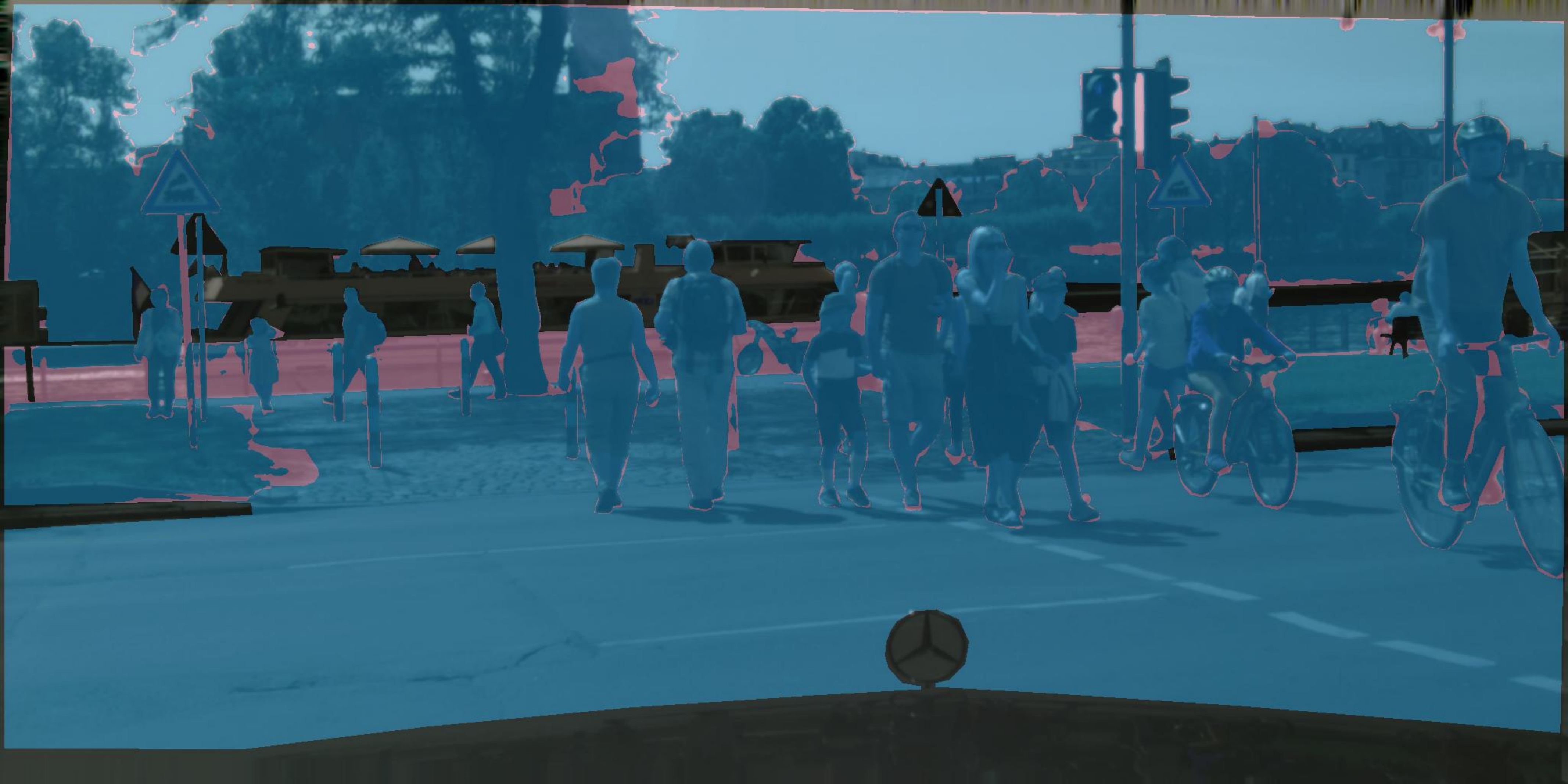}\\[1mm]
    \includegraphics[width=\linewidth]{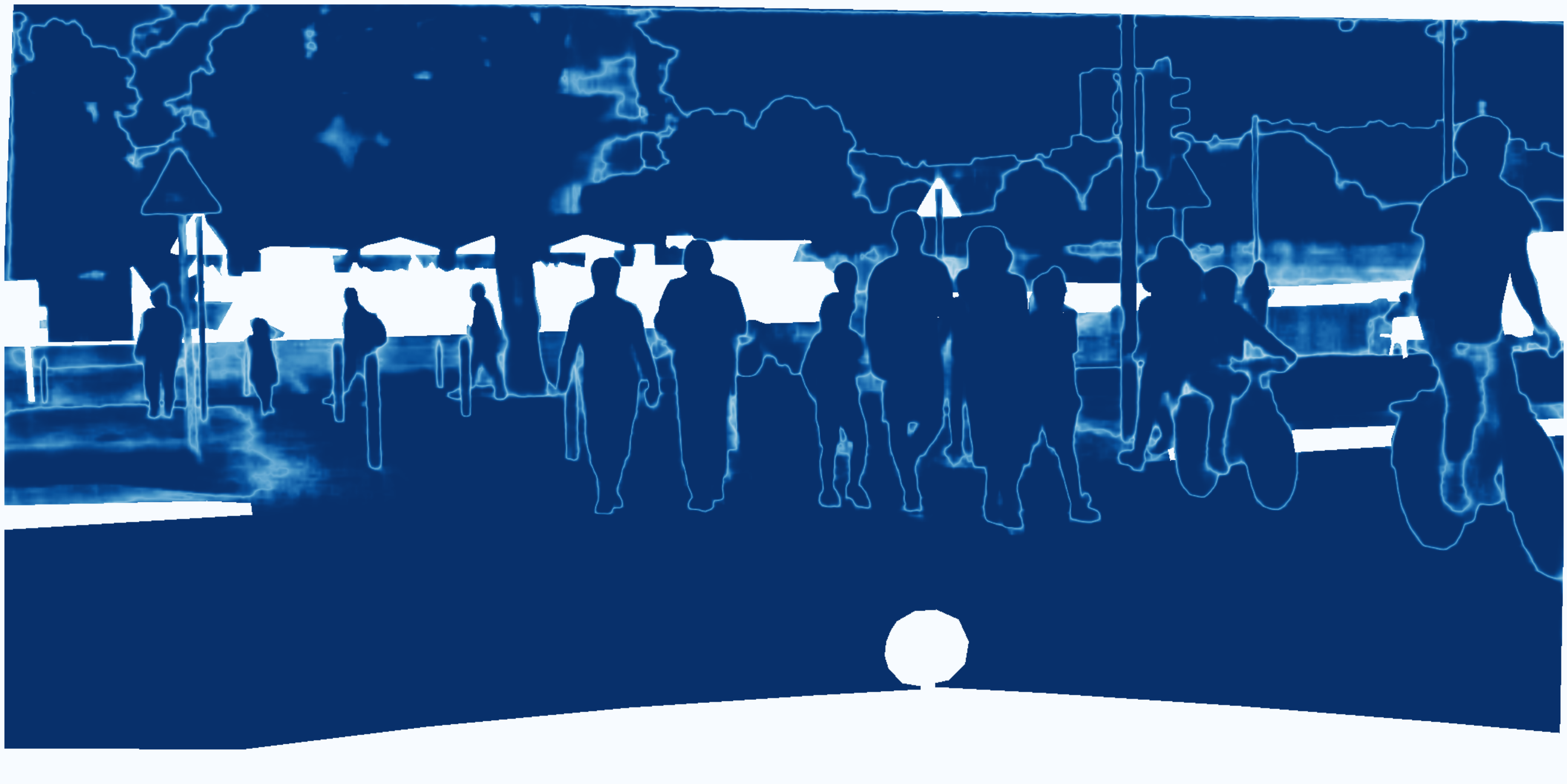}
  \label{fig:cs1_1}
  \end{subfigure}
  \begin{subfigure}{0.33\linewidth}
    \centering
    \eucname\\\vspace{0.3em}
     \includegraphics[width=\linewidth]{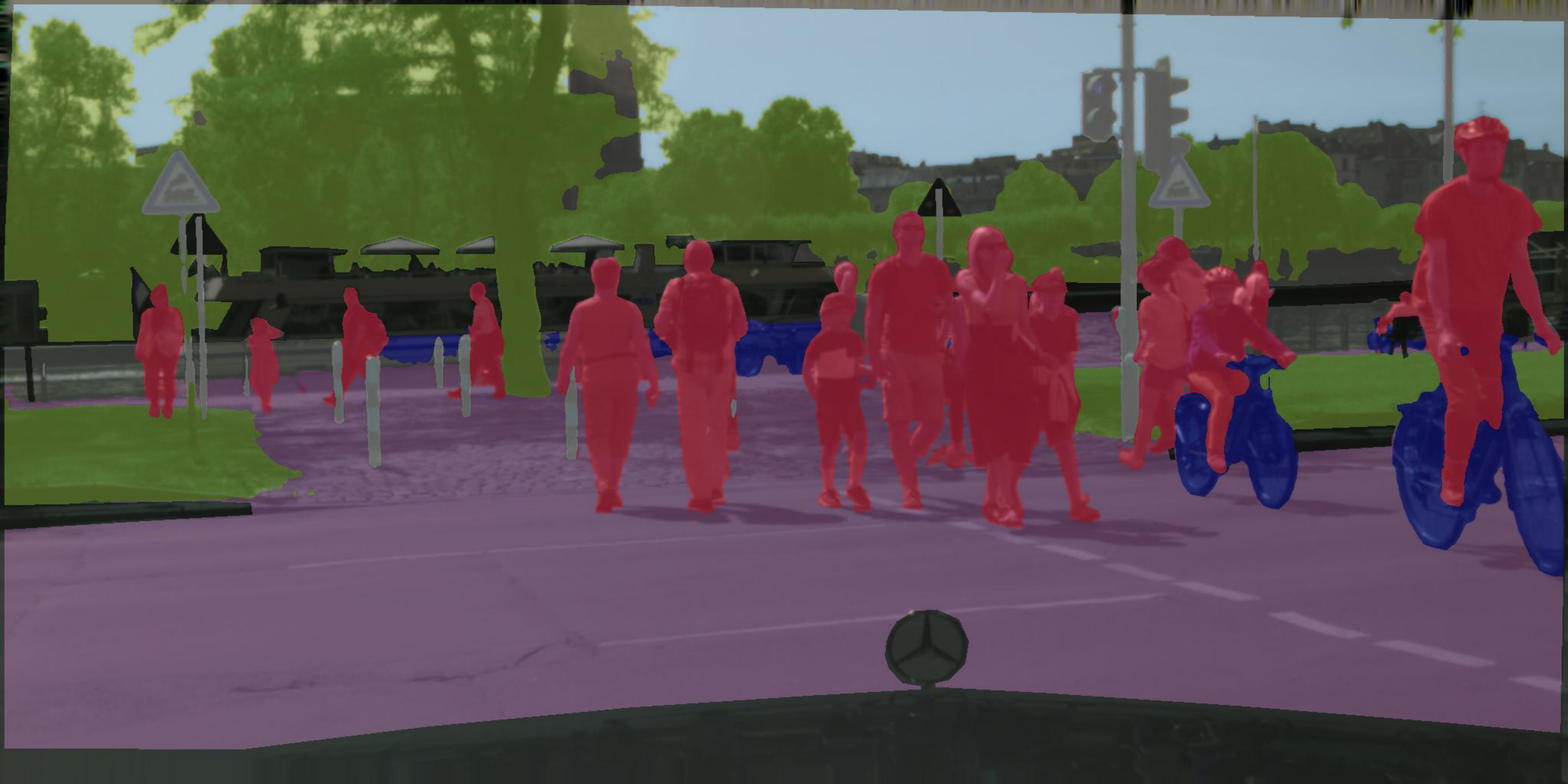}\\[1mm]
    \includegraphics[width=\linewidth]{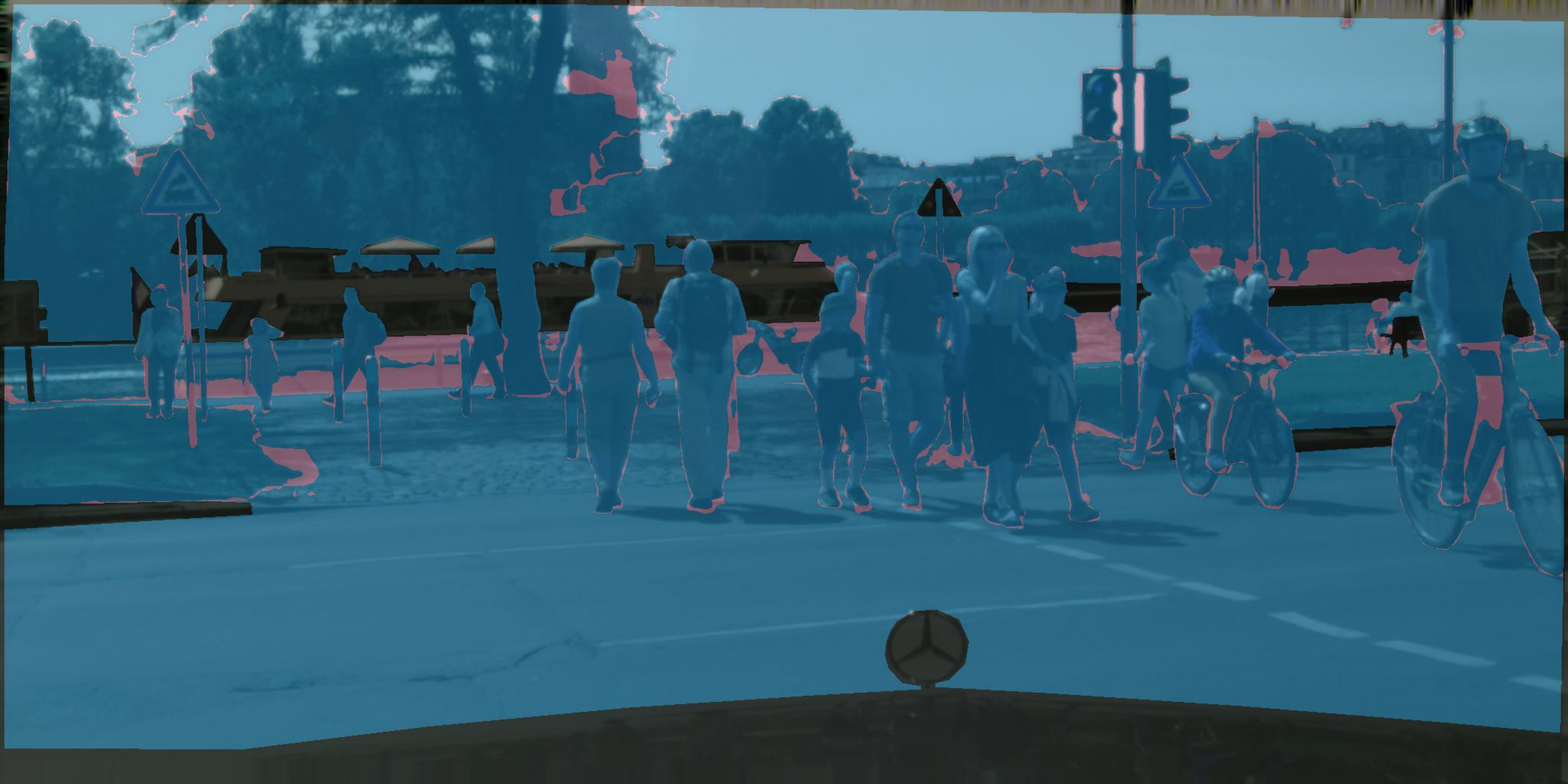}\\[1mm]
    \includegraphics[width=\linewidth]{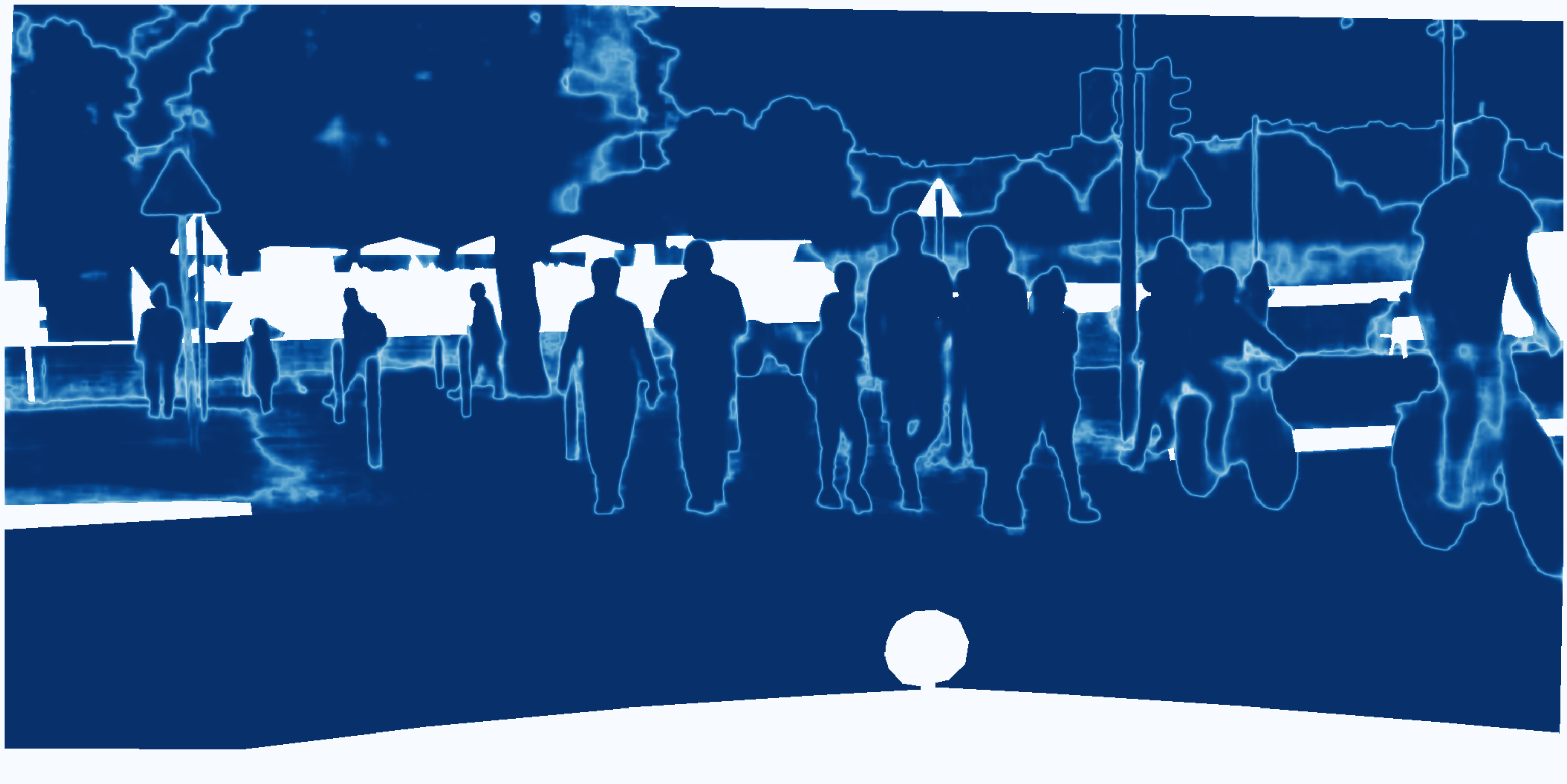}
    \label{fig:cs1_2}
  \end{subfigure}
  \begin{subfigure}{0.33\linewidth}
    \centering
    \hypname\\\vspace{0.2em}
     \includegraphics[width=\linewidth]{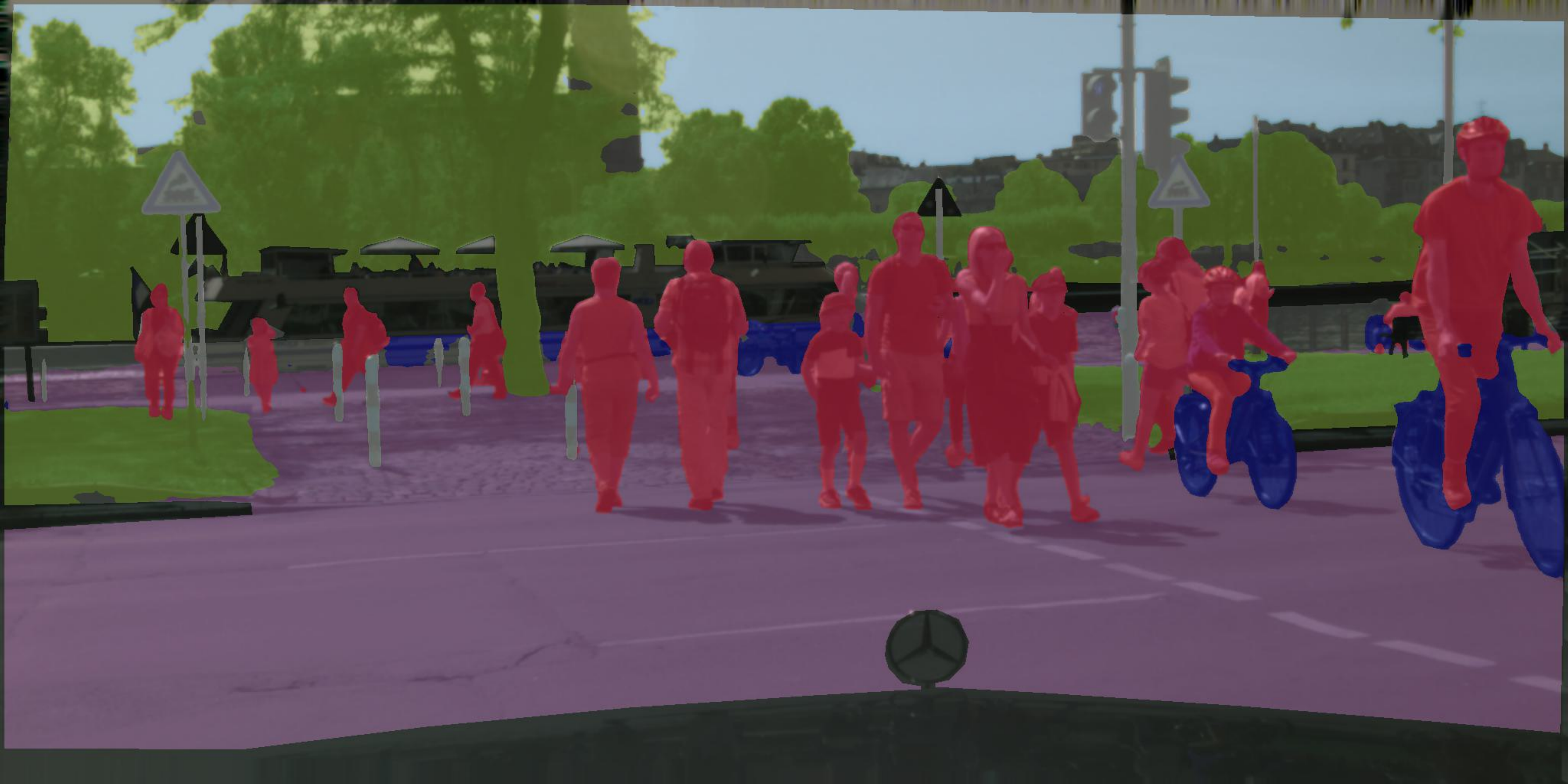}\\[1mm]
    \includegraphics[width=\linewidth]{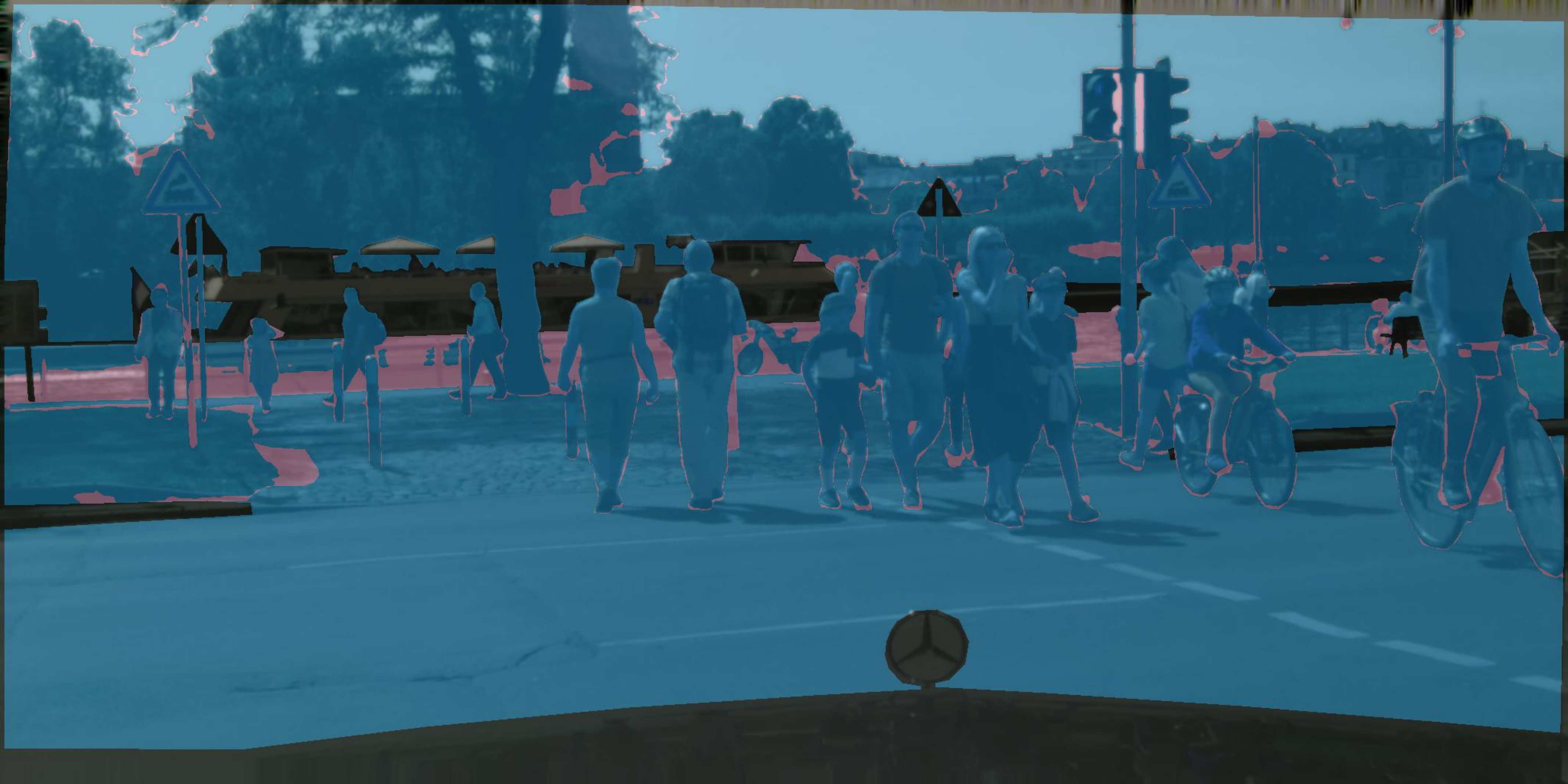}\\[1mm]
    \includegraphics[width=\linewidth]{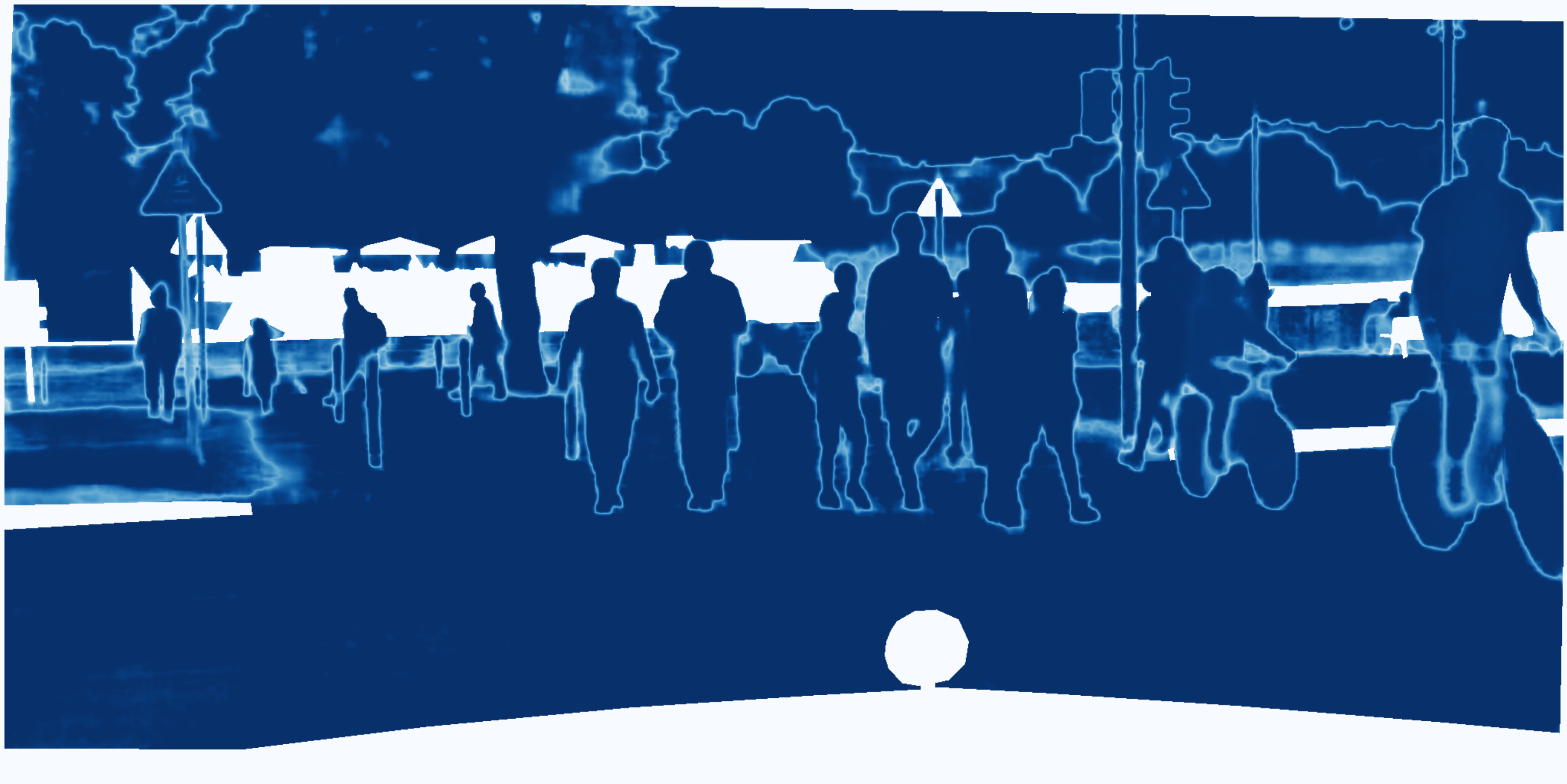}
    \label{fig:cs1_3}
  \end{subfigure}\\
  
    \begin{subfigure}{\linewidth}
        \centering
        \fboxsep 2pt
        \begin{minipage}{0.7\linewidth}
            \colorbox{flat}{\strut \color{white}{flat}}
            \colorbox{construction}{\strut \color{white}{construction}}
            \colorbox{object}{\strut  object}
            \colorbox{nature}{\strut \color{white}{nature}}
            \colorbox{sky}{\strut \color{white}{sky}}
            \colorbox{human}{\strut human}
            \colorbox{vehicle}{\strut \color{white}{vehicle}}\hspace{2em}
            \colorbox{ignore}{\strut \color{white}{ignore}}\hspace{2em}
            \colorbox{true}{\strut \color{white}{true}}
            \colorbox{false}{\strut false}\hspace{2em}%
        \end{minipage}%
        \begin{minipage}{0.3\linewidth}
        \begin{tikzpicture}
        \node [rectangle, left color=left!10!white, right color=left, anchor=north, minimum width=\linewidth, minimum height=0.5cm] (box) at (current page.north){0 \hspace{12em}  \color{white}{1}};
        \end{tikzpicture}
        \end{minipage}
    \end{subfigure}\\
    \vspace{1em}
  
  \begin{subfigure}{0.33\linewidth}
    \centering
    \includegraphics[width=\linewidth]{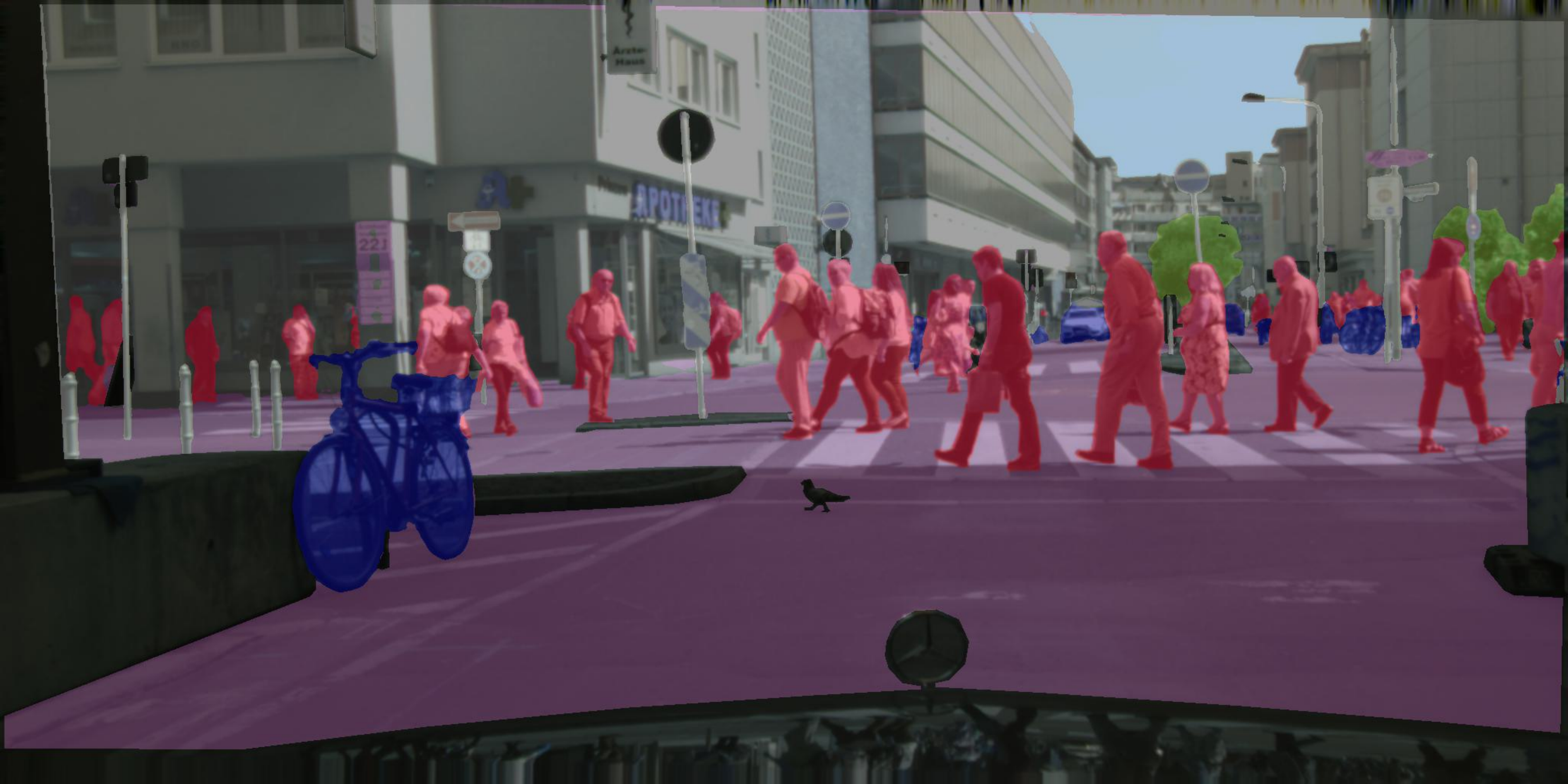}\\[1mm]
    \includegraphics[width=\linewidth]{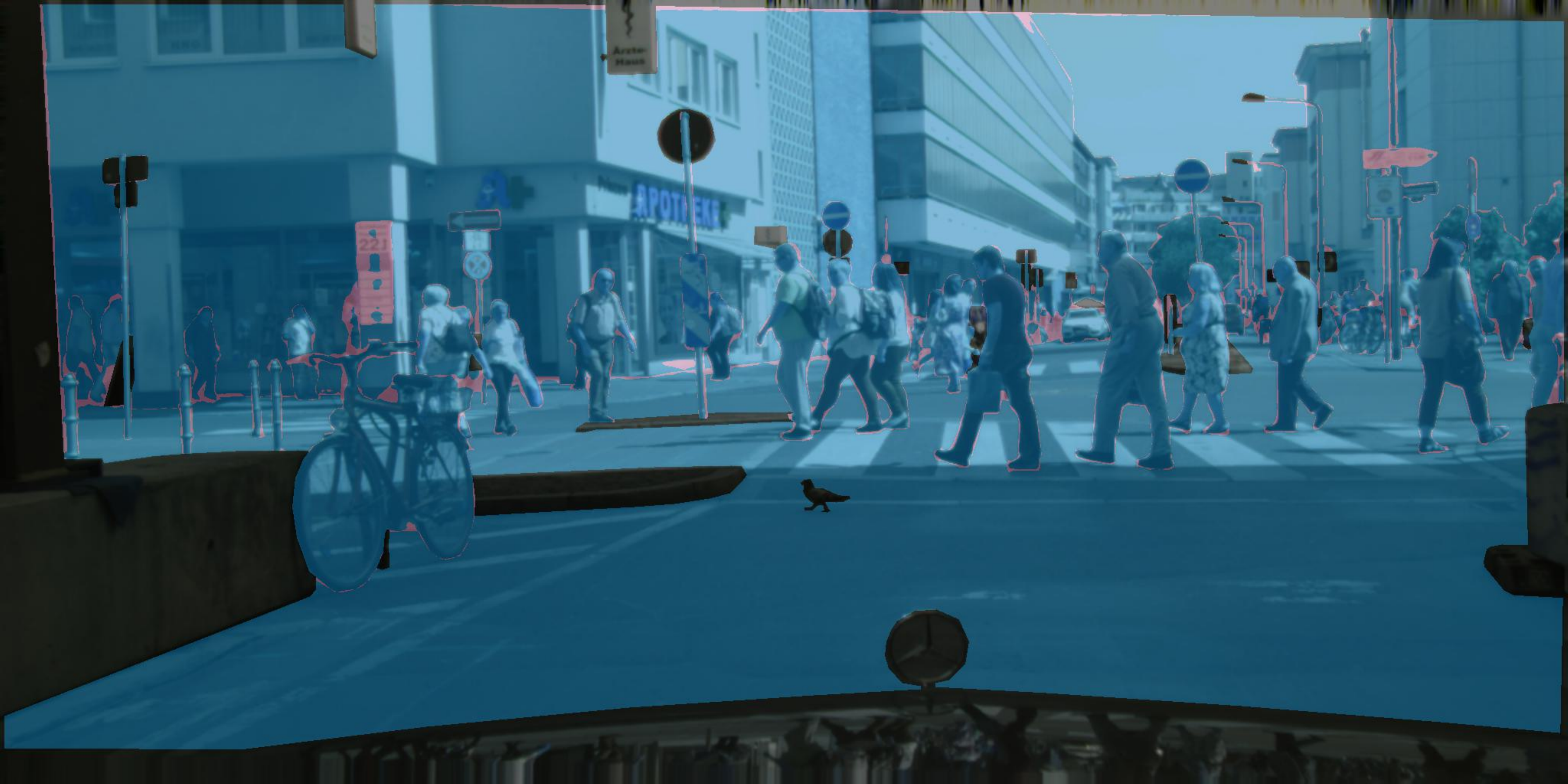}\\[1mm]
    \includegraphics[width=\linewidth]{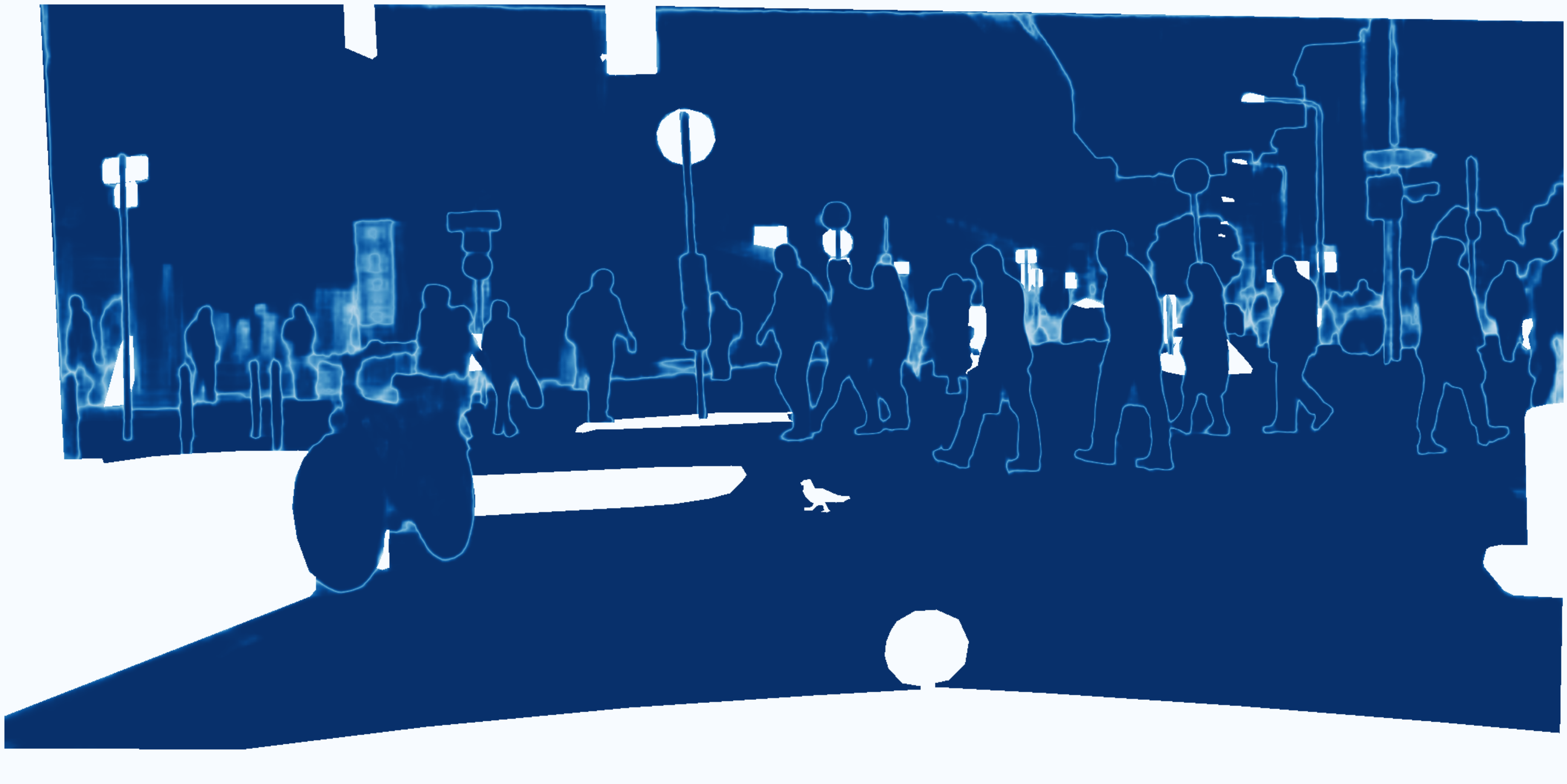}
  \label{fig:cs2_1}
  \end{subfigure}
  \begin{subfigure}{0.33\linewidth}
    \centering
     \includegraphics[width=\linewidth]{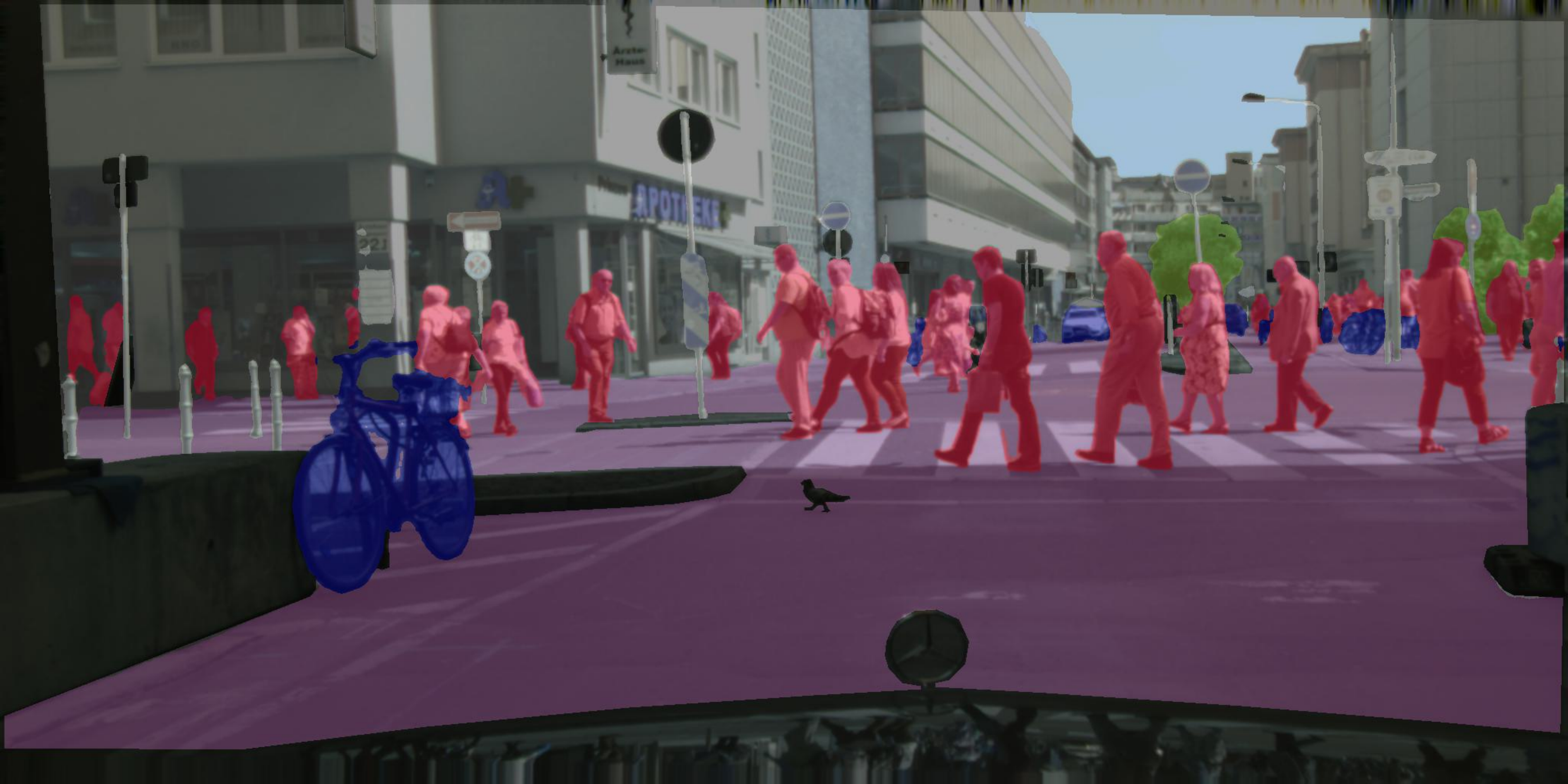}\\[1mm]
    \includegraphics[width=\linewidth]{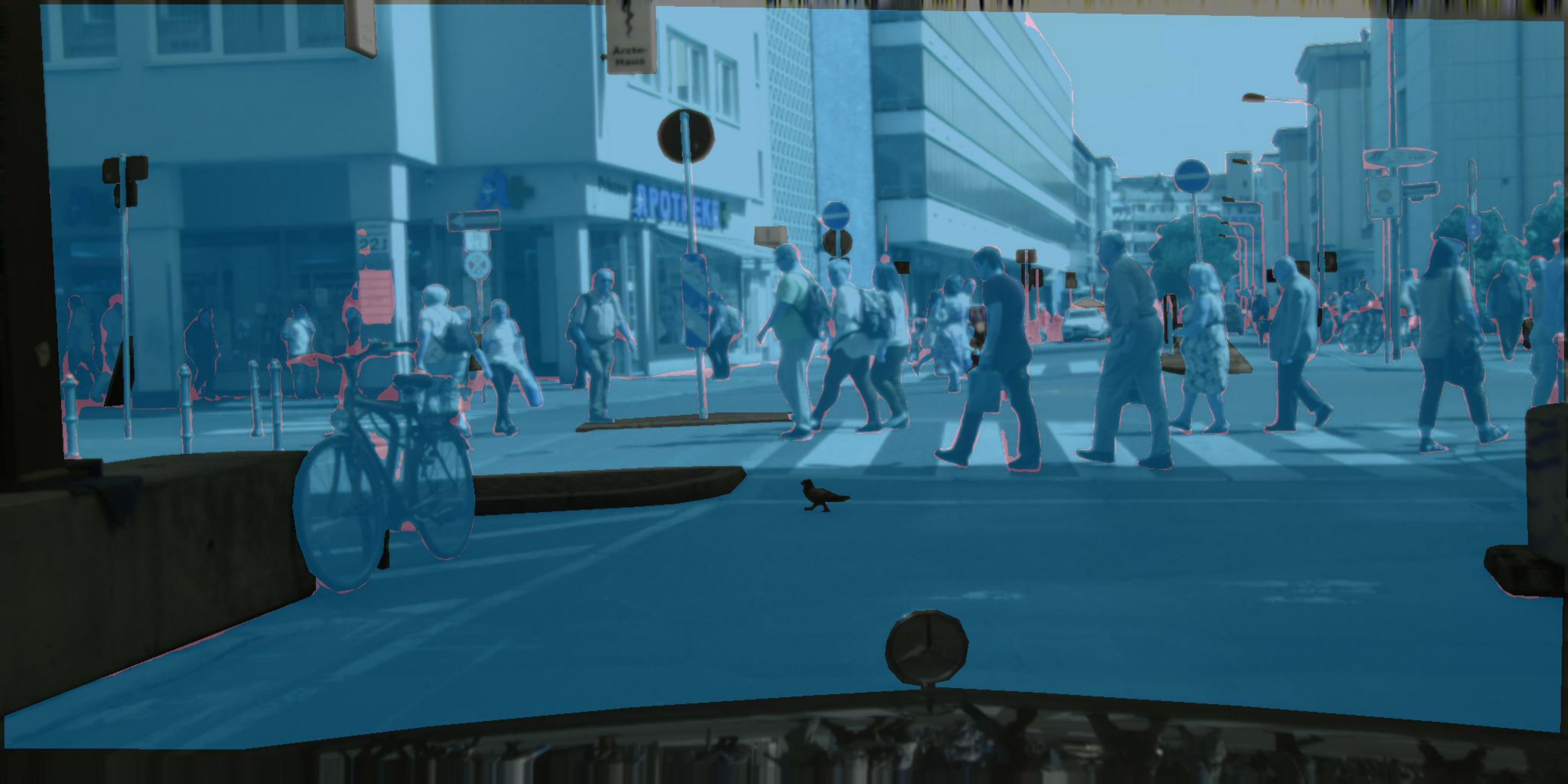}\\[1mm]
    \includegraphics[width=\linewidth]{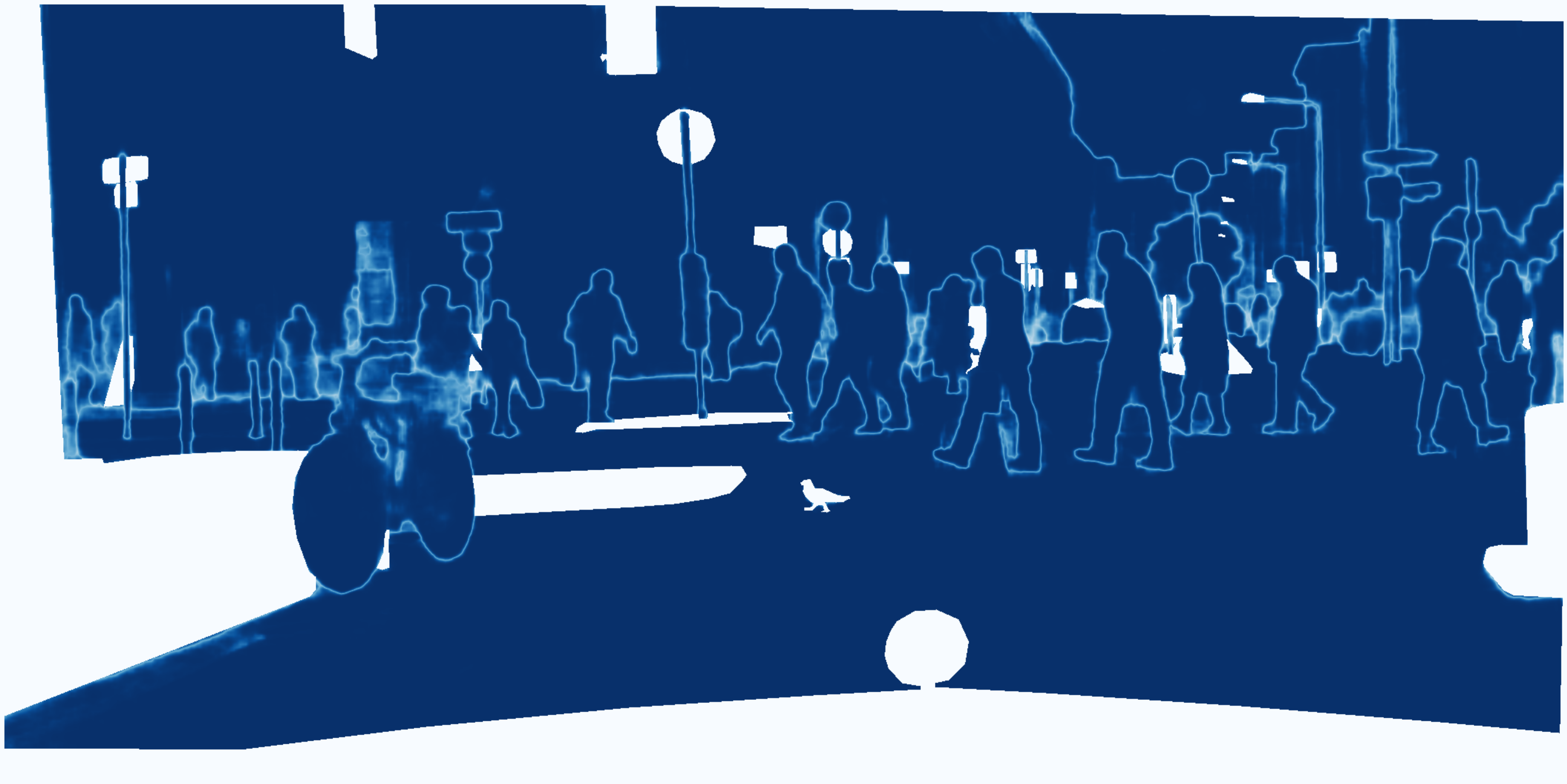}
    \label{fig:cs2_2}
  \end{subfigure}
  \begin{subfigure}{0.33\linewidth}
    \centering
     \includegraphics[width=\linewidth]{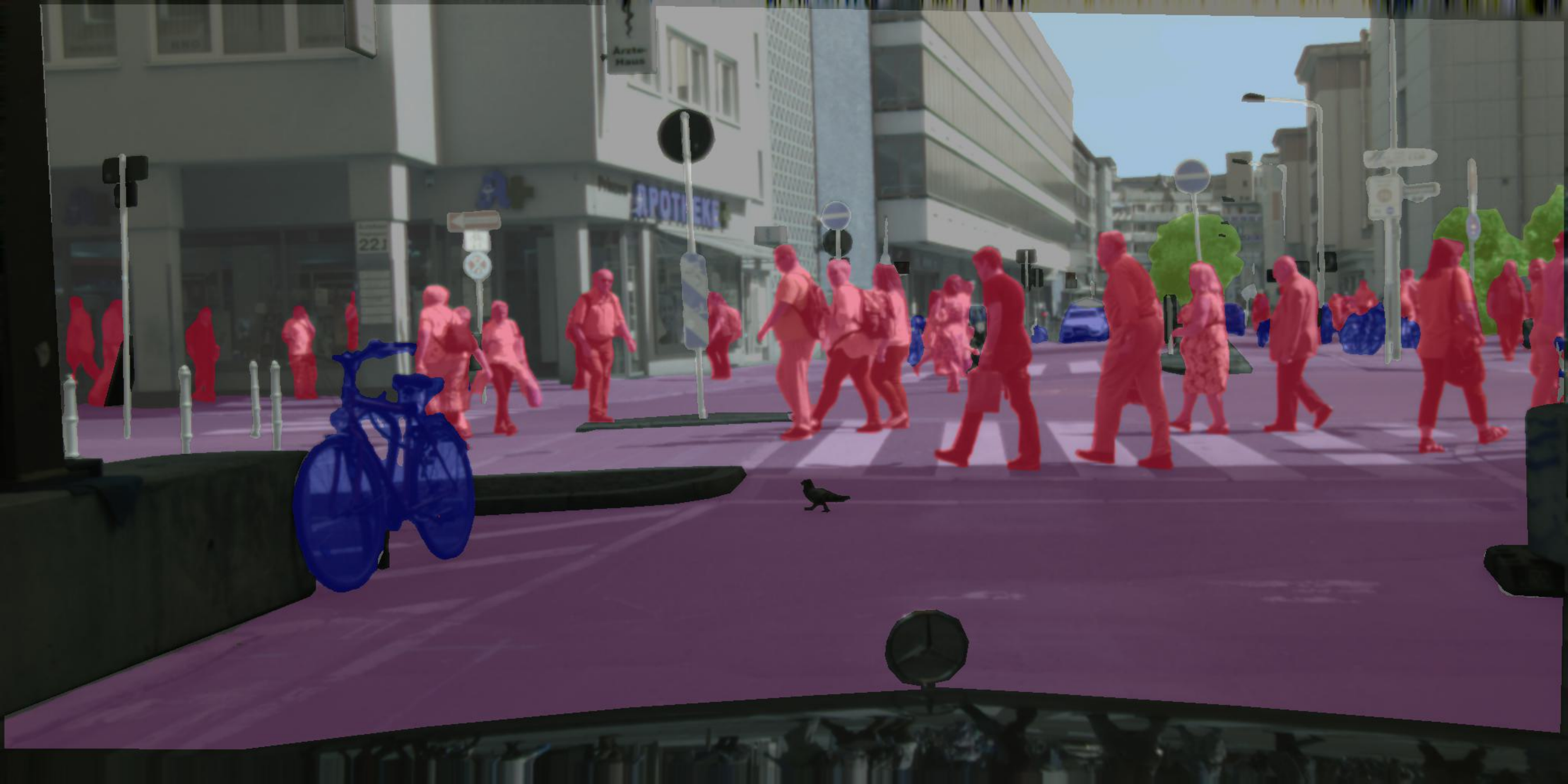}\\[1mm]
    \includegraphics[width=\linewidth]{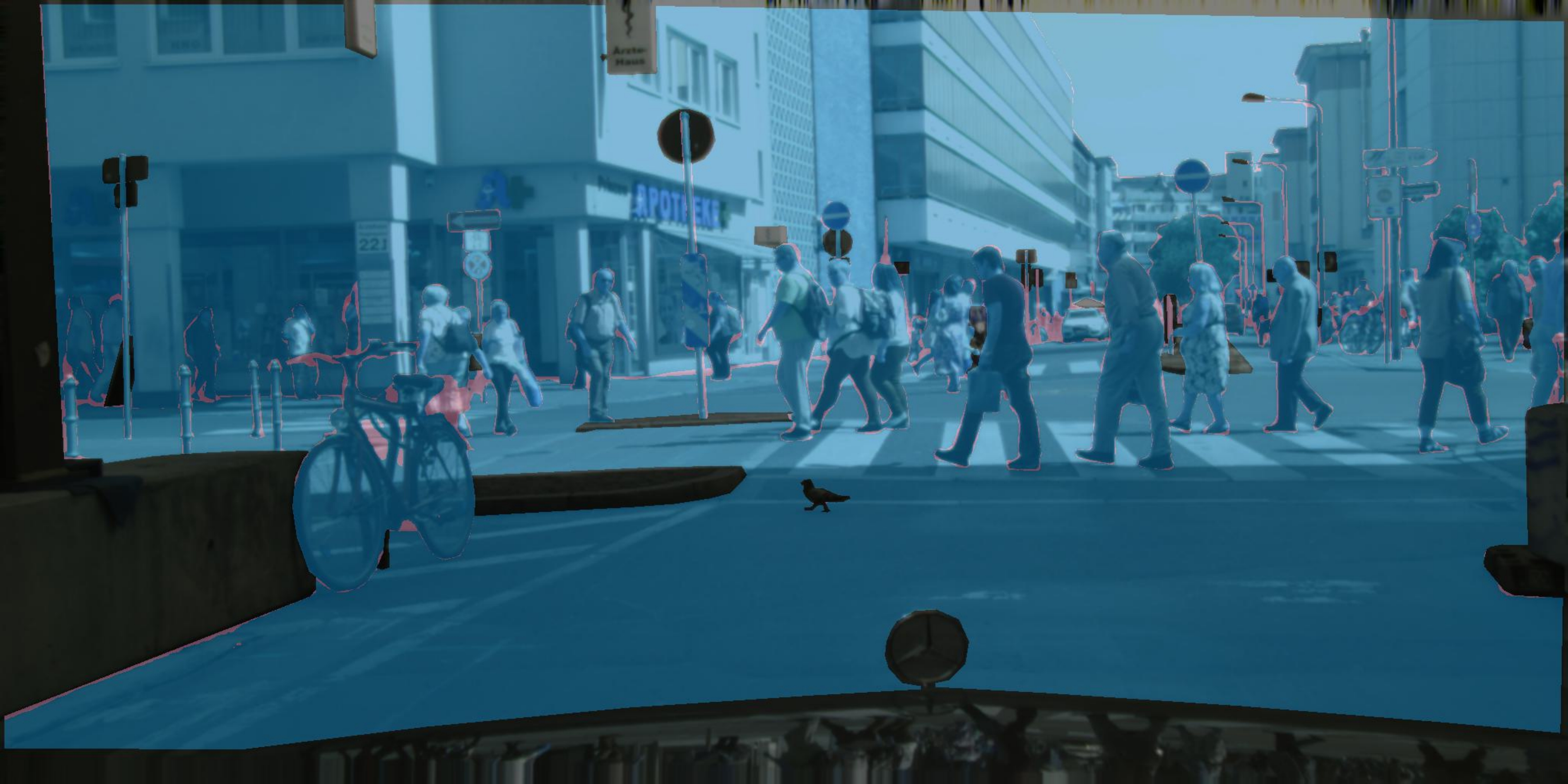}\\[1mm]
    \includegraphics[width=\linewidth]{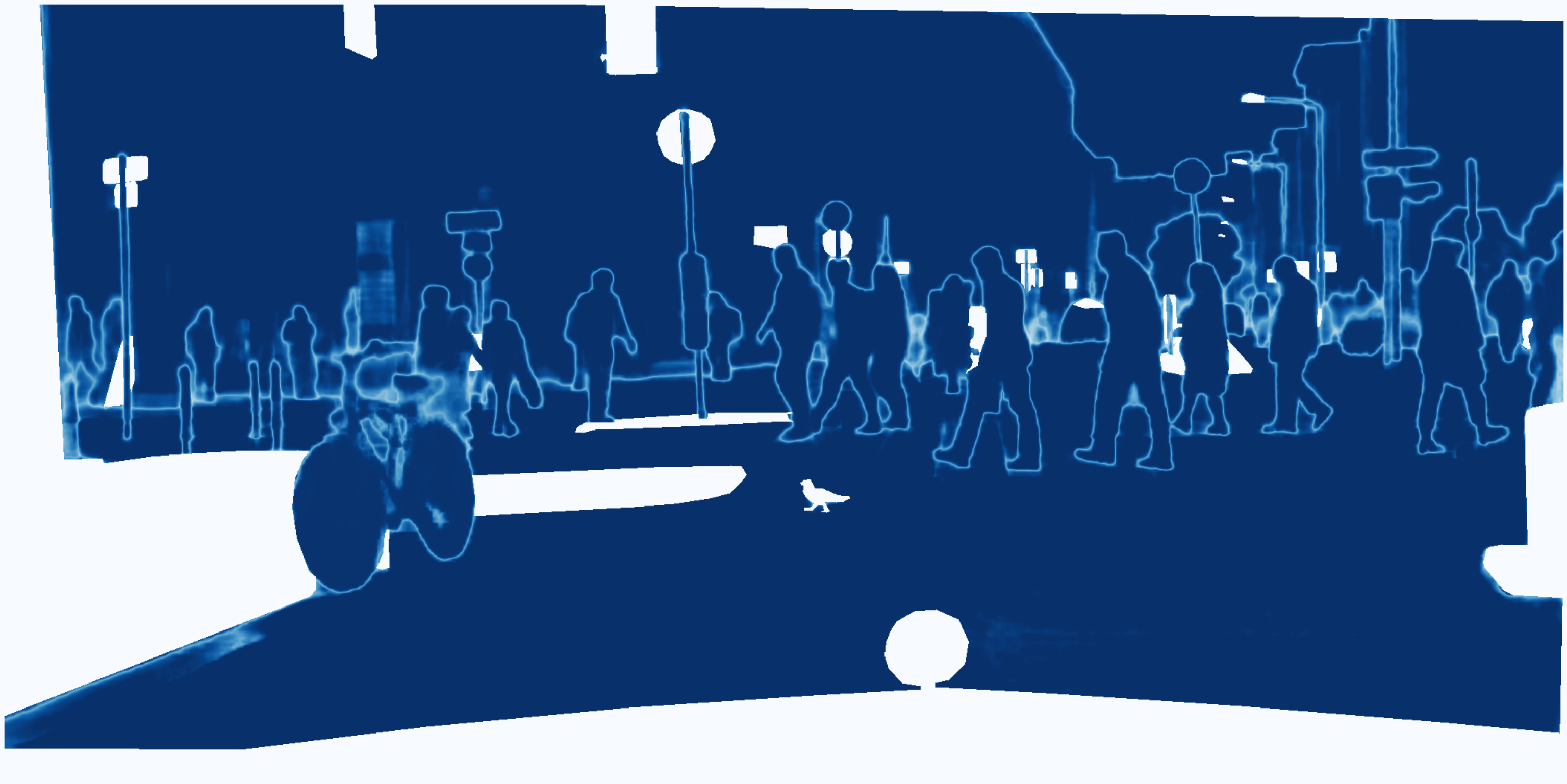}
    \label{fig:cs2_3}
  \end{subfigure}
  \caption{\textbf{Two qualitative examples from Cityscapes.} See \cref{sec:qualitative} for analysis.}
  \label{fig:qaulics1}
  \vspace{-0.7em}
\end{figure*}

\begin{figure*}[t]
\centering
  \begin{subfigure}{0.33\linewidth}
    \centering
    HSSN\\\vspace{0.3em}
    \includegraphics[width=\linewidth]{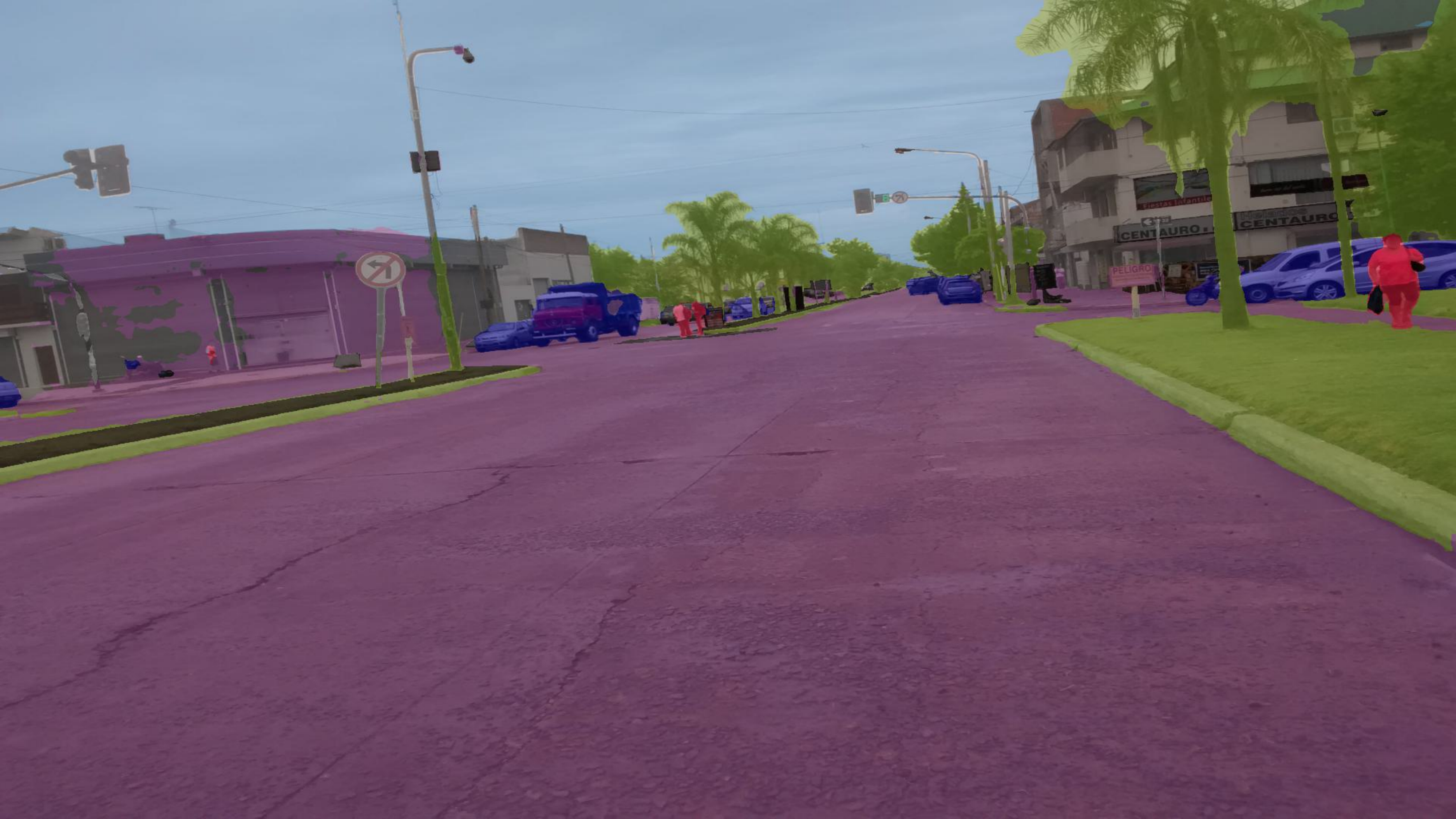}\\[1mm]
    \includegraphics[width=\linewidth]{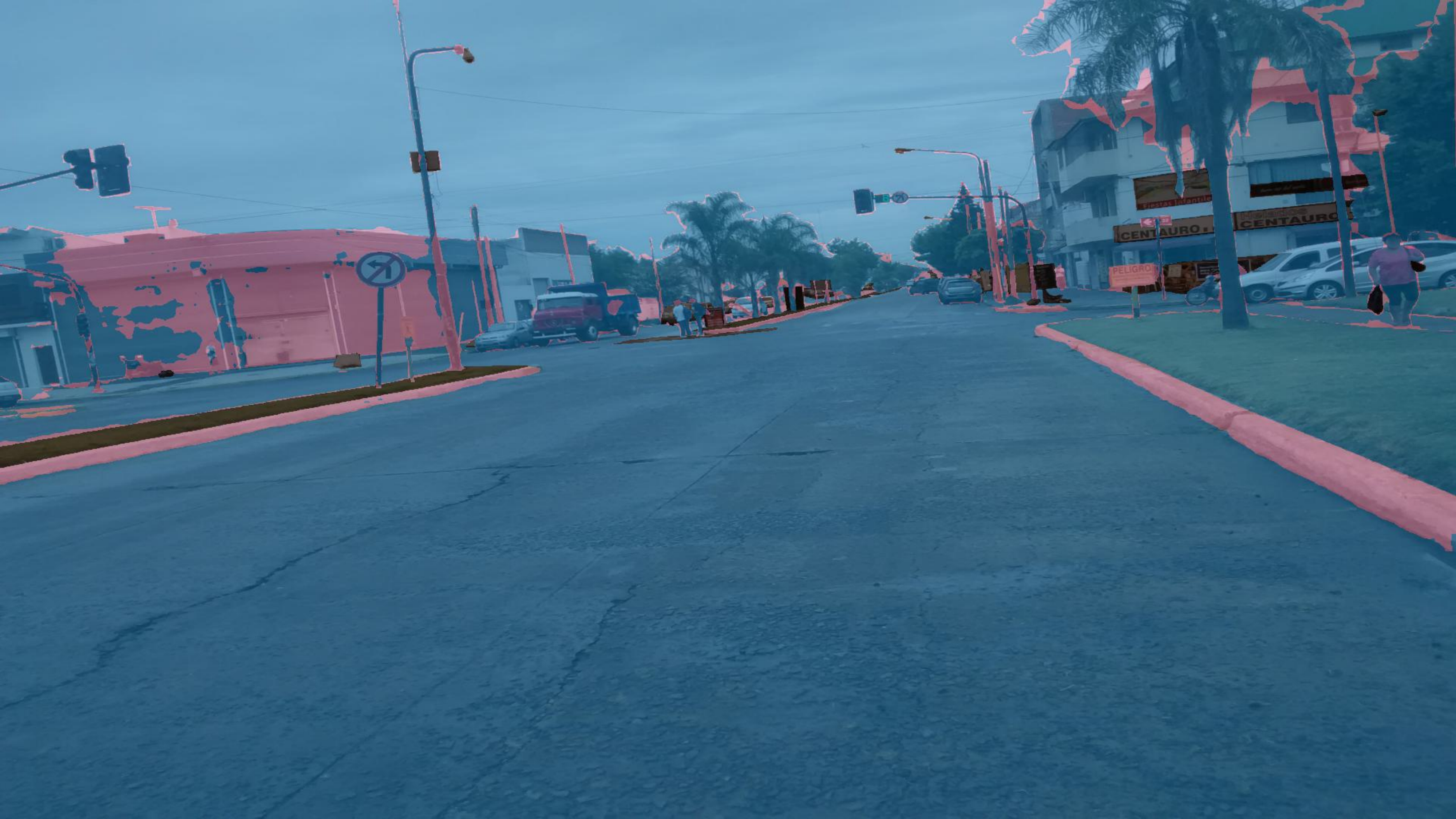}\\[1mm]
    \includegraphics[width=\linewidth]{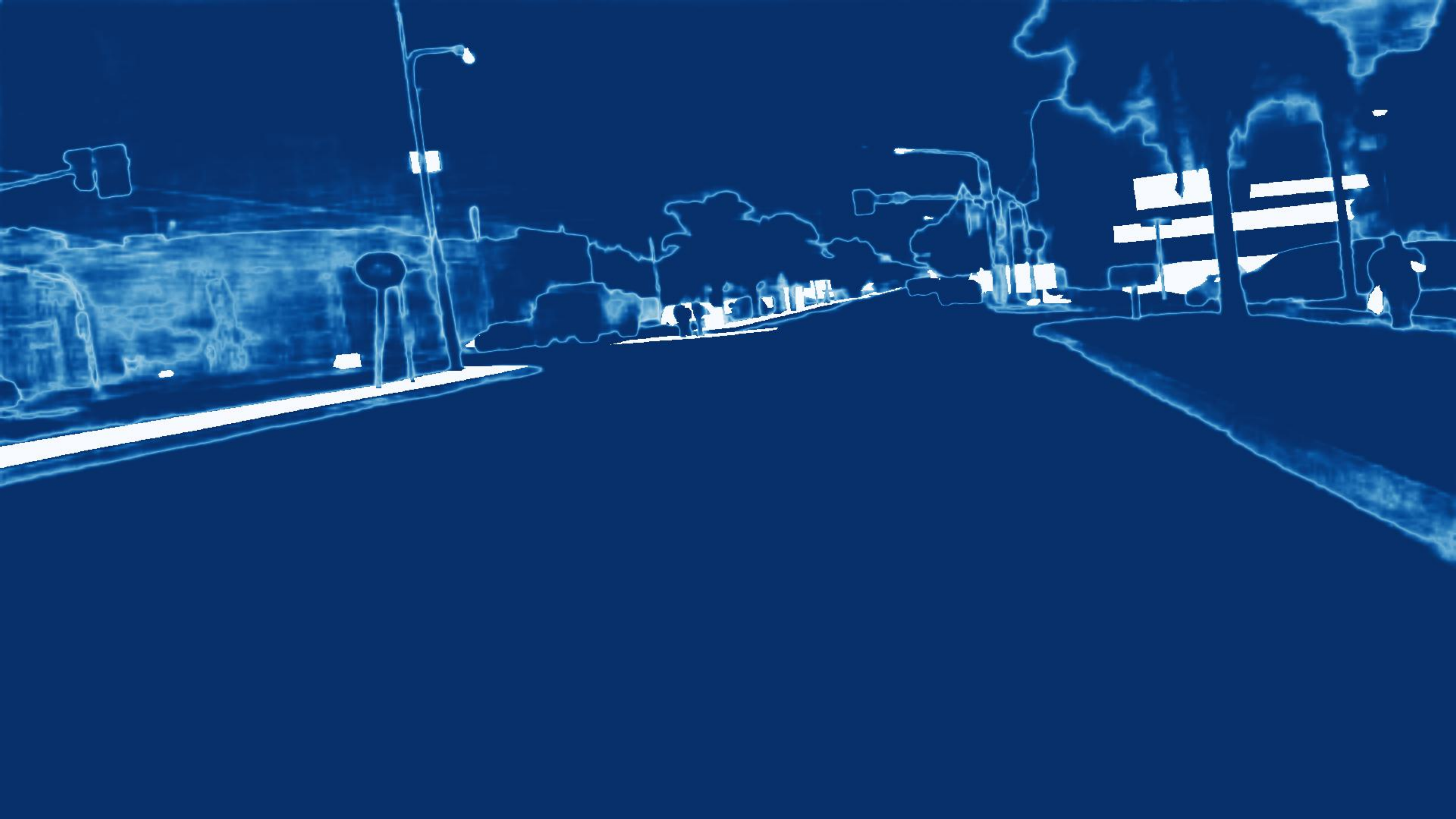}
  \label{fig:map1_1}
  \end{subfigure}
  \begin{subfigure}{0.33\linewidth}
    \centering
    \eucname\\\vspace{0.3em}
     \includegraphics[width=\linewidth]{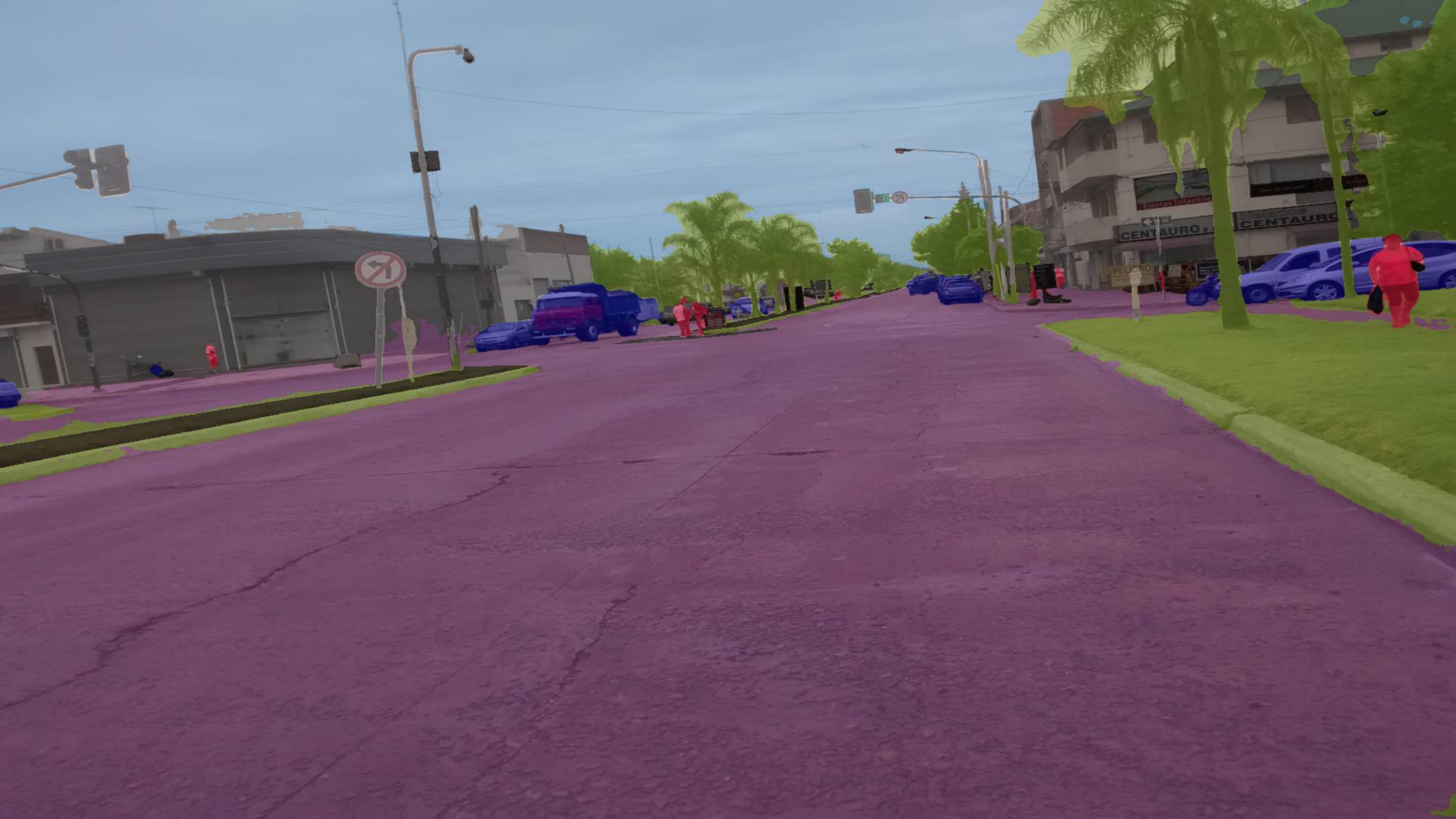}\\[1mm]
    \includegraphics[width=\linewidth]{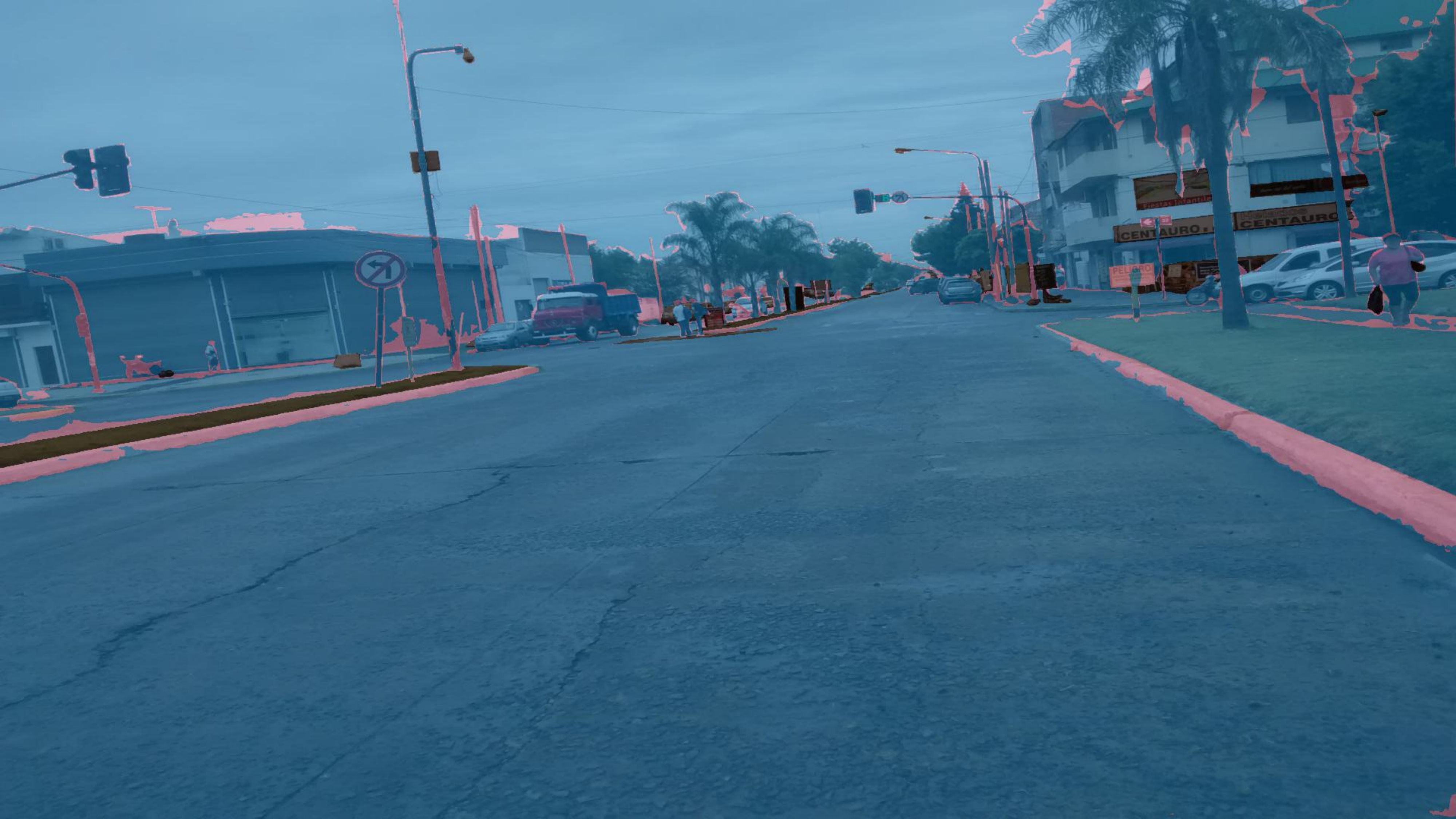}\\[1mm]
    \includegraphics[width=\linewidth]{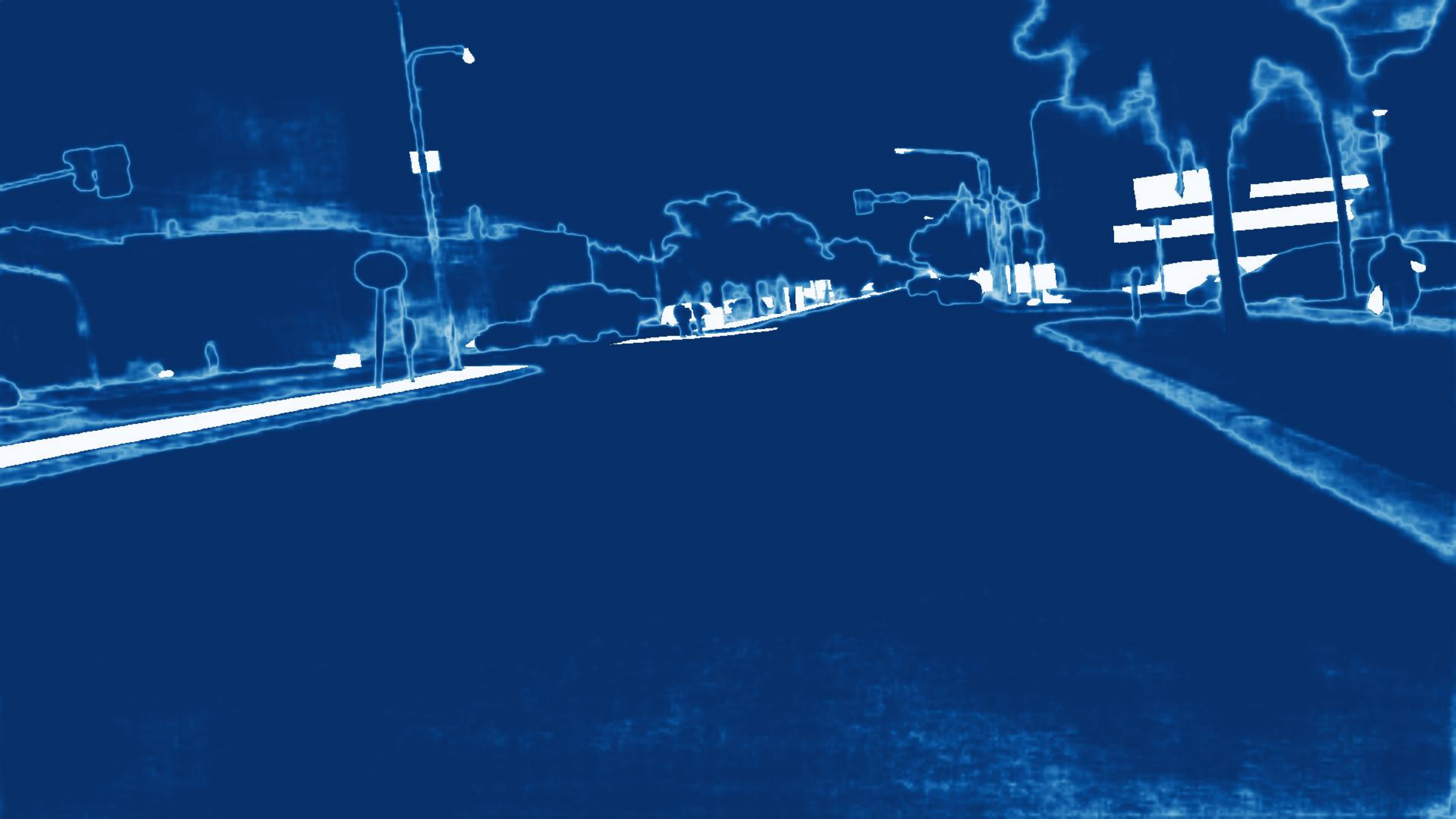}
    \label{fig:map1_2}
  \end{subfigure}
  \begin{subfigure}{0.33\linewidth}
    \centering
    \hypname\\\vspace{0.2em}
     \includegraphics[width=\linewidth]{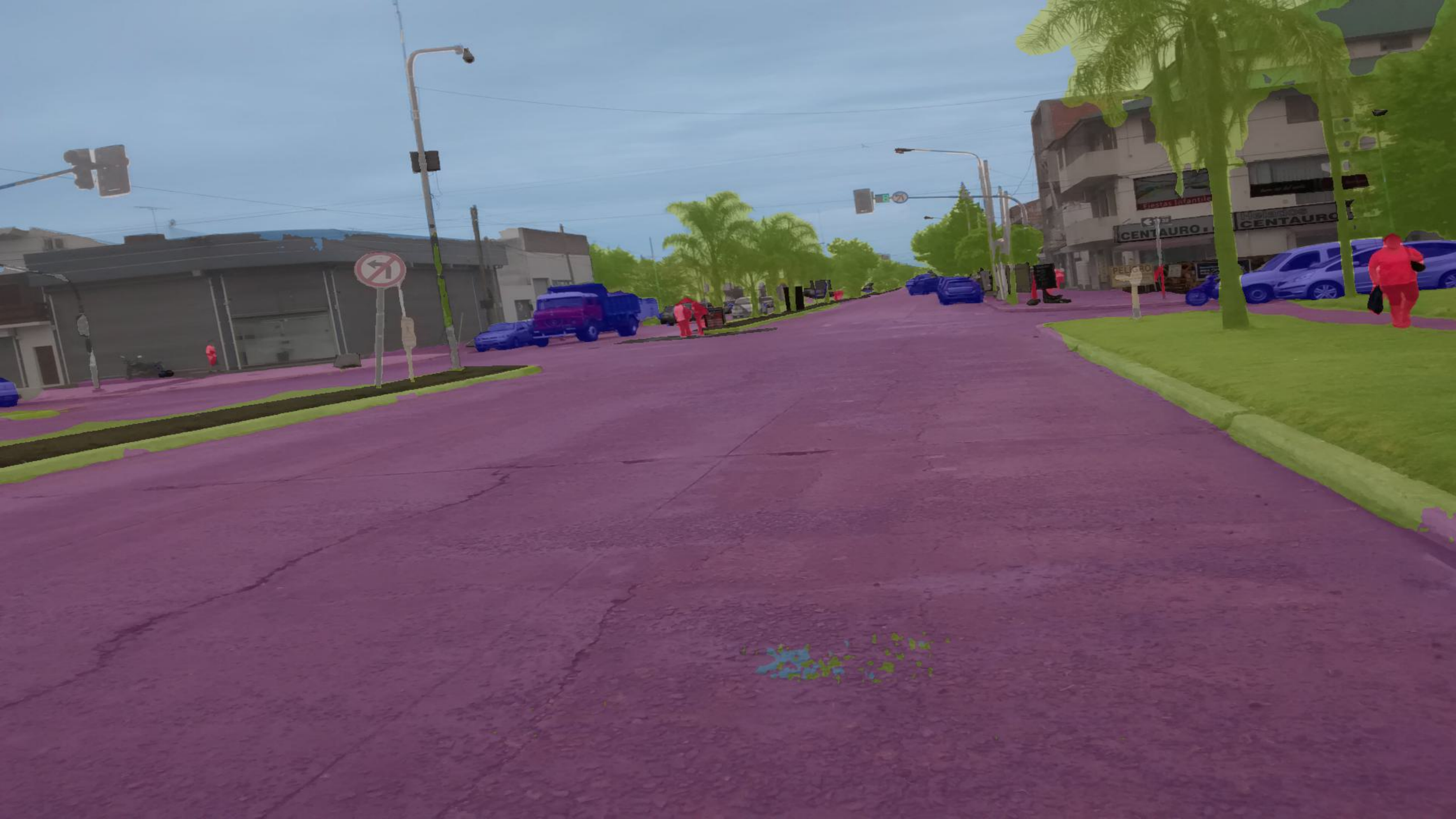}\\[1mm]
    \includegraphics[width=\linewidth]{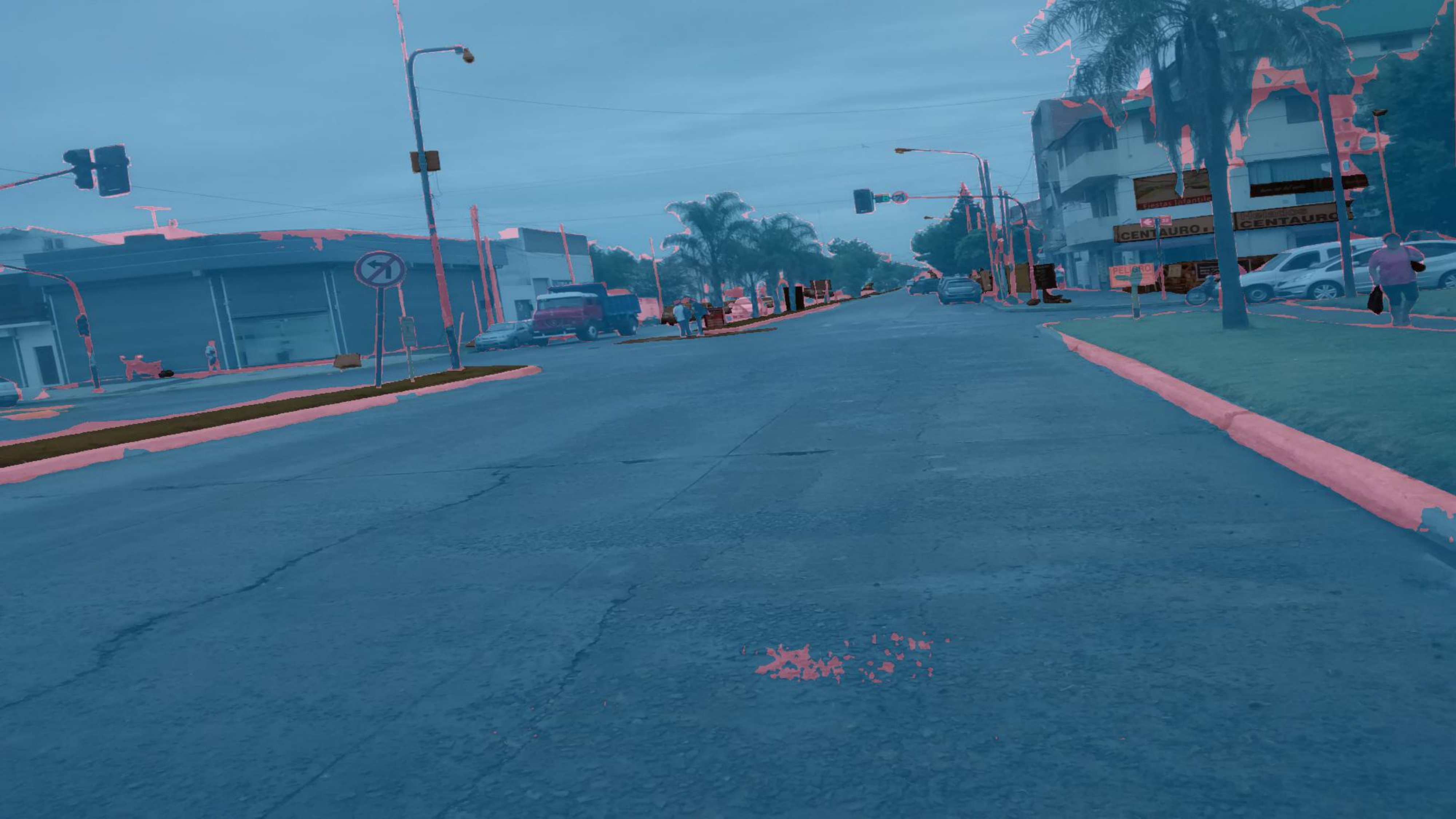}\\[1mm]
    \includegraphics[width=\linewidth]{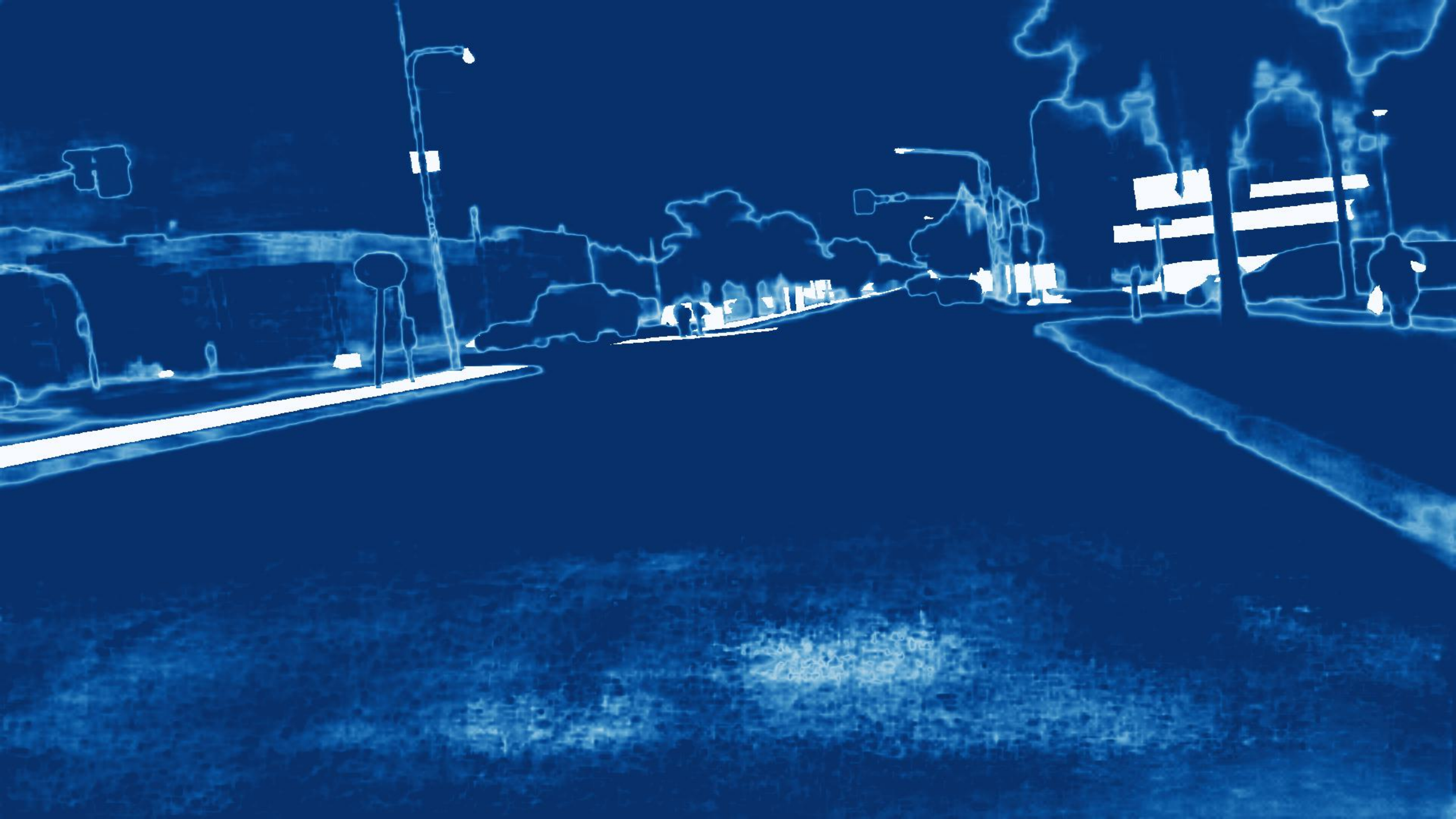}
    \label{fig:map1_3}
  \end{subfigure}
    \begin{subfigure}{\linewidth}
        \centering
        \fboxsep 2pt
        \begin{minipage}{0.7\linewidth}
            \colorbox{flat}{\strut \color{white}{flat}}
            \colorbox{construction}{\strut \color{white}{construction}}
            \colorbox{object}{\strut  object}
            \colorbox{nature}{\strut \color{white}{nature}}
            \colorbox{sky}{\strut \color{white}{sky}}
            \colorbox{human}{\strut human}
            \colorbox{vehicle}{\strut \color{white}{vehicle}}\hspace{2em}
            \colorbox{ignore}{\strut \color{white}{ignore}}\hspace{2em}
            \colorbox{true}{\strut \color{white}{true}}
            \colorbox{false}{\strut false}\hspace{2em}%
        \end{minipage}%
        \begin{minipage}{0.3\linewidth}
        \begin{tikzpicture}
        \node [rectangle, left color=left!10!white, right color=left, anchor=north, minimum width=\linewidth, minimum height=0.5cm] (box) at (current page.north){0 \hspace{12em}  \color{white}{1}};
        \end{tikzpicture}
        \end{minipage}
    \end{subfigure}\\
    \vspace{1em}
    
  \begin{subfigure}{0.33\linewidth}
    \centering
    \includegraphics[width=\linewidth]{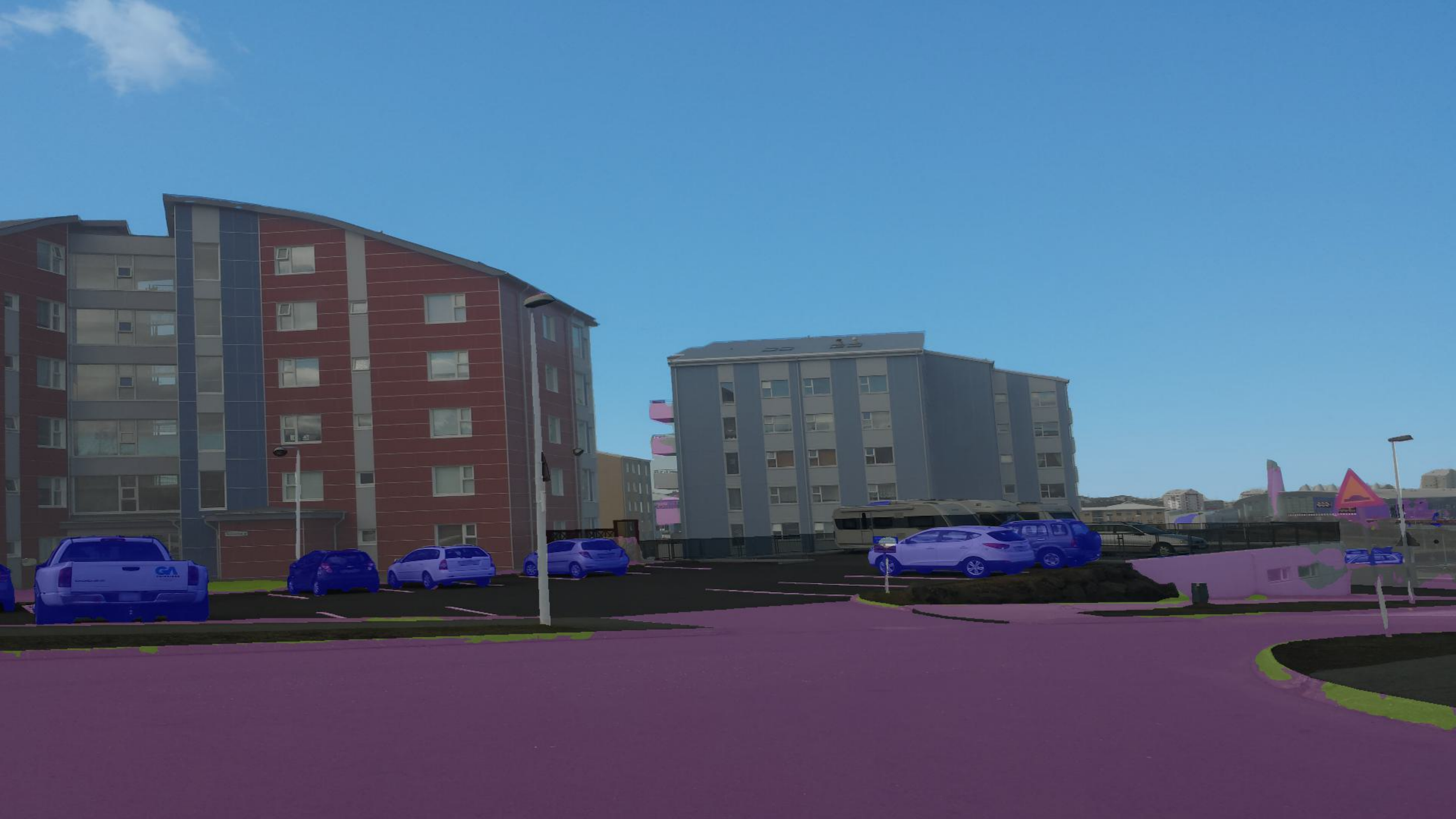}\\[1mm]
    \includegraphics[width=\linewidth]{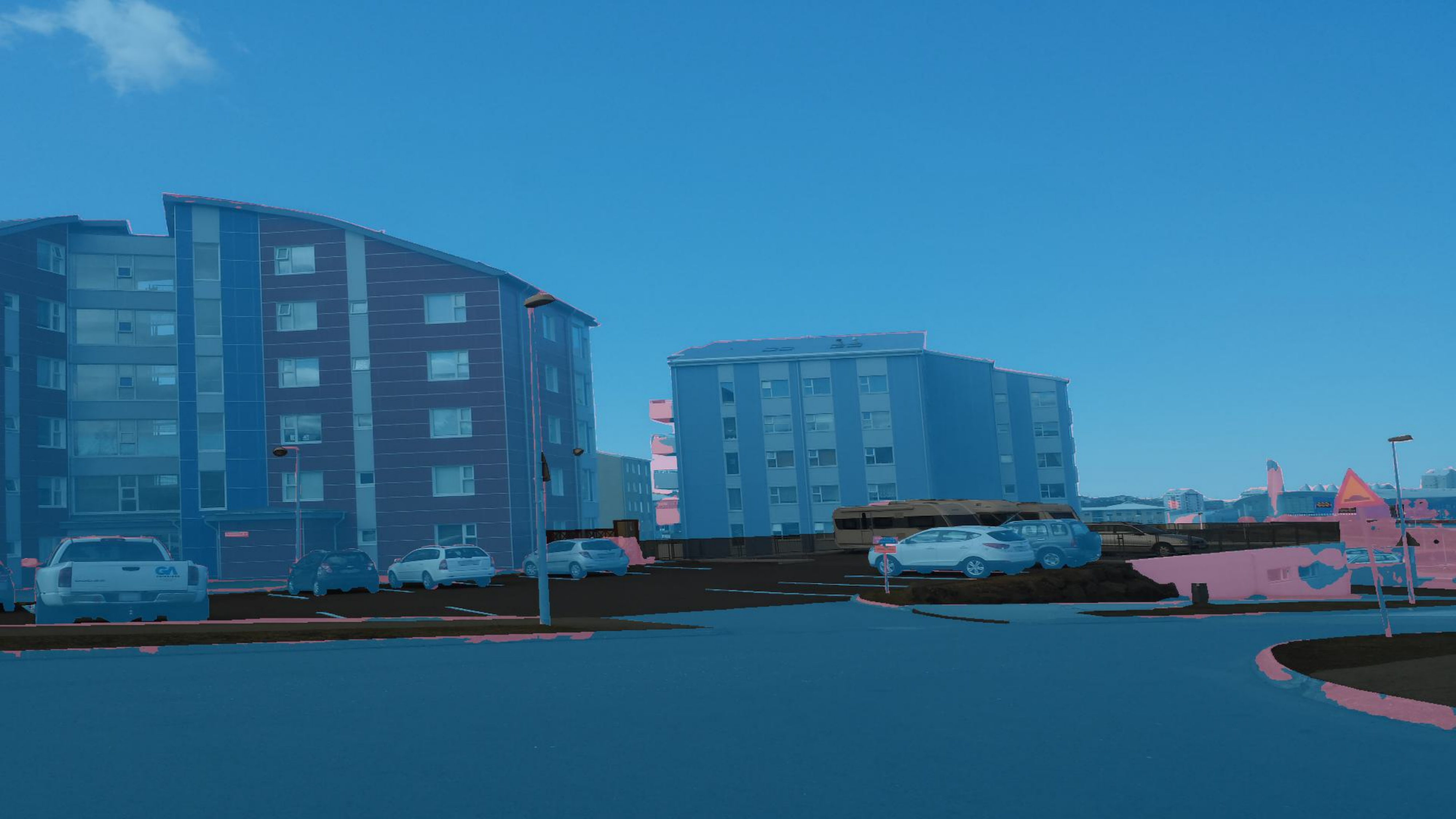}\\[1mm]
    \includegraphics[width=\linewidth]{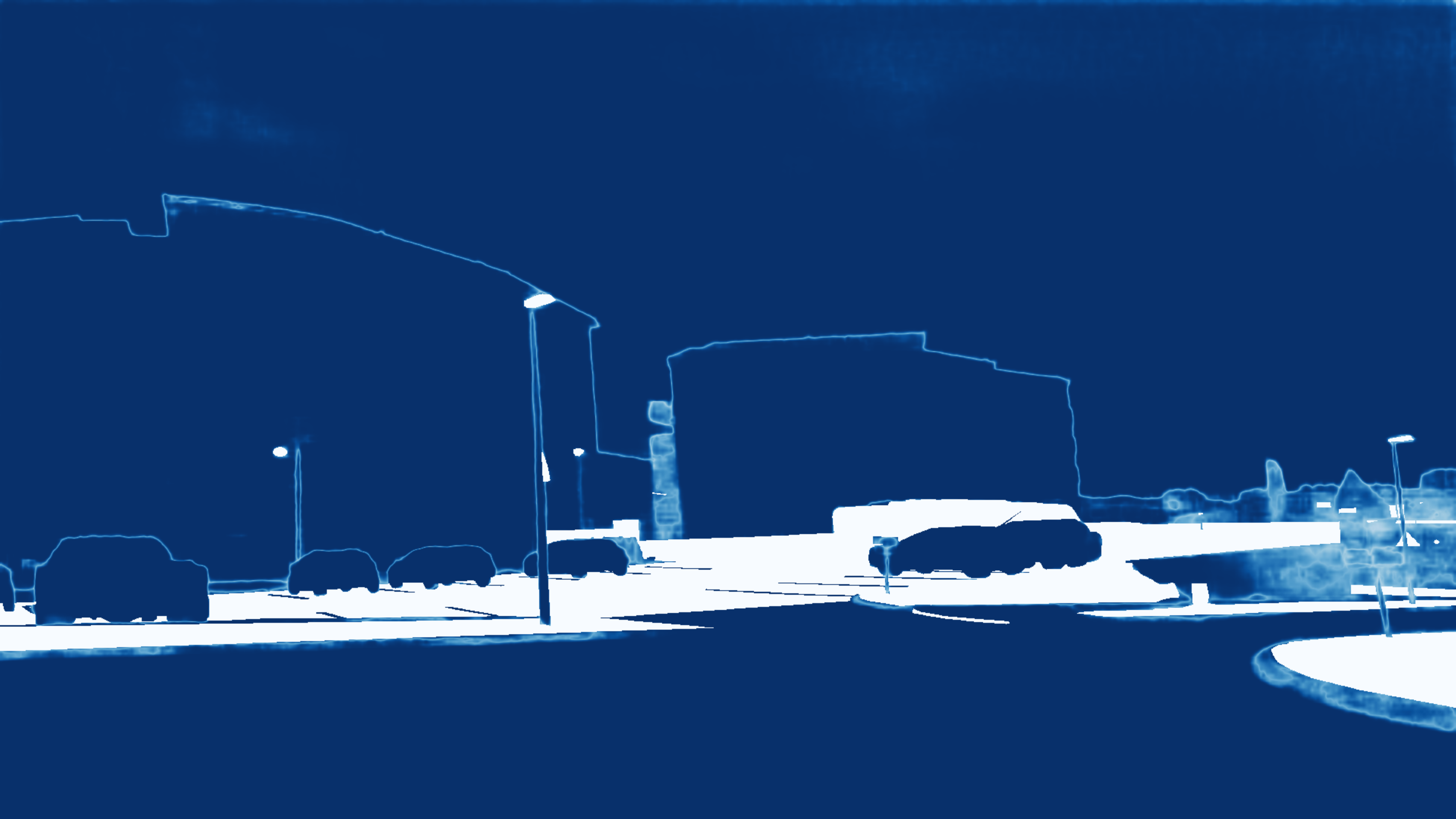}
  \label{fig:map2_1}
  \end{subfigure}
  \begin{subfigure}{0.33\linewidth}
    \centering
     \includegraphics[width=\linewidth]{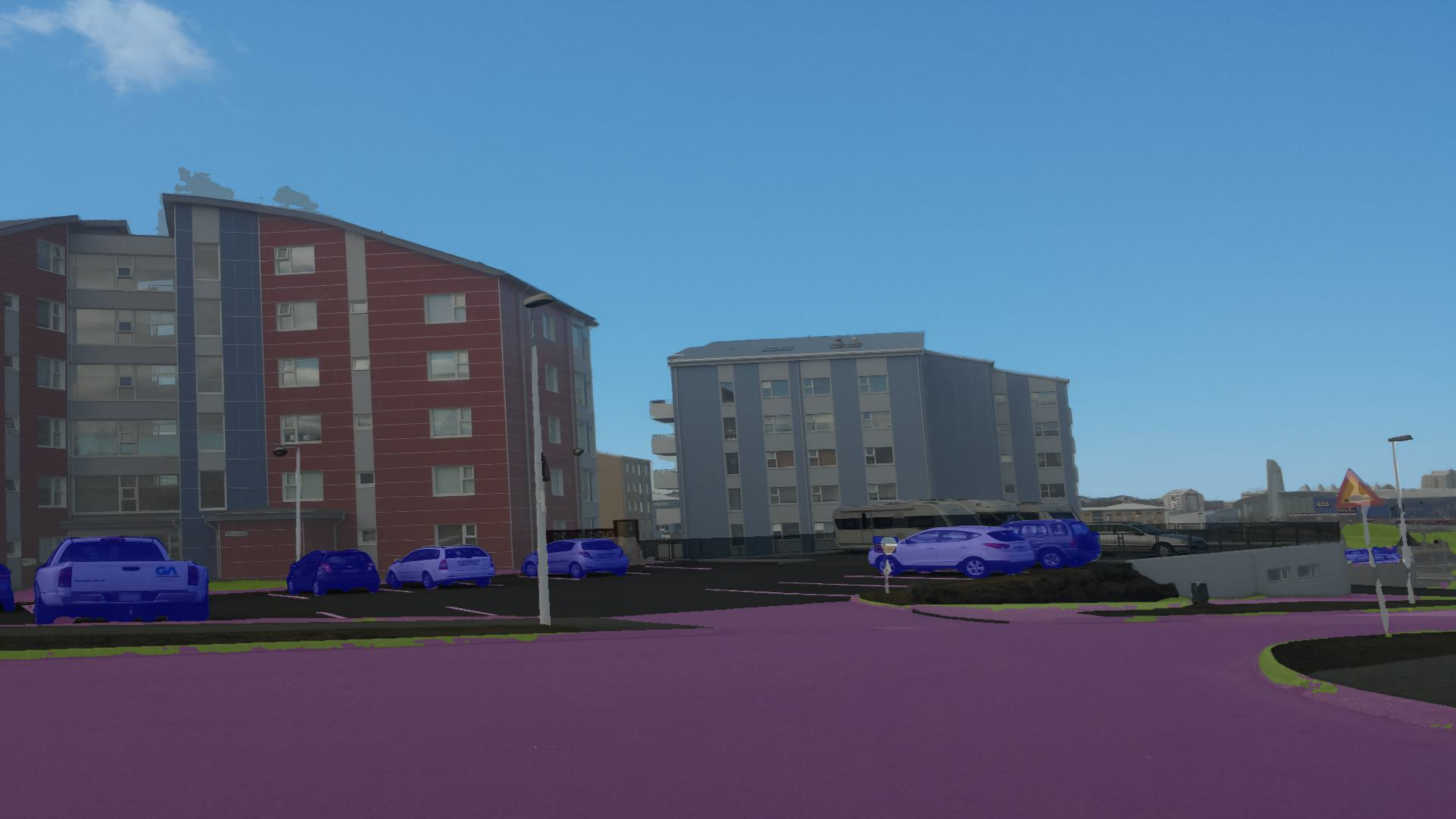}\\[1mm]
    \includegraphics[width=\linewidth]{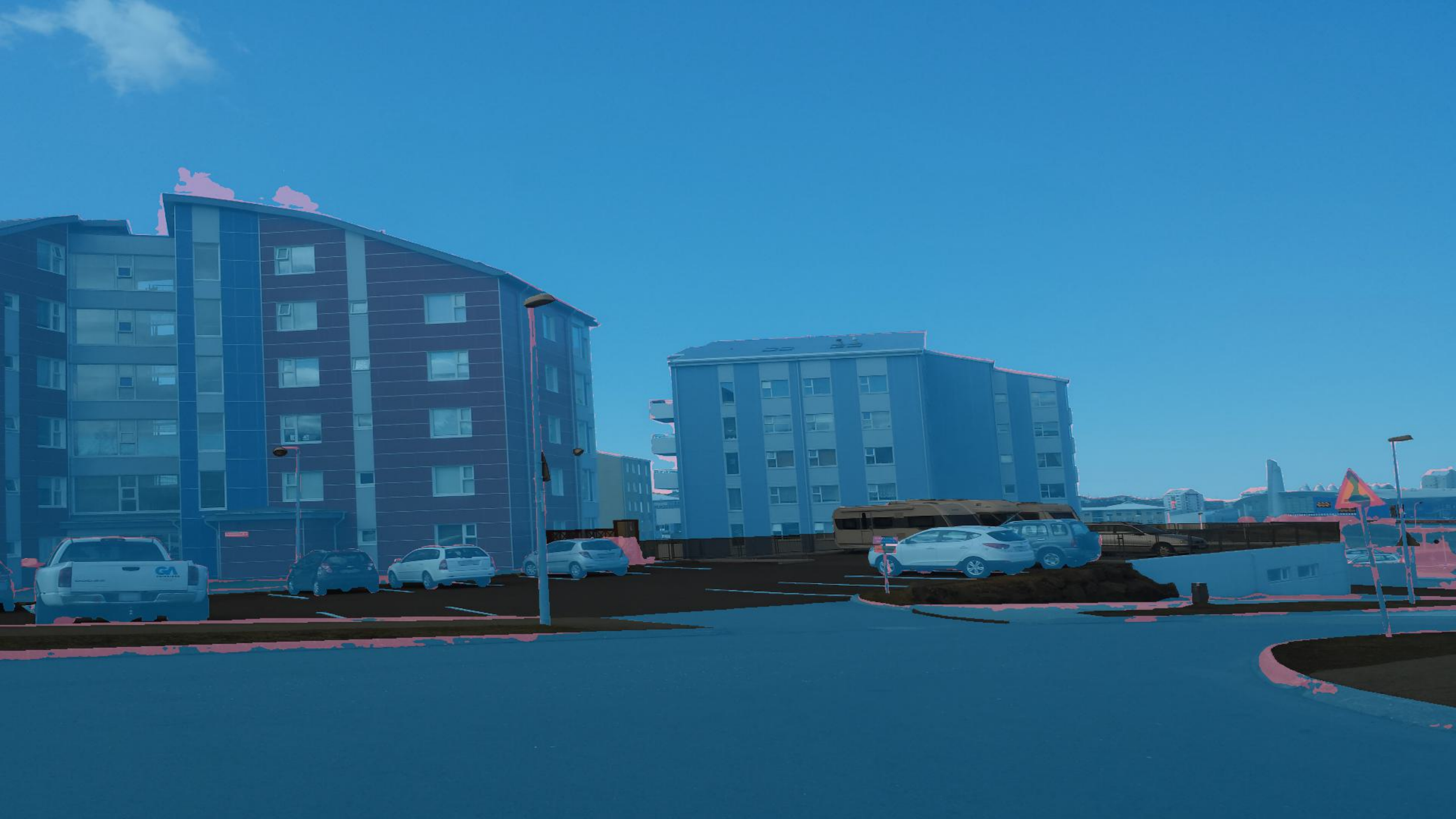}\\[1mm]
    \includegraphics[width=\linewidth]{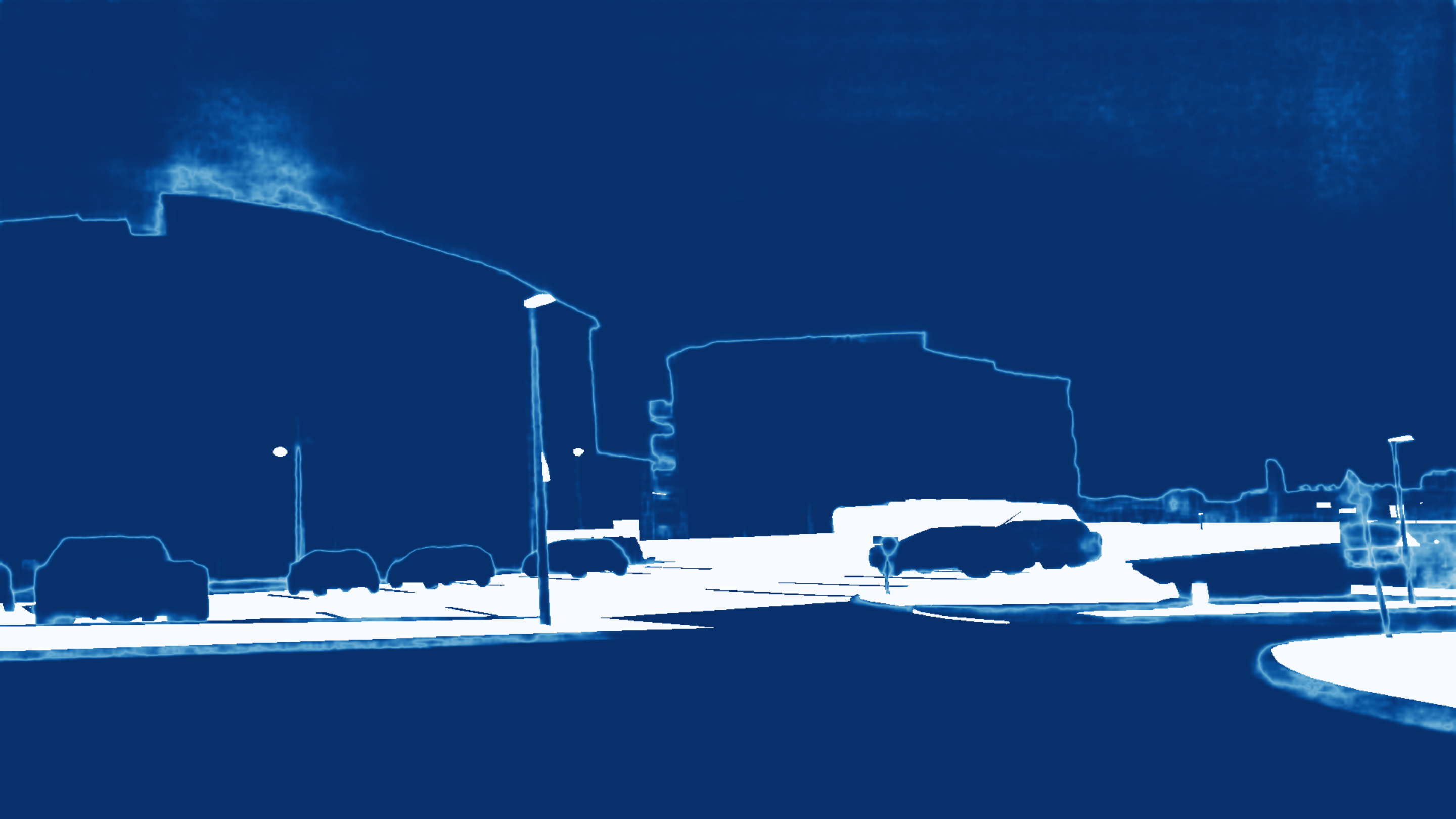}
    \label{fig:map2_2}
  \end{subfigure}
  \begin{subfigure}{0.33\linewidth}
    \centering
     \includegraphics[width=\linewidth]{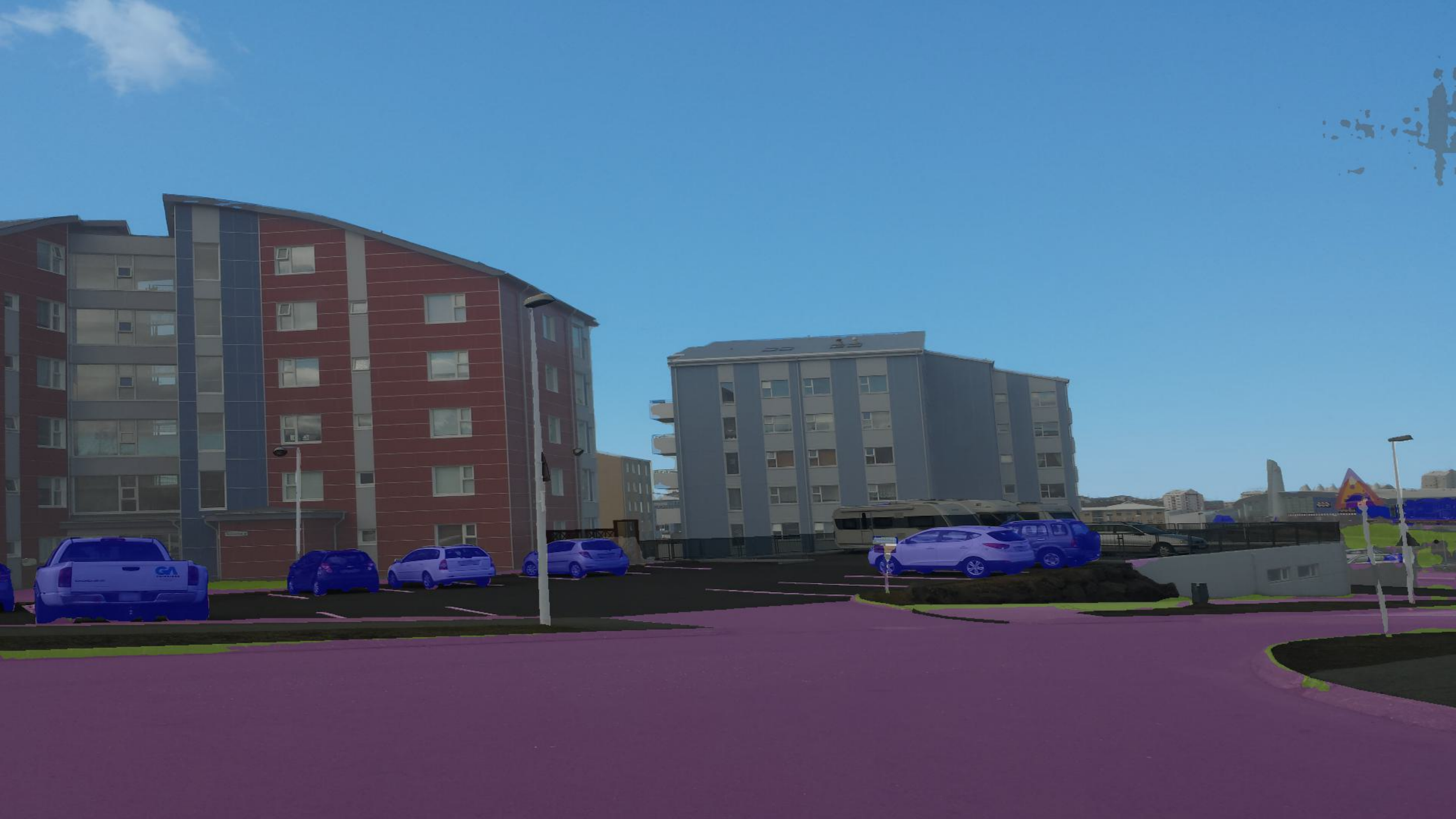}\\[1mm]
    \includegraphics[width=\linewidth]{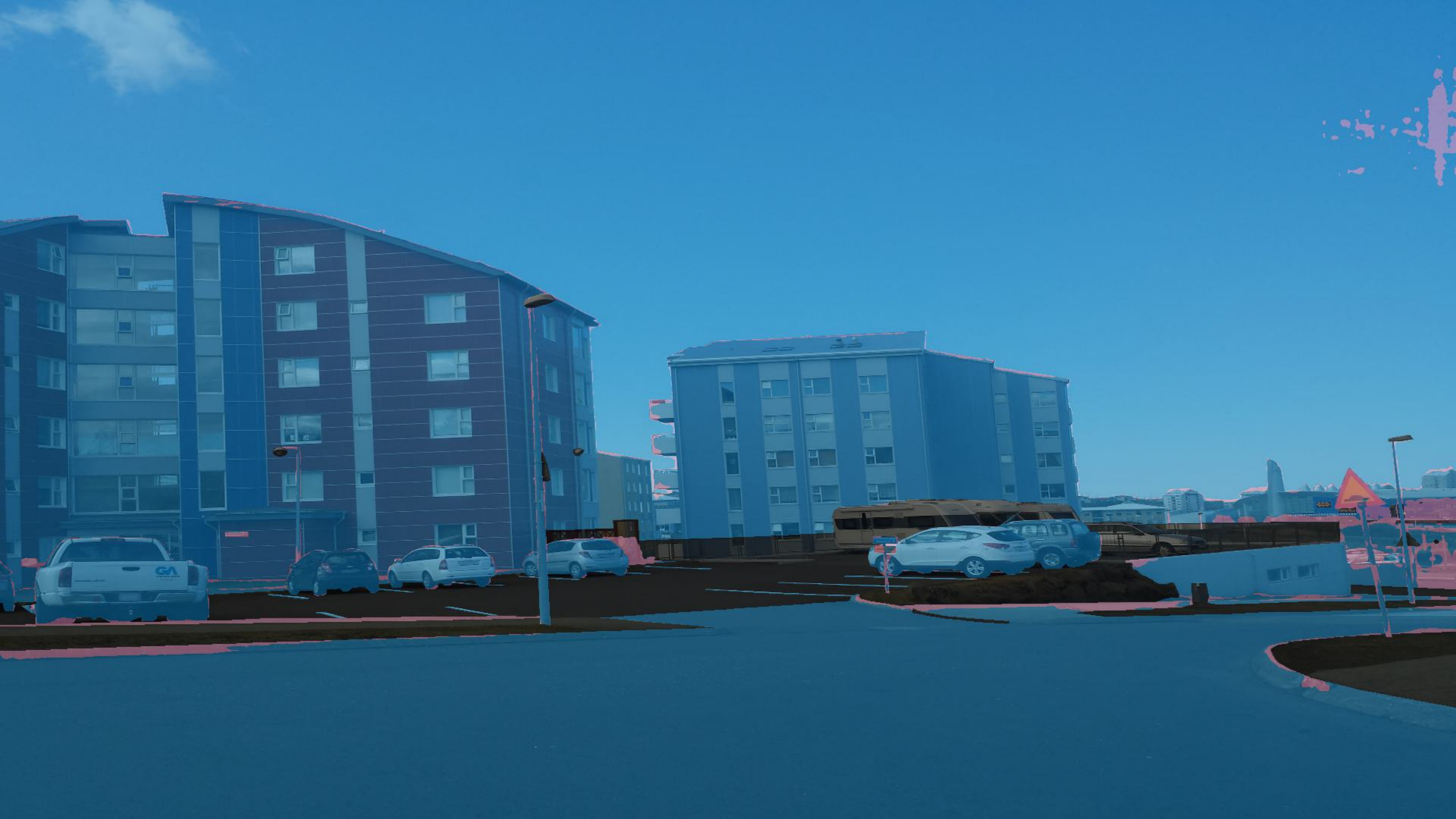}\\[1mm]
    \includegraphics[width=\linewidth]{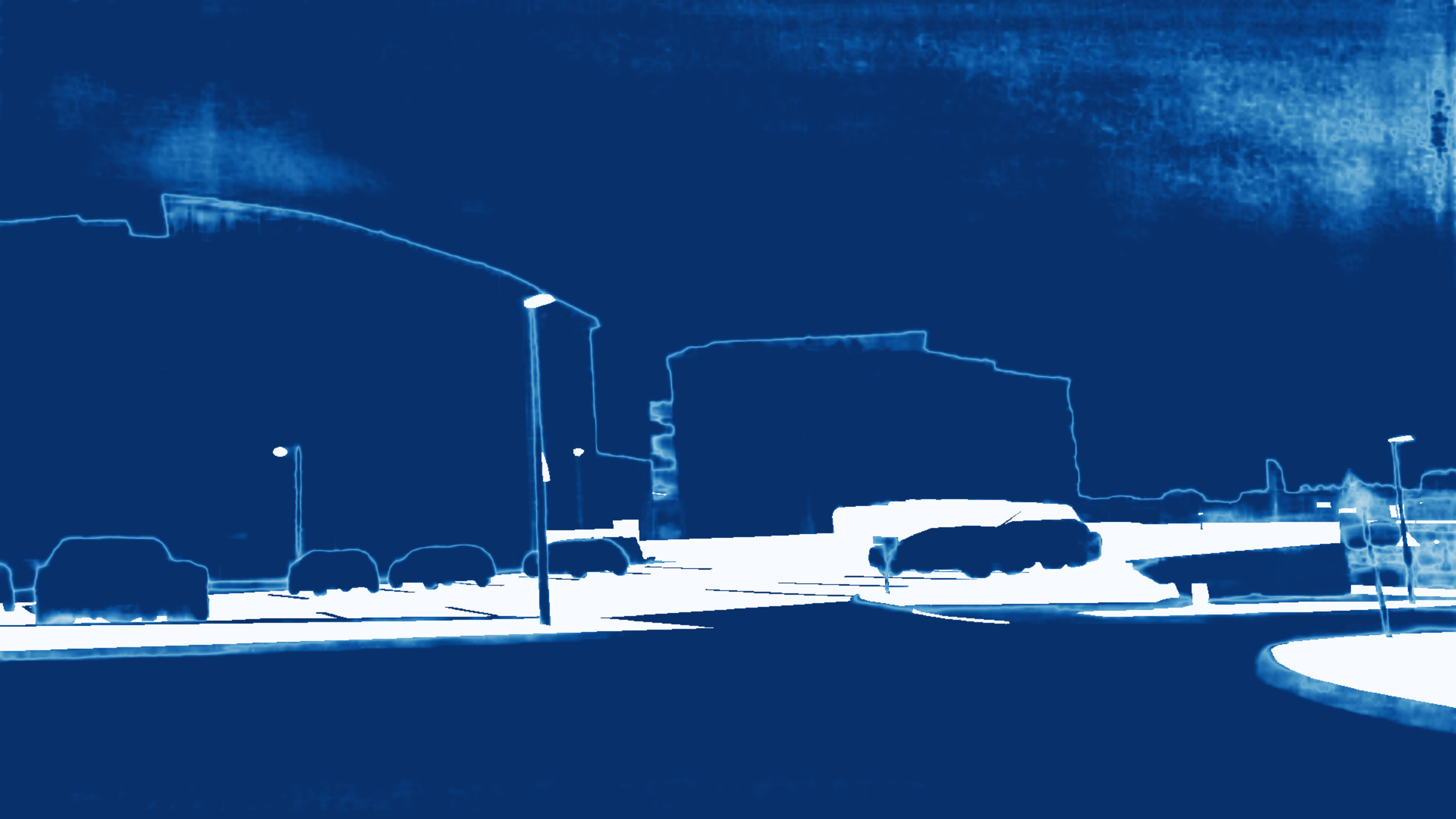}
    \label{fig:map2_3}
  \end{subfigure}
  \caption{\textbf{Two qualitative examples from Mappilary.} See \cref{sec:qualitative} for analysis.}
  \label{fig:qaulimap2}
  \vspace{-0.7em}
\end{figure*}

\begin{figure*}[t]
\centering
  \begin{subfigure}{0.33\linewidth}
    \centering
    HSSN\\\vspace{0.3em}
    \includegraphics[width=\linewidth]{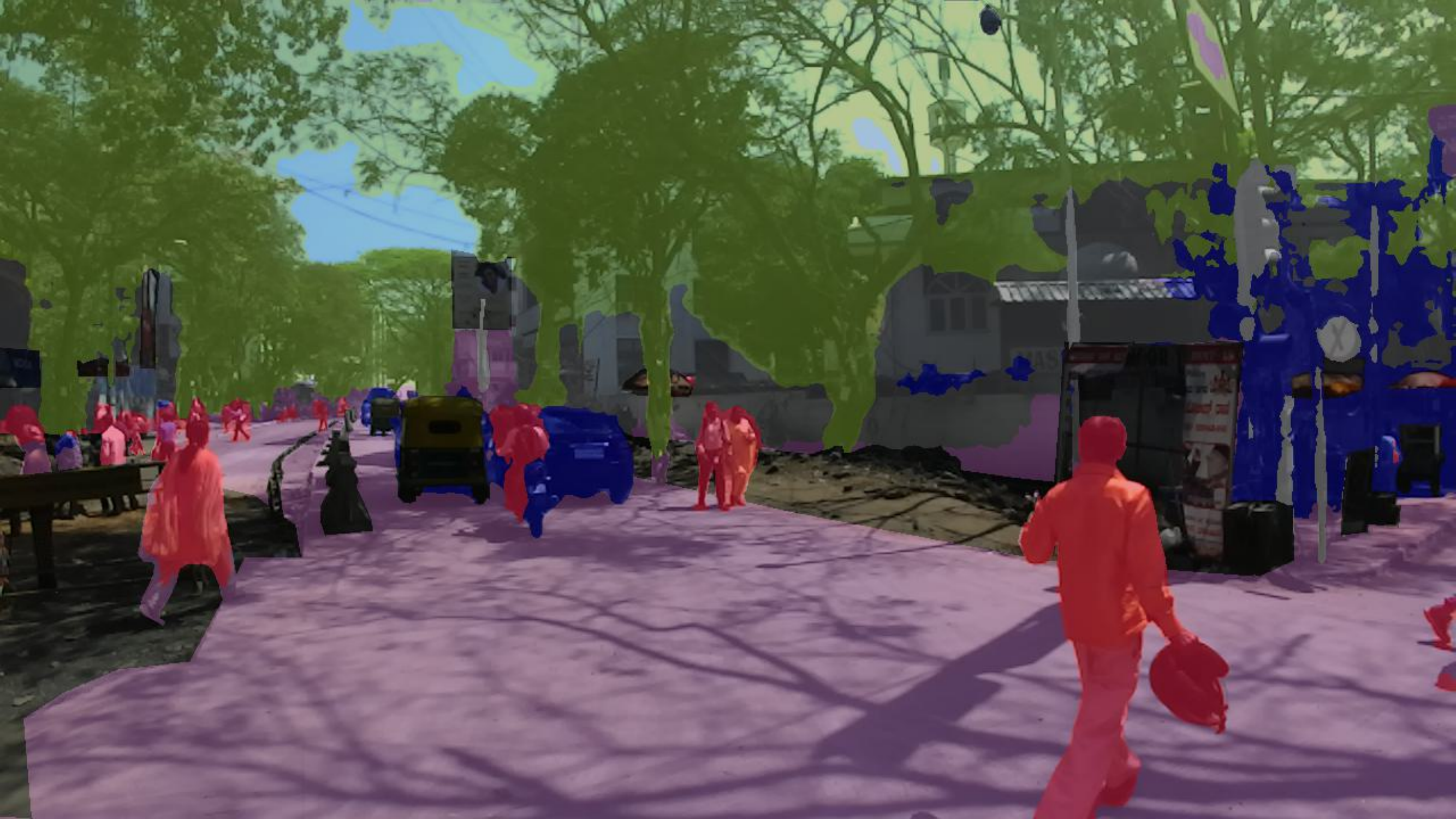}\\[1mm]
    \includegraphics[width=\linewidth]{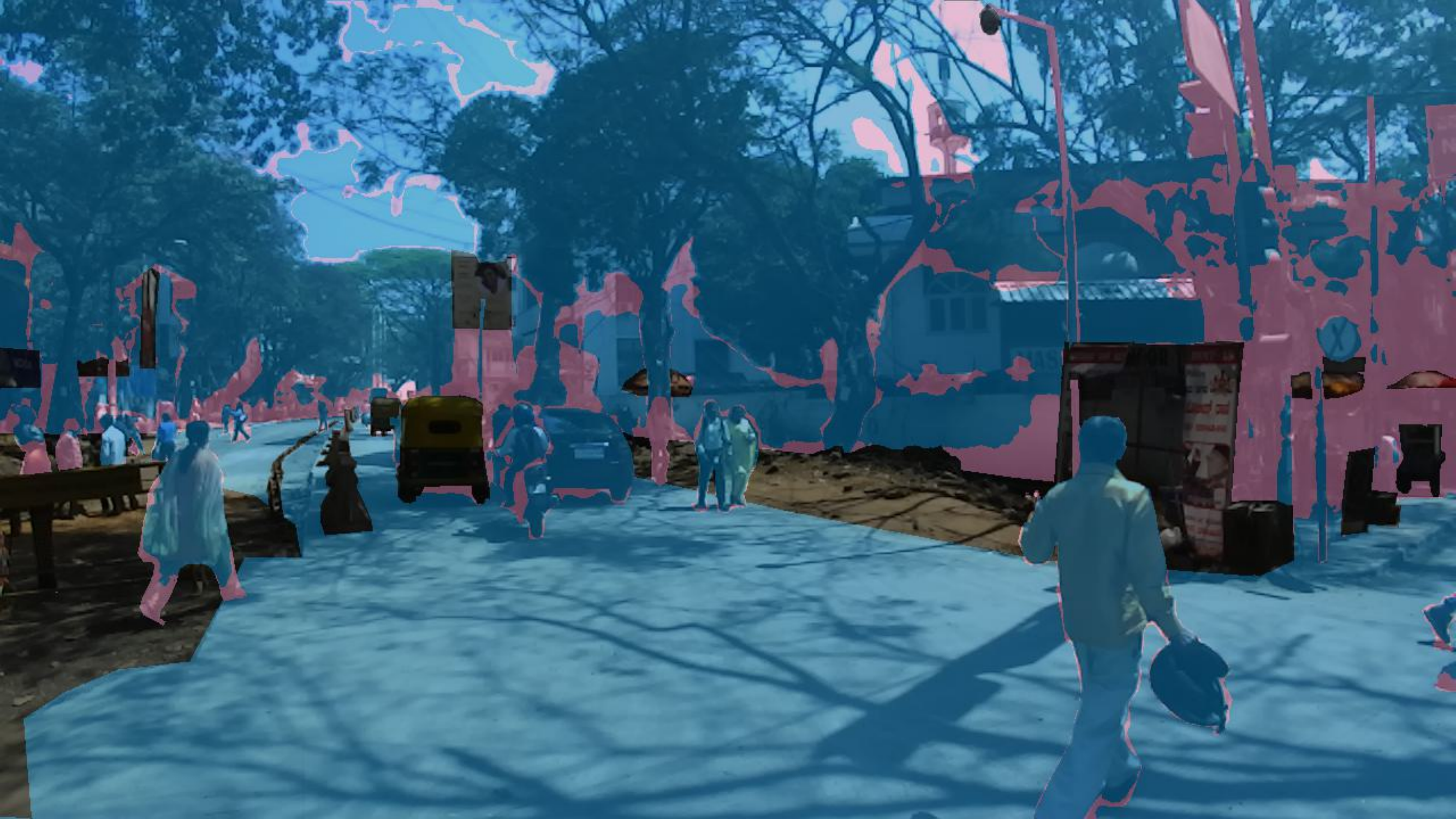}\\[1mm]
    \includegraphics[width=\linewidth]{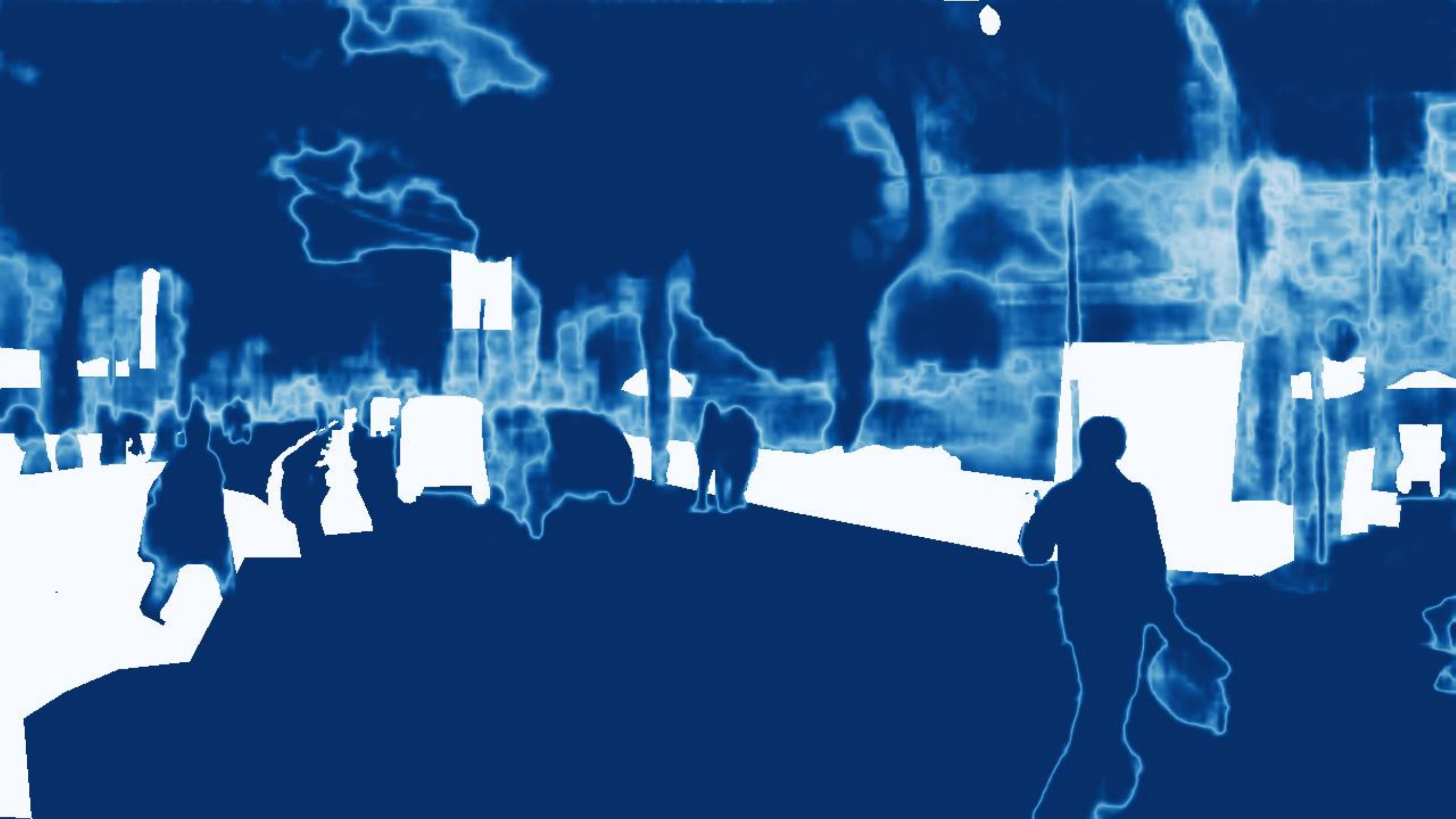}
  \label{fig:idd1_1}
  \end{subfigure}
  \begin{subfigure}{0.33\linewidth}
    \centering
    \eucname\\\vspace{0.3em}
    \includegraphics[width=\linewidth]{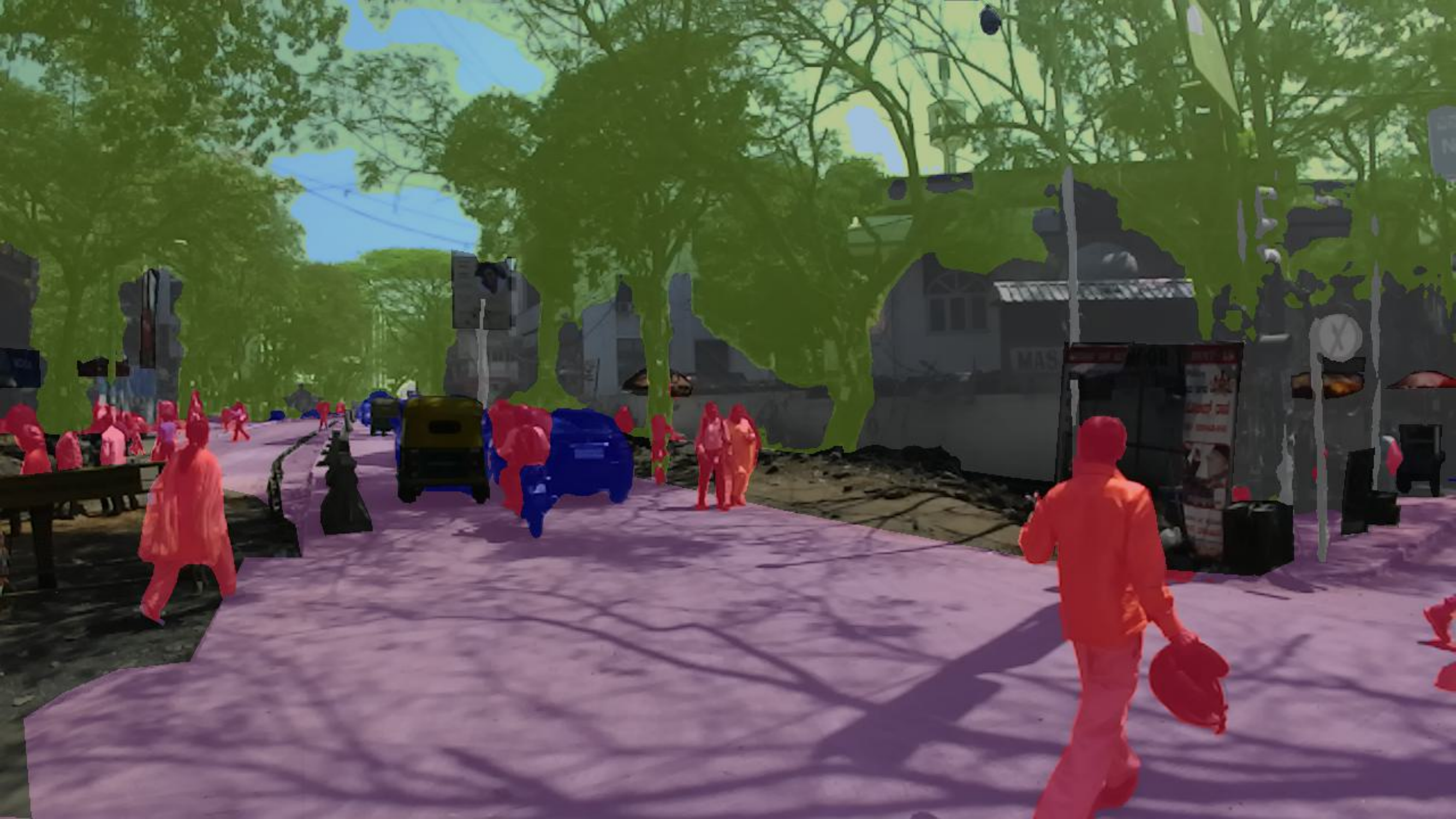}\\[1mm]
    \includegraphics[width=\linewidth]{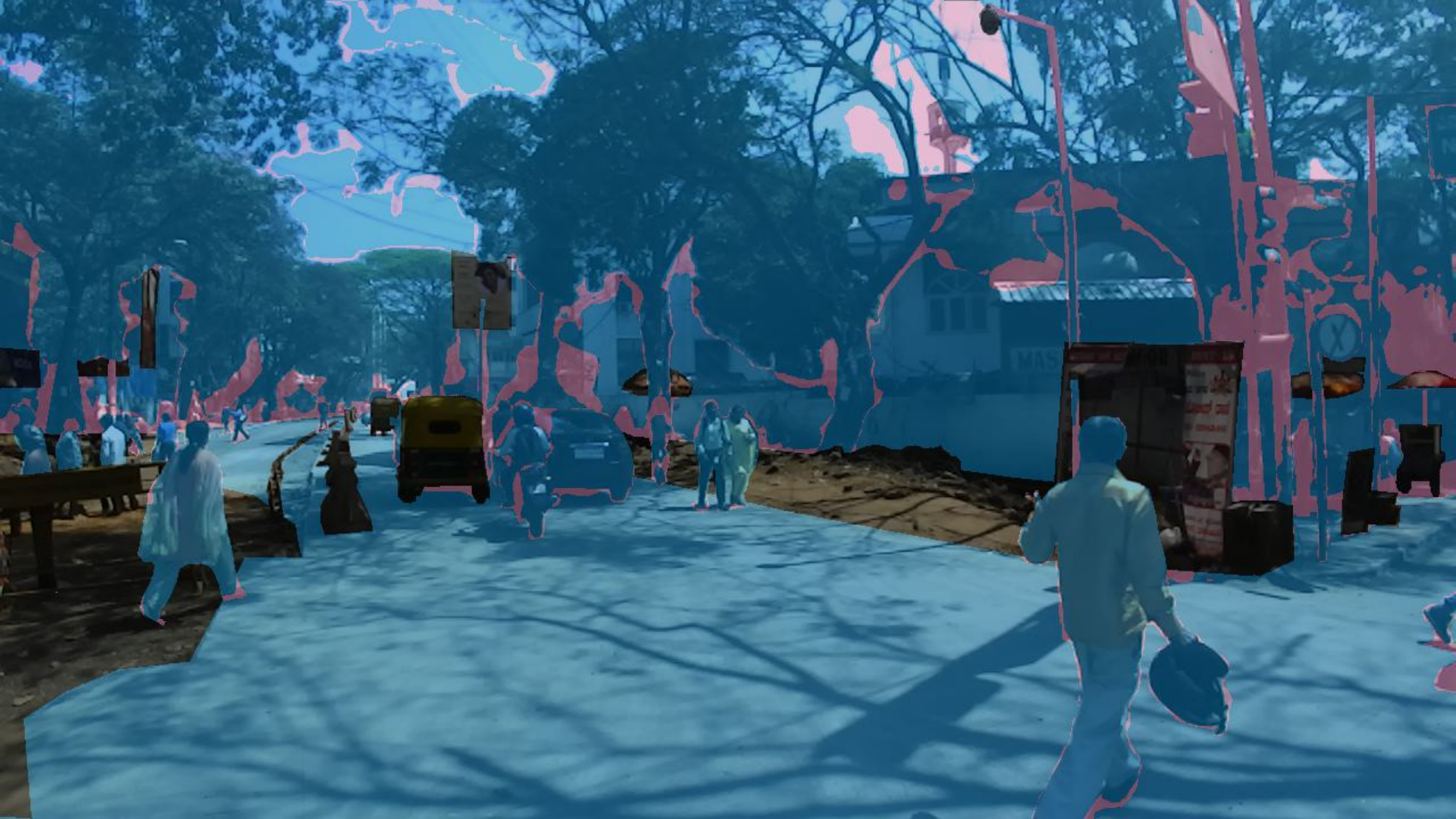}\\[1mm]
    \includegraphics[width=\linewidth]{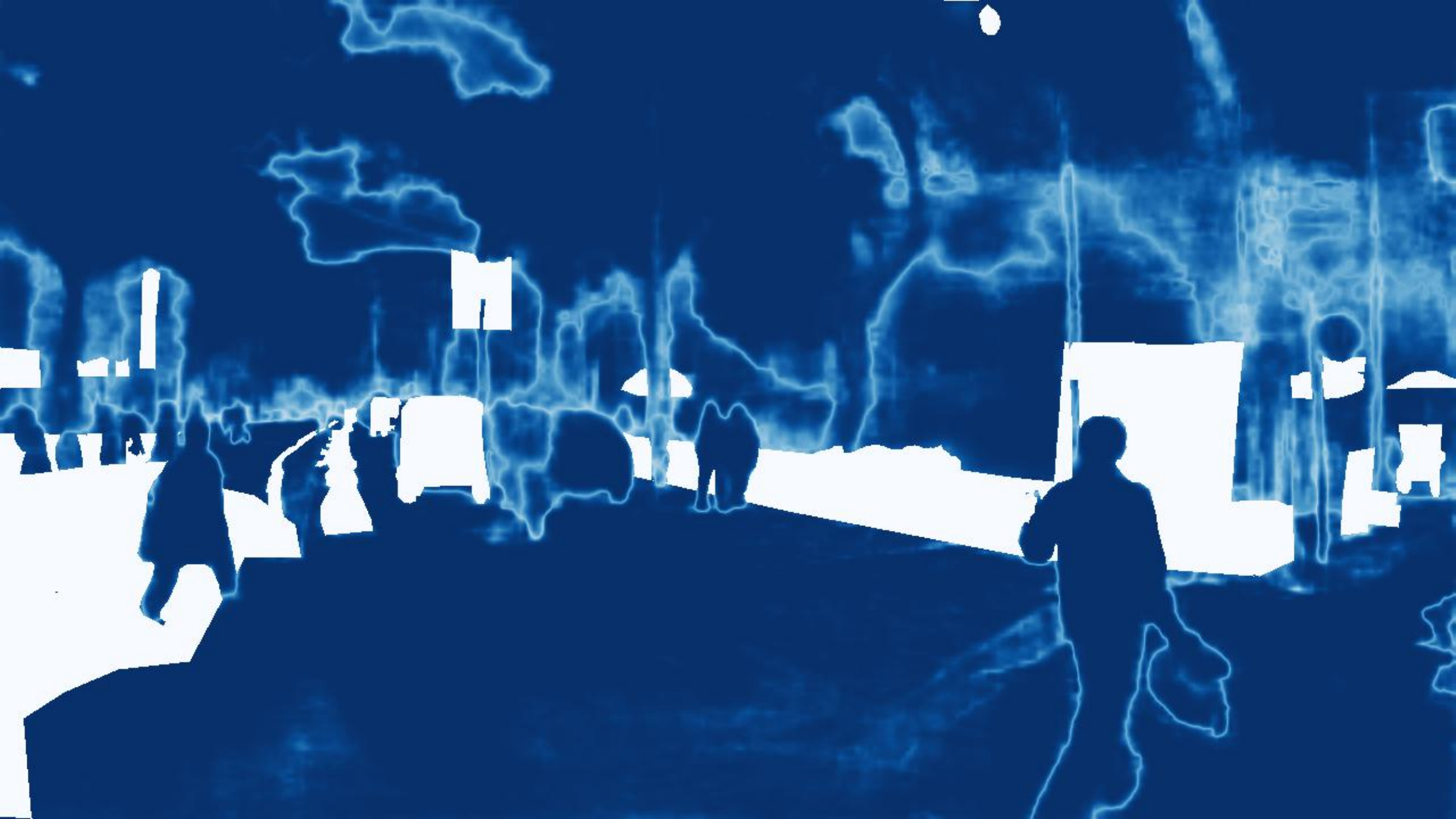}
    \label{fig:idd1_2}
  \end{subfigure}
  \begin{subfigure}{0.33\linewidth}
    \centering
    \hypname\\\vspace{0.2em}
    \includegraphics[width=\linewidth]{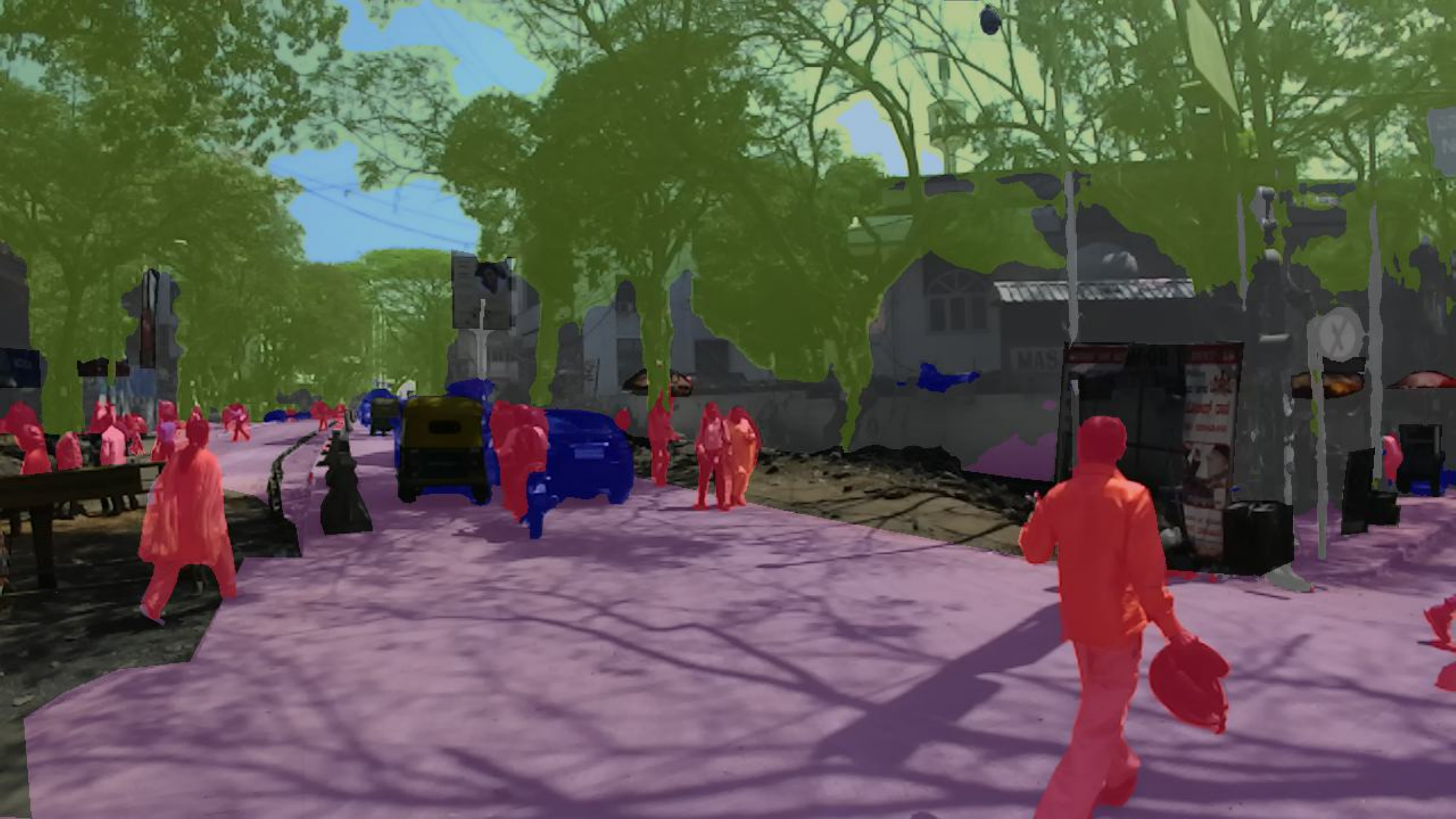}\\[1mm]
    \includegraphics[width=\linewidth]{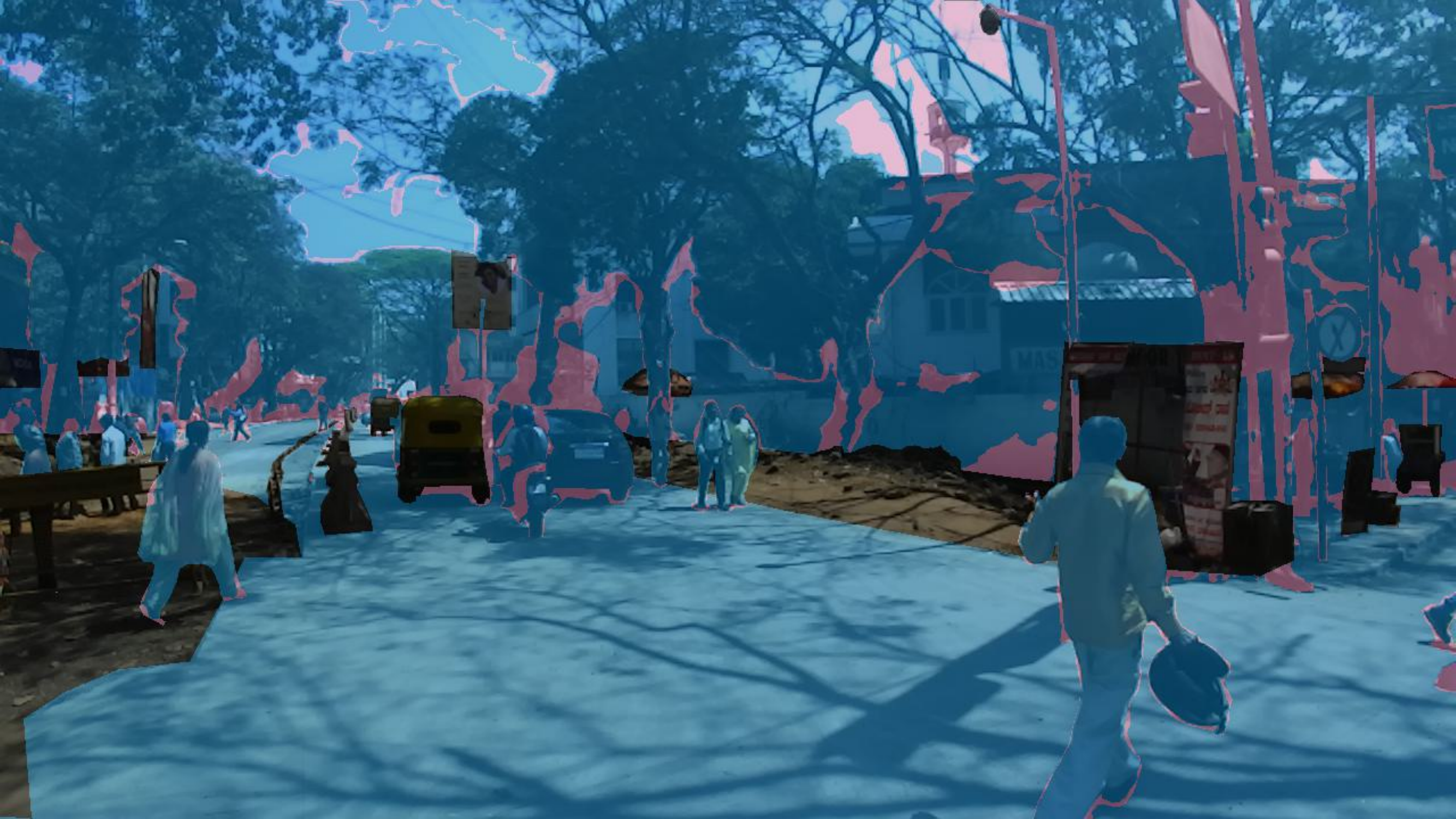}\\[1mm]
    \includegraphics[width=\linewidth]{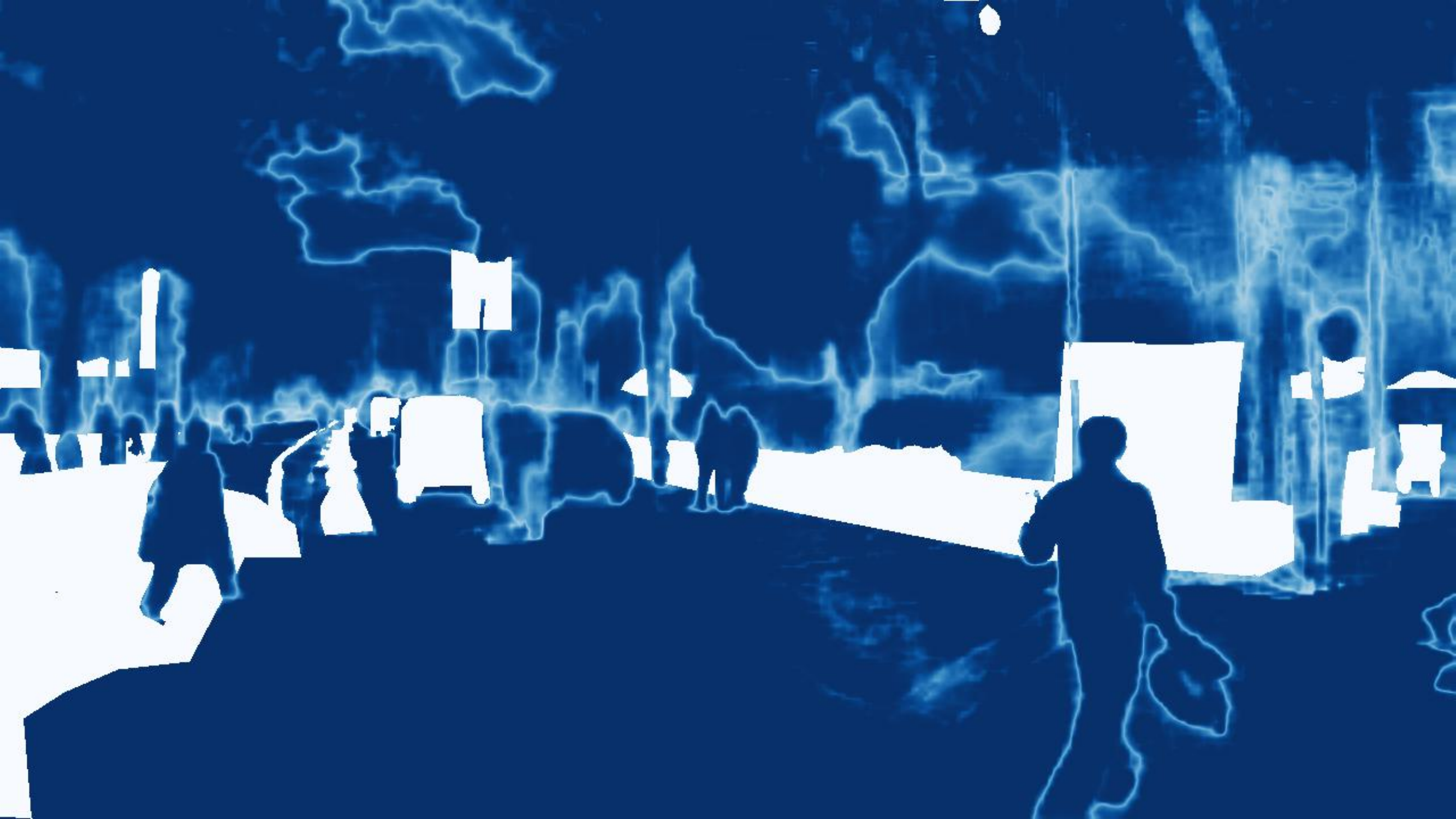}
    \label{fig:idd1_3}
  \end{subfigure}
    \begin{subfigure}{\linewidth}
        \centering
        \fboxsep 2pt
        \begin{minipage}{0.7\linewidth}
            \colorbox{flat}{\strut \color{white}{flat}}
            \colorbox{construction}{\strut \color{white}{construction}}
            \colorbox{object}{\strut  object}
            \colorbox{nature}{\strut \color{white}{nature}}
            \colorbox{sky}{\strut \color{white}{sky}}
            \colorbox{human}{\strut human}
            \colorbox{vehicle}{\strut \color{white}{vehicle}}\hspace{2em}
            \colorbox{ignore}{\strut \color{white}{ignore}}\hspace{2em}
            \colorbox{true}{\strut \color{white}{true}}
            \colorbox{false}{\strut false}\hspace{2em}%
        \end{minipage}%
        \begin{minipage}{0.3\linewidth}
        \begin{tikzpicture}
        \node [rectangle, left color=left!10!white, right color=left, anchor=north, minimum width=\linewidth, minimum height=0.5cm] (box) at (current page.north){0 \hspace{12em}  \color{white}{1}};
        \end{tikzpicture}
        \end{minipage}
    \end{subfigure}\\
    \vspace{1em}
  \begin{subfigure}{0.33\linewidth}
    \centering
    \includegraphics[width=\linewidth]{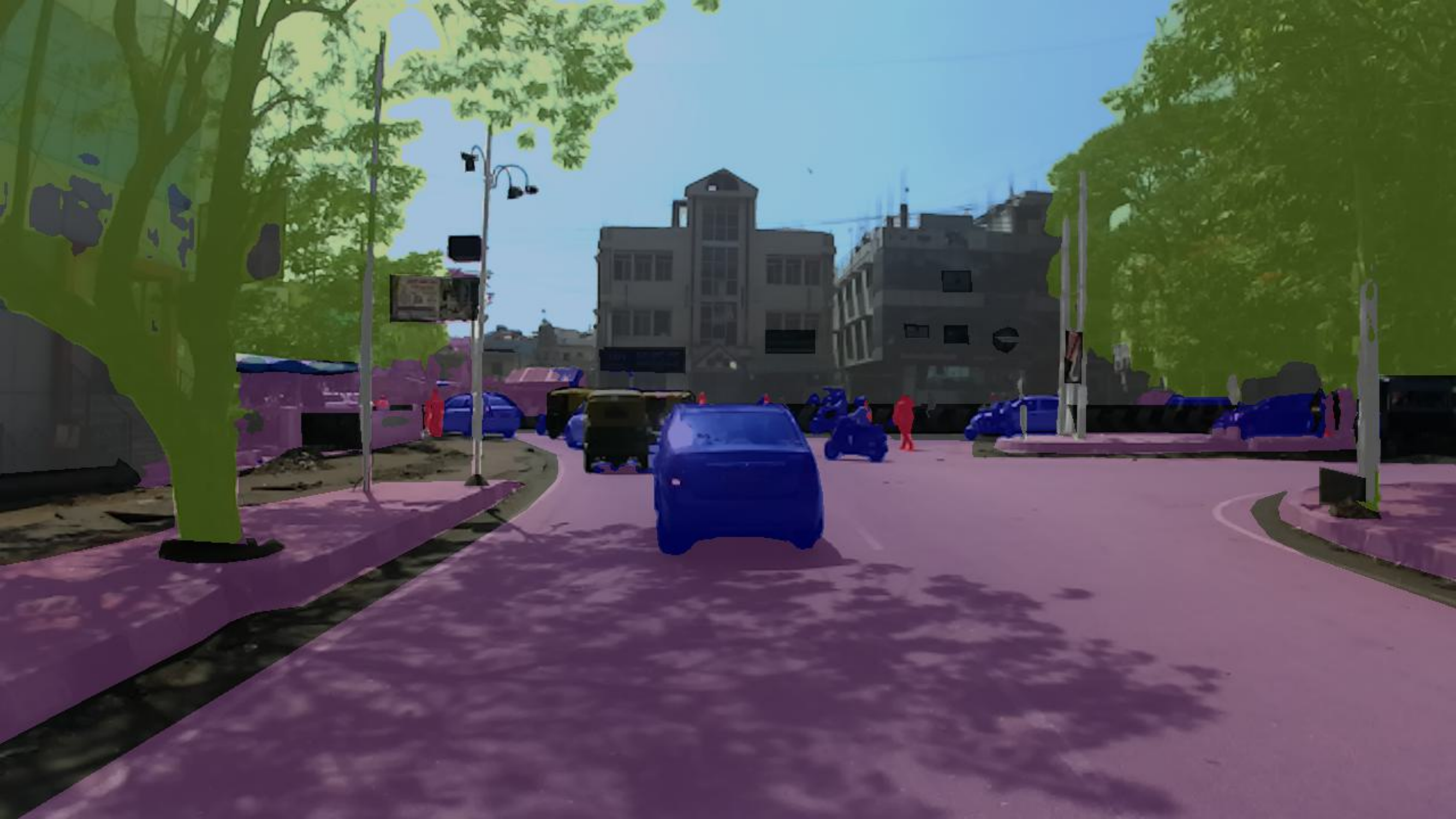}\\[1mm]
    \includegraphics[width=\linewidth]{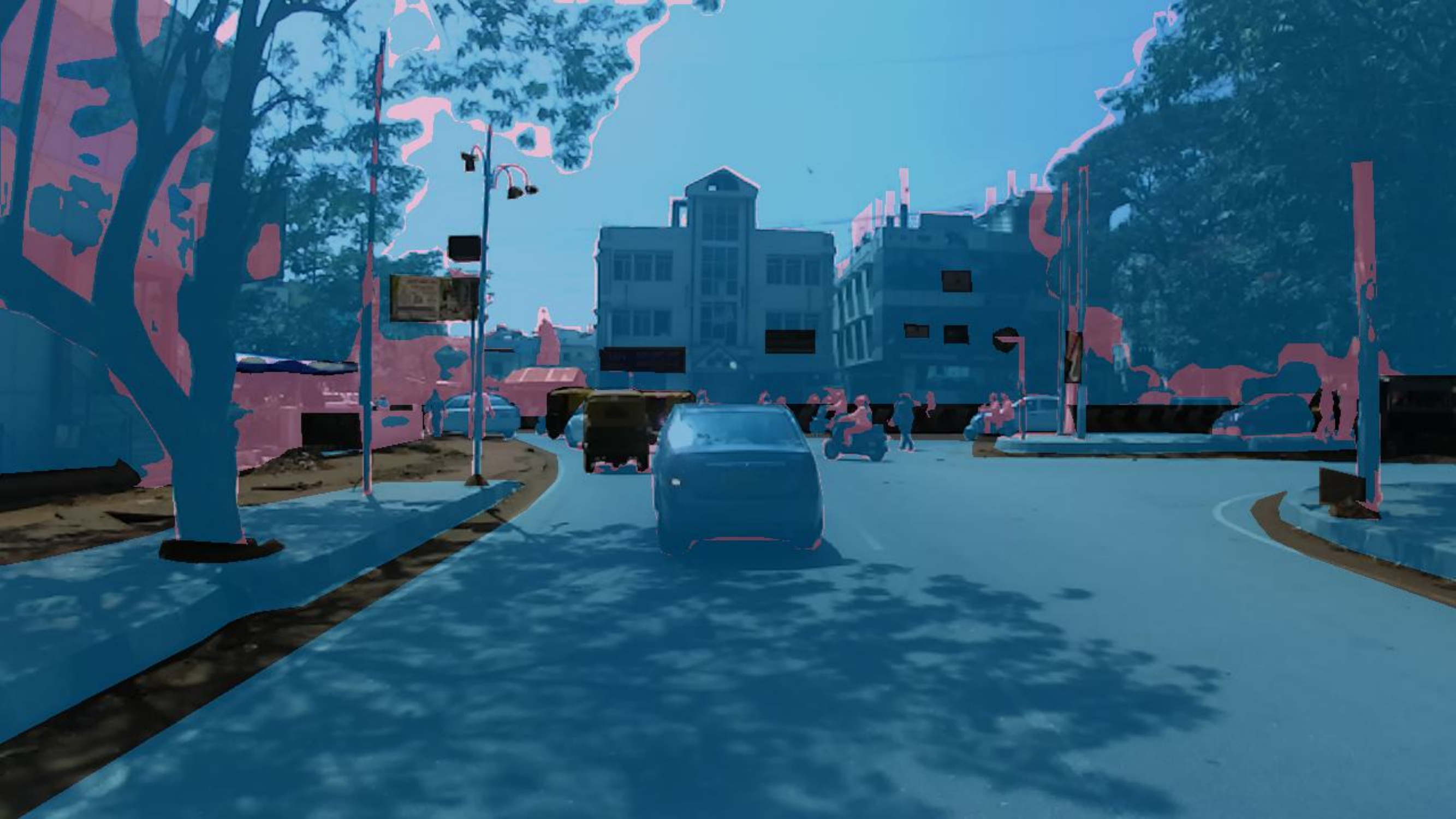}\\[1mm]
    \includegraphics[width=\linewidth]{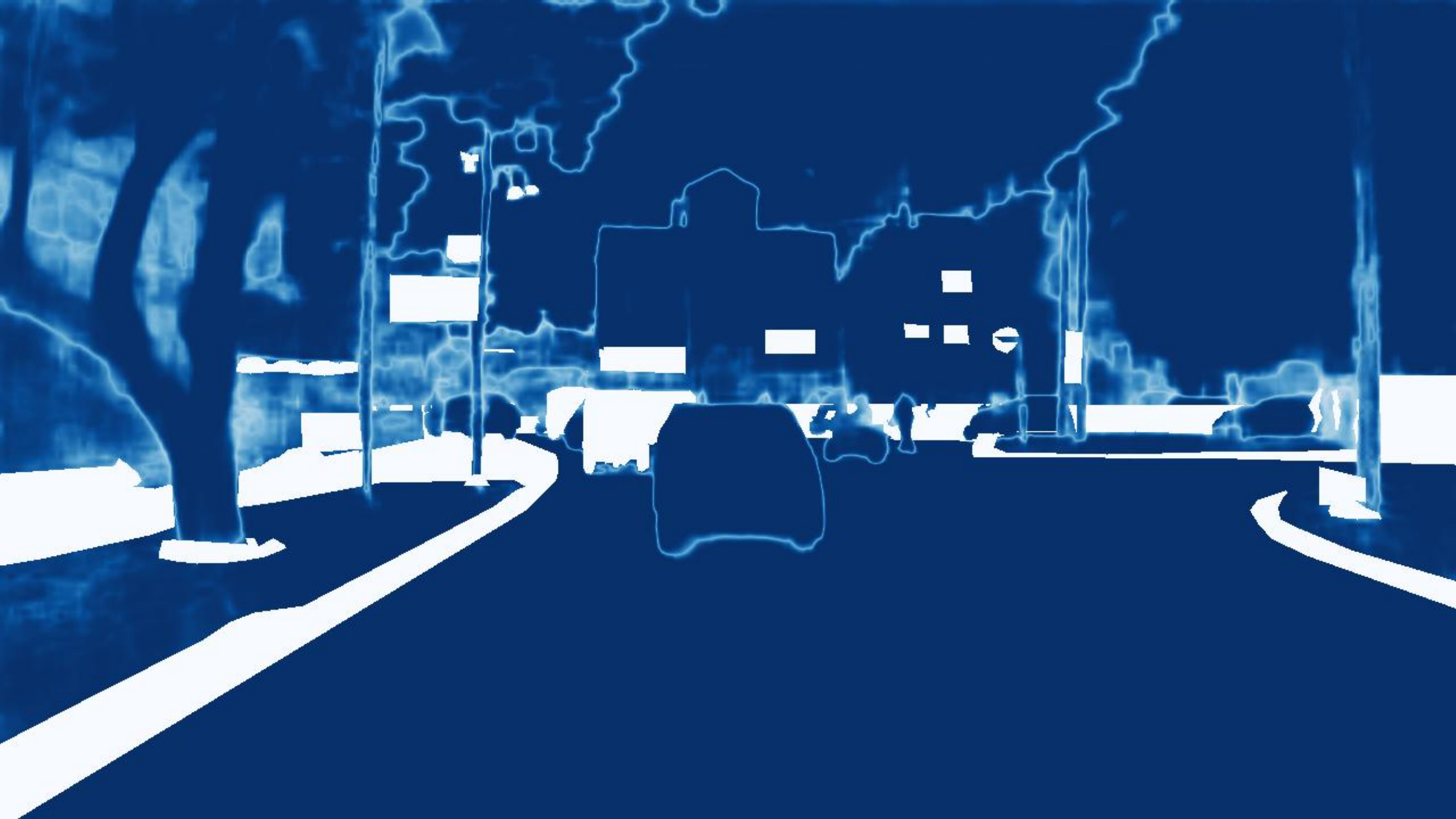}
  \label{fig:idd2_1}
  \end{subfigure}
  \begin{subfigure}{0.33\linewidth}
    \centering
    \includegraphics[width=\linewidth]{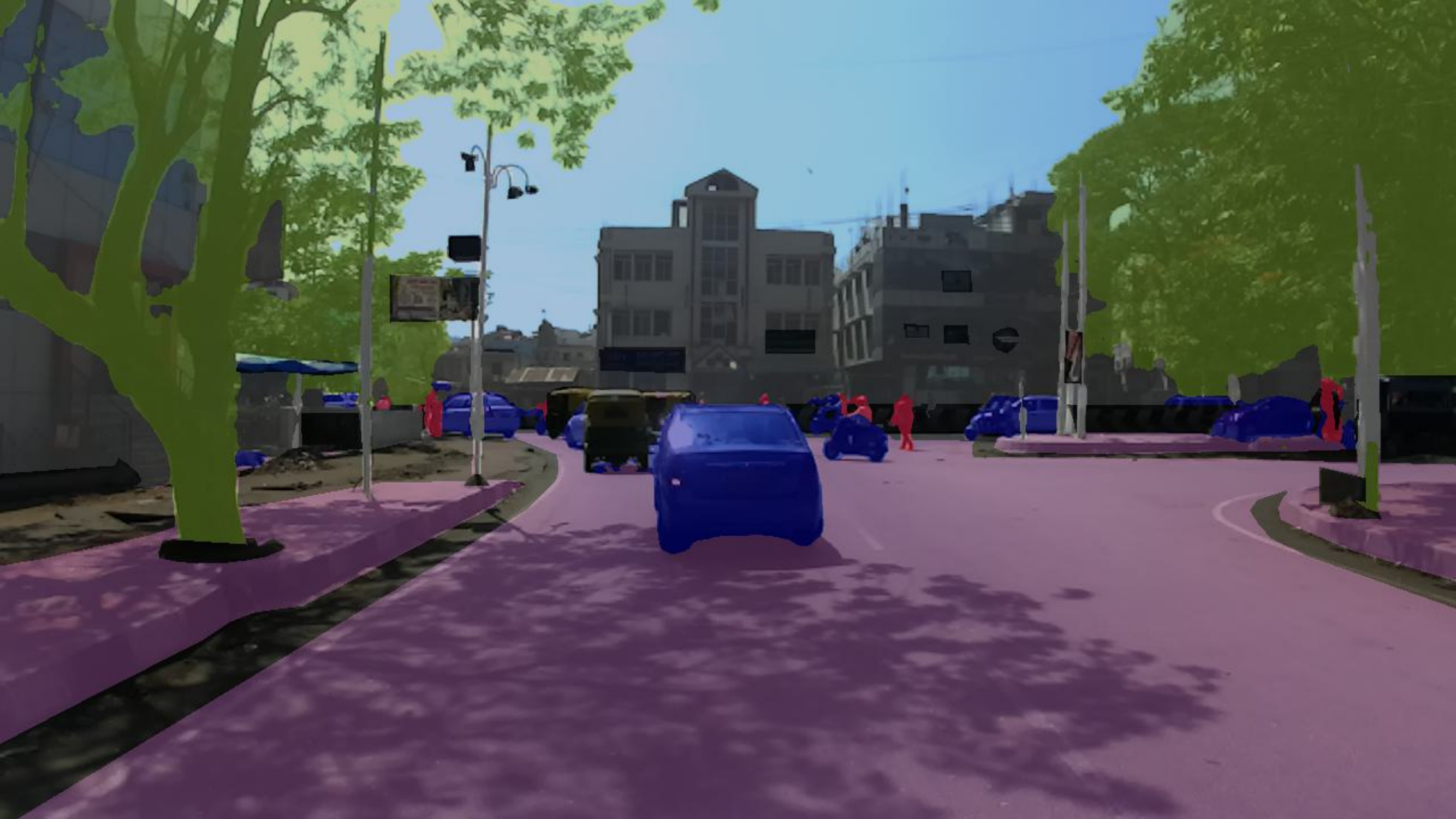}\\[1mm]
    \includegraphics[width=\linewidth]{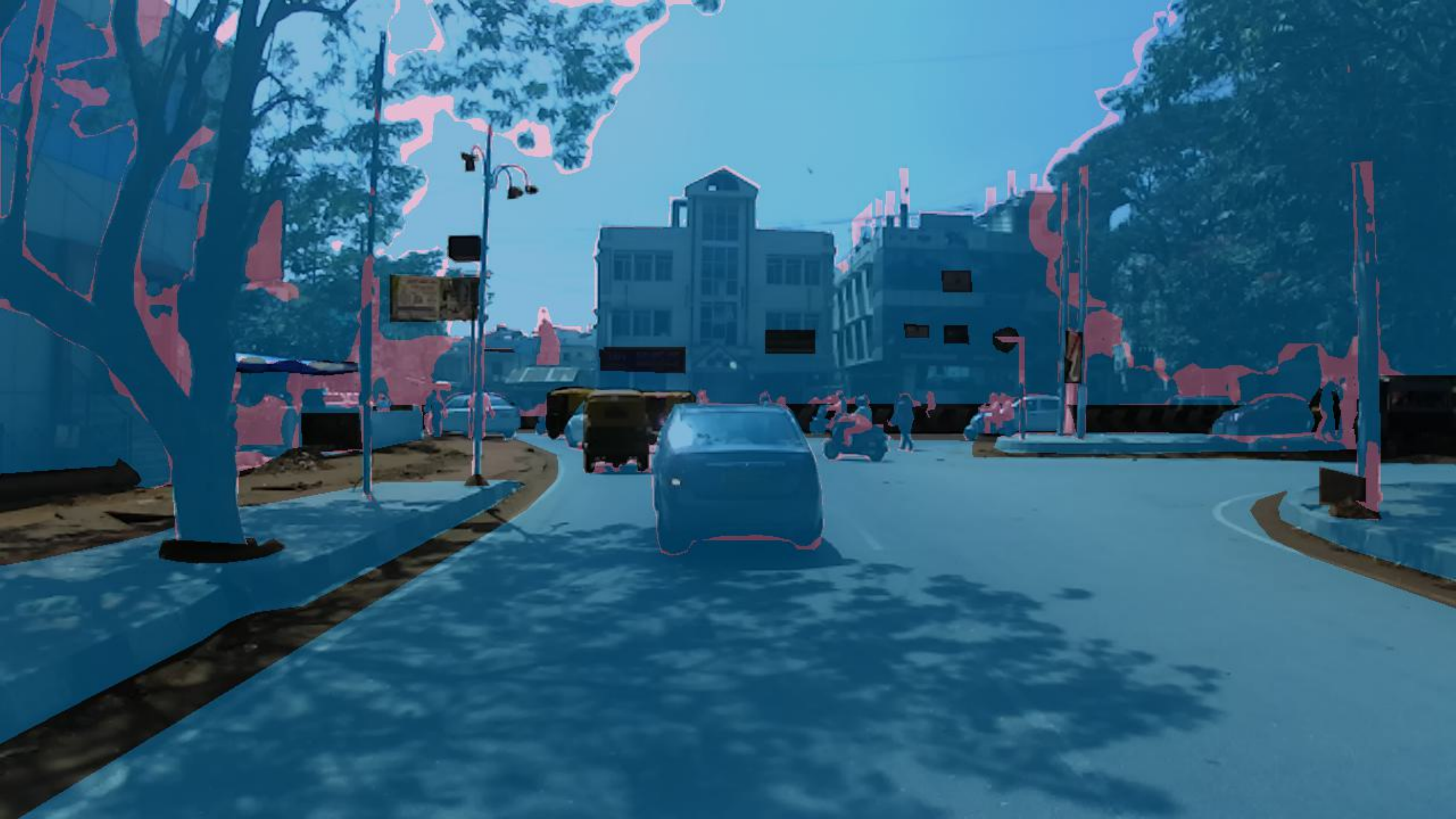}\\[1mm]
    \includegraphics[width=\linewidth]{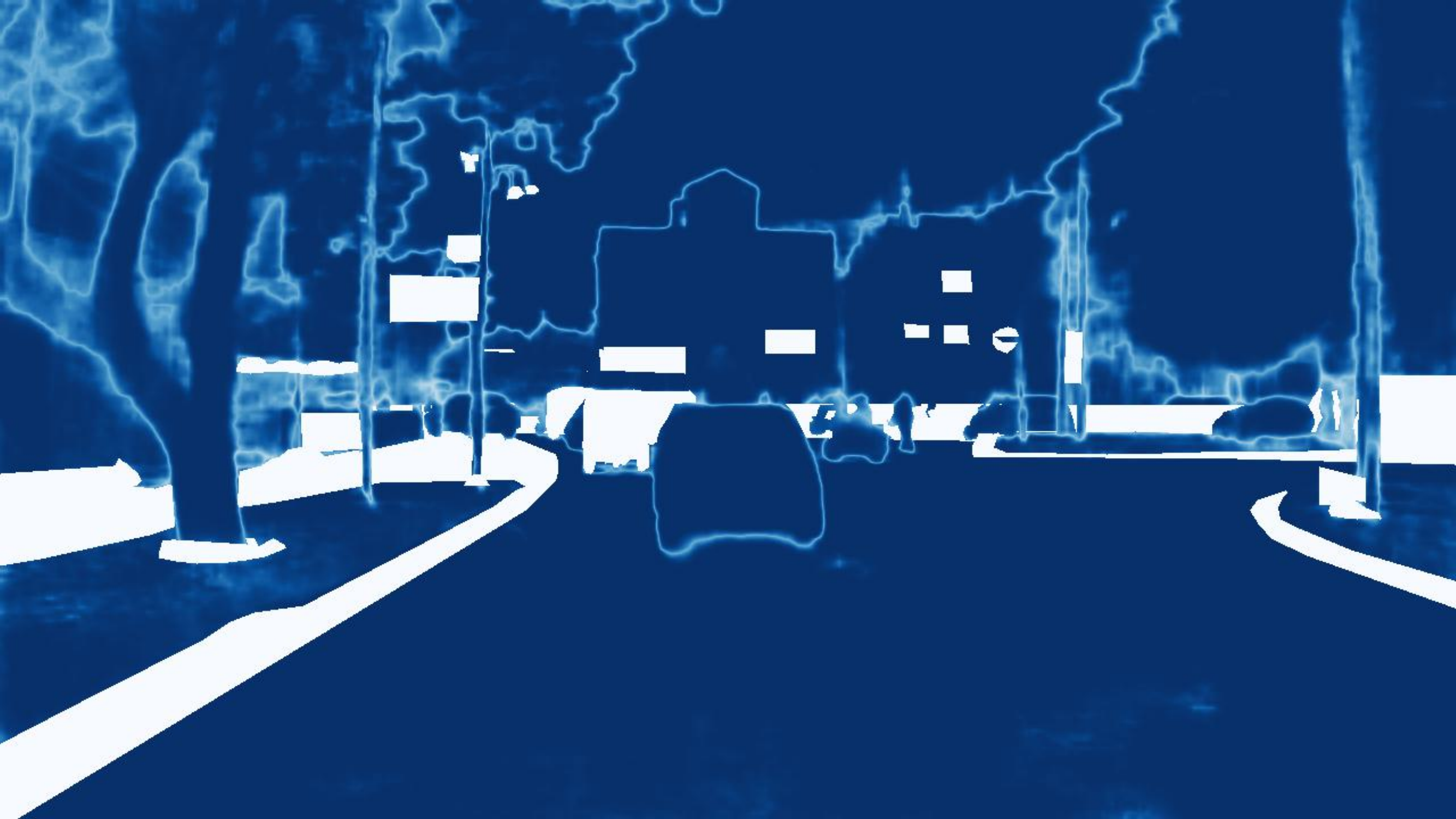}
    \label{fig:idd2_2}
  \end{subfigure}
  \begin{subfigure}{0.33\linewidth}
    \centering
    \includegraphics[width=\linewidth]{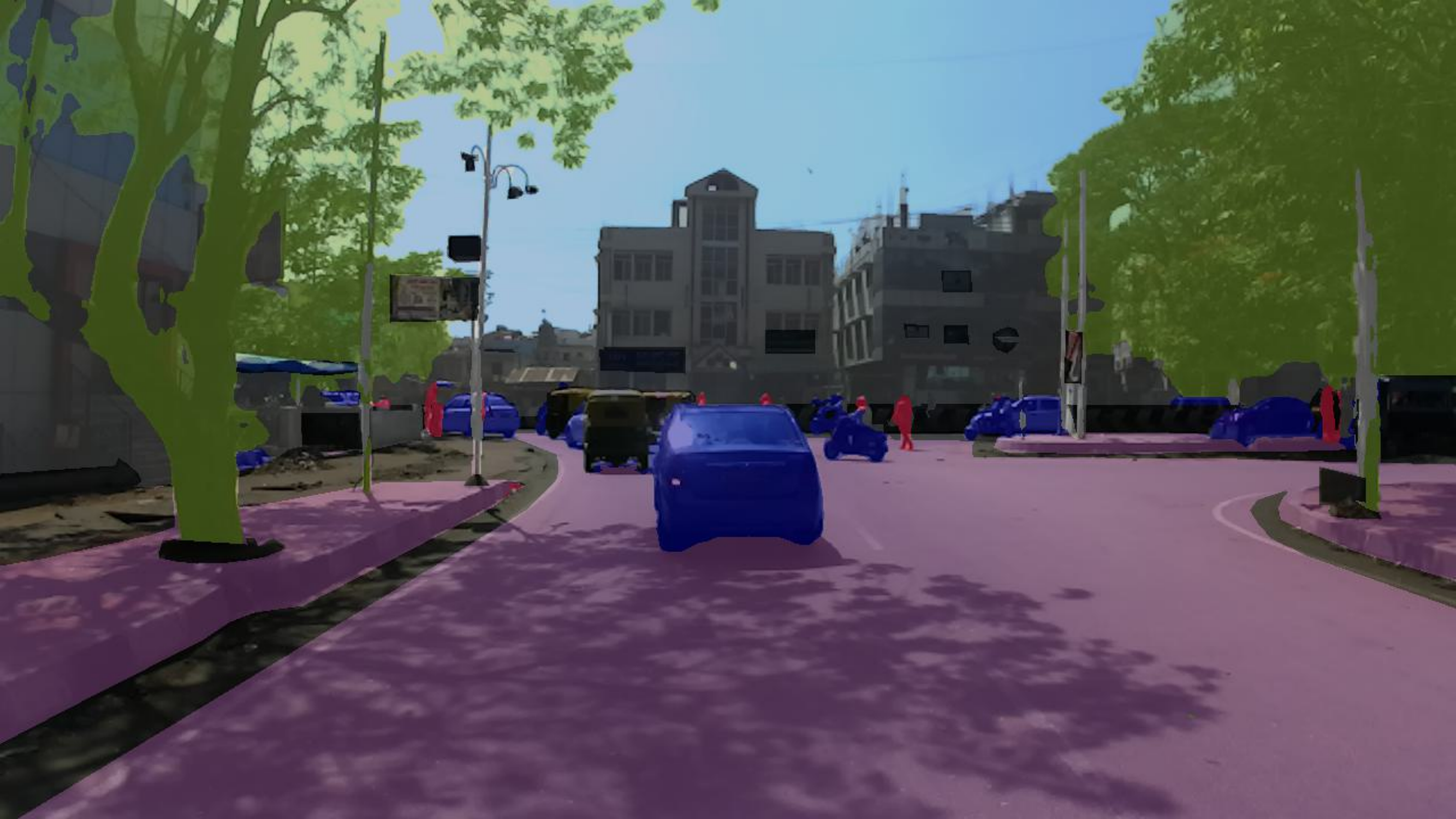}\\[1mm]
    \includegraphics[width=\linewidth]{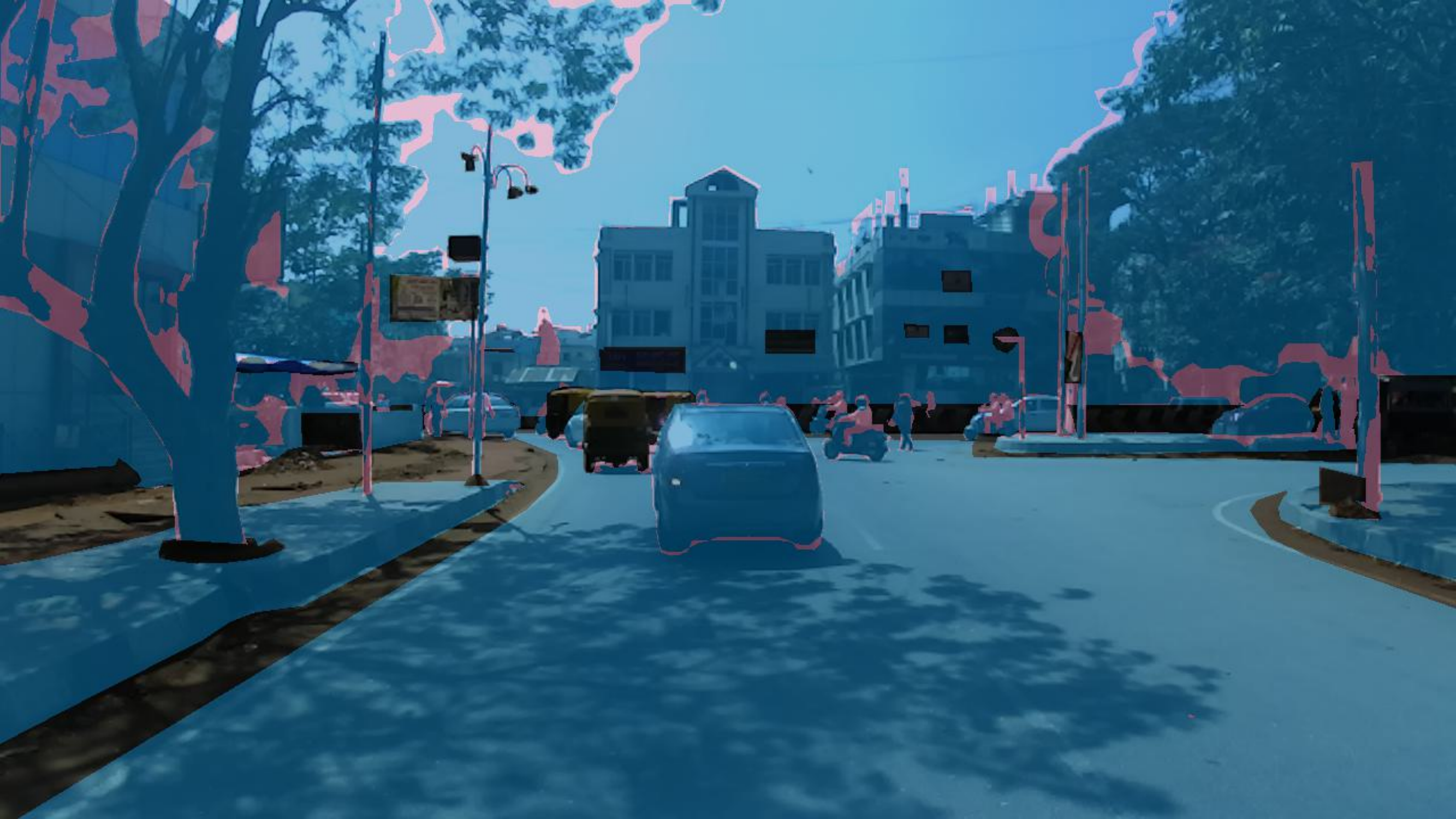}\\[1mm]
    \includegraphics[width=\linewidth]{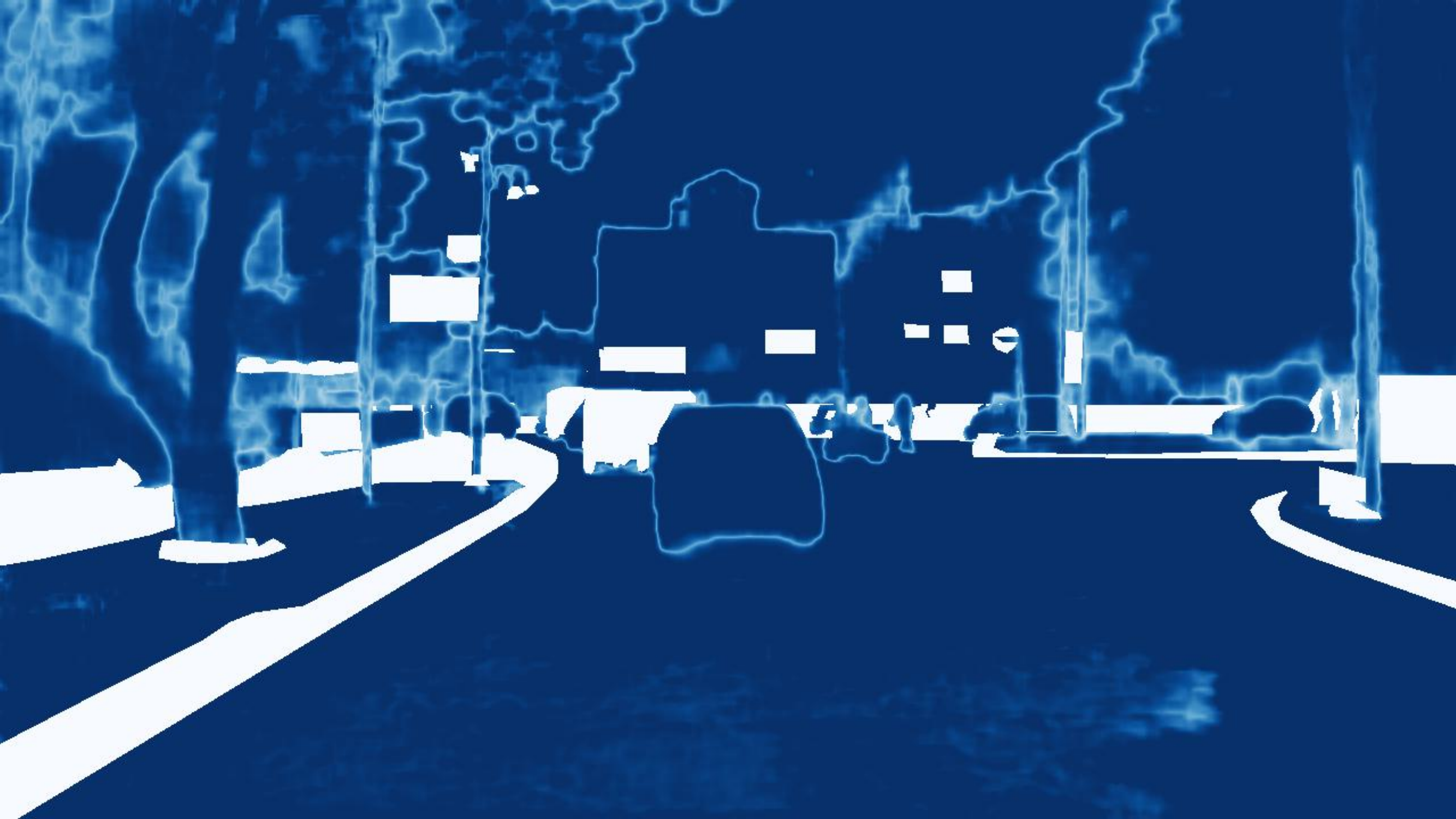}
    \label{fig:idd2_3}
  \end{subfigure}
  \caption{\textbf{Two qualitative examples from IDD.} See \cref{sec:qualitative} for analysis.}
  \label{fig:qaulidd2}
  \vspace{-0.7em}
\end{figure*}

\begin{figure*}[t]
\centering
  \begin{subfigure}{0.33\linewidth}
    \centering
    HSSN\\\vspace{0.3em}
    \includegraphics[width=\linewidth]{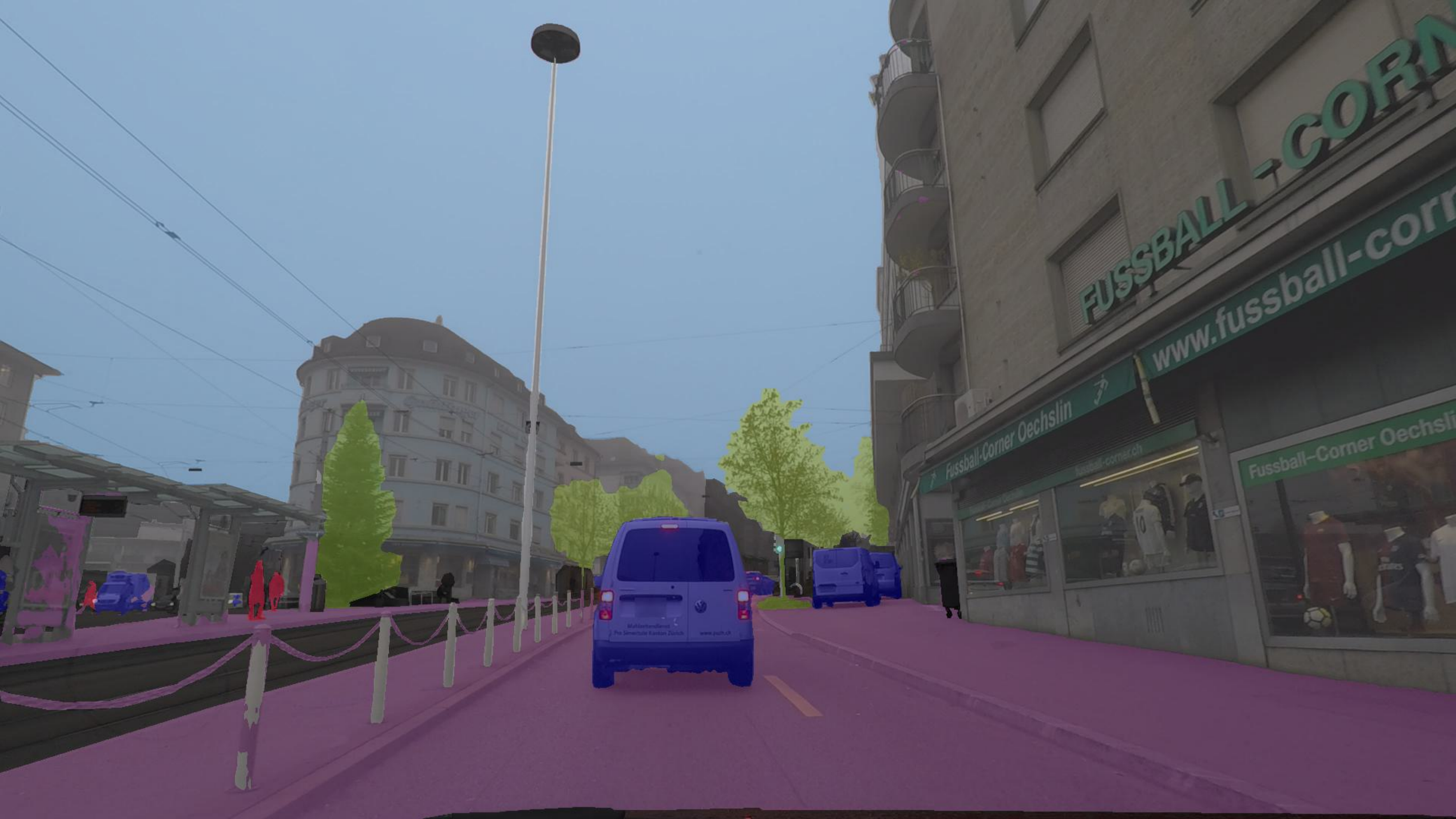}\\[1mm]
    \includegraphics[width=\linewidth]{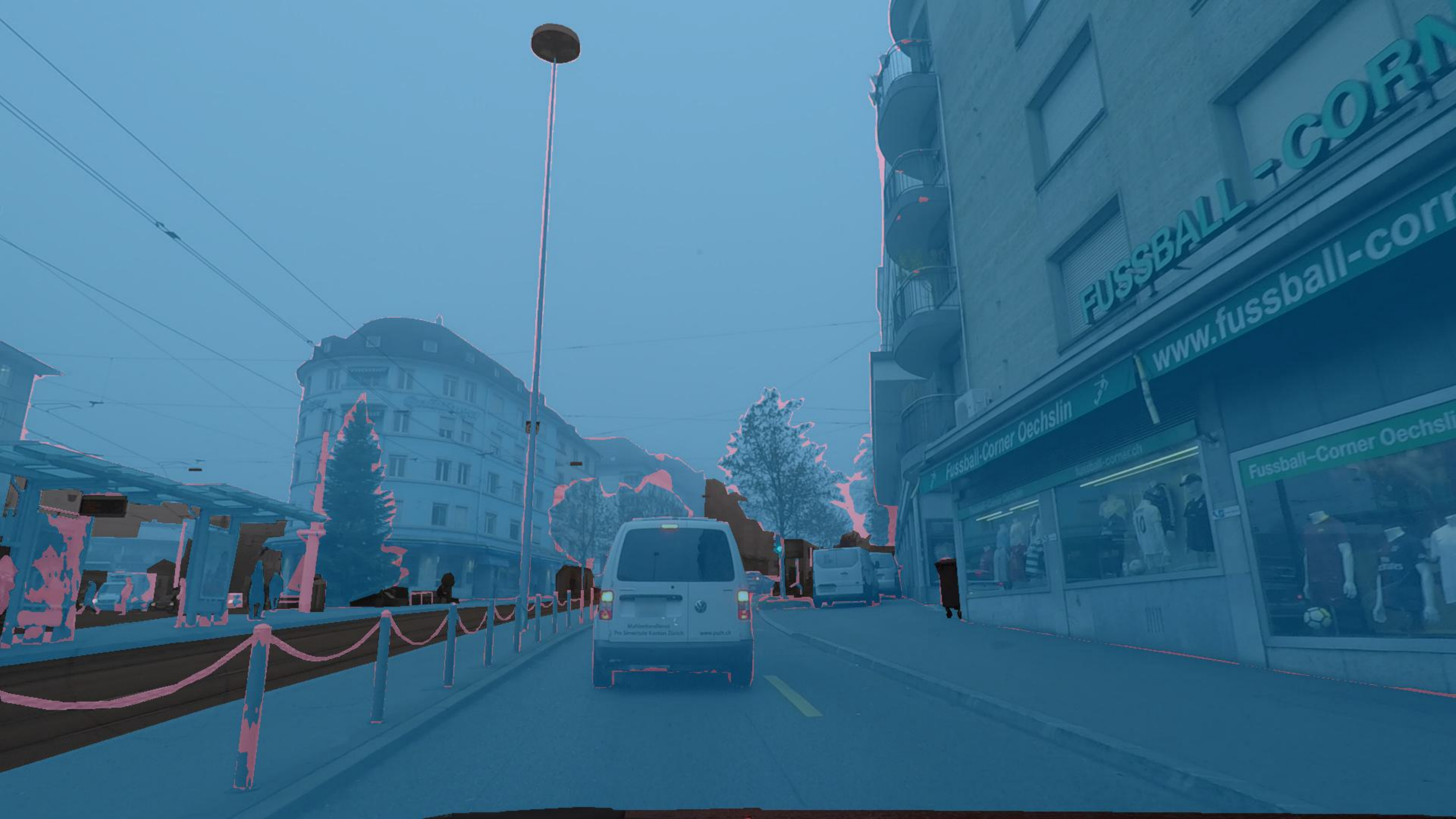}\\[1mm]
    \includegraphics[width=\linewidth]{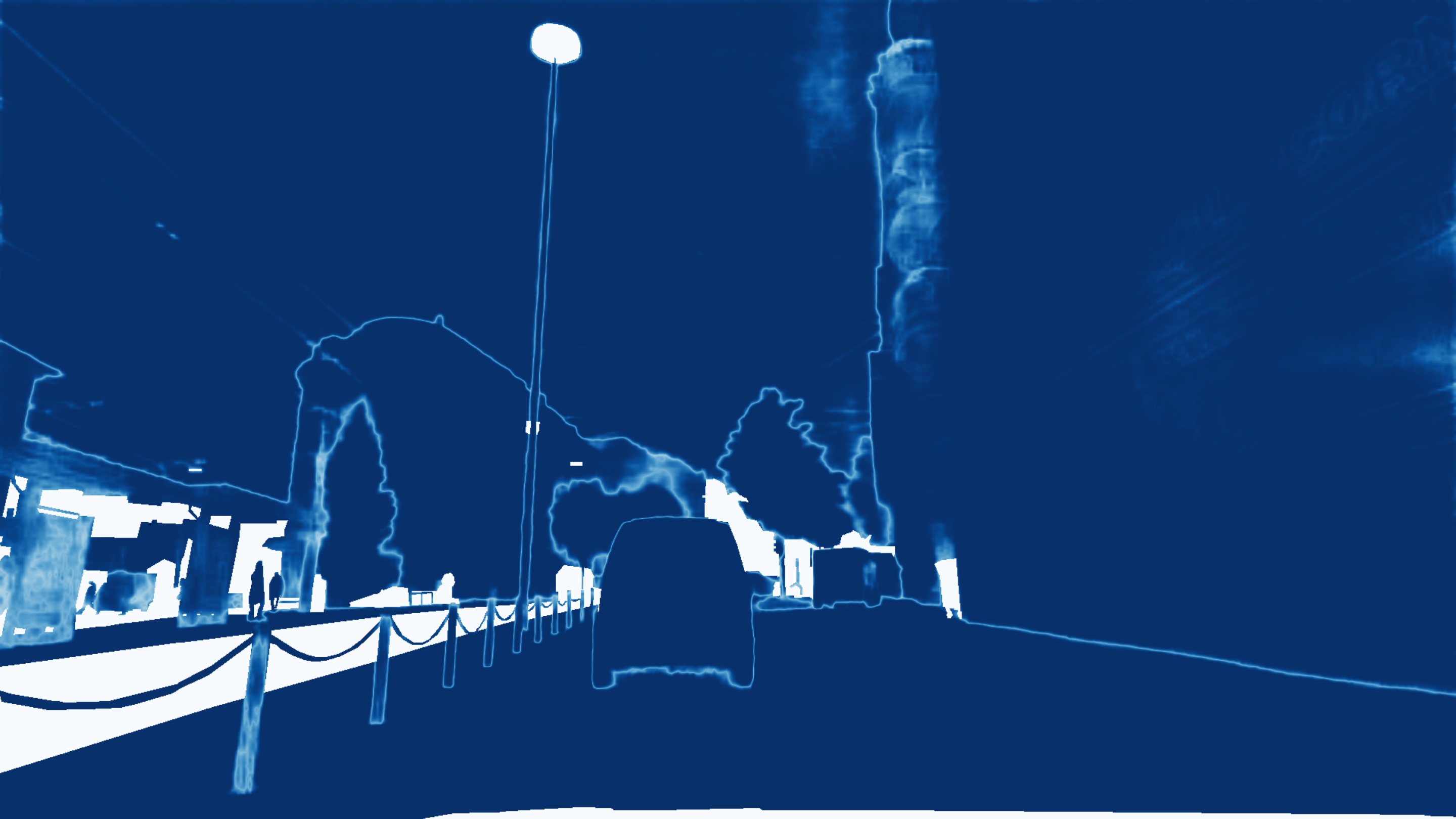}
  \label{fig:acdc1_1}
  \end{subfigure}
  \begin{subfigure}{0.33\linewidth}
    \centering
    \eucname\\\vspace{0.3em}
    \includegraphics[width=\linewidth]{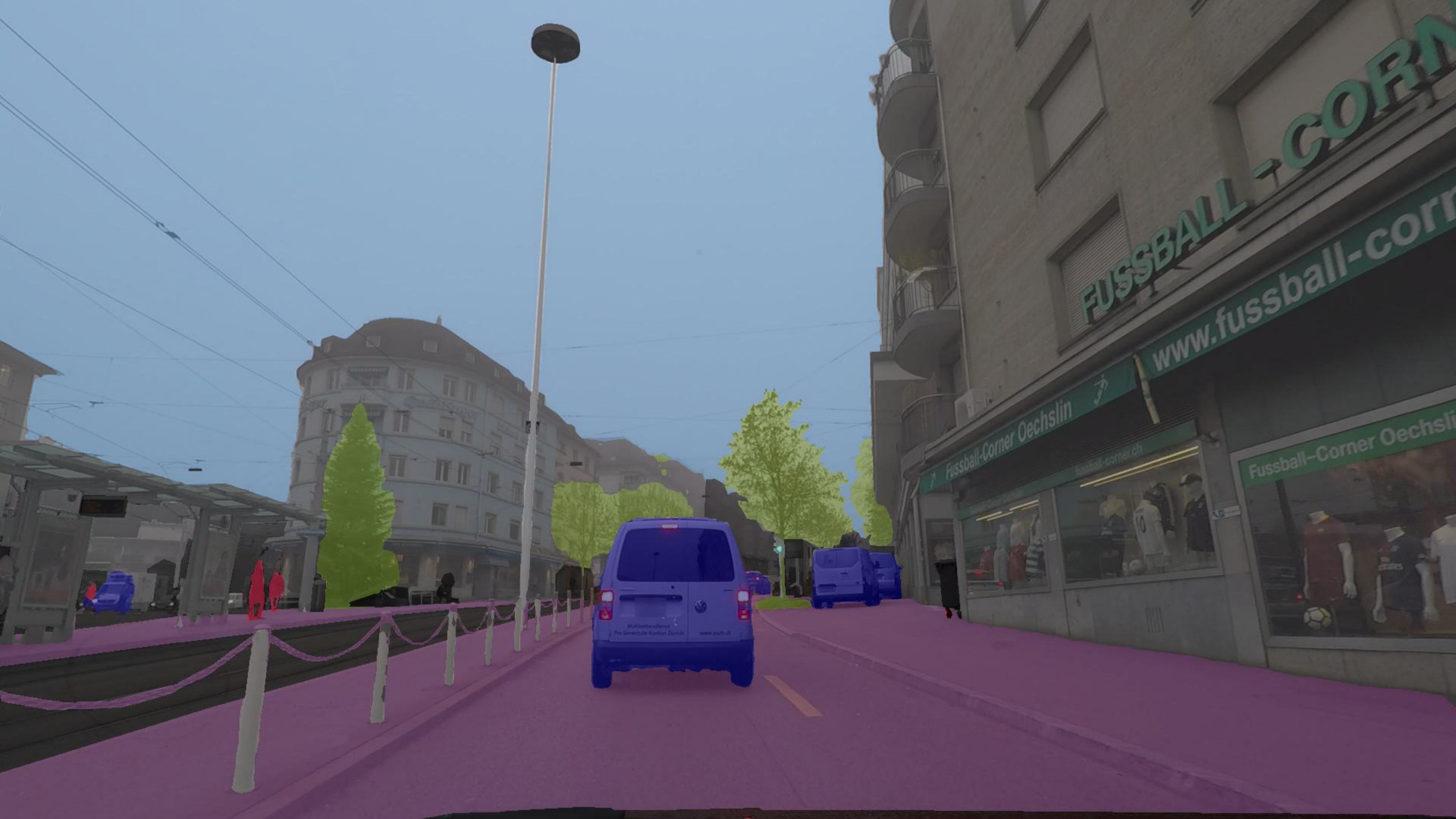}\\[1mm]
    \includegraphics[width=\linewidth]{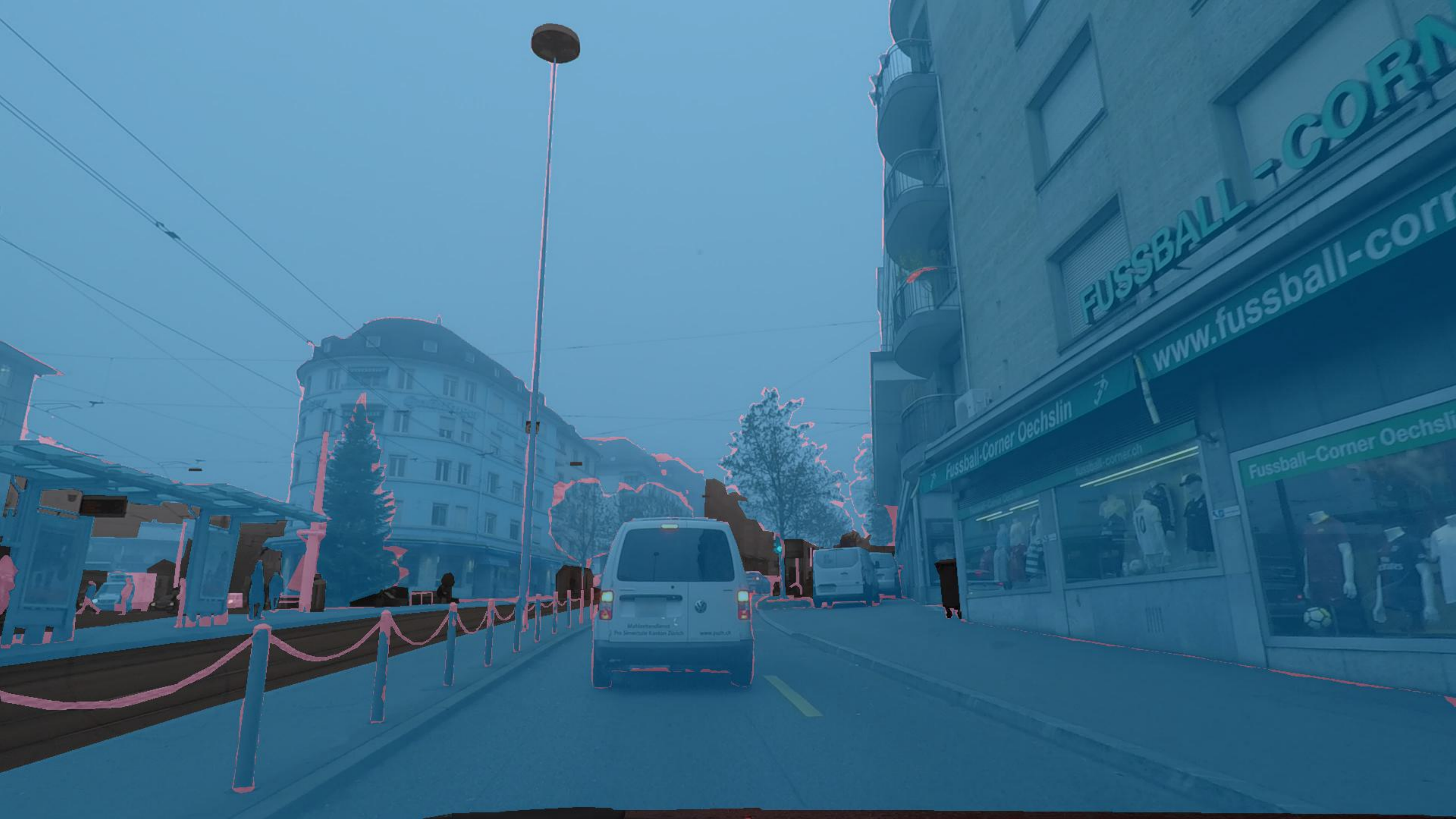}\\[1mm]
    \includegraphics[width=\linewidth]{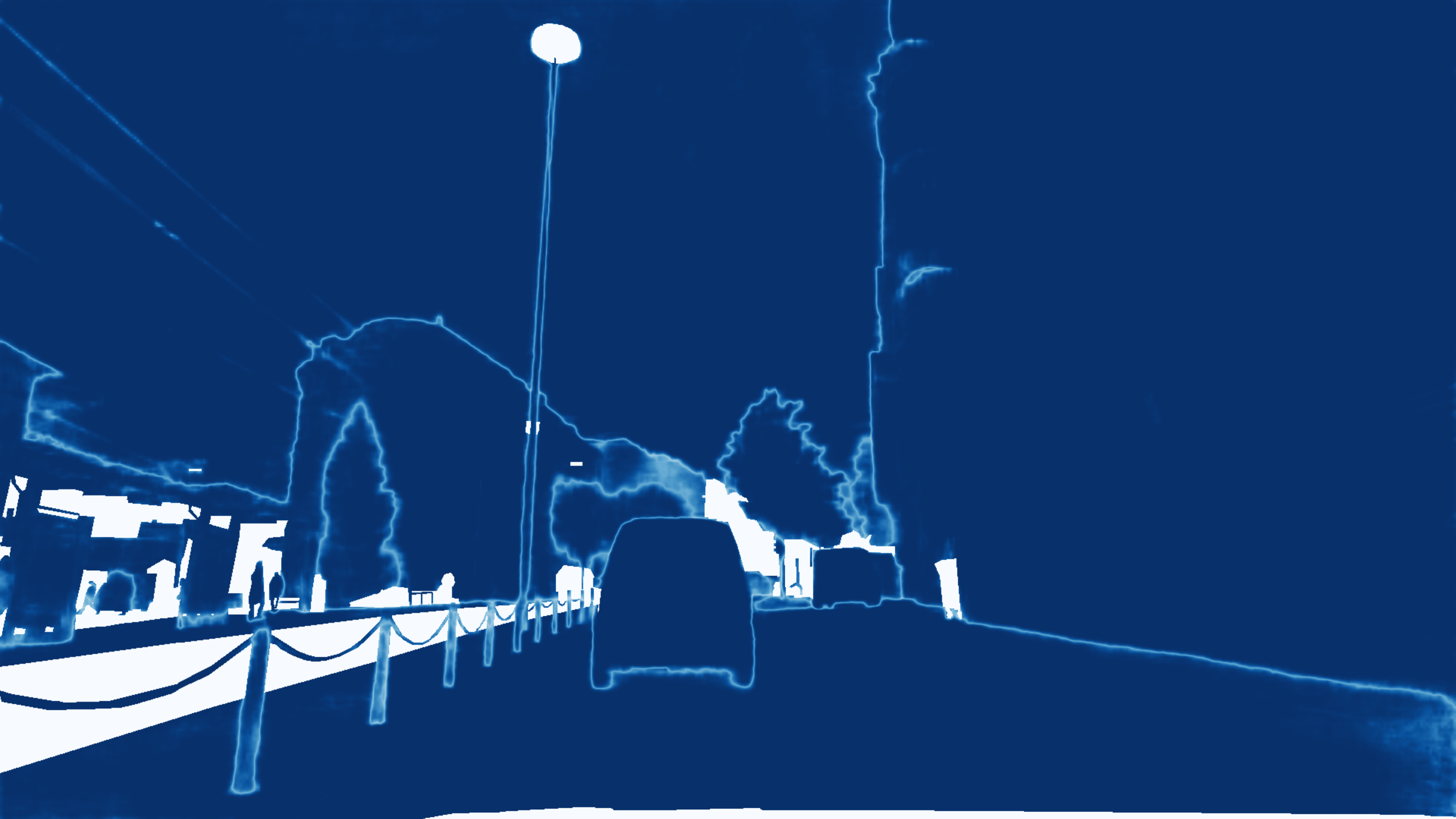}
    \label{fig:acdc1_2}
  \end{subfigure}
  \begin{subfigure}{0.33\linewidth}
    \centering
    \hypname\\\vspace{0.2em}
    \includegraphics[width=\linewidth]{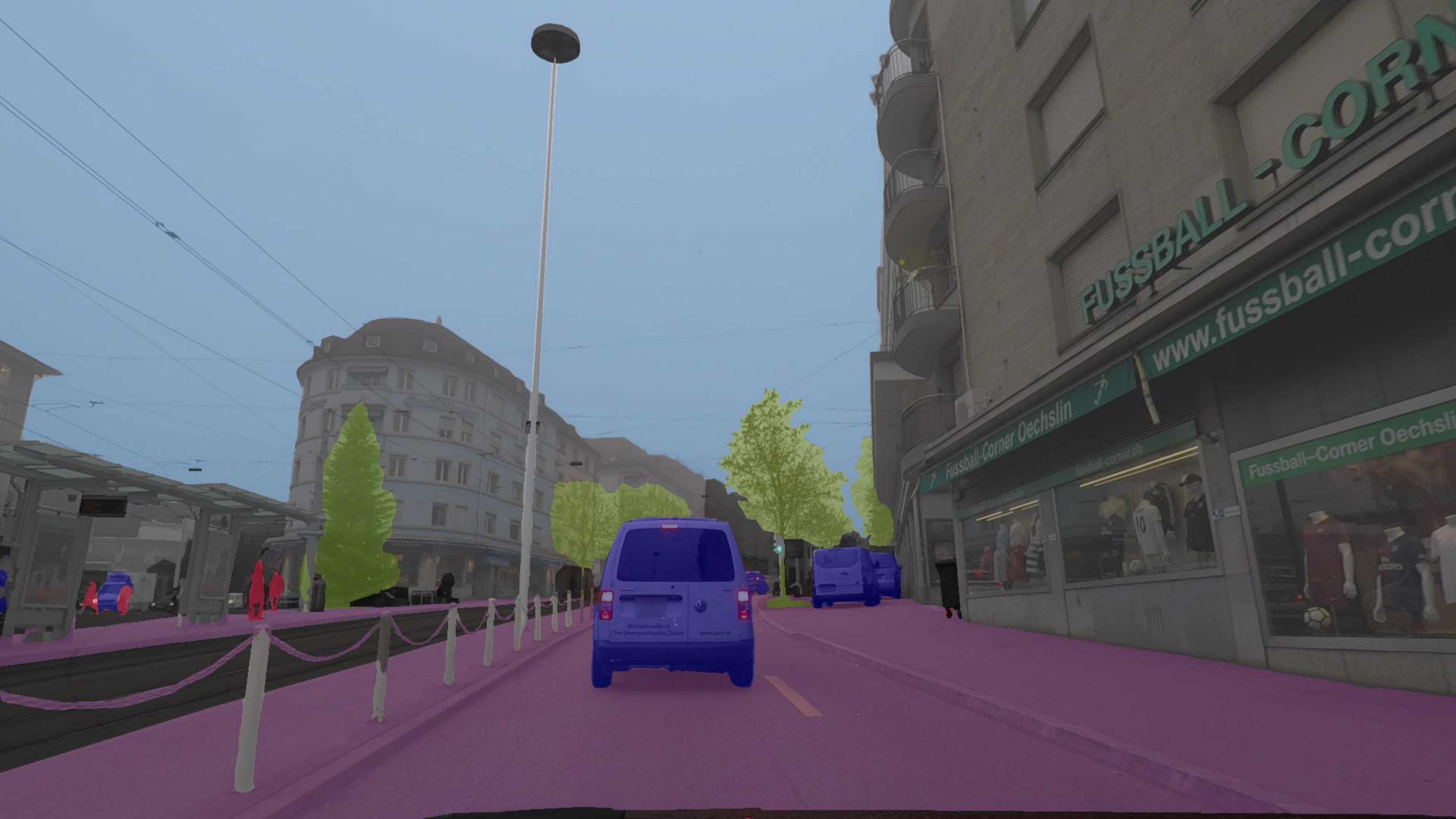}\\[1mm]
    \includegraphics[width=\linewidth]{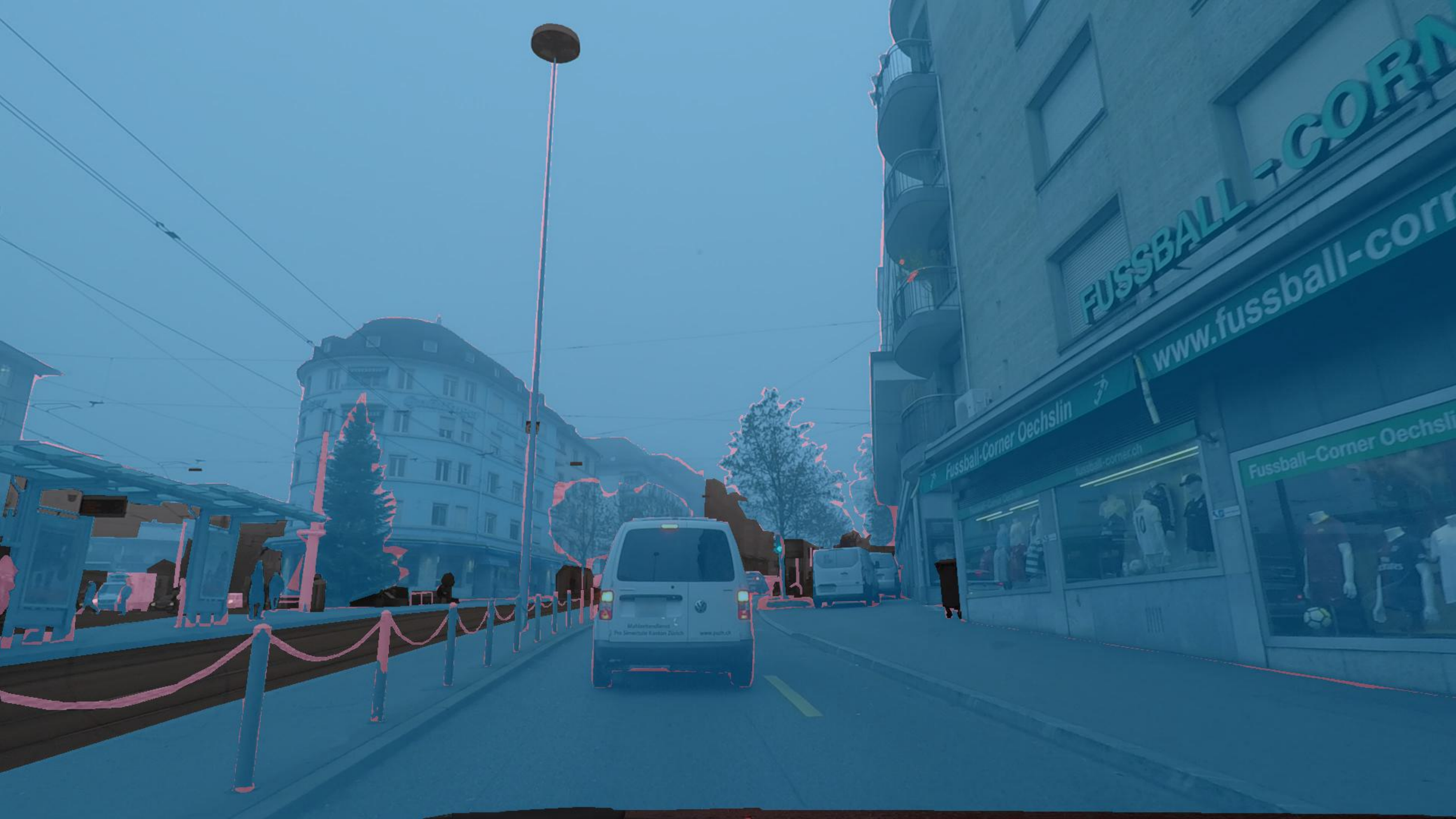}\\[1mm]
    \includegraphics[width=\linewidth]{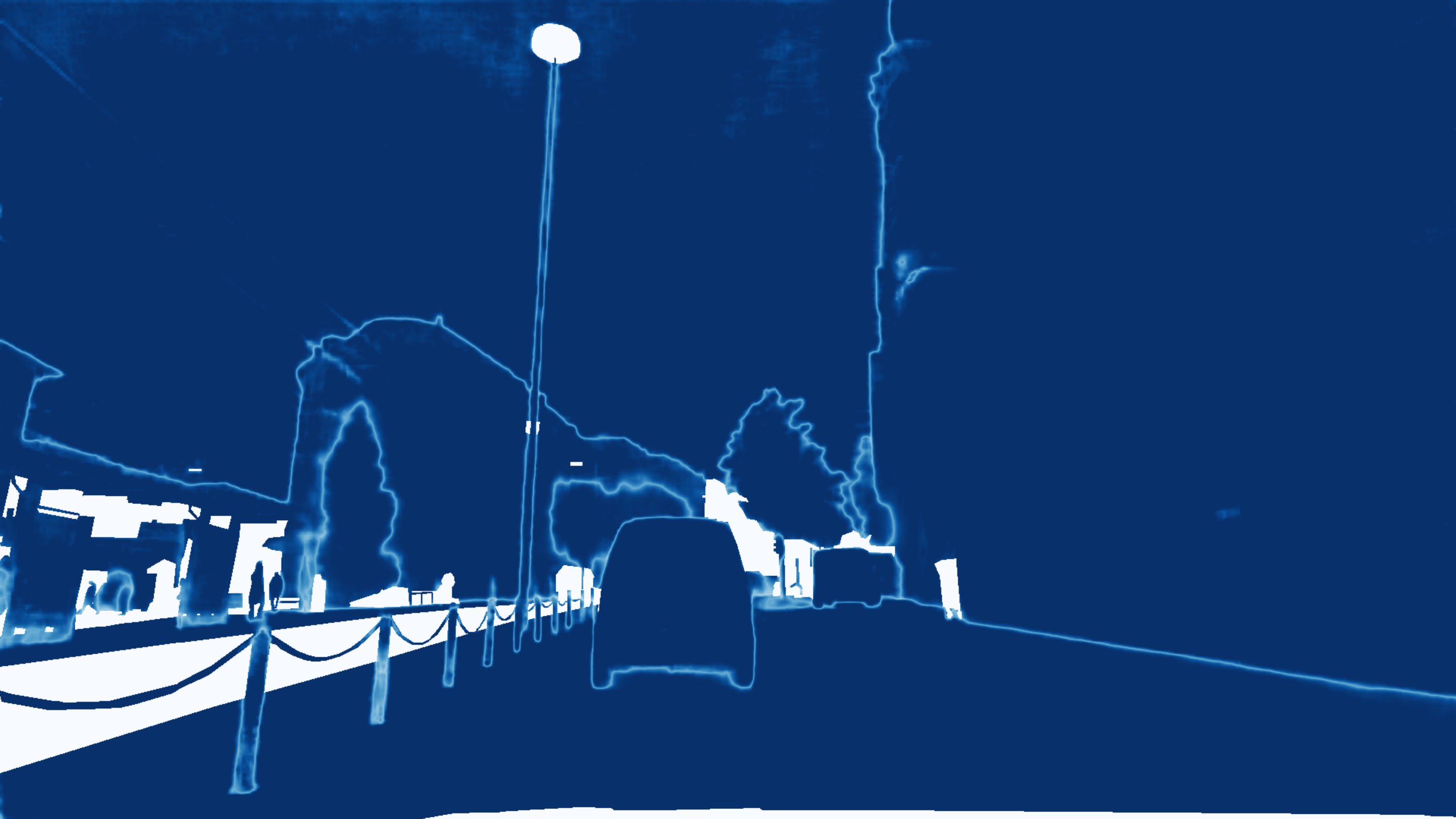}
    \label{fig:acdc1_3}
  \end{subfigure}
    \begin{subfigure}{\linewidth}
        \centering
        \fboxsep 2pt
        \begin{minipage}{0.7\linewidth}
            \colorbox{flat}{\strut \color{white}{flat}}
            \colorbox{construction}{\strut \color{white}{construction}}
            \colorbox{object}{\strut  object}
            \colorbox{nature}{\strut \color{white}{nature}}
            \colorbox{sky}{\strut \color{white}{sky}}
            \colorbox{human}{\strut human}
            \colorbox{vehicle}{\strut \color{white}{vehicle}}\hspace{2em}
            \colorbox{ignore}{\strut \color{white}{ignore}}\hspace{2em}
            \colorbox{true}{\strut \color{white}{true}}
            \colorbox{false}{\strut false}\hspace{2em}%
        \end{minipage}%
        \begin{minipage}{0.3\linewidth}
        \begin{tikzpicture}
        \node [rectangle, left color=left!10!white, right color=left, anchor=north, minimum width=\linewidth, minimum height=0.5cm] (box) at (current page.north){0 \hspace{12em}  \color{white}{1}};
        \end{tikzpicture}
        \end{minipage}
    \end{subfigure}\\
    \vspace{1em}
  \begin{subfigure}{0.33\linewidth}
    \centering
    
    \includegraphics[width=\linewidth]{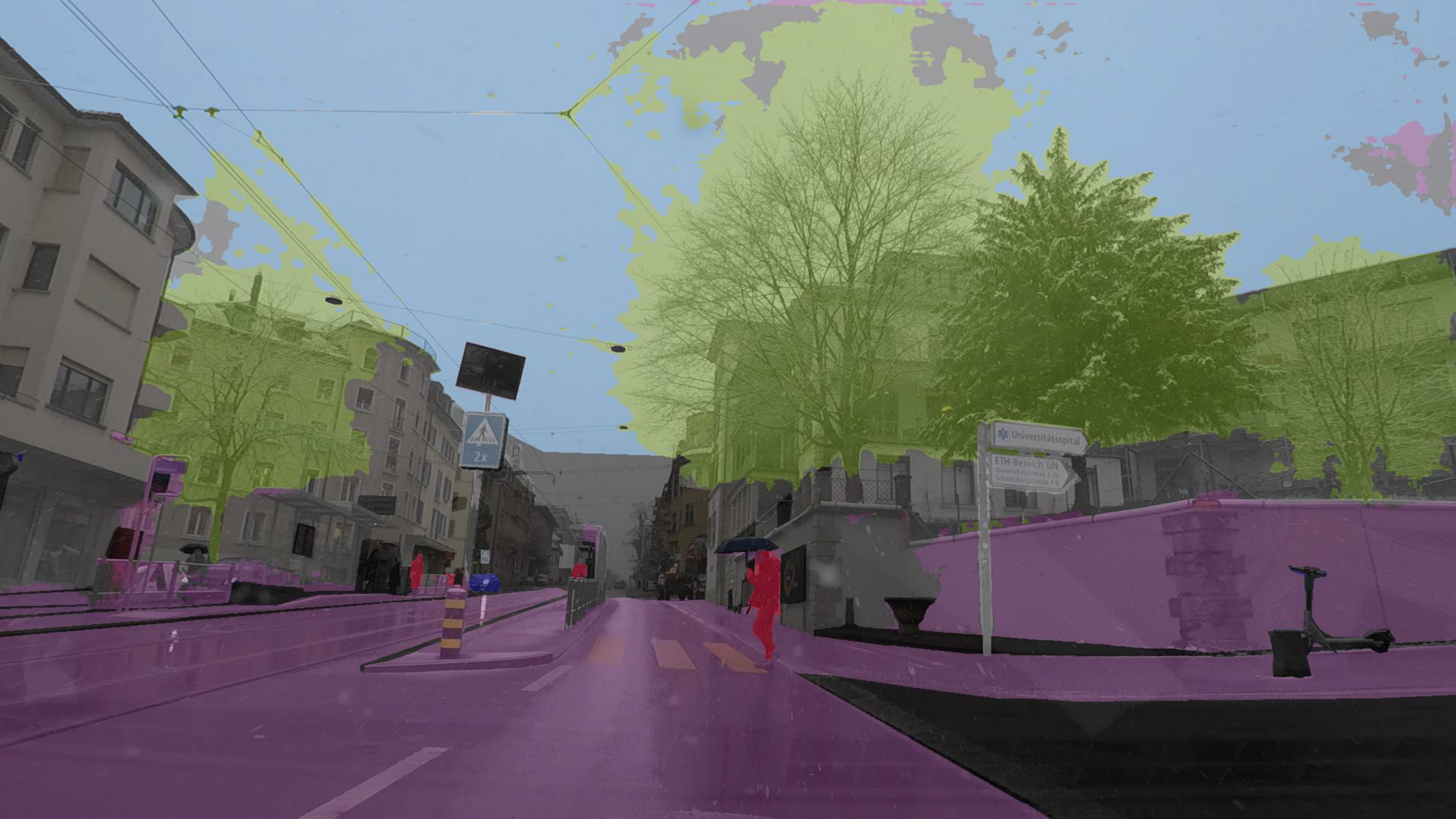}\\[1mm]
    \includegraphics[width=\linewidth]{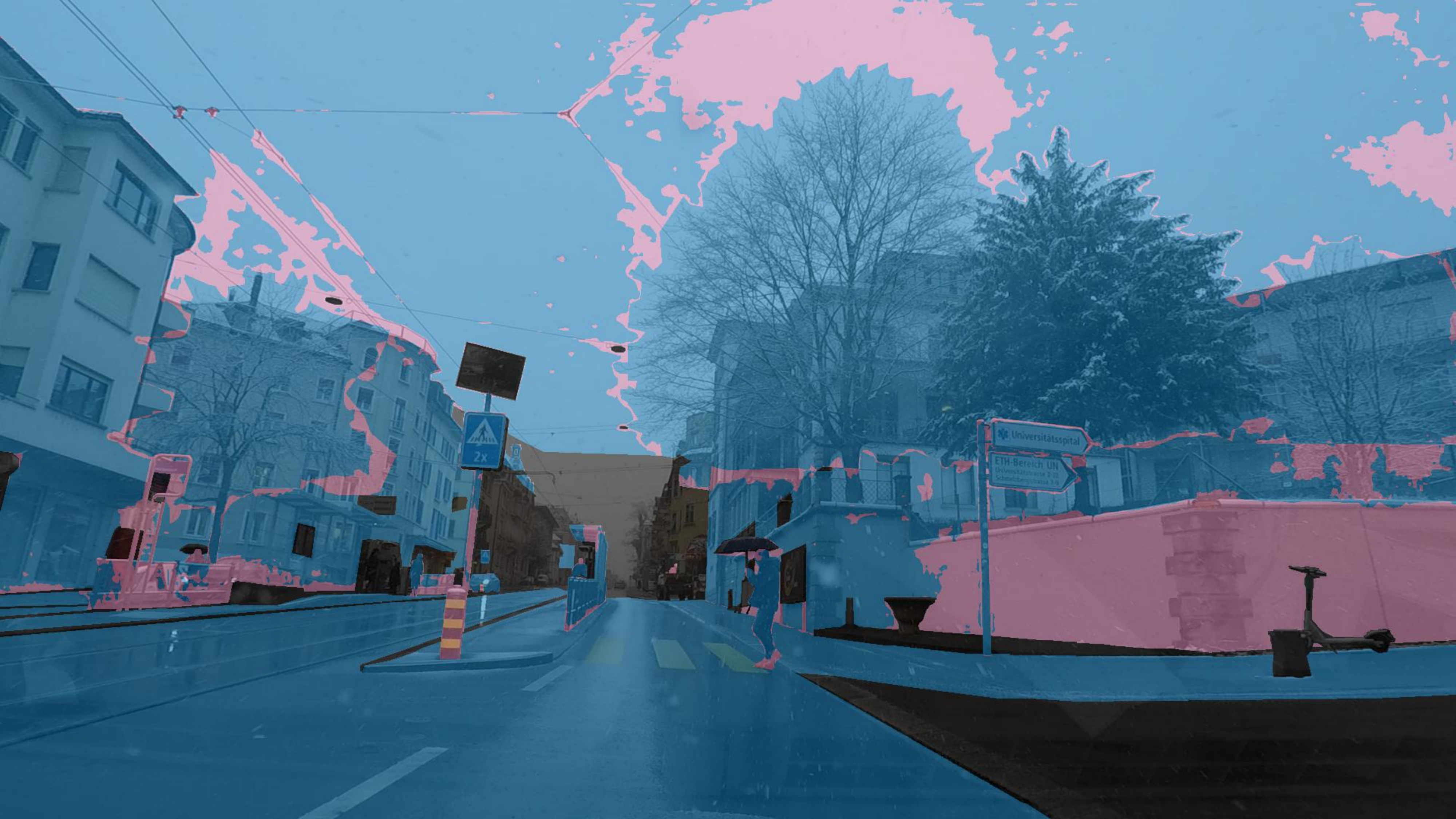}\\[1mm]
    \includegraphics[width=\linewidth]{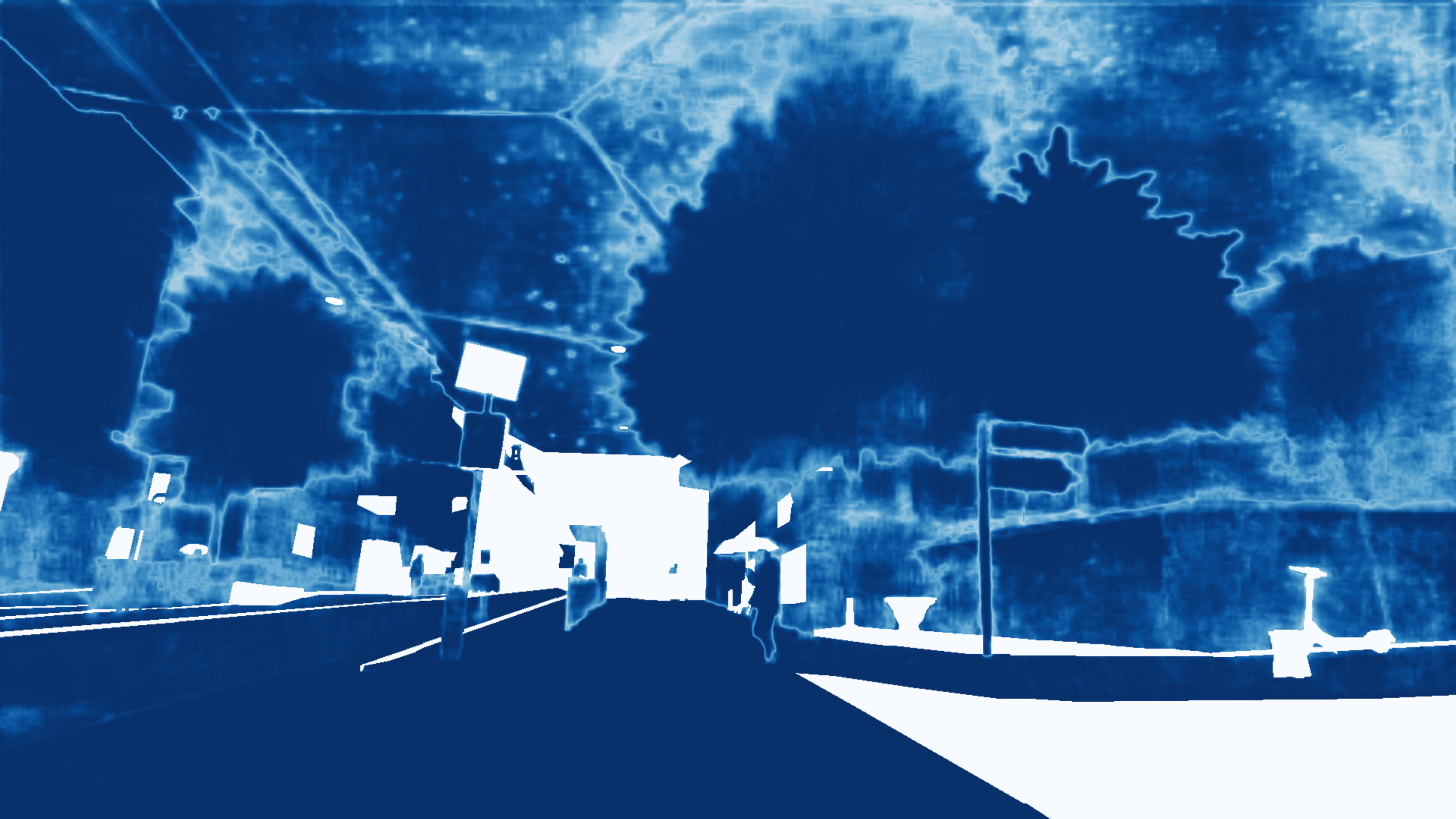}
  \label{fig:acdc2_1}
  \end{subfigure}
  \begin{subfigure}{0.33\linewidth}
    \centering
    \includegraphics[width=\linewidth]{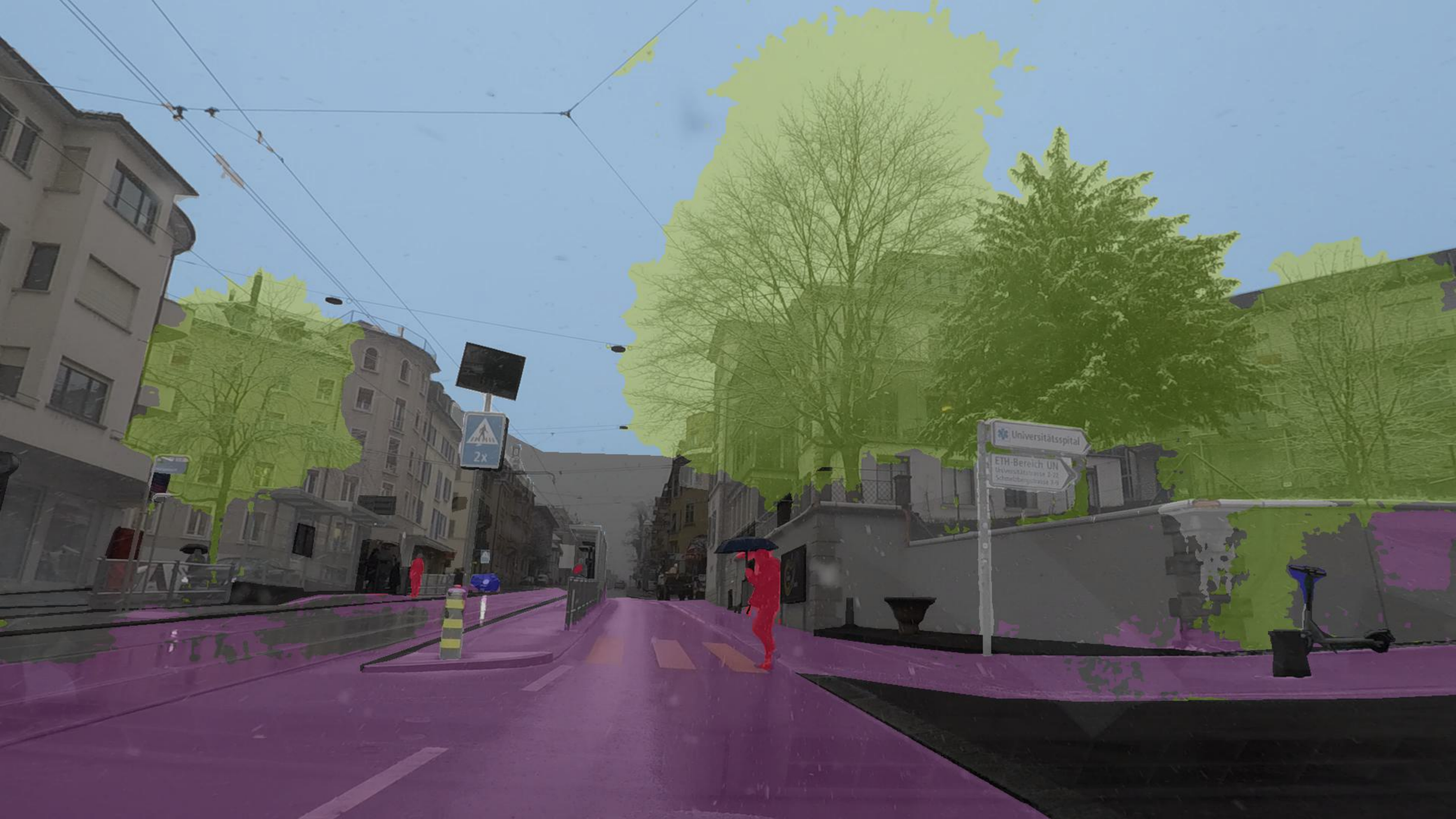}\\[1mm]
    \includegraphics[width=\linewidth]{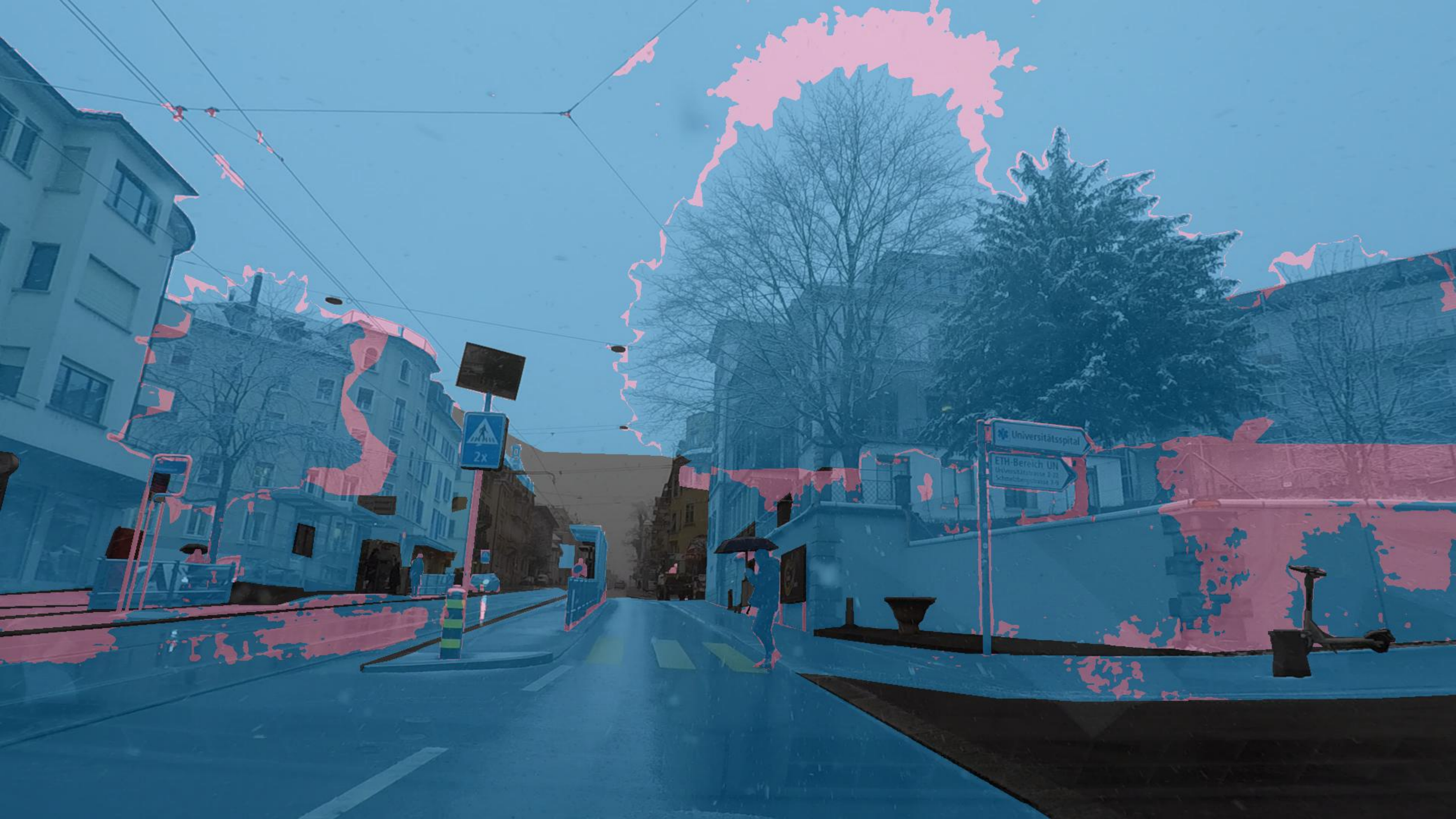}\\[1mm]
    \includegraphics[width=\linewidth]{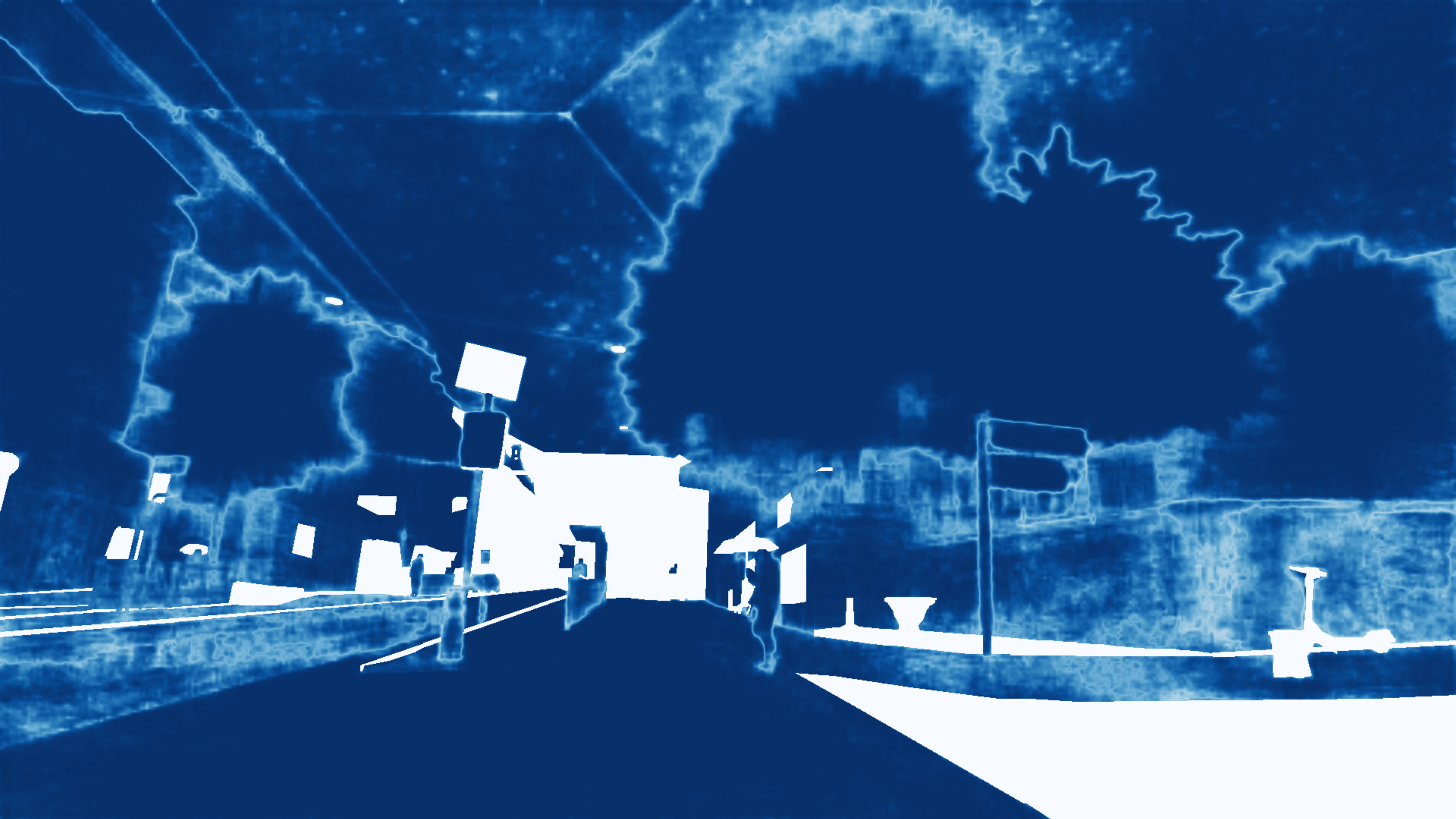}
    \label{fig:acdc2_2}
  \end{subfigure}
  \begin{subfigure}{0.33\linewidth}
    \centering
    \includegraphics[width=\linewidth]{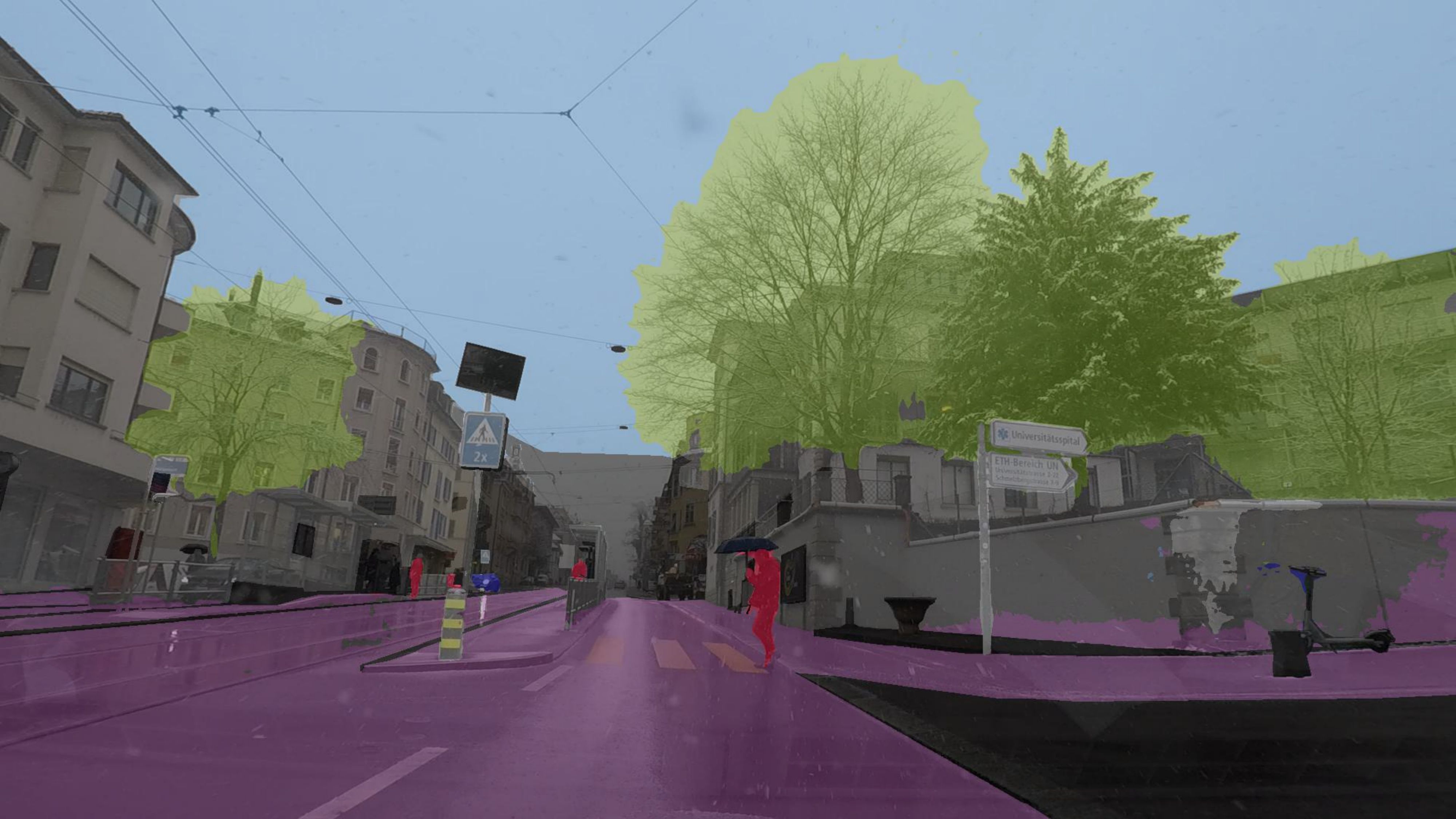}\\[1mm]
    \includegraphics[width=\linewidth]{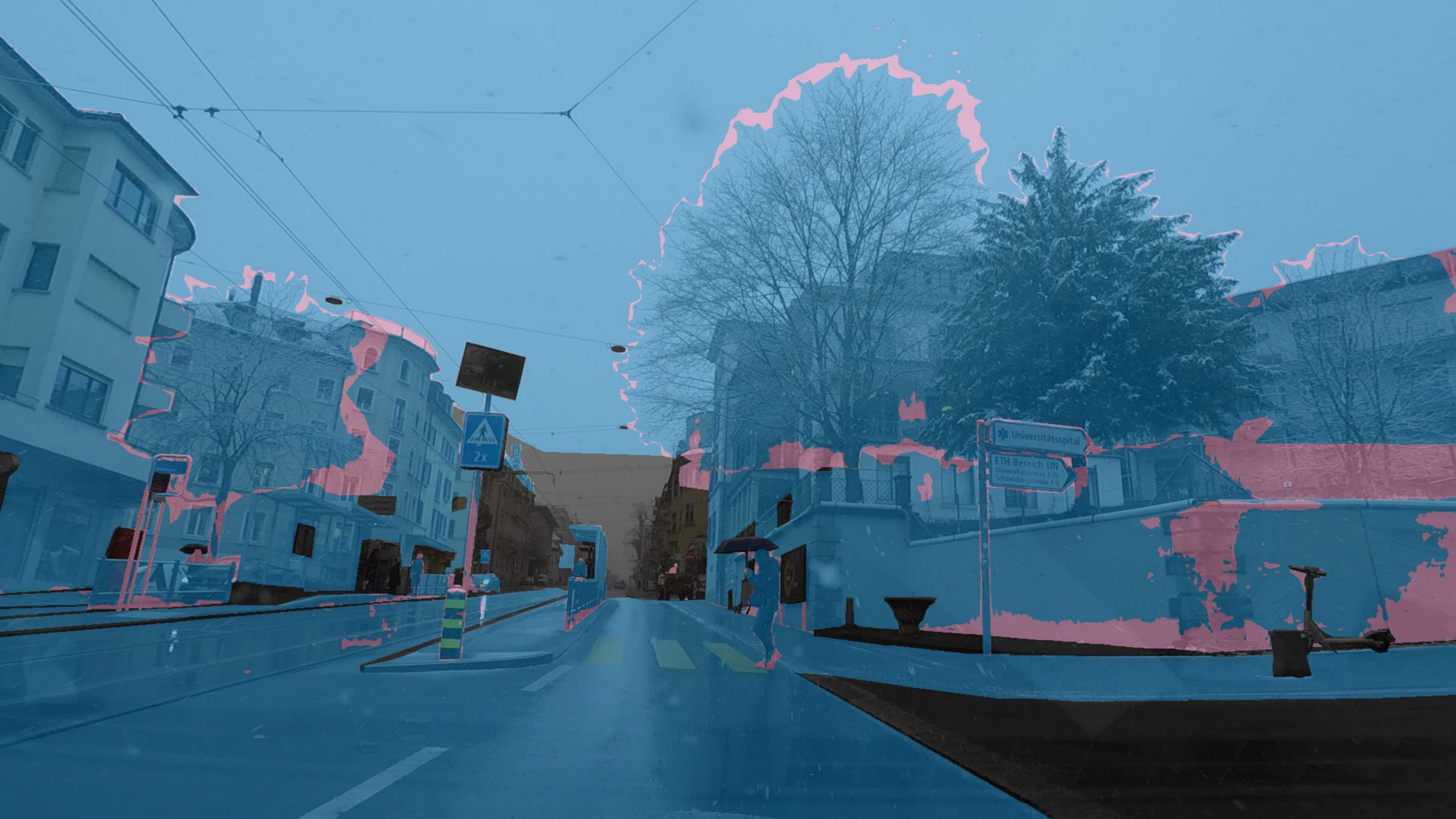}\\[1mm]
    \includegraphics[width=\linewidth]{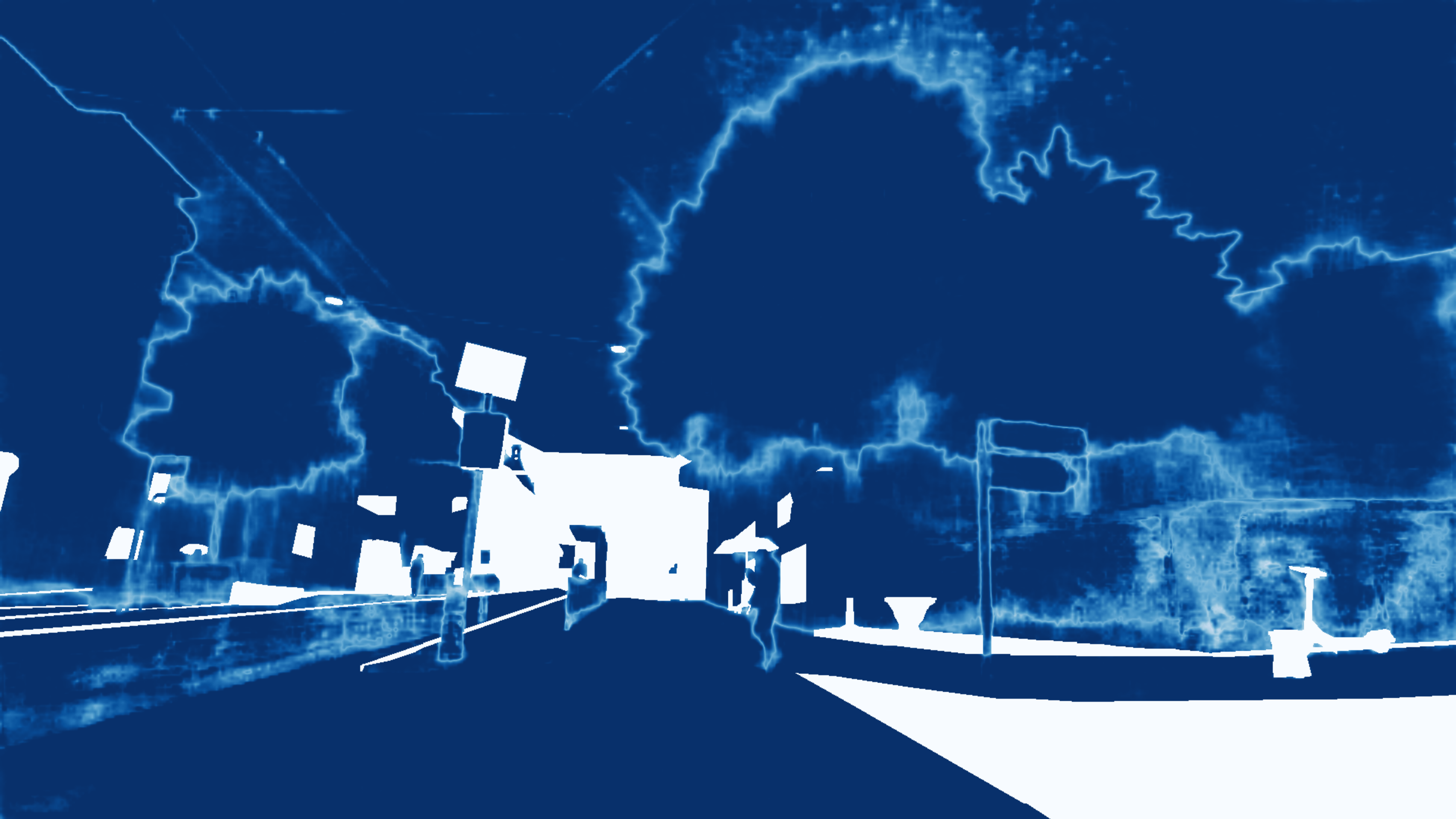}
    \label{fig:acdc2_3}
  \end{subfigure}
  \caption{\textbf{Two qualitative examples from ACDC.} See \cref{sec:qualitative} for analysis.}
  \label{fig:qauliacdc}
  \vspace{-0.7em}
\end{figure*}

\begin{figure*}[t]
\centering
  \begin{subfigure}{0.33\linewidth}
    \centering
    HSSN\\\vspace{0.3em}
    \includegraphics[width=\linewidth]{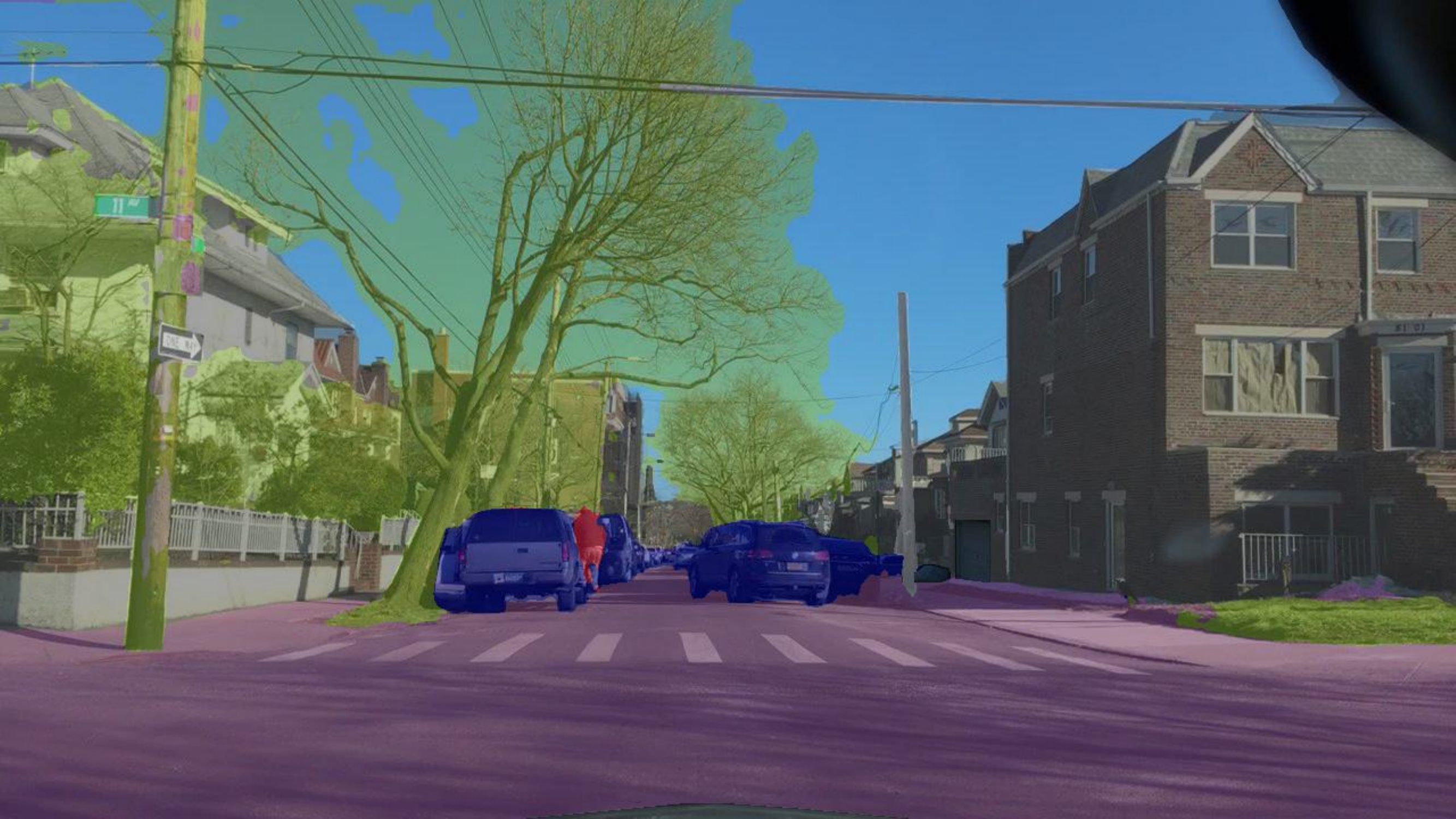}\\[1mm]
    \includegraphics[width=\linewidth]{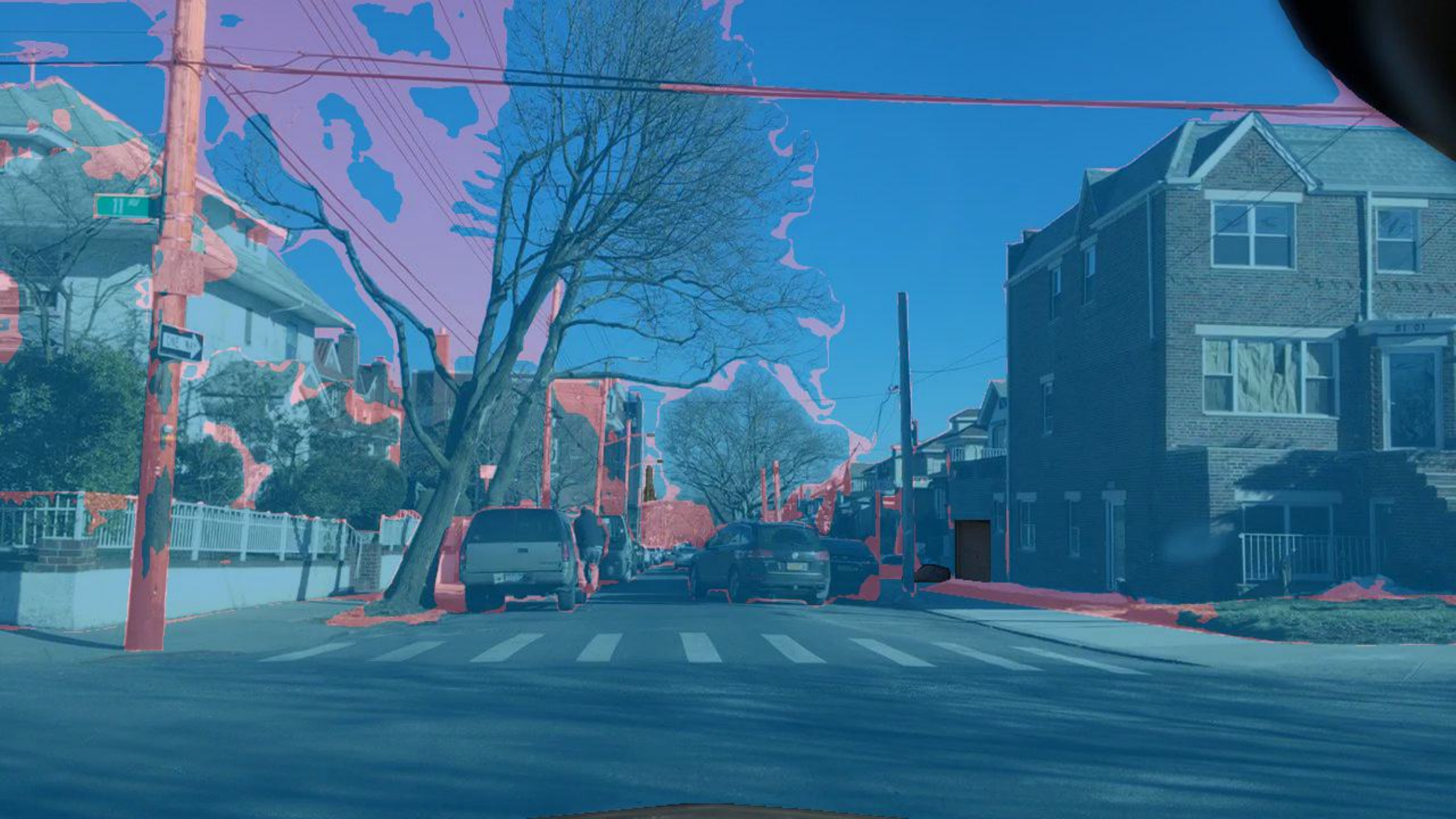}\\[1mm]
    \includegraphics[width=\linewidth]{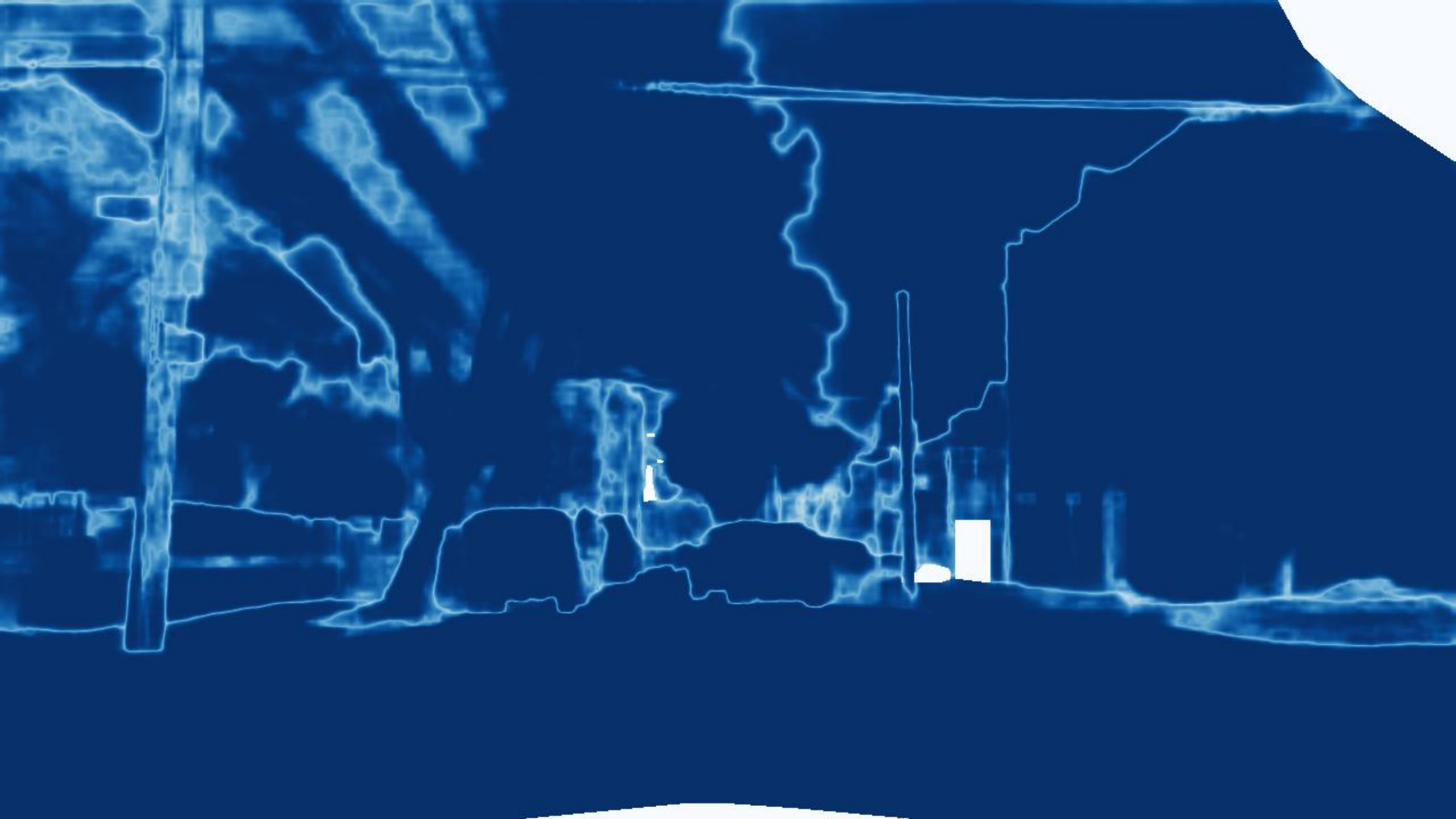}
  \label{fig:bdd1_1}
  \end{subfigure}
  \begin{subfigure}{0.33\linewidth}
    \centering
    \eucname\\\vspace{0.3em}
    \includegraphics[width=\linewidth]{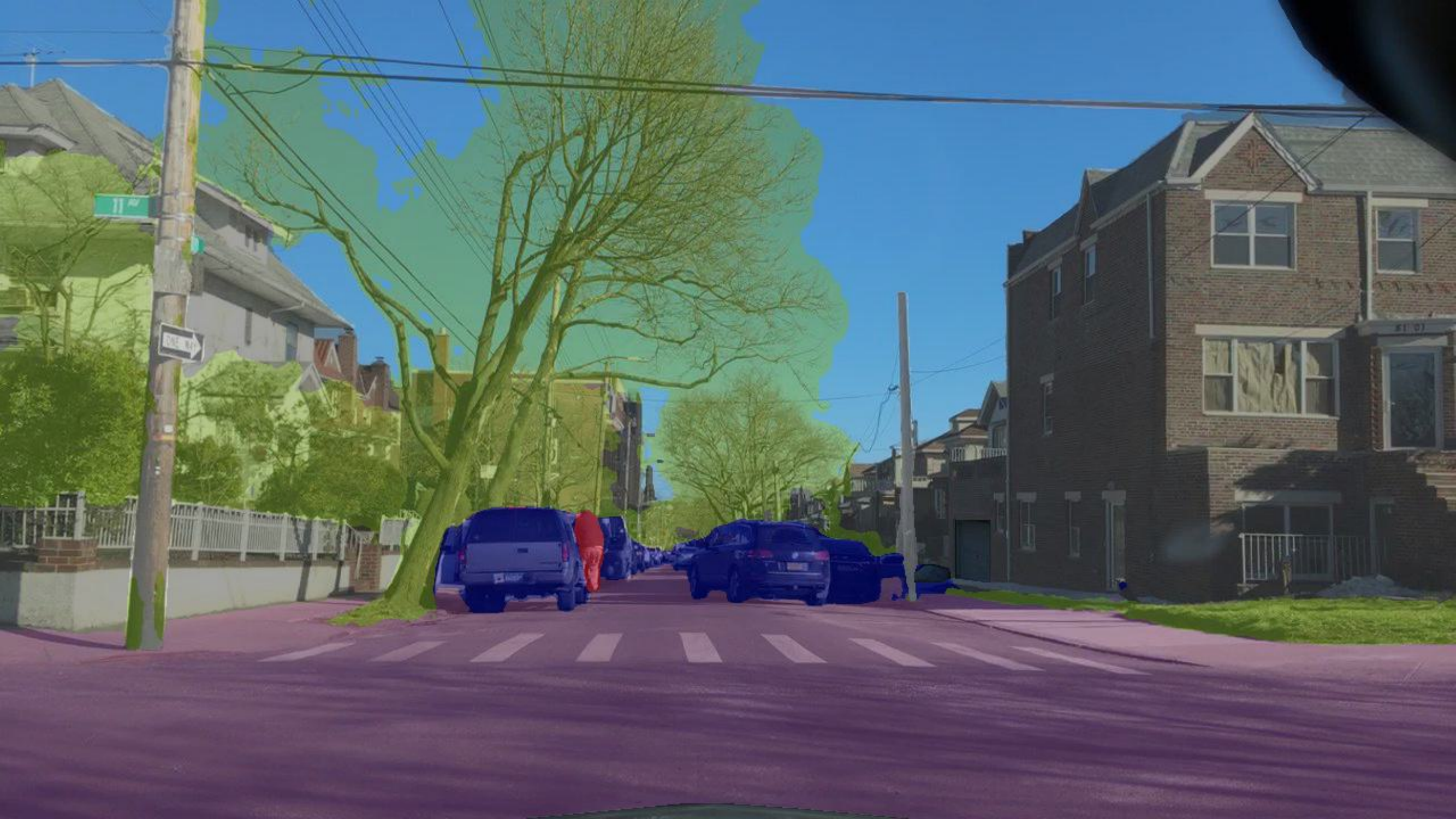}\\[1mm]
    \includegraphics[width=\linewidth]{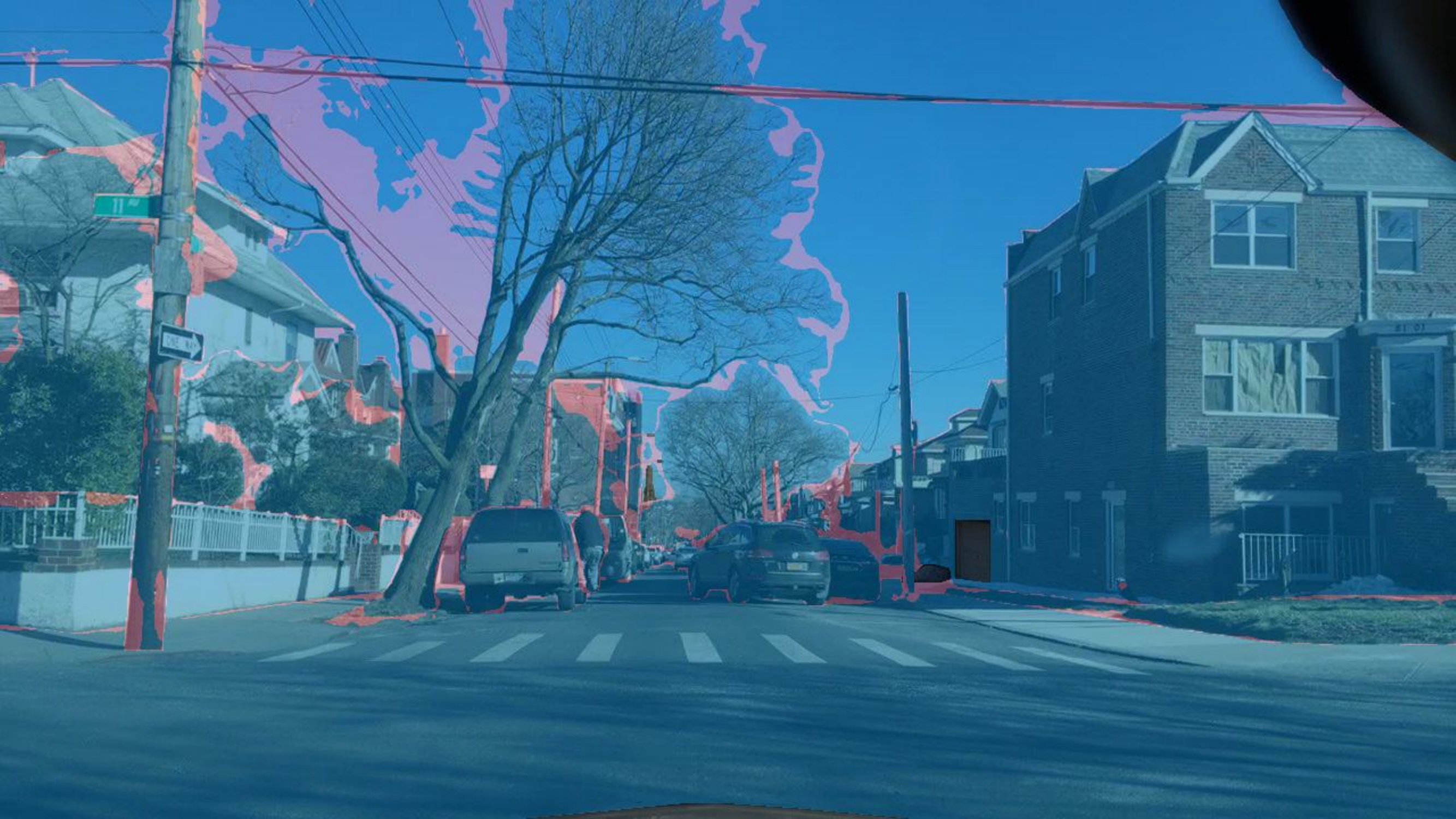}\\[1mm]
    \includegraphics[width=\linewidth]{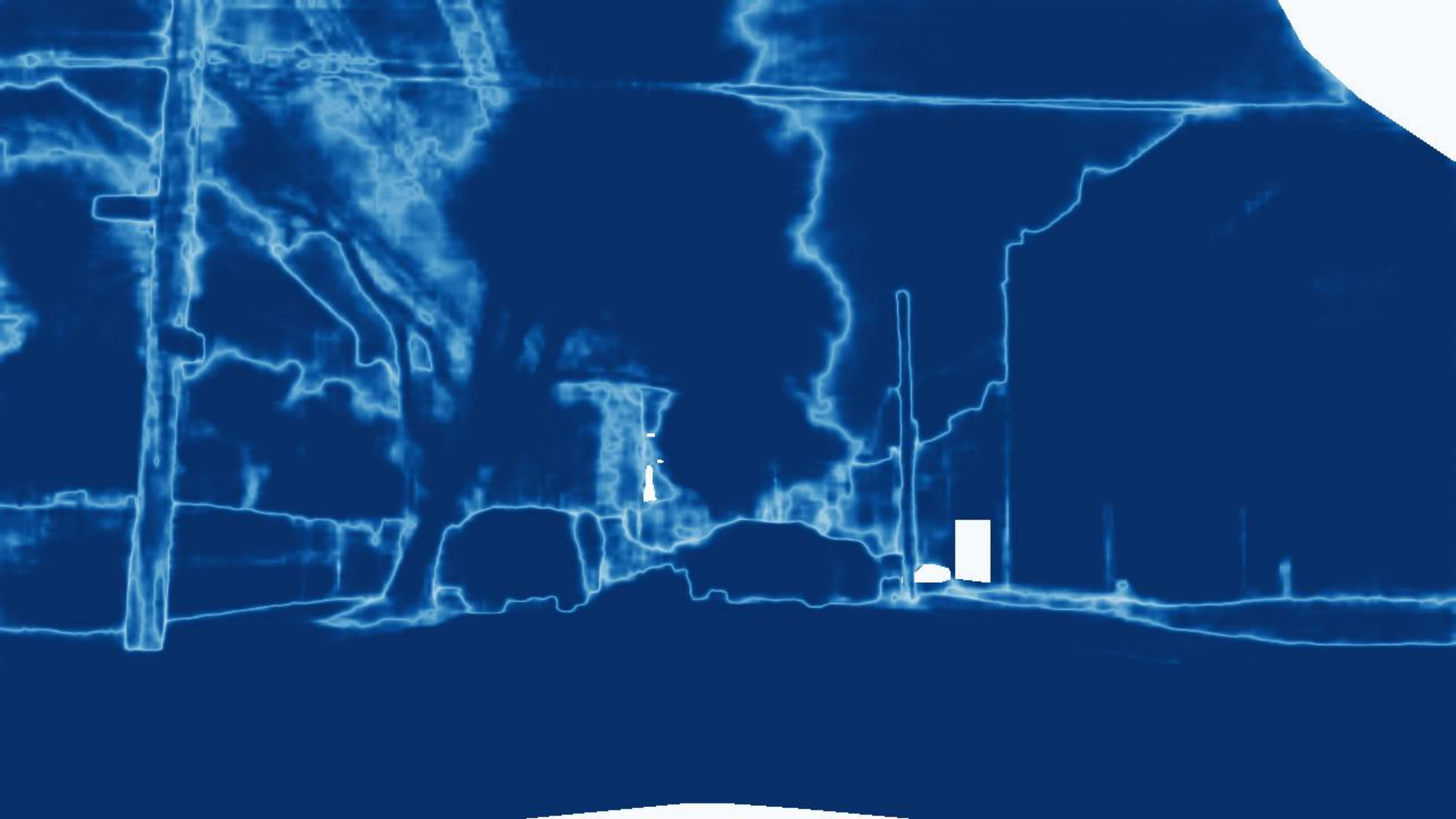}
    \label{fig:bdd1_2}
  \end{subfigure}
  \begin{subfigure}{0.33\linewidth}
    \centering
    \hypname\\\vspace{0.2em}
    \includegraphics[width=\linewidth]{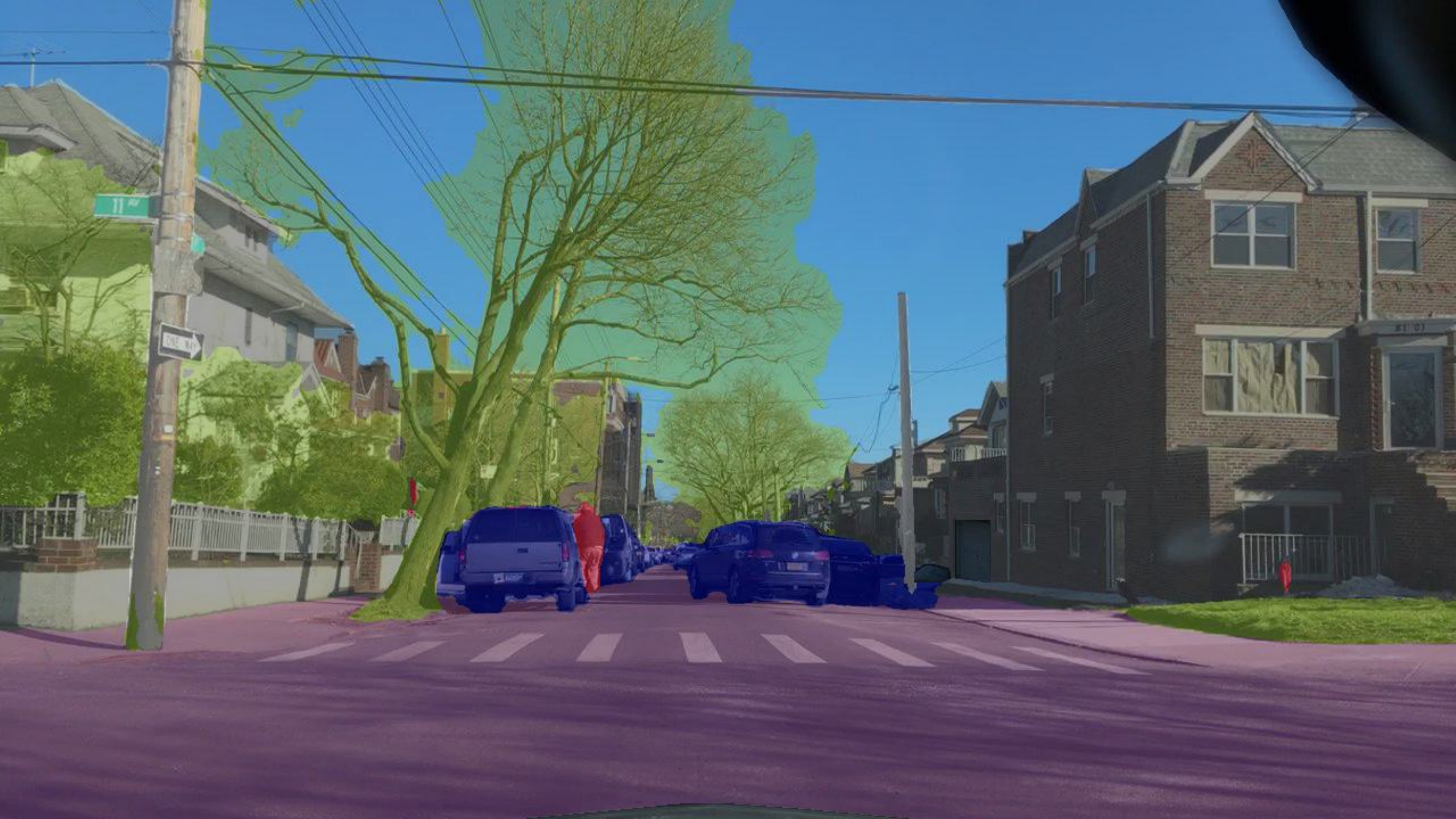}\\[1mm]
    \includegraphics[width=\linewidth]{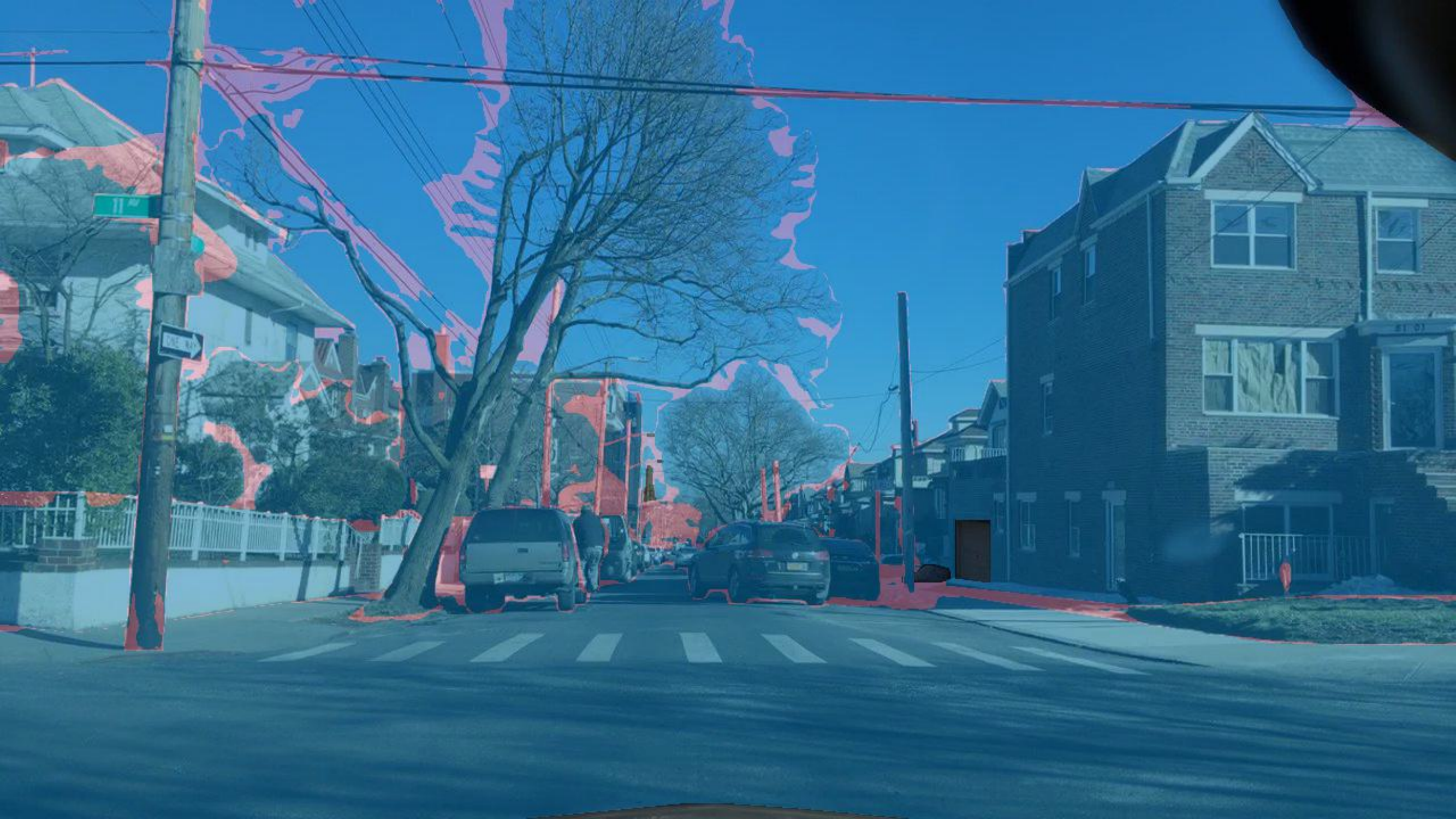}\\[1mm]
    \includegraphics[width=\linewidth]{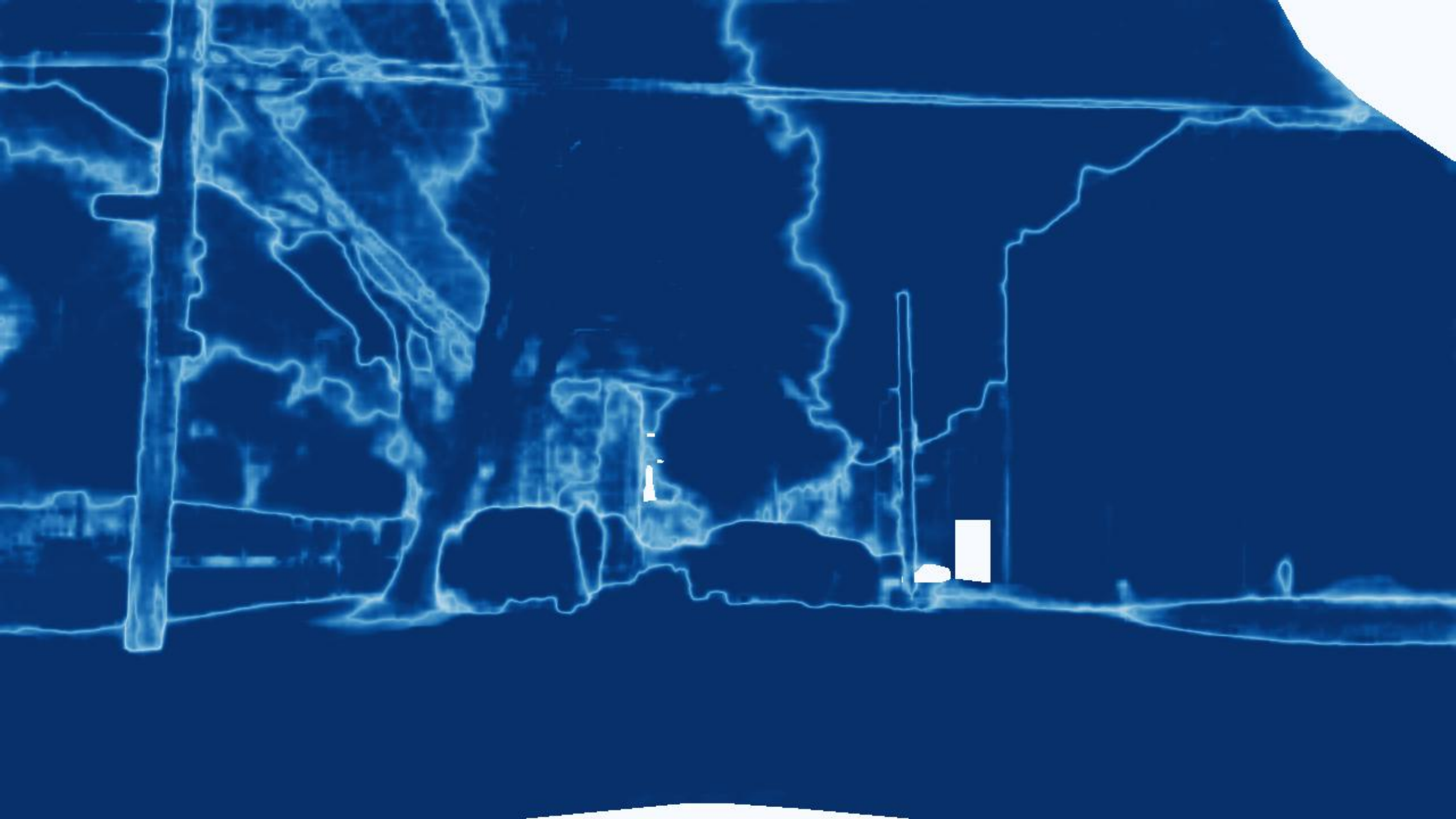}
    \label{fig:bdd1_3}
  \end{subfigure}
    \begin{subfigure}{\linewidth}
        \centering
        \fboxsep 2pt
        \begin{minipage}{0.7\linewidth}
            \colorbox{flat}{\strut \color{white}{flat}}
            \colorbox{construction}{\strut \color{white}{construction}}
            \colorbox{object}{\strut  object}
            \colorbox{nature}{\strut \color{white}{nature}}
            \colorbox{sky}{\strut \color{white}{sky}}
            \colorbox{human}{\strut human}
            \colorbox{vehicle}{\strut \color{white}{vehicle}}\hspace{2em}
            \colorbox{ignore}{\strut \color{white}{ignore}}\hspace{2em}
            \colorbox{true}{\strut \color{white}{true}}
            \colorbox{false}{\strut false}\hspace{2em}%
        \end{minipage}%
        \begin{minipage}{0.3\linewidth}
        \begin{tikzpicture}
        \node [rectangle, left color=left!10!white, right color=left, anchor=north, minimum width=\linewidth, minimum height=0.5cm] (box) at (current page.north){0 \hspace{12em}  \color{white}{1}};
        \end{tikzpicture}
        \end{minipage}
    \end{subfigure}\\
    \vspace{1em}
  \begin{subfigure}{0.33\linewidth}
    \centering
    \includegraphics[width=\linewidth]{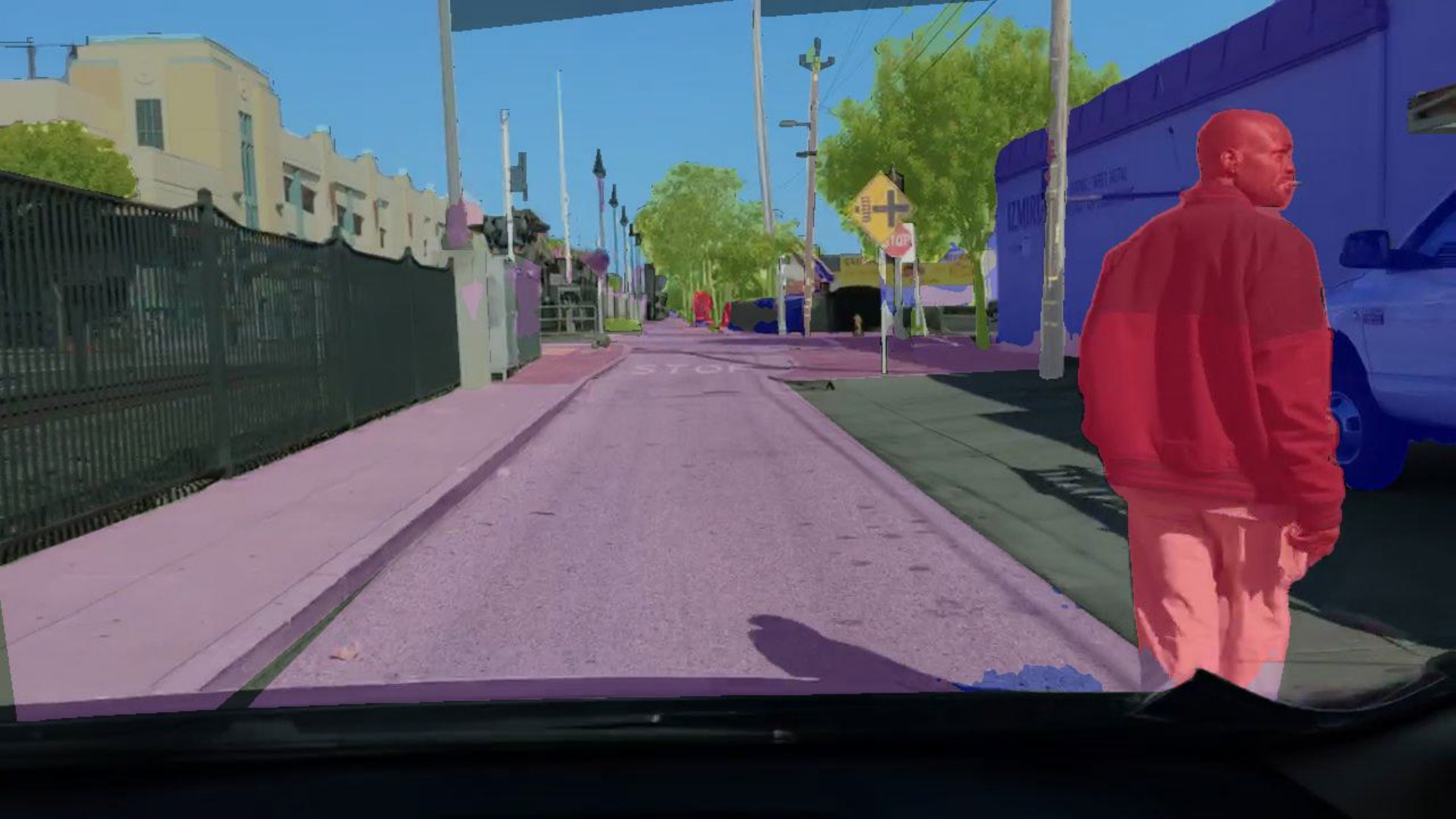}\\[1mm]
    \includegraphics[width=\linewidth]{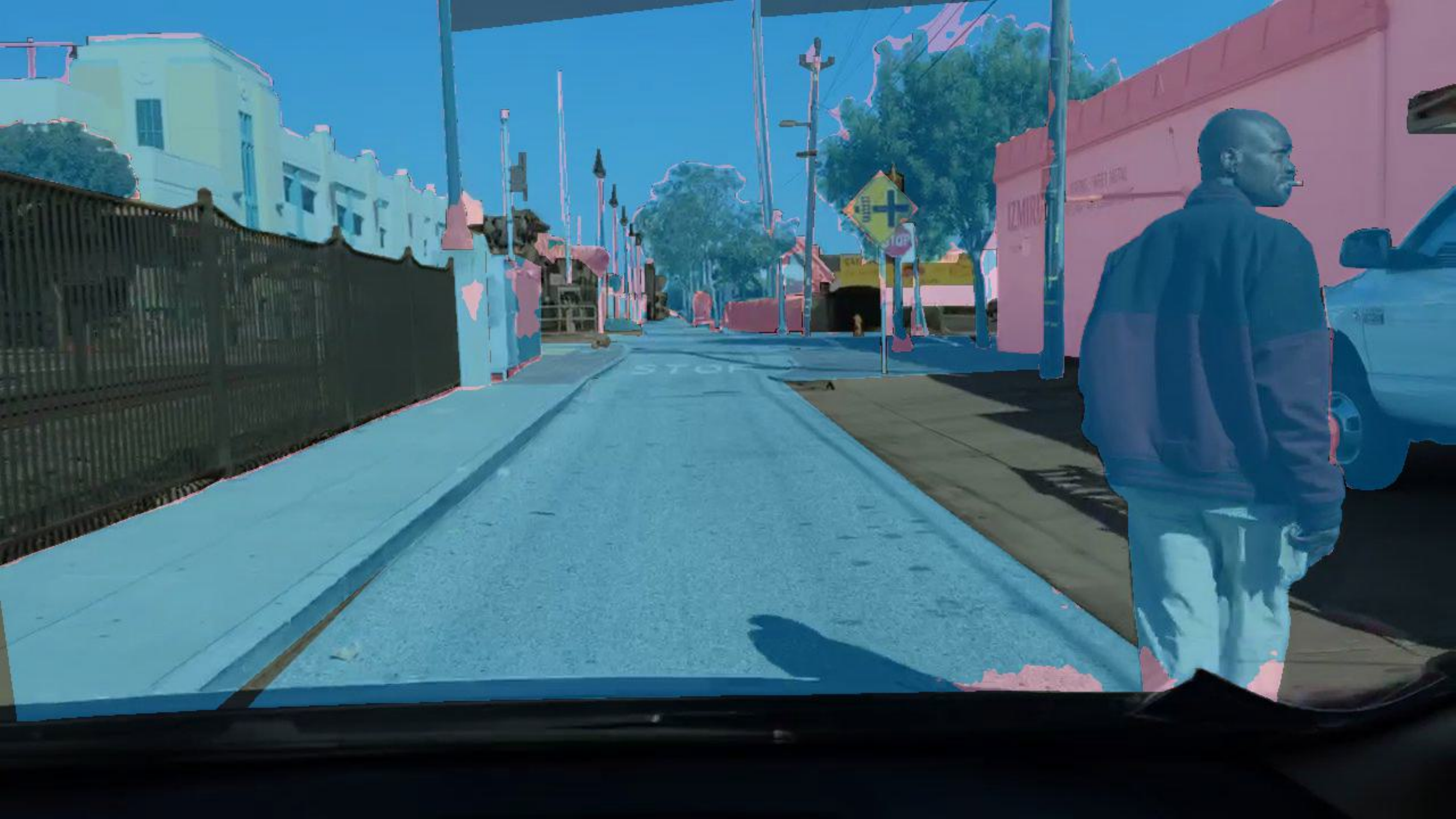}\\[1mm]
    \includegraphics[width=\linewidth]{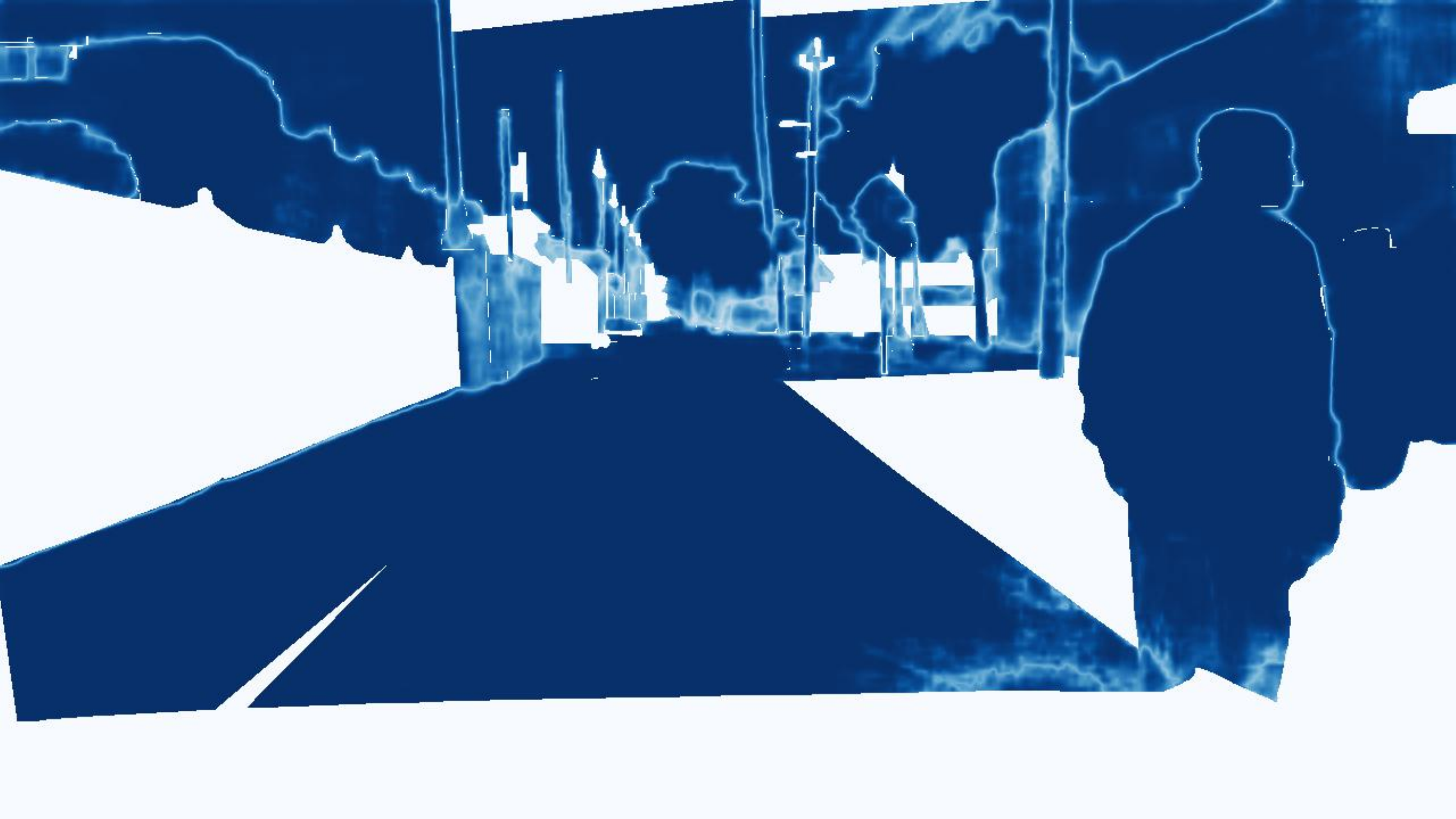}
  \label{fig:bdd2_1}
  \end{subfigure}
  \begin{subfigure}{0.33\linewidth}
    \centering
    \includegraphics[width=\linewidth]{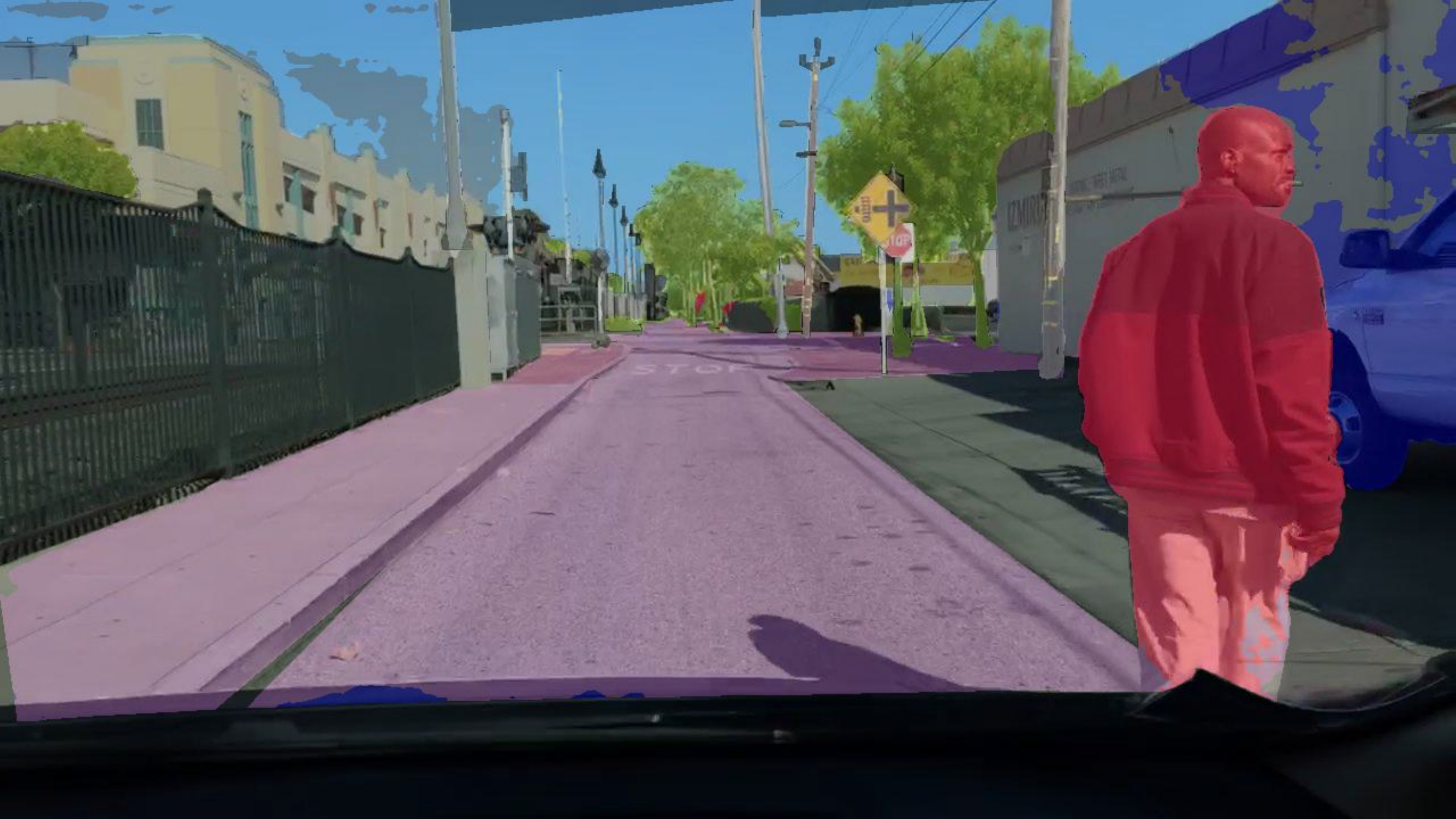}\\[1mm]
    \includegraphics[width=\linewidth]{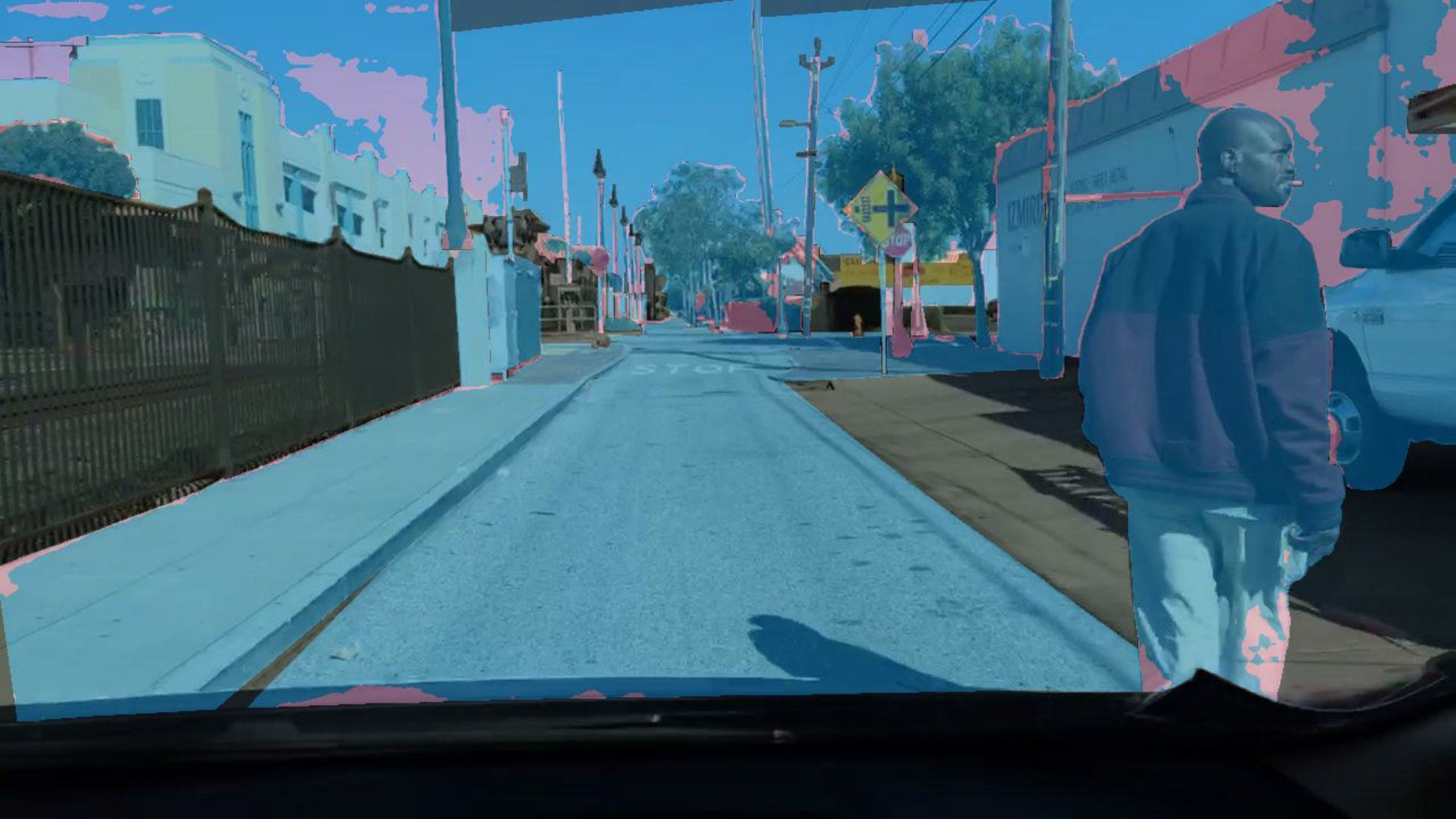}\\[1mm]
    \includegraphics[width=\linewidth]{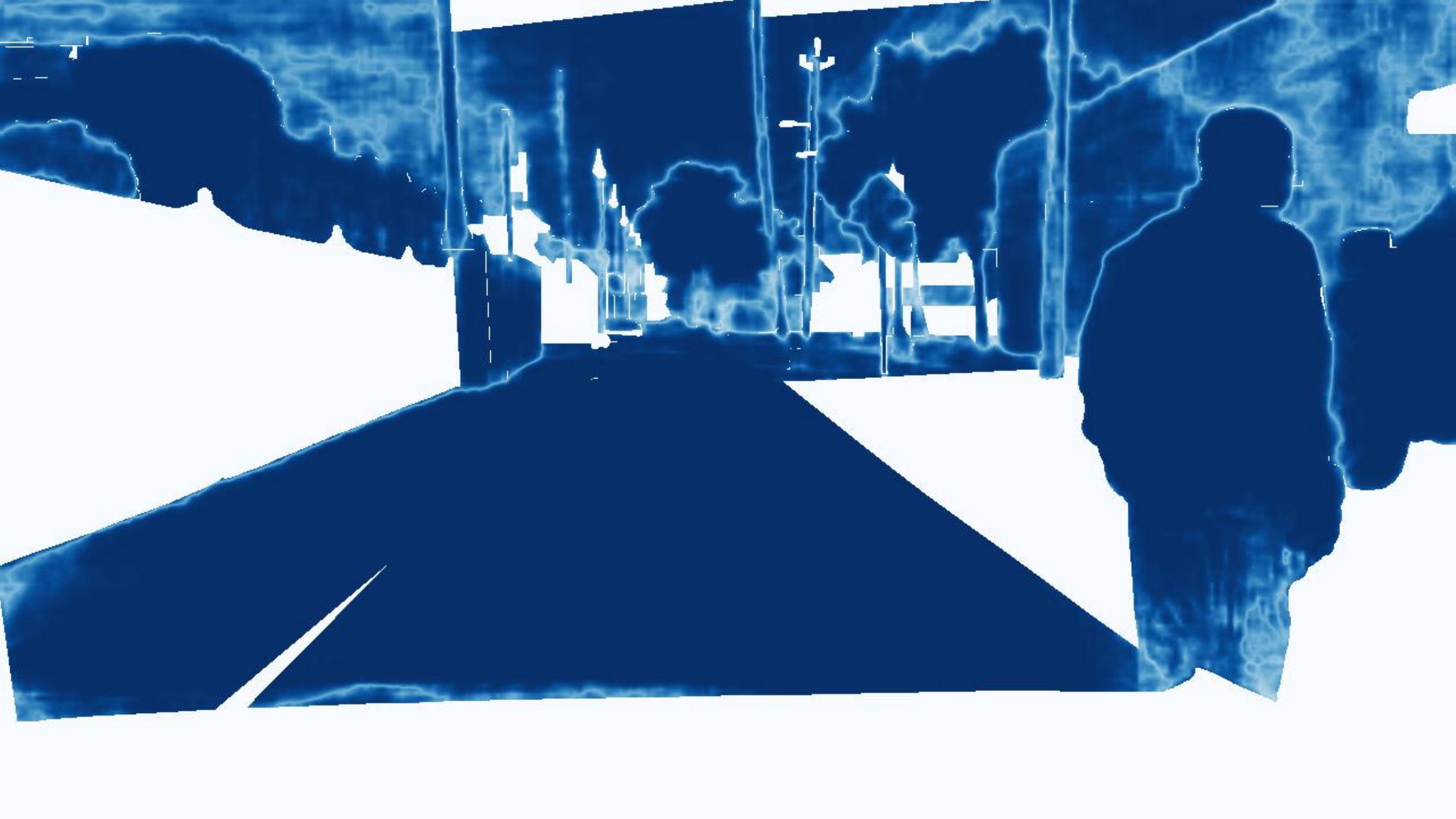}
    \label{fig:bdd2_2}
  \end{subfigure}
  \begin{subfigure}{0.33\linewidth}
    \centering
    \includegraphics[width=\linewidth]{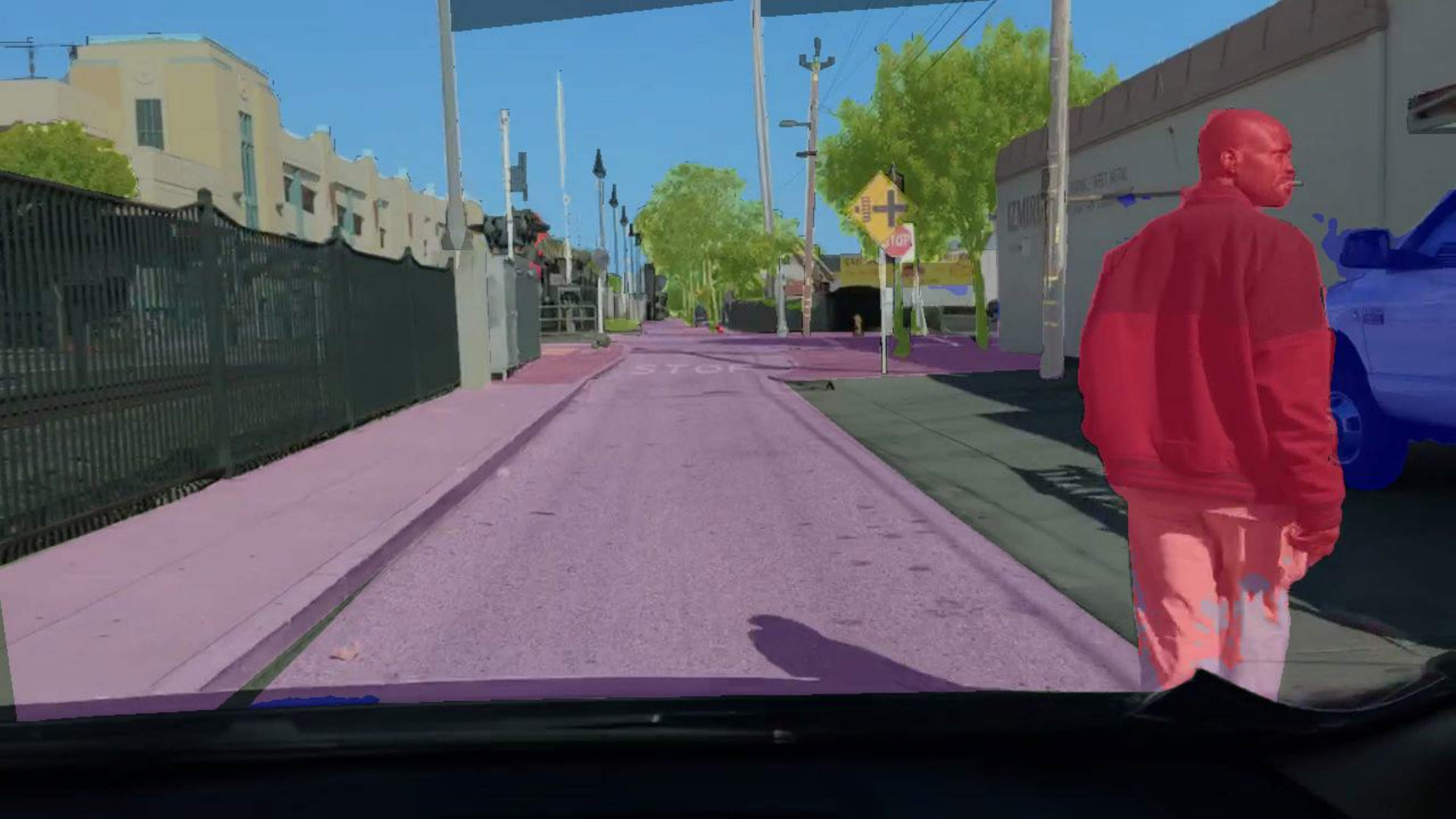}\\[1mm]
    \includegraphics[width=\linewidth]{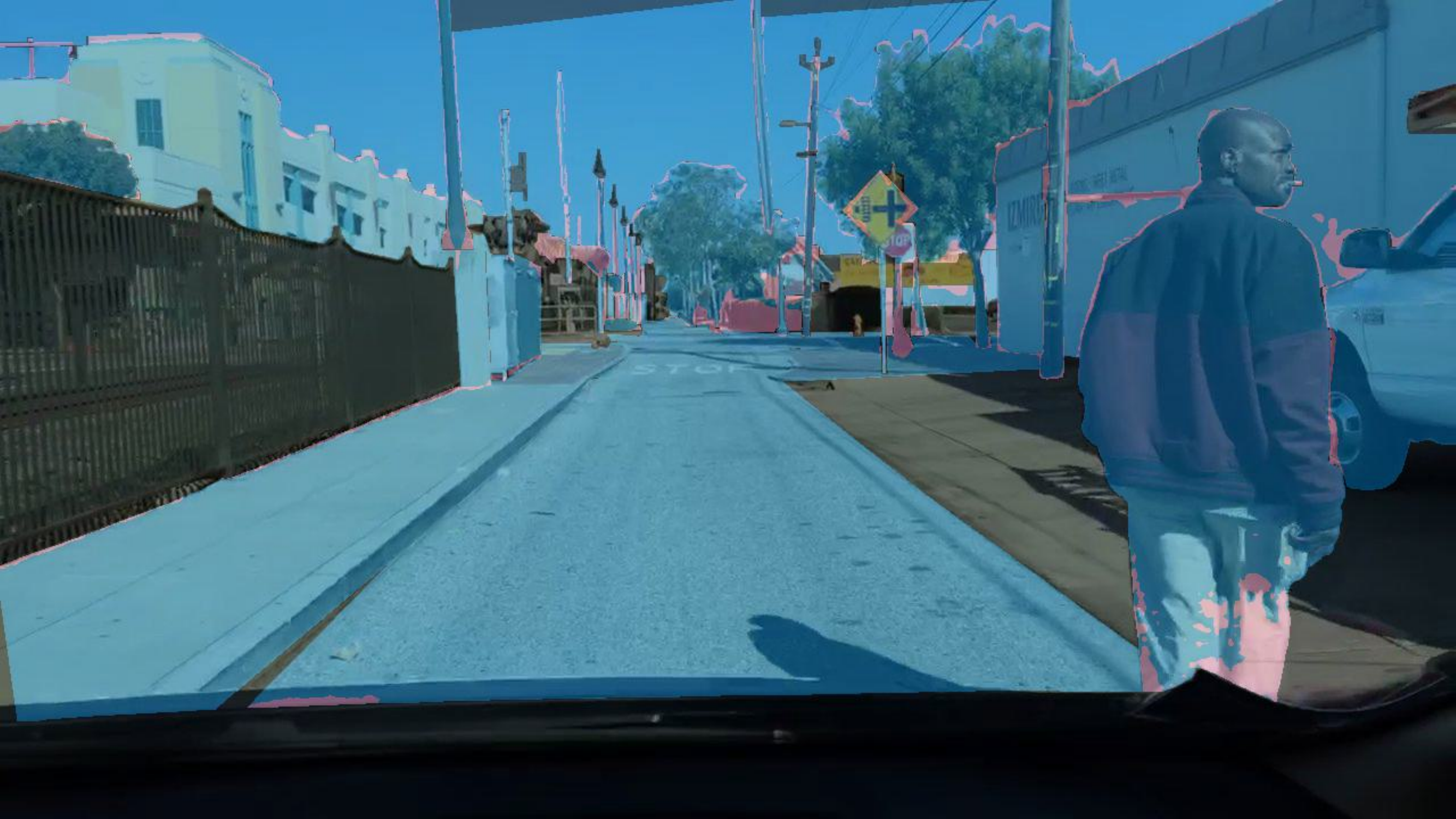}\\[1mm]
    \includegraphics[width=\linewidth]{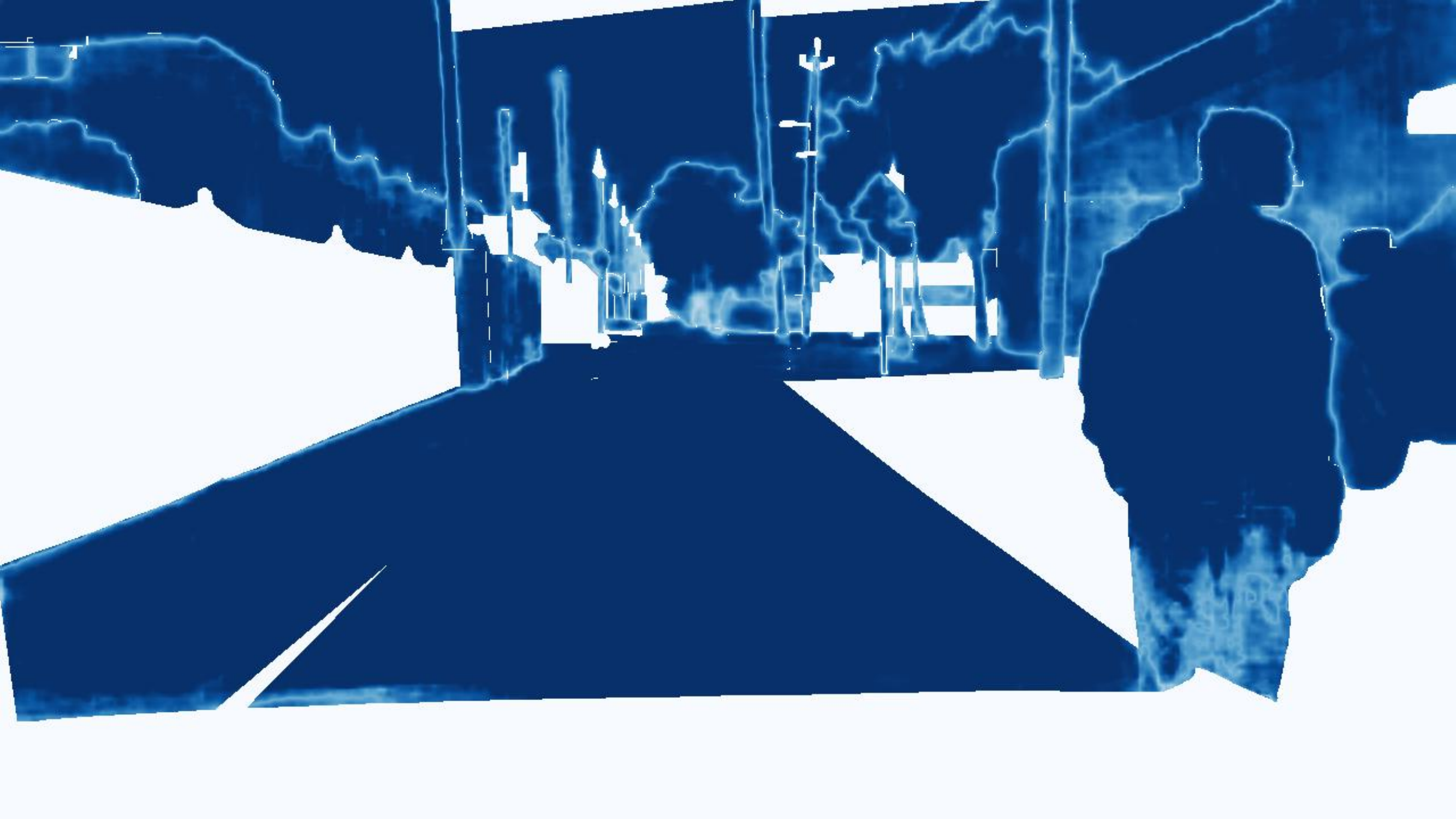}
    \label{fig:bdd2_3}
  \end{subfigure}
  \caption{\textbf{Two qualitative examples from BDD.} See \cref{sec:qualitative} for analysis.}
  \label{fig:qaulbdd2}
  \vspace{-0.7em}
\end{figure*}

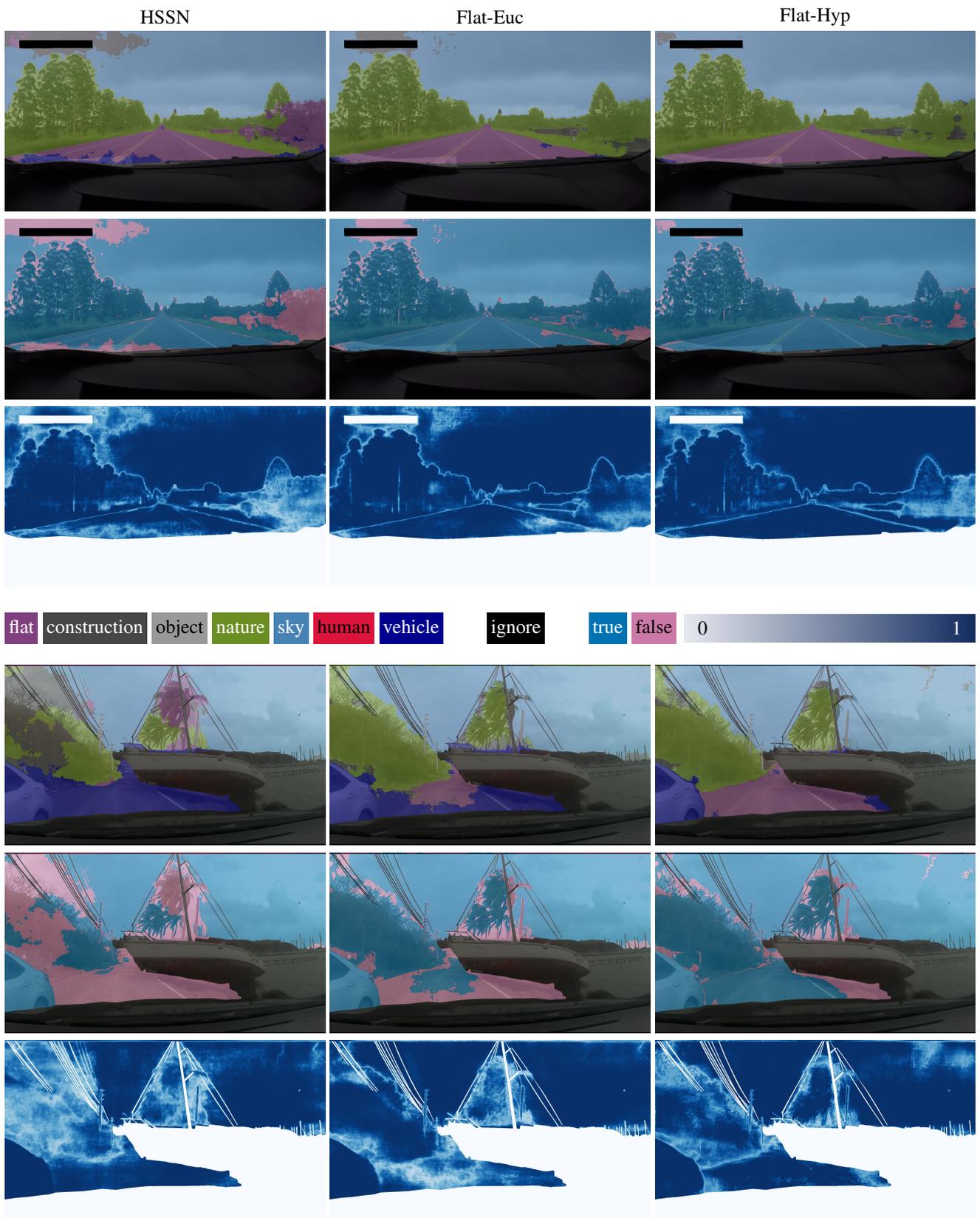
\begin{figure*}[t]
\centering
    \begin{subfigure}{0.33\linewidth}
    \centering
    HSSN\\\vspace{0.3em}
    \includegraphics[width=\linewidth]{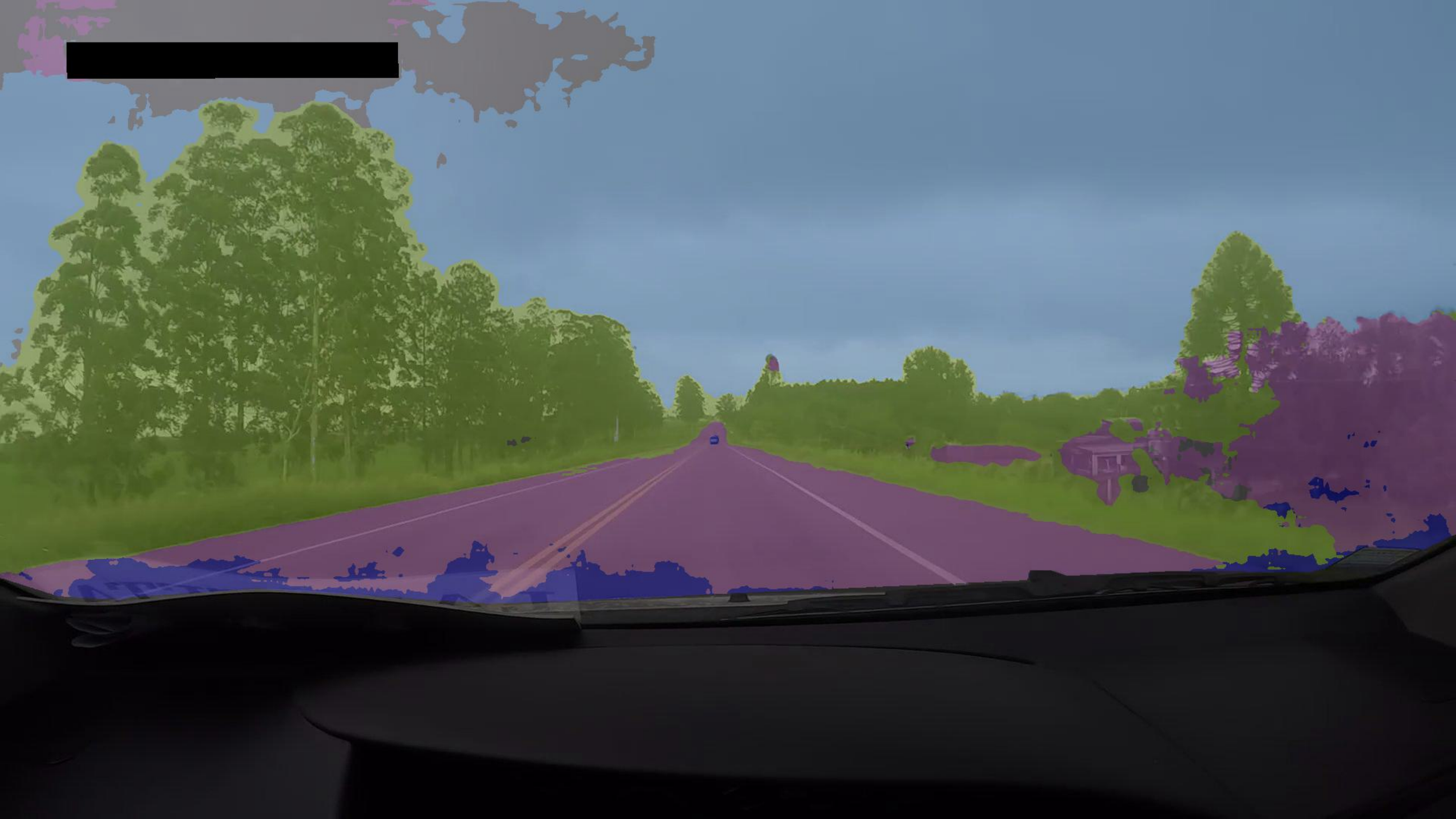}\\[1mm]
    \includegraphics[width=\linewidth]{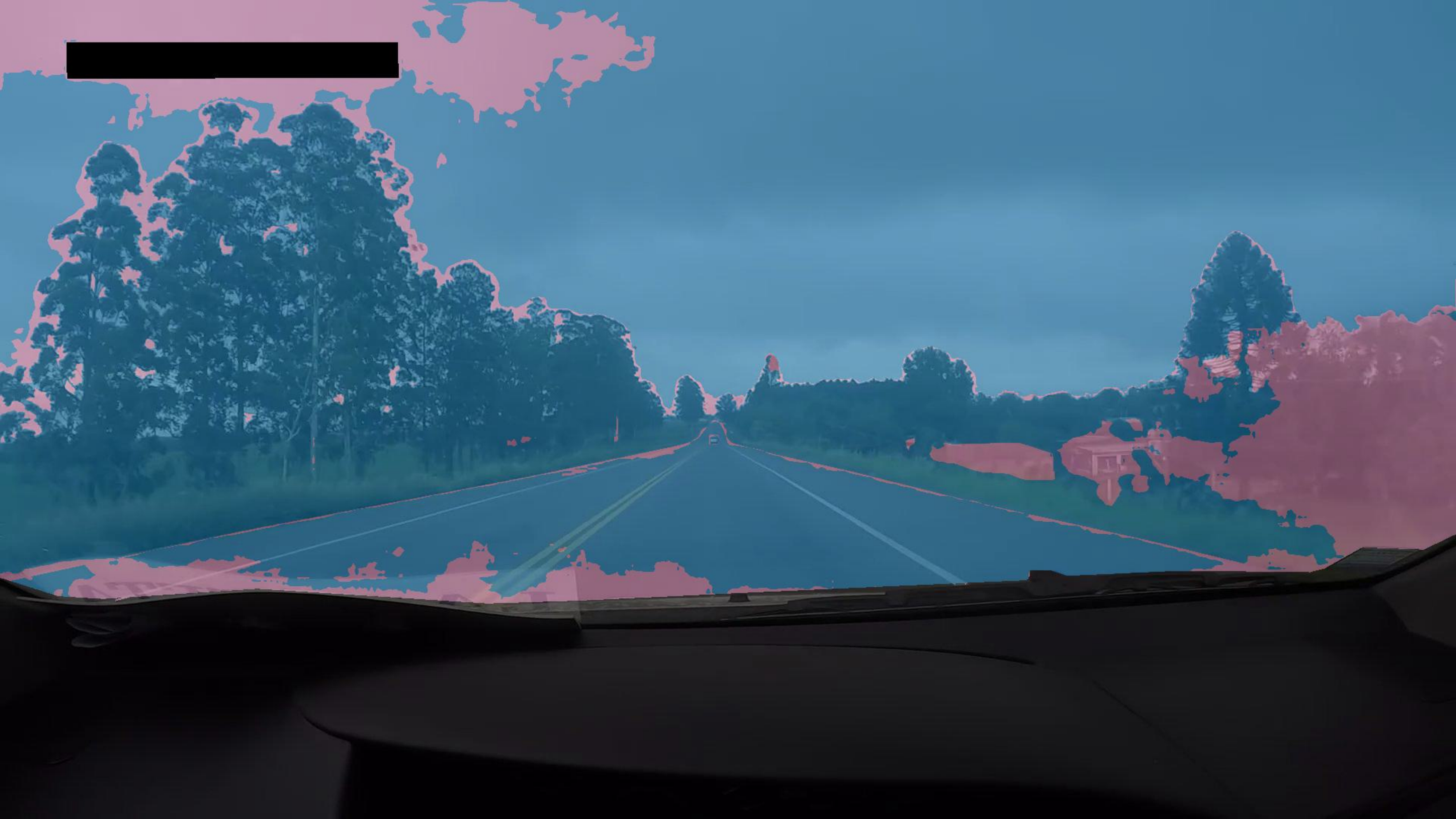}\\[1mm]
    \includegraphics[width=\linewidth]{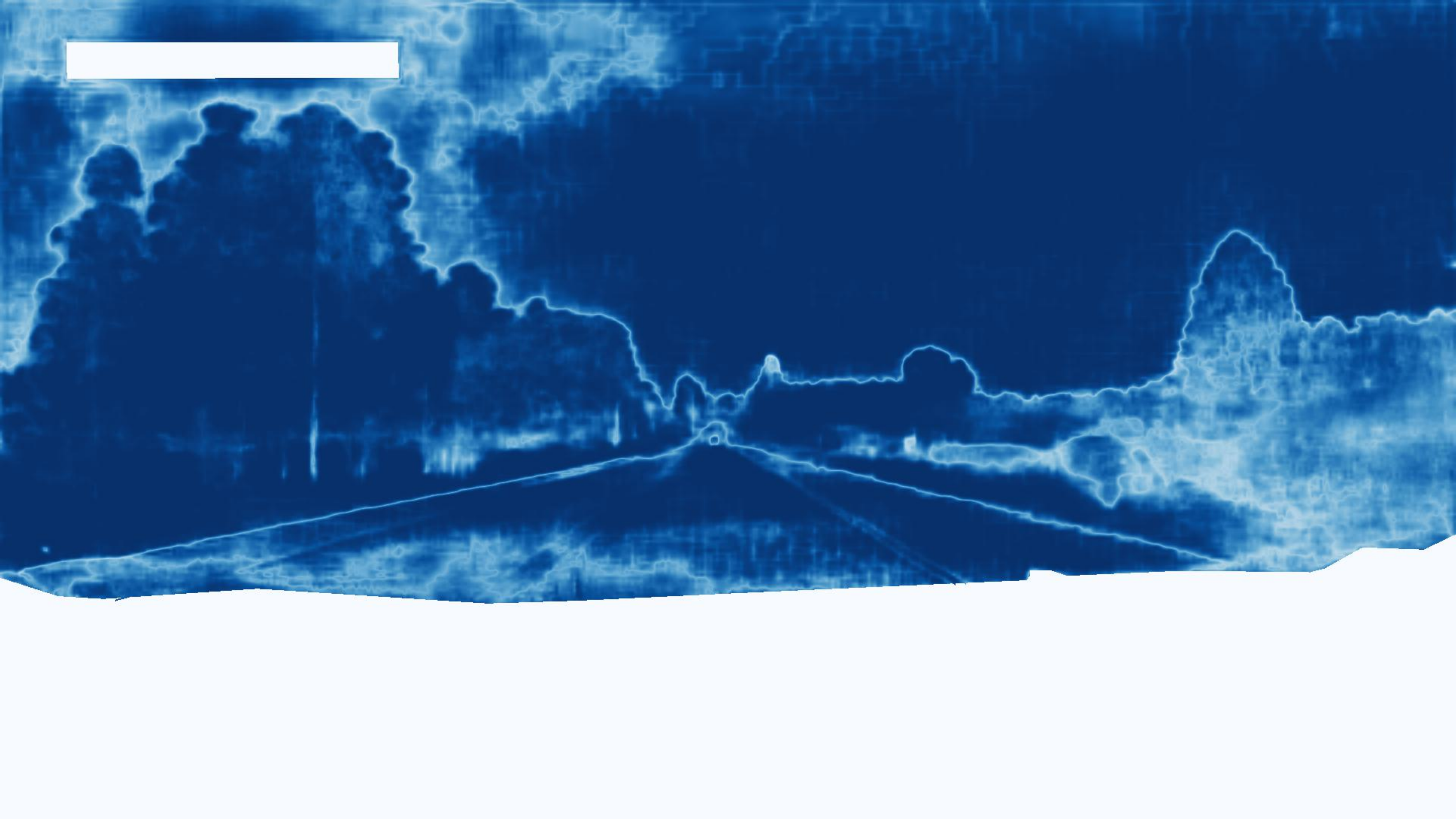}
    \label{fig:wd1_1}
  \end{subfigure}
  \begin{subfigure}{0.33\linewidth}
    \centering
    \eucname\\\vspace{0.3em}
    \includegraphics[width=\linewidth]{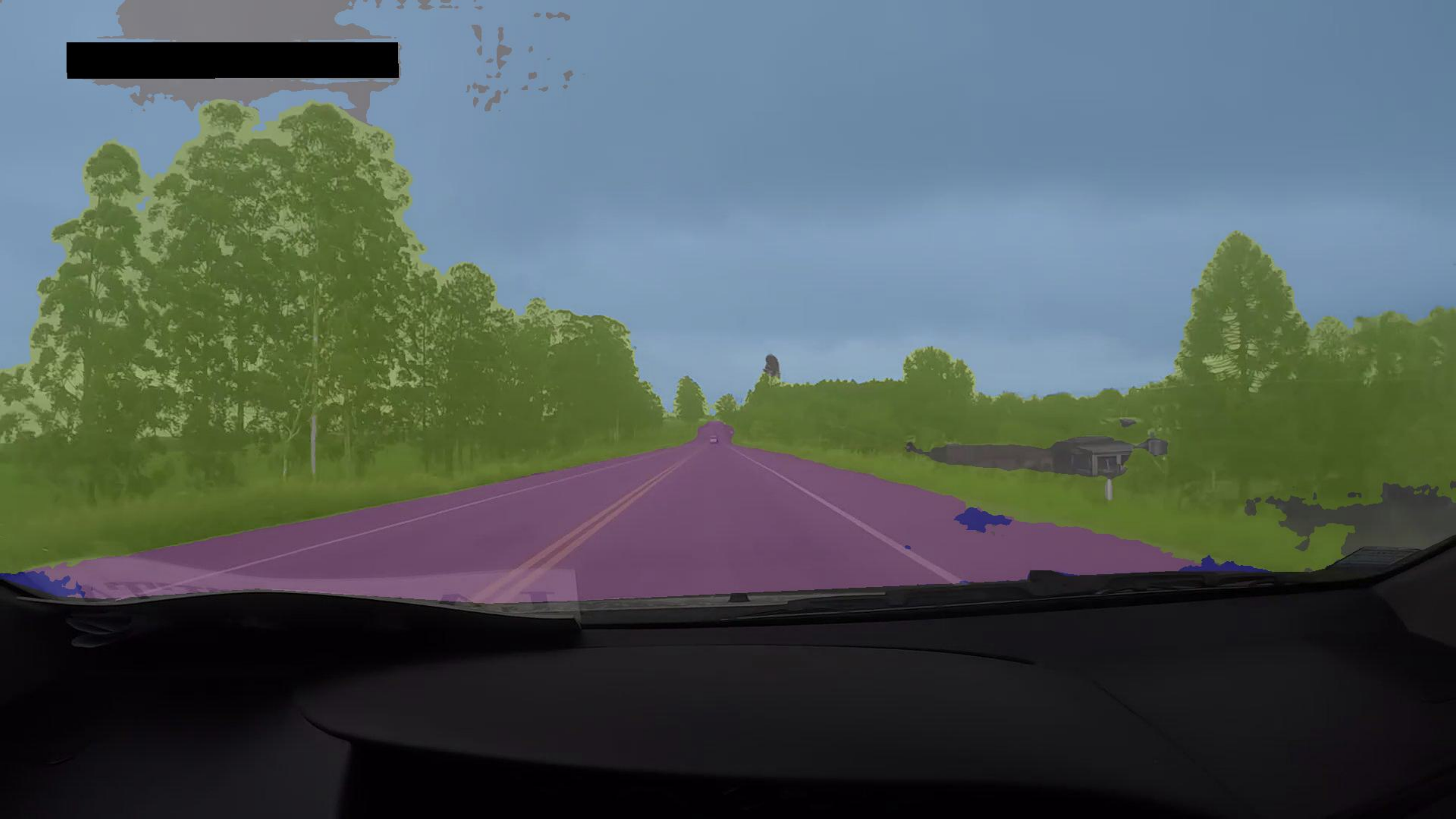}\\[1mm]
    \includegraphics[width=\linewidth]{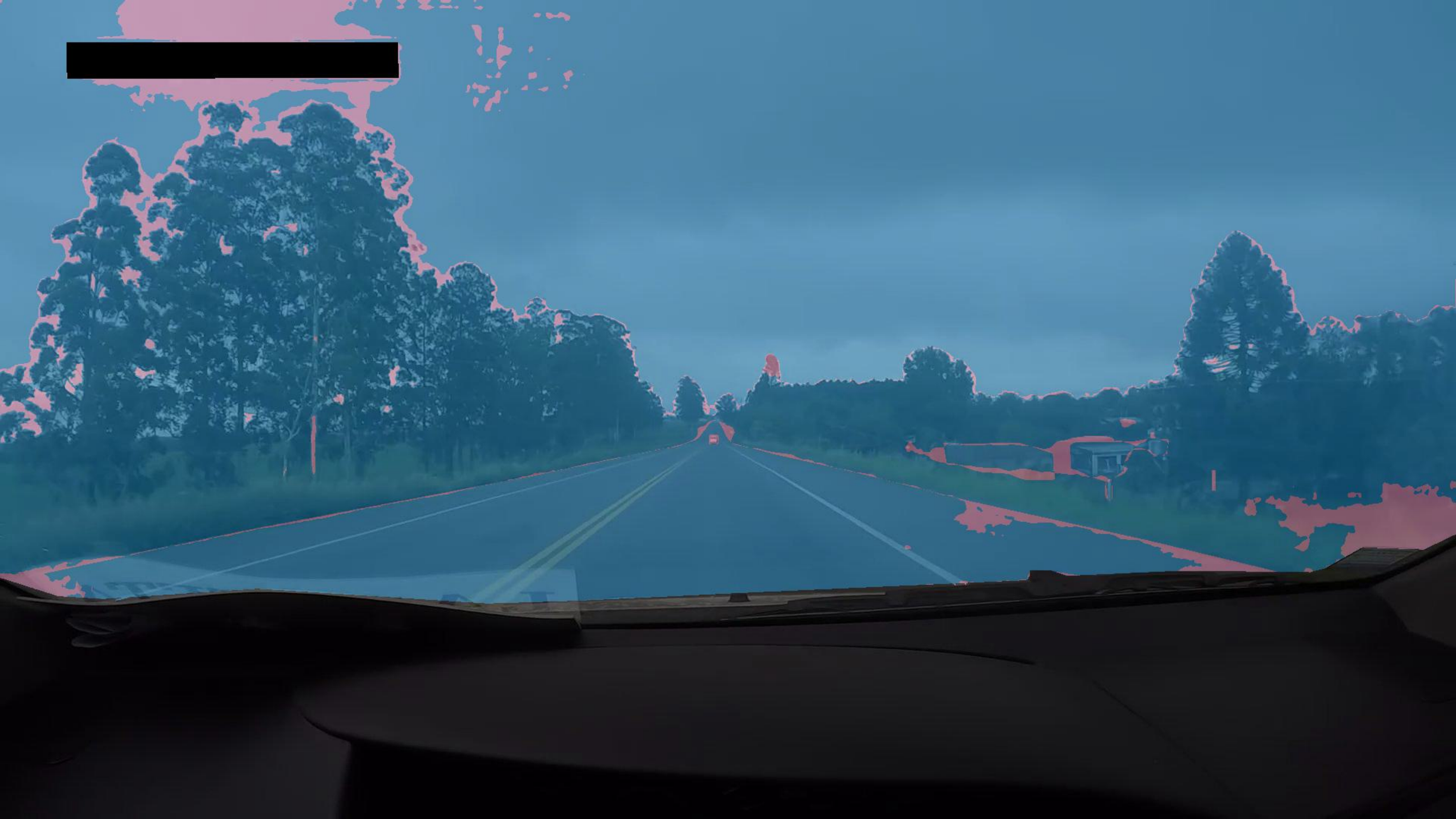}\\[1mm]
    \includegraphics[width=\linewidth]{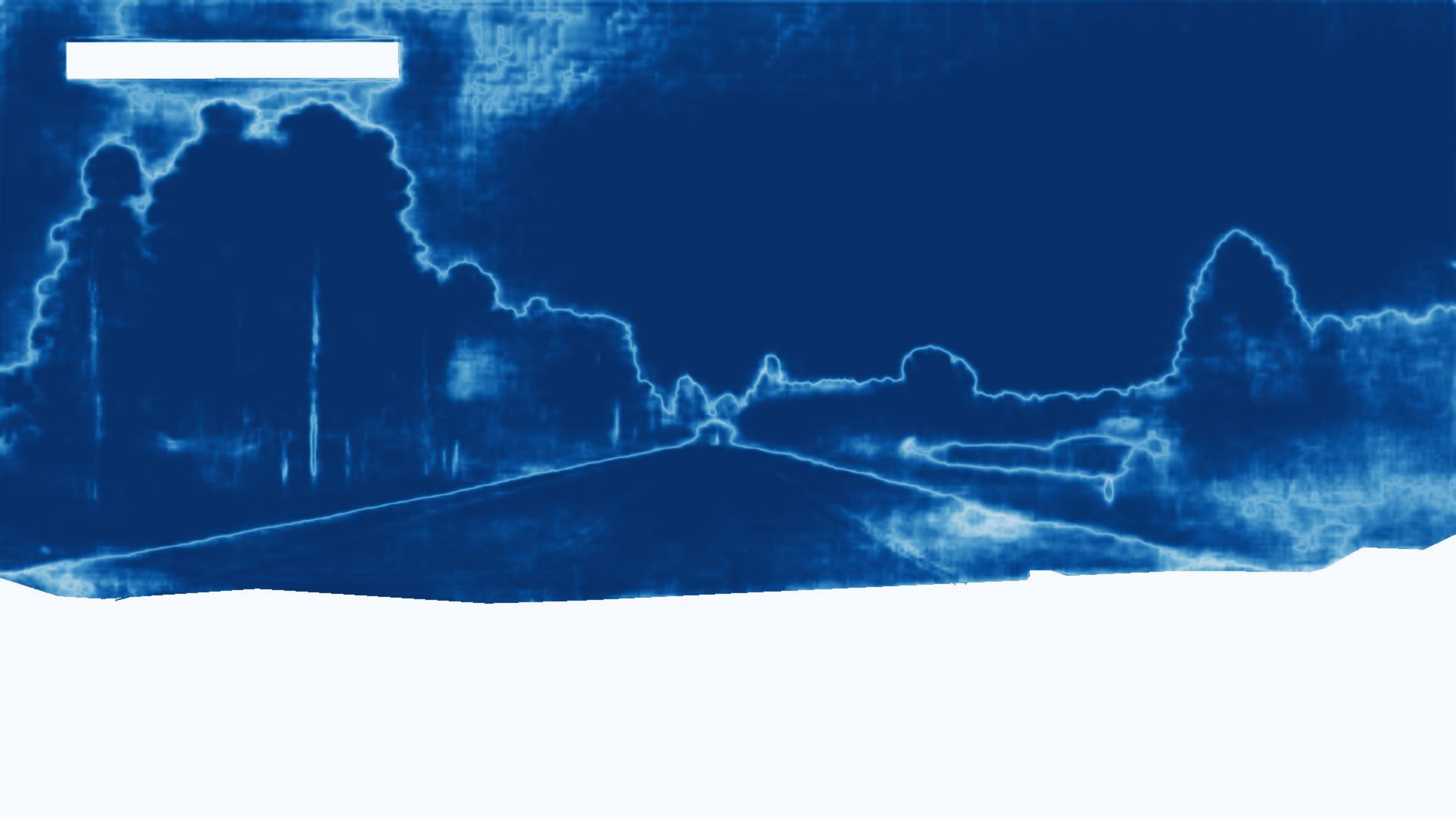}
    \label{fig:wd1_2}
  \end{subfigure}
  \begin{subfigure}{0.33\linewidth}
    \centering
    \hypname\\\vspace{0.2em}
    \includegraphics[width=\linewidth]{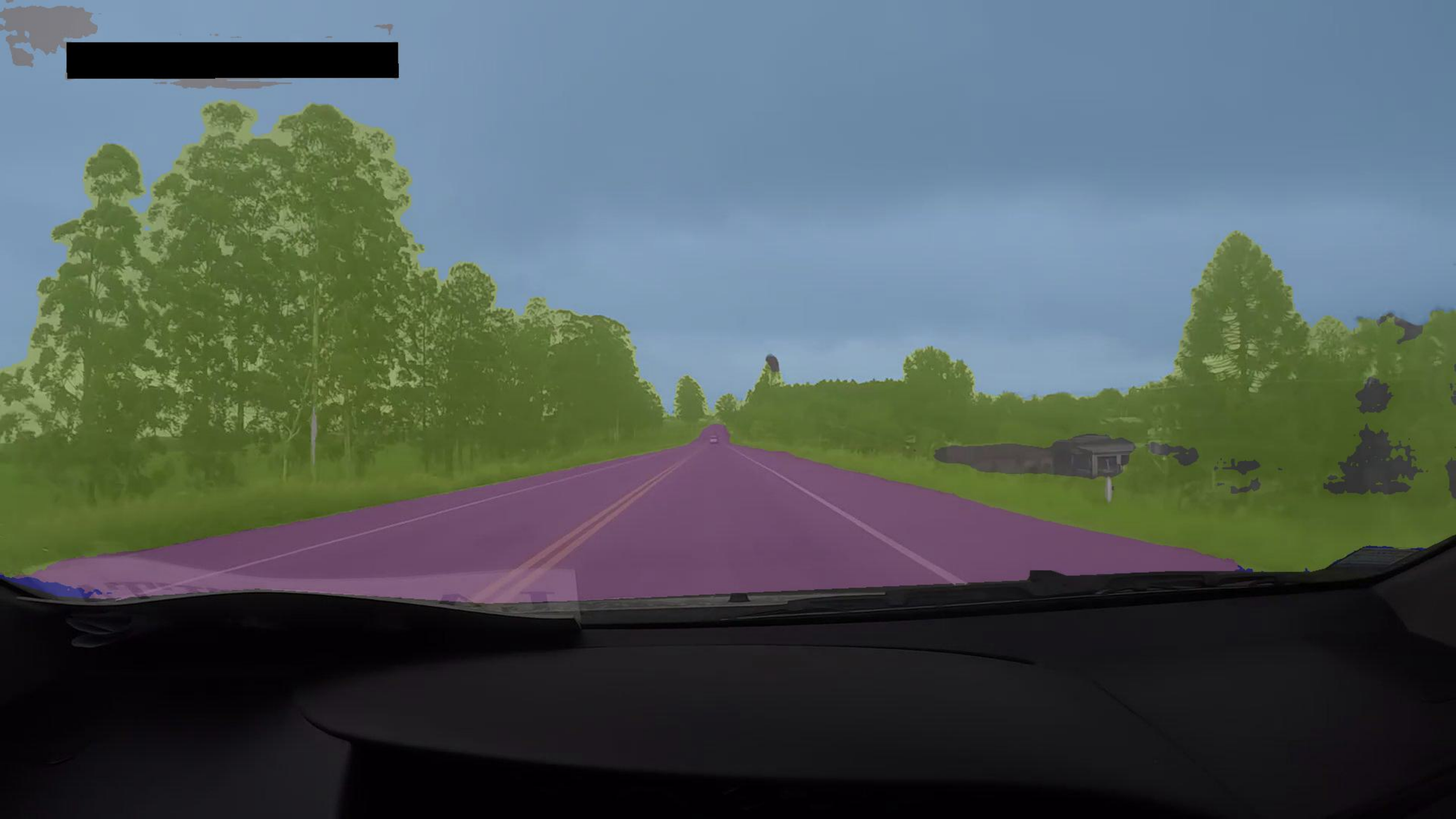}\\[1mm]
    \includegraphics[width=\linewidth]{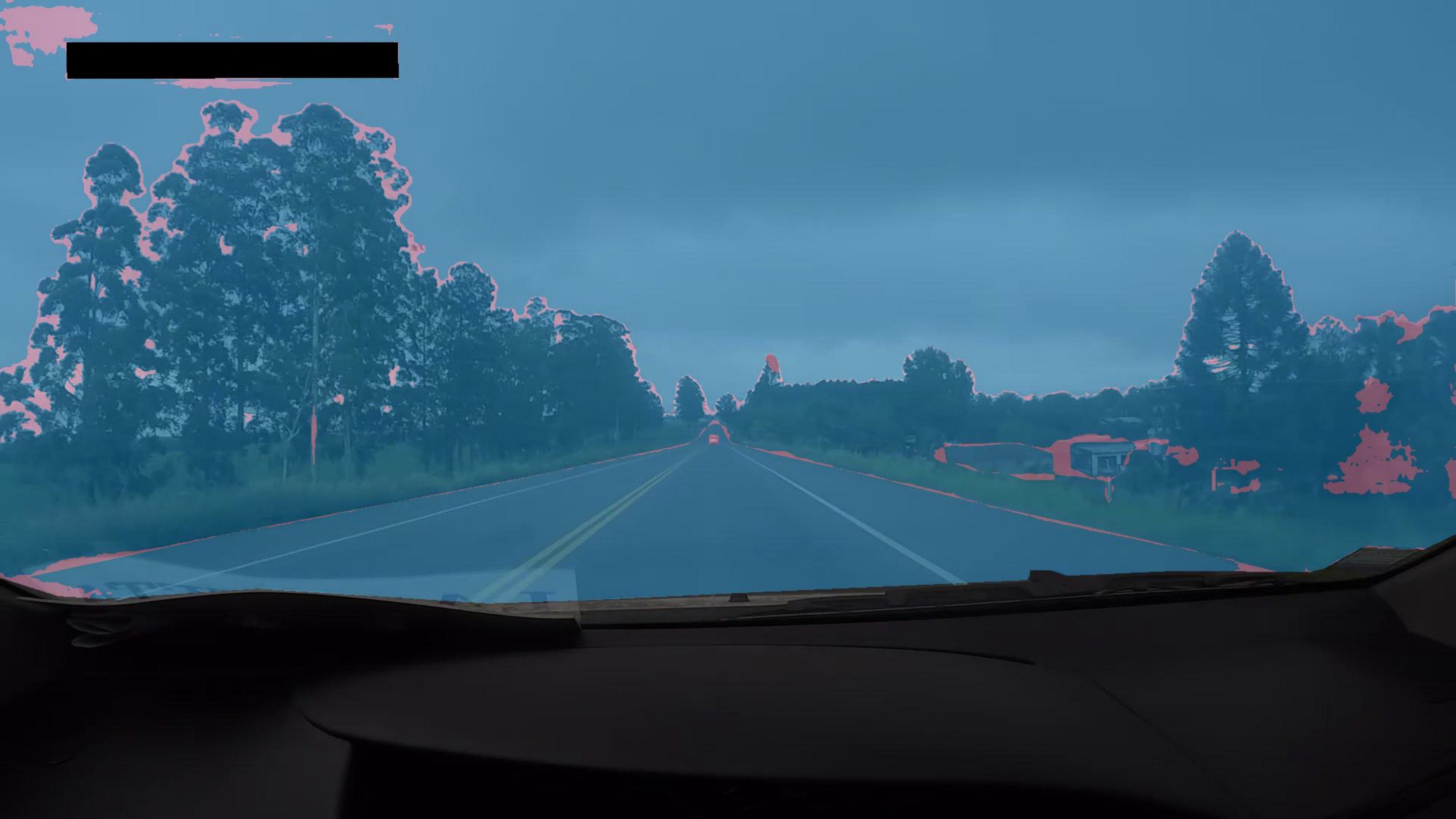}\\[1mm]
    \includegraphics[width=\linewidth]{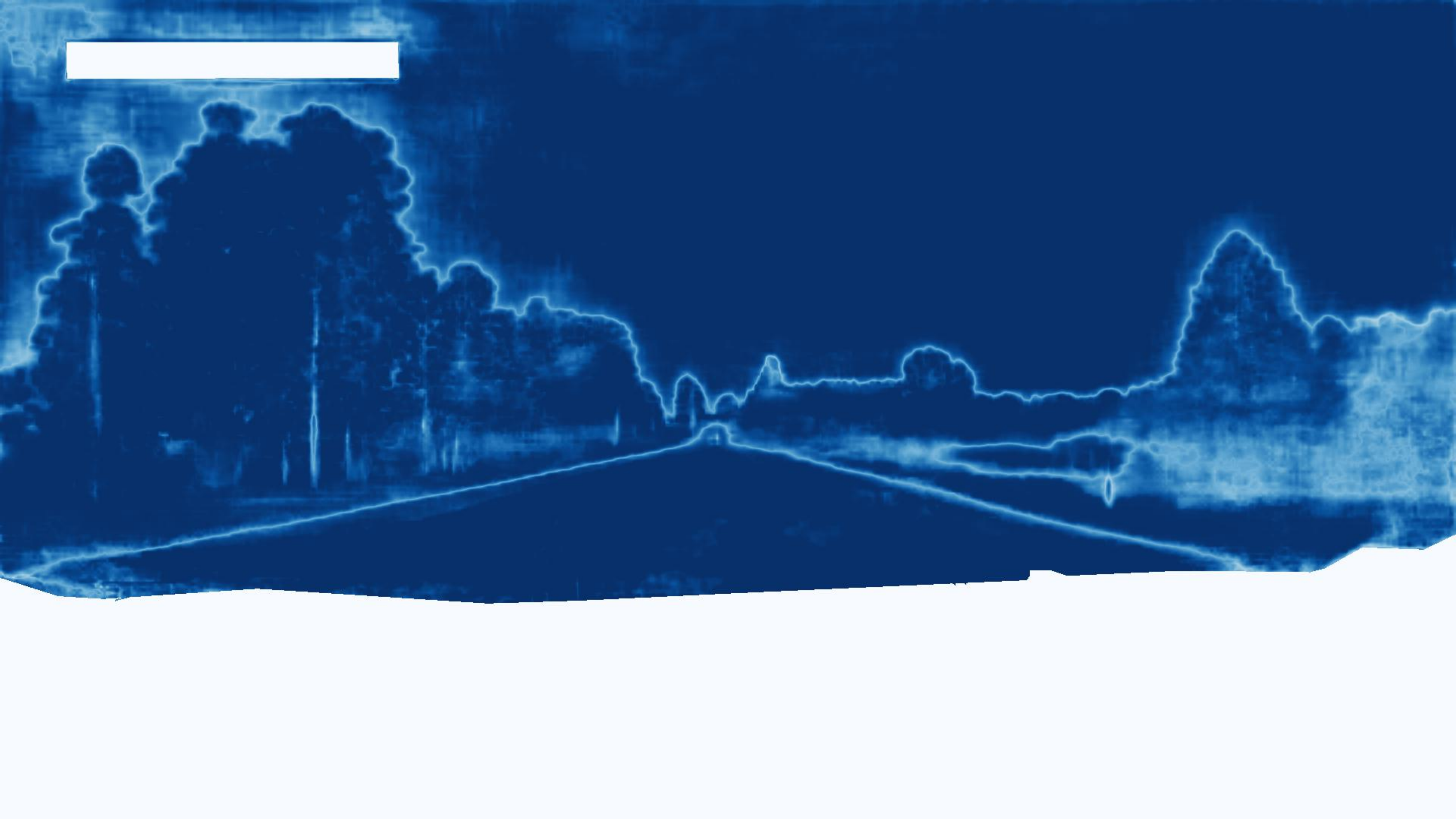}
  \label{fig:wd1_3}
  \end{subfigure}
    \begin{subfigure}{\linewidth}
        \centering
        \fboxsep 2pt
        \begin{minipage}{0.7\linewidth}
            \colorbox{flat}{\strut \color{white}{flat}}
            \colorbox{construction}{\strut \color{white}{construction}}
            \colorbox{object}{\strut  object}
            \colorbox{nature}{\strut \color{white}{nature}}
            \colorbox{sky}{\strut \color{white}{sky}}
            \colorbox{human}{\strut human}
            \colorbox{vehicle}{\strut \color{white}{vehicle}}\hspace{2em}
            \colorbox{ignore}{\strut \color{white}{ignore}}\hspace{2em}
            \colorbox{true}{\strut \color{white}{true}}
            \colorbox{false}{\strut false}\hspace{2em}%
        \end{minipage}%
        \begin{minipage}{0.3\linewidth}
        \begin{tikzpicture}
        \node [rectangle, left color=left!10!white, right color=left, anchor=north, minimum width=\linewidth, minimum height=0.5cm] (box) at (current page.north){0 \hspace{12em}  \color{white}{1}};
        \end{tikzpicture}
        \end{minipage}
    \end{subfigure}\\
    \vspace{1em}
    \begin{subfigure}{0.33\linewidth}
    \centering
    \includegraphics[width=\linewidth]{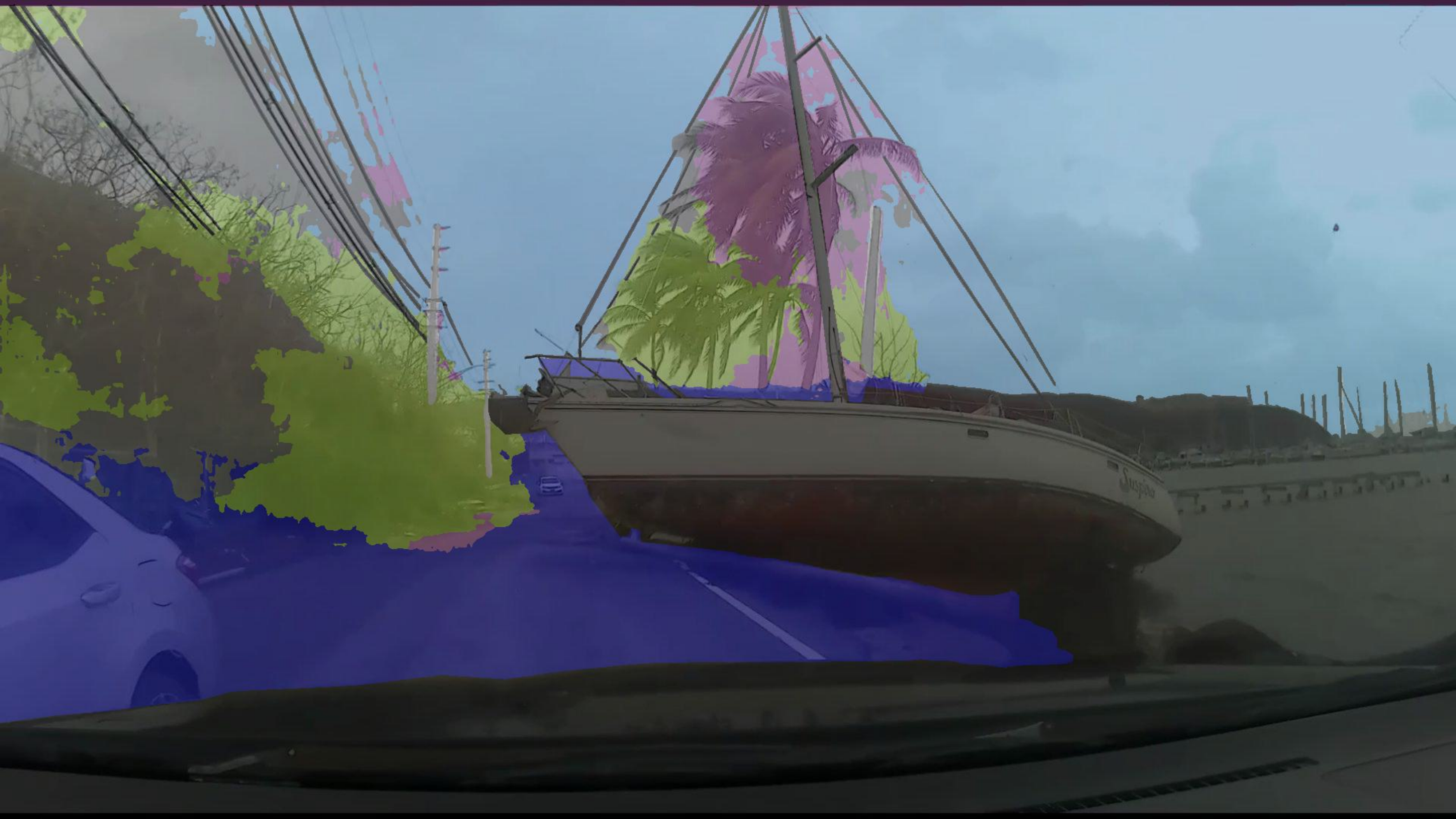}\\[1mm]
    \includegraphics[width=\linewidth]{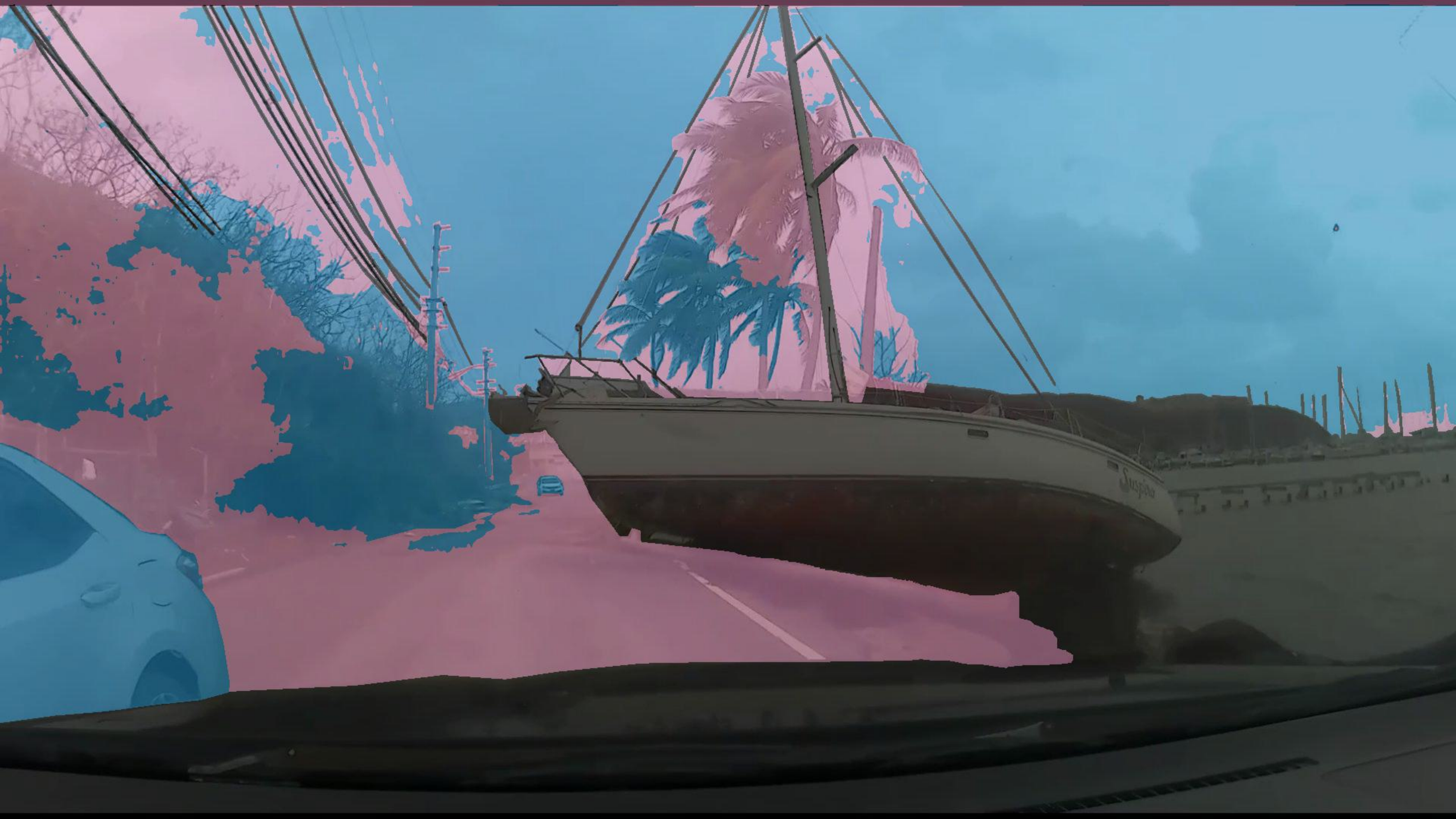}\\[1mm]
    \includegraphics[width=\linewidth]{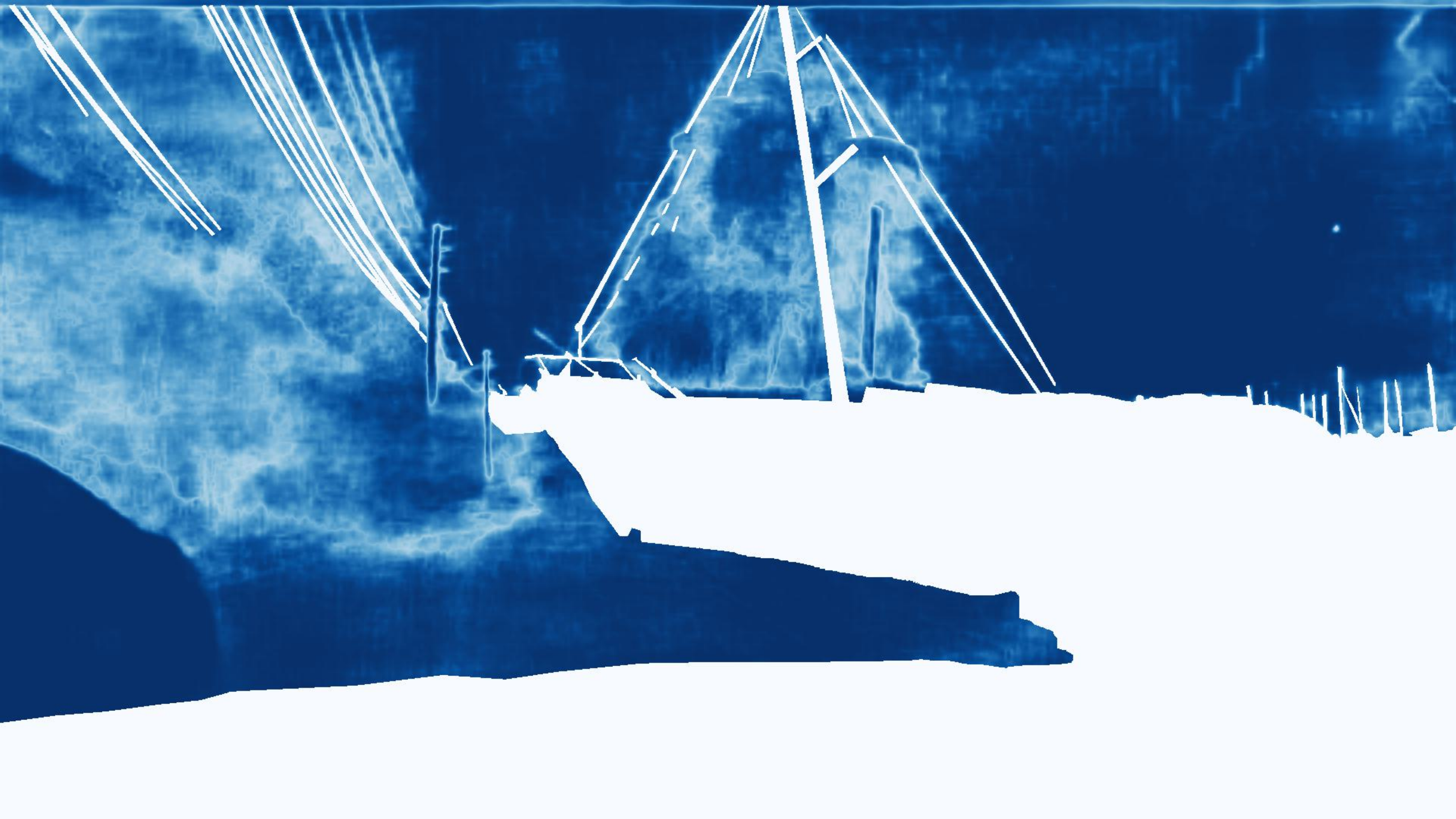}
    \label{fig:wd2_1}
  \end{subfigure}
  \begin{subfigure}{0.33\linewidth}
    \centering
    \includegraphics[width=\linewidth]{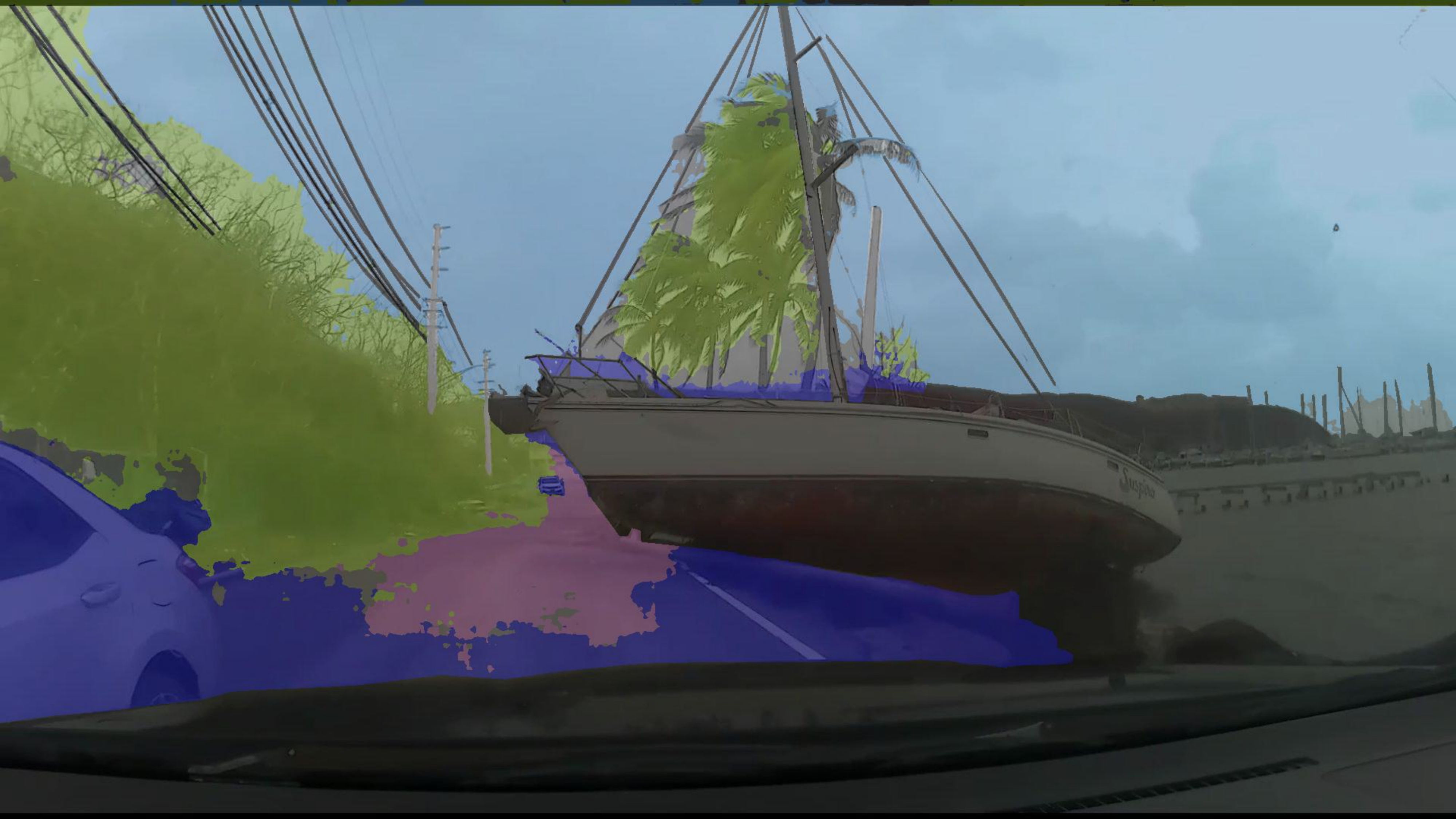}\\[1mm]
    \includegraphics[width=\linewidth]{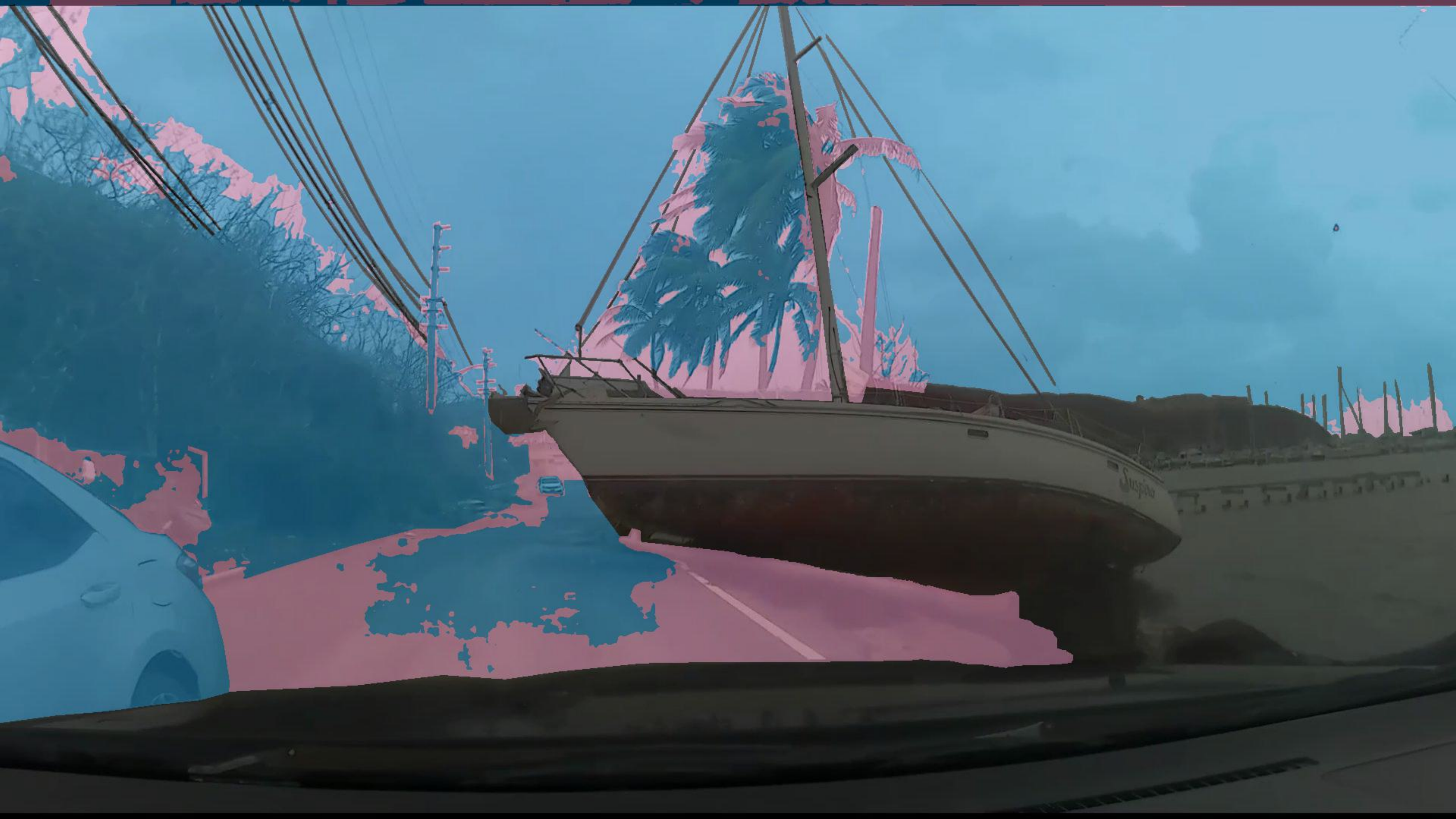}\\[1mm]
    \includegraphics[width=\linewidth]{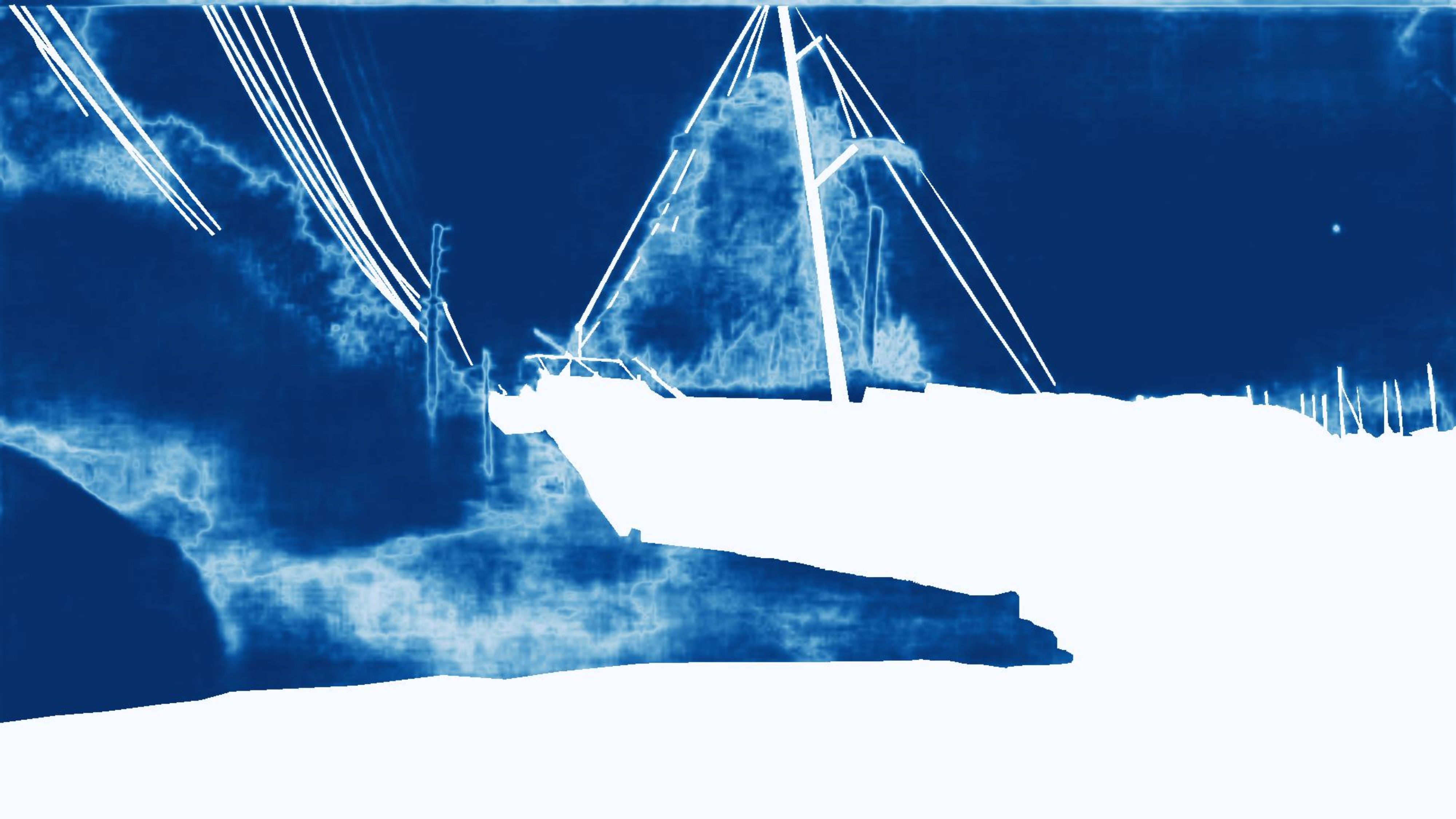}
    \label{fig:wd2_2}
  \end{subfigure}
  \begin{subfigure}{0.33\linewidth}
    \centering
    \includegraphics[width=\linewidth]{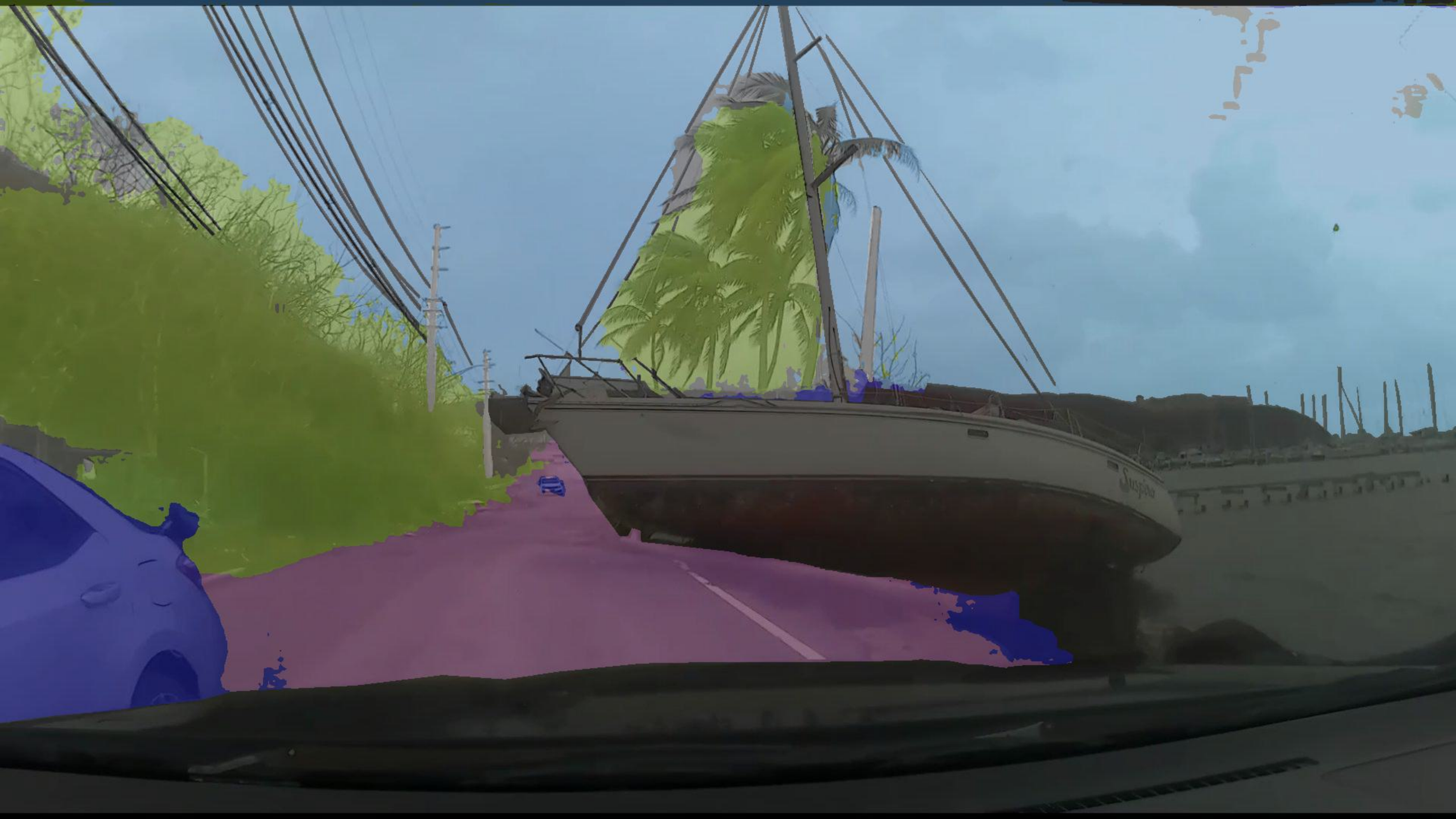}\\[1mm]
    \includegraphics[width=\linewidth]{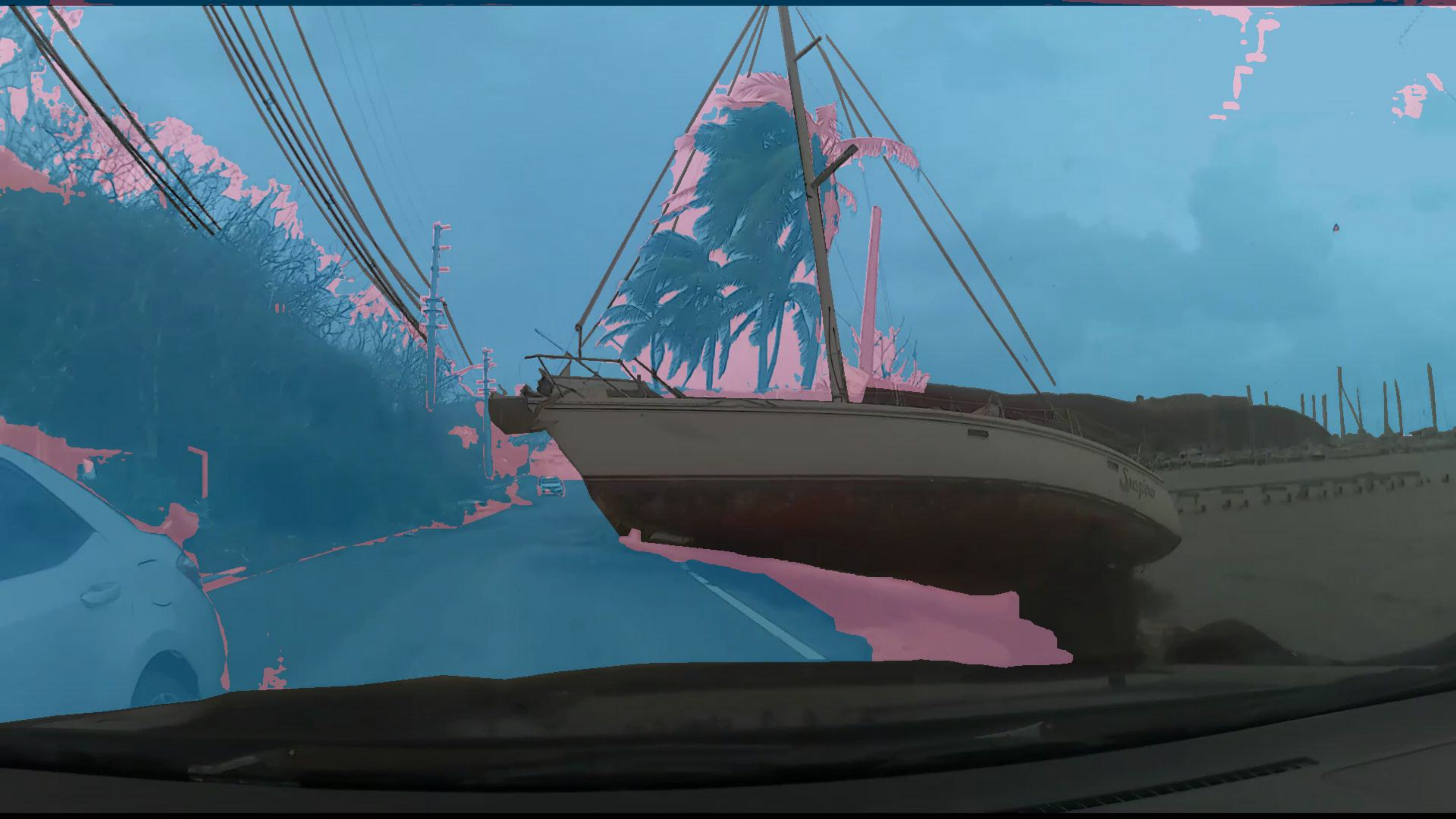}\\[1mm]
    \includegraphics[width=\linewidth]{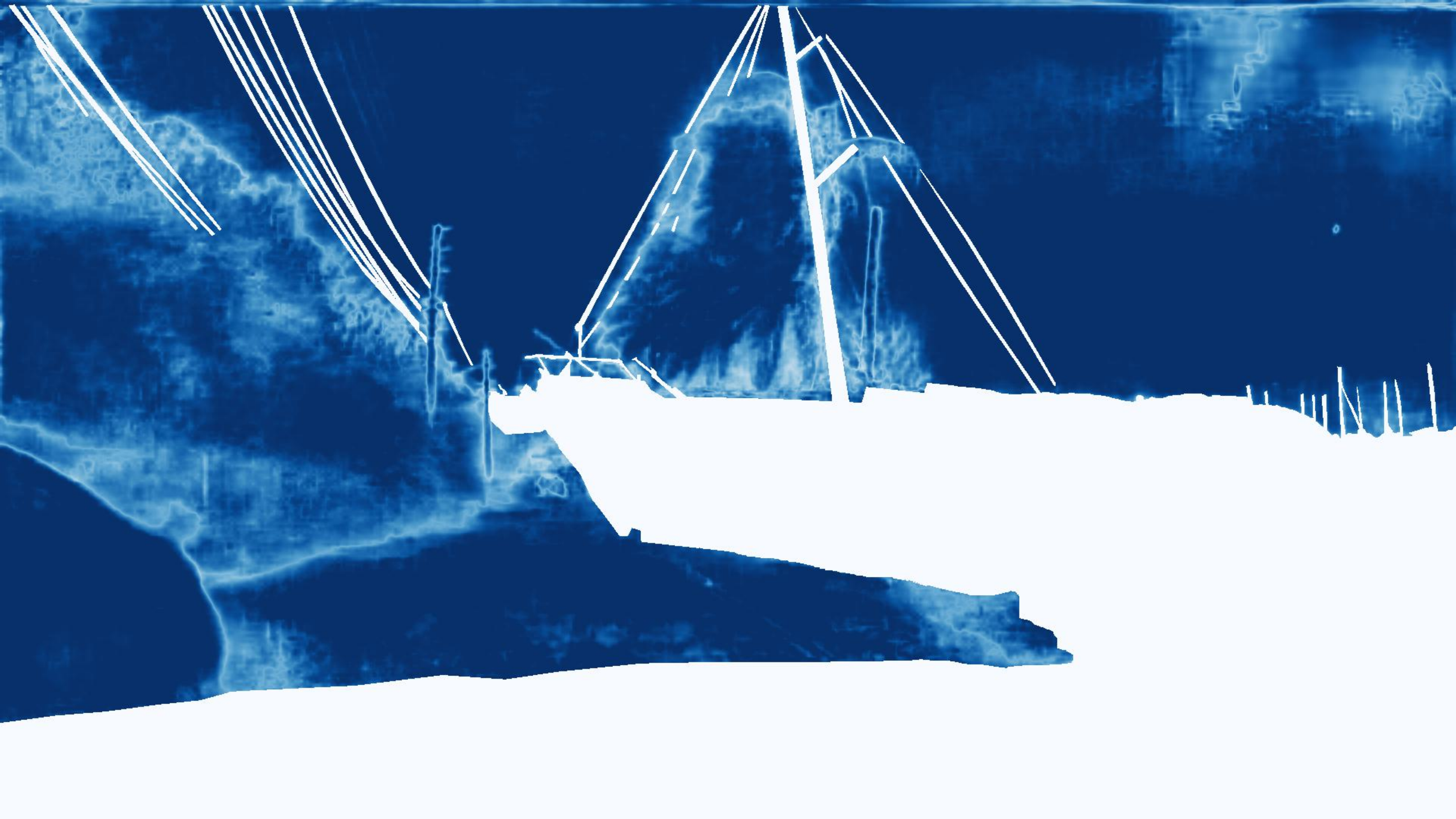}
  \label{fig:wd2_3}
  \end{subfigure}
  \caption{\textbf{Two qualitative examples from Wilddash.} See \cref{sec:qualitative} for analysis.}
  \label{fig:qaulwd2}
  \vspace{-0.7em}
\end{figure*}

\end{document}